\documentclass[11pt]{article}
\usepackage{fullpage}
\usepackage{amsmath,amsfonts,amsthm,amssymb}
\usepackage{url}
\usepackage{color}
\usepackage[usenames,dvipsnames,svgnames,table]{xcolor}
\usepackage[colorlinks=true, linkcolor=red, urlcolor=blue, citecolor=gray, backref]{hyperref}
\usepackage[capitalise]{cleveref}
\usepackage{graphicx}
\usepackage{caption} \usepackage{subcaption, enumitem}
\usepackage{hhline}
\usepackage{algorithm}
\usepackage{algpseudocode}
\usepackage{float}
\usepackage{threeparttable}
\makeatletter

\newcommand{\Var}{\operatorname{Var}}
\newcommand{\eps}{\ensuremath{\epsilon}}

\DeclareMathOperator*{\E}{\mathbb{E}}
\let\Pr\relax
\DeclareMathOperator*{\Pr}{\mathrm{Pr}}

\DeclareMathOperator{\rank}{rank}
\DeclareMathOperator{\colspan}{span}
\newcommand{\R}{\mathbb{R}}

\DeclareMathOperator*{\argmin}{arg\,min}

\newcommand{\poly}{\mathop\mathrm{poly}}

\DeclareMathOperator{\nnz}{nnz}

\usepackage{color}

\makeatletter

\newtheorem*{rep@theorem}{\rep@title}
\newcommand{\newreptheorem}[2]{%
\newenvironment{rep#1}[1]{%
 \def\rep@title{#2 \ref{##1}}%
 \begin{rep@theorem}}%
 {\end{rep@theorem}}}
\makeatother
\newtheorem{theorem}{Theorem}[section]
\newreptheorem{theorem}{Theorem}
\newreptheorem{fact}{Fact}

\newtheorem{corollary}[theorem]{Corollary}
\newtheorem{lemma}[theorem]{Lemma}
\newtheorem*{lemma*}{Lemma}
\newreptheorem{lemma}{Lemma}
\newtheorem{fact}[theorem]{Fact}
\newtheorem{claim}[theorem]{Claim}

\newtheorem{definition}[theorem]{Definition}
\newtheorem{problem}[theorem]{Problem}
\newtheorem{remark}[theorem]{Remark}

\usepackage{tai_macros}
\usepackage{footnotebackref}
\usepackage{tikz}
\usepackage{thm-restate}
\def\OPT{\mathsf{OPT}}
\DeclareMathOperator{\diam}{diam}

\title{\vspace{-2em}Active Linear Regression for $\ell_p$ Norms and Beyond}

\author{
Cameron Musco\\UMass Amherst\\\texttt{cmusco@cs.umass.edu} \and Christopher Musco\\ NYU\\\texttt{cmusco@nyu.edu} \and David P. Woodruff \\ CMU\\ \texttt{dwoodruf@cs.cmu.edu} \and Taisuke Yasuda \\ CMU\\ \texttt{taisukey@cs.cmu.edu}
}
\date{}

  \usepackage{nth}
  \usepackage{intcalc}

\begin{document}

\maketitle
 \thispagestyle{empty}

\begin{abstract}
We study active sampling algorithms for linear regression, which aim to query only a small number of entries of a target vector $\bfb \in \mathbb{R}^n$ and output a near minimizer to $\min_{\bfx \in \mathbb{R}^d} \| \bfA\bfx - \bfb\|$, where $\bfA \in \mathbb{R}^{n \times d}$ is a design matrix and $\| \cdot \|$ is some loss function. 

For $\ell_p$ norm regression for any $0<p<\infty$, we give an algorithm based on Lewis weight sampling which outputs a $(1+\epsilon)$-approximate solution using just $\tilde{O}(d/\eps^2)$ queries to $\bfb$ for $p\in(0,1)$, $\tilde{O}(d/\eps)$ queries for $p\in(1,2)$, and $\tilde{O}(d^{p/2}/\eps^{p})$ queries for $p\in(2,\infty)$. For $p\in(0,2)$, our bounds are optimal up to logarithmic factors, thus settling the query complexity for this range of $p$. For $p\in(2,\infty)$, our dependence on $d$ is optimal, while our dependence on $\eps$ is off by at most a single $\eps$ factor, up to logarithmic factors. Our result resolves an open question of Chen and Derezi\'{n}ski, who gave near optimal bounds for the $\ell_1$ norm, but required at least $d^2/\eps^2$ samples for $\ell_p$ regression with $p \in (1,2)$, and gave no bounds for $p\in(2,\infty)$ or $p\in(0,1)$. 

We also provide the first total sensitivity upper bound of $O(d^{\max\{1,p/2\}}\log^2 n)$ for loss functions with at most degree $p$ polynomial growth. This improves a recent result of Tukan, Maalouf, and Feldman. By combining this with our techniques for $\ell_p$ regression, we obtain an active regression algorithm making $\tilde O(d^{1+\max\{1,p/2\}}/\poly(\eps))$ queries for such loss functions, including the important cases of the Tukey and Huber losses. This answers another question of Chen and Derezi\'{n}ski. For the Huber loss, we further improve our bound to a sample complexity of $\tilde O(d^{4-2\sqrt 2}/\poly(\eps))$ where $4-2\sqrt 2\approx 1.17157$. 
Our sensitivity bounds also give improvements to a variety of previous results using sensitivity sampling, including Orlicz norm subspace embeddings, robust subspace approximation, and dimension reduction for smoothed $p$-norms. 

Finally, our active sampling results give the first sublinear time algorithms for Kronecker product regression under every $\ell_p$ norm. Previous results required reading the entire $\bfb$ vector in the kernel feature space.
\end{abstract}

\clearpage
\setcounter{page}{1}
\section{Introduction}\label{sec:introduction}

We consider a classic active learning problem: given a design matrix $\bfA\in \R^{n\times d}$ and query access to entries of an unknown target (measurement) vector $\bfb \in \R^n$, how can we compute an approximate minimizer of the regression problem $\min_{\bfx\in\R^d} \|\bfA\bfx - \bfb\|$ while querying \emph{as few entries of $\bfb$ as possible}? This problem arises in applications where labeled data is expensive: viewing a single entry of $\bfb$ might require running a survey, physical experiment, or time-intensive computer simulation \cite{SacksWelchMitchell:1989,Pukelsheim:2006}. 
Concretely, we study the following problem for general vector norms\footnote{Our work will also extend to other loss functions of the form $\sum_{i=1}^n M([\bfA\bfx-\bfb]_i)$ that are not necessarily norms.} $\|\cdot \|$:
\begin{problem}
	\label{prob:informal_problem}
For $\bfA \in \R^{n\times d}, \bfb\in \R^n$, and accuracy parameter $0 < \epsilon \leq 1$, find $\tilde{\bfx}\in \R^d$ satisfying:
\begin{align*}
	\|\bfA\tilde{\bfx} - \bfb\| \leq (1+\epsilon)\cdot \min_{\bfx\in\R^d} \|\bfA\bfx - \bfb\|,
\end{align*}
while reading as few of the entries $\{\bfb(1), \ldots, \bfb(n)\}$ of the target vector $\bfb$ as possible.\footnote{In principal, entries of $\bfb$ can be read \emph{adaptively} -- i.e., we can select indices to query based on the results of other queries. However, the benefits of adaptivity appear limited. Most methods for solving Problem \ref{prob:informal_problem} and those studied in this paper are non-adaptive.} 
\end{problem}
 Notably, the formulation of Problem \ref{prob:informal_problem} makes {no assumptions} on $\bfA$ and $\bfb$. For example, we do not assume that there exists a ground truth $\bar{\bfx}$ and that  $\bfA\bar{\bfx} - \bfb$ is bounded in magnitude, or follows some distribution (e.g., has random Gaussian entries). Under these stronger assumptions, much is known about the problem, which has been studied for decades in the statistics literature on ``optimal design of experiments'', as well as in machine learning \cite{KieferWolfowitz:1959,Pukelsheim:2006,ChaudhuriKakadeNetrapalli:2015}. 
 
 In contrast, progress on the assumption-free version of the problem has only come in recent years, thanks to advances in random matrix theory and randomized numerical linear algebra. This is for good reason: solving Problem \ref{prob:informal_problem} inherently requires choosing which entries of $\bfb$ to query in a \emph{randomized way}: an adversary can easily ``fool'' any deterministic algorithm by concentrating error in $\bfA\bfx - \bfb$ on the indices of $\bfb$ that will be deterministically queried. 
 
\subsection{Prior Work}
\label{sec:prior_work}
\noindent\textbf{Euclidean Norm.} 
Problem \ref{prob:informal_problem} is fully understood when the error is measured in the $\ell_2$ norm, $\|\bfw\|_2 = \left(\sum_{i=1}^n |\bfw_i|^2\right)^{1/2}$ -- i.e., for least squares regression. The typical approach is to subsample and reweight rows (i.e., constraints) of the regression problem and to let $\tilde{\bfx}$ be the minimizer of this sampled problem, which only involves a fraction of the entries in $\bfb$. I.e., letting $\bfS \in \R^{m\times n}$ be a sampling matrix with $m < n$ rows ($\bfS$ has one non-zero entry per row), set $\tilde{\bfx} = \argmin_\bfx \|\bfS \bfA\bfx - \bfS\bfb\|$. When constraints are selected with probability proportional to the \emph{statistical leverage scores} of $\bfA$'s rows, Problem \ref{prob:informal_problem} can be solved with $O(d/\epsilon \cdot \log d)$ samples, and thus $O(d/\epsilon \cdot \log d)$ queries to $\bfb$ \cite{Sarlos:2006,Woodruff:2014,DerezinskiWarmuthHsu:2018}.\footnote{All query complexity bounds in this section are stated for solving Problem \ref{prob:informal_problem} with high constant probability -- e.g., probability $99/100$. In later sections we will include an explicit dependence on a failure probability $\delta$.} 
Using tools from spectral graph sparsification \cite{BatsonSpielmanSrivastava:2012,LeeSun:2018}, Chen and Price recently improved the leverage score sampling result to $O(d/\epsilon)$, which is optimal \cite{ChenPrice:2019a}. 

In practice, methods based on leverage score sampling (also known as ``coherence motivated sampling'') have found many applications. They are widely used in high-dimensional function fitting problems arising in the solution of parametric partial differential equations, where even mild assumptions on $\bfA$ and $\bfb$ are undesirable \cite{CohenDavenportLeviatan:2013,CohenMigliorati:2017,HamptonDoostan:2015}. Methods for solving Problem \ref{prob:informal_problem} in the $\ell_2$ norm also yield robust methods for interpolating sparse Fourier functions, bandlimited and multiband functions, and for data-efficient kernel learning \cite{ChenKanePrice:2016,AvronKapralovMusco:2019,ErdelyiMuscoMusco:2020}.
 
\medskip

\noindent\textbf{Other Norms.} Much less was known about Problem \ref{prob:informal_problem} beyond the $\ell_2$ norm until recent work of Chen and Derezi\'{n}ski \cite{ChenDerezinski:2021}, which proves an upper bound of ${O}(d /\epsilon^2 \cdot \log d)$ queries for the $\ell_1$ norm, $\|\bfw\|_1 = \sum_{i=1}^n |\bfw_i|$. This result is tight up to the $\log d$ factor. A similar result is obtained in \cite{ParulekarParulekarPrice:2021}. Chen and Derezi\'{n}ski also prove a result for $\ell_p$ norms, $\|\bfw\|_p = \left(\sum_{i=1}^n |\bfw_i|^p\right)^{1/p}$, for $p\in (1,2)$, in which they show that $O(d^2/\epsilon^2 \cdot \log d)$ queries suffice to solve Problem \ref{prob:informal_problem}. As for the $\ell_2$ norm, the results for $\ell_1$ and $\ell_p$ are obtained by subsampling rows of the regression problem independently at random. However, instead of sampling with probabilities proportional to the leverage scores, \cite{ChenDerezinski:2021,ParulekarParulekarPrice:2021} employ a natural generalization of these scores known as the \emph{$\ell_p$ Lewis weights} \cite{CohenPeng:2015}. They left open the question of whether a linear in $d$ dependence is possible for $1 < p < 2$, and any bounds at all for $p > 2$. 

Beyond norms, if $\bfb$ is a $\{-1,1\}$ label vector, and the error is measured via the logistic loss, Munteanu et al. \cite{MunteanuSchwiegelshohnSohler:2018} show that $\poly(d, \mu, 1/\epsilon)$ samples suffice, where $\mu$ is a complexity measure of $\bfA$.
 This bound has recently been tightened to $\tilde O(d \mu^2/\epsilon^2)$ \cite{MaiMuscoRao:2021}, using Lewis weight sampling.\footnote{Throughout, $\tilde{O}$ is used to suppress polylogarithmic factors in the argument.} For other loss functions, such as the Tukey loss and Huber's $M$-estimators for robust regression \cite{Fox:2002}, we are not aware of any known results solving Problem \ref{prob:informal_problem}. Chen and Derezi\'{n}ski also pose the open question of obtaining active regression bounds for other loss functions, in particular the Tukey and Huber losses, which are important in practice.

\subsection{Our Contributions}

\paragraph{$\ell_p$ Active Regression.}
Our first main result is a new algorithm for solving Problem \ref{prob:informal_problem} for the $\ell_p$ norm for any $0 < p < \infty$\footnote{Note that for $p\in(0,1)$, $\norm*{\cdot}_p$ is not a norm, but we refer to it as a norm by a standard abuse of notation.}. While near-optimal bounds are known for $p\in\{1,2\}$ \cite{ChenPrice:2019a, ChenDerezinski:2021, ParulekarParulekarPrice:2021}, the problem is far from settled for all other $p$. Previously, active $\ell_p$ regression for $p>2$ and $0<p<1$ had no known nontrivial algorithms with $(1+\eps)$ relative error, and the only known approach was to read all $n$ entries of $\bfb$ and solve the problem using offline results. A natural question is whether a sublinear query complexity is possible in these regimes. For $p\in(1,2)$, \cite{ChenDerezinski:2021} achieved an algorithm making $O(d^2/\eps^2\cdot \log d)$ queries, thus achieving the first sublinear query complexity. One of their main open questions is whether the dependence on $d$ can be improved to linear or not. Our main result answers all of these questions.



\begin{theorem}[Main Result for Active $\ell_p$ Regression]\label{thm:main} Given $0 < p < \infty$, $\bfA \in \R^{n \times d}$, and query access to $\bfb\in \R^n$, there is an algorithm (Algorithm \ref{alg:clip2}) that solves Problem \ref{prob:informal_problem} for the $\ell_p$-norm with probability $99/100$ which makes $m$ queries in $\bfb$, where
\[
	m = \begin{dcases}
    O\parens*{\frac{d}{\eps^{2}}(\log d)^2(\log(d/\eps))} & p\in (0,1) \\
    O\parens*{\frac{d}{\eps}(\log d)^2(\log(d/\eps))} & p\in (1,2) \\
    O\parens*{\frac{d^{p/2}}{\eps^{p}}(\log d)^2(\log(d/\eps))^{p-1}} & p\in(2,\infty)
    \end{dcases}.
\]
\end{theorem}

Our main algorithm, Algorithm \ref{alg:clip2}, is introduced and analyzed in Section \ref{sec:main}, culminating in Theorem \ref{thm:epsL}. We complement our algorithmic result with various new lower bounds which show the tightness of our algorithm, proven in Section \ref{sec:lower}. For $p\in(0,2)$, our dependence on $d$ and $\eps$ in the query complexity are simultaneously tight up to polylogarithmic factors; we show an $\Omega(d/\eps^2)$ lower bound for $p\in(0,1)$ and an $\Omega(d/\eps)$ lower bound for $p\in(1,2)$. For $p>2$, our dependence on $d$ is tight due to a lower bound of $\Omega(d^{p/2})$ which we show, while our $\eps$ dependence is off by at most factor of $\eps$ due to an $\Omega(\eps^{1-p})$ lower bound for the one-dimensional \emph{$\ell_p$ power means} problem in Theorem 3 of \cite{CSS2021}. Note that our active regression lower bounds for $p\in(0,2)$ improve this previous power means lower bound.

Notably, we achieve a linear dependence on $\eps$ for $p\in(1,2)$, which is perhaps surprising given that all previous known approaches to dimension reduction for $\ell_p$ regression relied on preserving the $\ell_p$ norm of all vectors in a subspace up to $(1\pm\eps)$ factors \cite{CohenPeng:2015}, which requires $\Omega(d/\eps^2)$ dimensions \cite{LWW20}. It also demonstrates a separation in the query complexity for $p \leq 1$ and $1 < p < 2$, due to a lower bound of $\Omega(d/\eps^2)$ for $p = 1$ \cite{ChenDerezinski:2021, ParulekarParulekarPrice:2021} as well as for $p\in(0,1)$ which we show.

Note that Theorem \ref{thm:main} is stated to solve Problem \ref{prob:informal_problem} with constant probability, $99/100$. In general, we show how to obtain $1-\delta$ probability with dependence on $\delta$ that is only polylogarithmic in $1/\delta$. In Section \ref{sec:lower}, we show that any algorithm that simply samples rows of the regression problem and solves the sampled problem must suffer a $1/\delta^{p-1}$ dependence. Indeed, such a loss is seen in the algorithm of \cite{ChenDerezinski:2021} for $p\in(1,2)$. Thus, a success probability boosting routine, as we give in Section \ref{sec:main}, is required to obtain a $\poly\log(1/\delta)$ dependence.

\bgroup
\def\arraystretch{1.25}
\begin{table}[h]
	\centering
	\caption{Upper and lower bounds for Problem \ref{prob:informal_problem} for various norms and loss functions. New results are highlighted in blue. 
	For simplicity, we suppress leading constants depending only on $p$, as well as $\poly\log n$ factors for $M$-estimator results. Our results significantly strengthen and generalize prior work, providing the first query complexity result with a tight $d$ dependence for $\ell_p$ norms. We also give the first results for $M$-estimators as well as $\ell_p$ norms for $p>2$ and $p\in(0,1)$, and matching lower bounds in many cases. }
	\begin{tabular}{ |l|ll|ll|ll|l|} 
		\hline
		\textbf{Loss Function} 
		&\multicolumn{2}{c|}{\textbf{Prior Work} } & \multicolumn{2}{c|}{\textbf{Our Work}} & \multicolumn{2}{c|}{\textbf{Lower Bound}} \\
		\hline
		$\ell_2$					& $O(d/\epsilon)$ &\hspace{-.5em}\cite{ChenPrice:2019a} & \multicolumn{2}{c|}{\textbf{--}} & $\Omega(d/\epsilon)$ & \cite{ChenPrice:2019a}\\ 
		$\ell_1$ 						& $\tilde{O}(d/\epsilon^2)$ &\hspace{-.5em}\cite{ChenDerezinski:2021} & \multicolumn{2}{c|}{\textbf{--}} & $\Omega(d/\epsilon^2)$& \cite{ChenDerezinski:2021} \\ 
		$\ell_p$, $p\in (1,2)$ & $\tilde{O}(d^2/\epsilon^2)$ &\hspace{-.5em}\cite{ChenDerezinski:2021} & \cellcolor{blue!15}$\tilde{O}(d/\eps)$ &\cellcolor{blue!15} \hspace{-.5em}(Thm. \ref{thm:main}) &\cellcolor{blue!15}$\Omega(d/\eps)$ &\cellcolor{blue!15} \hspace{-1em}(Thm. \ref{thm:p-in-1-2-lb-d-dim}) \\ 
		$\ell_p$, $p > 2$ 		& \multicolumn{2}{c|}{\textbf{--}} &\cellcolor{blue!15}$\tilde{O}(d^{\frac{p}{2}}/\eps^p)$ & \cellcolor{blue!15} \hspace{-.5em}(Thm. \ref{thm:main}) &\cellcolor{blue!15}$\Omega(d^{\frac{p}{2}}+\eps^{1-p})$ &\cellcolor{blue!15} \hspace{-1em}(Thm. \ref{thm:general_lb}) \\ 
		$\ell_p$, $p \in (0,1)$ 		& \multicolumn{2}{c|}{\textbf{--}} &\cellcolor{blue!15}$\tilde{O}(d/\eps^2)$ & \cellcolor{blue!15} \hspace{-.5em}(Thm. \ref{thm:main}) &\cellcolor{blue!15}$\Omega(d/\eps^2)$ &\cellcolor{blue!15} \hspace{-1em}(Thm. \ref{thm:p-in-0-1-lb-d-dim}) \\ 
		$M$-estimators		& \multicolumn{2}{c|}{\textbf{--}} &\cellcolor{blue!15}$\tilde{O}(d^{\frac{p}2+O(1)}/\eps^c)$ & \cellcolor{blue!15} \hspace{-.5em}(Thm. \ref{thm:m-active-regression}) &$\Omega(d)$& \\ 
		Huber loss		& \multicolumn{2}{c|}{\textbf{--}} &\cellcolor{blue!15}$\tilde{O}(d^{4-2\sqrt2}/\eps^c)$ & \cellcolor{blue!15} \hspace{-.5em}(Thm. \ref{thm:active-huber}) &$\Omega(d)$ &\\ 
		Tukey loss		& \multicolumn{2}{c|}{\textbf{--}} &\cellcolor{blue!15}$\tilde{O}(d^{\frac{p}{2}+O(1)}/\eps^c)$ & \cellcolor{blue!15} \hspace{-.5em}(Thm. \ref{thm:tuk}) &$\Omega(d)$ &\\ 
		\hline
	\end{tabular}
\label{tab:main_results}
\end{table}
\egroup


\paragraph{Sensitivity Bounds and Active Regression for General Losses.}
We show that our approach to solving Problem \ref{prob:informal_problem} for $\ell_p$ norms generalizes to a broad class of loss functions known as \emph{$M$-estimators} \cite{ClarksonWoodruff:2015b}, which take the form $\sum_{i=1}^n M([\bfA\bfx - \bfb]_i)$. The only properties that we require are that we can (1) compute a constant factor approximation to Problem \ref{prob:informal_problem} (2) the loss function obeys approximate variants of the triangle inequality and (3) we can bound the so-called \emph{sensitivities} of the loss, which bound the fraction of the total loss that can be concentrated at any coordinate $i\in[n]$ (see Equation \eqref{eq:sensitivity-def}). 

To the best of our knowledge, the only prior result achieving sensitivity bounds for general loss functions is \cite{TukanMaaloufFeldman:2020}. However, this work makes use of L\"owner-John ellipsoids, which leads to practically inefficient algorithms, and loses a factor of $\sqrt d$ in the total sensitivity due to the ellipsoidal rounding. As our second main result, we develop new sensitivity bounds for $M$-estimators that significantly simplify and improve this result. 

\begin{theorem}[Main Result for Sensitivity Bounds, Informal Version of Theorem \ref{thm:m-estimator-alg}]
Let $\bfA\in\mathbb R^{n\times d}$ and let $M$ be an $M$-estimator loss with at most degree $p$ growth. Then, with probability at least $99/100$, Algorithm \ref{alg:m-estimator-alg-sensitivity} computes $M$-sensitivity upper bounds which sum to at most $O(d^{1\lor (p/2)}\log^2 n + \tau)$\footnote{Here, $a\lor b$ denotes $\max(a,b)$, and $a\land b$ denotes $\min(a,b)$.} in time at most $\tilde O(\nnz(\bfA) + nd^{C}/\tau)$ for some $C = O(1)$.
\end{theorem}

Our approach to sensitivity bounds only relies on hashing and the computation of $\ell_p$ Lewis weights \cite{CohenPeng:2015, FazelLeePadmanabhan:2021}, and avoids the computation of L\"owner-John ellipsoids. This allows for input sparsity time algorithms, and answers an open question of \cite{TukanMaaloufFeldman:2020} on avoiding L\"owner--John ellipsoids in the computation of sensitivities. Note that our dependence on $d$ matches the sensitivity bounds for the $\ell_p$ loss and is thus tight. We also show in Section \ref{subsec:sensitivity-lower-bounds} that the dependence on $n$ is necessary for loss functions such as the Huber and Tukey losses. Furthermore, our algorithm can be turned into a non-algorithmic proof that the sensitivities sum to at most $O(d^{1\lor(p/2)}\log n)$ for these $M$-estimators as shown in Section \ref{subsec:sharper-sensitivity}; this is in fact tight for the Tukey loss by our lower bound of $\Omega(d\log n)$ in Section \ref{subsec:sensitivity-lower-bounds}. Thus, we obtain the first tight bounds on the sum of sensitivities, for losses other than $\ell_p$. Overall, we make significant progress on generalizing the theory of matrix approximation beyond $\ell_p$ losses to handle general $M$-estimators, which is a direction that has recently received much attention \cite{FS2012, ClarksonWoodruff:2015b, ClarksonWoodruff:2015a, ClarksonWangWoodruff:2019, SongWoodruffZhong:2019, TukanMaaloufFeldman:2020}.

Combined with our active regression techniques, our sensitivity bounds yield active regression algorithms for general loss functions, including the Huber and Tukey losses, answering an open question of Chen and Derezi\'{n}ski \cite{ChenDerezinski:2021}. Note that prior to our work, no sublinear query complexity was known for any $M$-estimator regression, besides the $\ell_2$ and $\ell_1$ losses. 

Furthermore, our new sensitivity bounds imply significant improvements in previous results using sensitivity sampling, beyond active regression, including Orlicz norm subspace embeddings \cite{SongWangYang:2019} and robust subspace approximation \cite{ClarksonWoodruff:2015a}. We believe that our general technique here will find other further applications, and leave it as an open question to do so.

Our new sensitivity computation algorithm for general losses is given in Algorithm \ref{alg:m-estimator-alg-sensitivity}, and its guarantees are stated and proven in Theorem \ref{thm:m-estimator-alg}. Its application to active regression is given in Theorem \ref{thm:m-active-regression}, its applications to Orlicz norm subspace embeddings are discussed in Section \ref{sec:orlicz}, and its applications to robust subspace approximation are discussed in Section \ref{sec:robust-subspace-approx}.

\paragraph{Subspace Embeddings for Orlicz Norms.}

\emph{Orlicz norms} can be viewed as scale-invariant extensions of $M$-estimators, and have recently attracted attention as a general class of norms that admit efficient dimensionality reduction results \cite{AndoniLinSheng:2018, SongWangYang:2019}. In particular, \cite{SongWangYang:2019} apply sensitivity sampling to obtain \emph{subspace embeddings} for Orlicz norms, which yields a small weighted subset $\tilde\bfA$ of rows of a matrix $\bfA\in\mathbb R^{n\times d}$ such that $\norm{\tilde\bfA\bfx} = (1\pm\eps)\norm{\bfA\bfx}$\footnote{For $a,b\geq 0$, $a \pm b$ denotes a number $c$ such that $a-b \leq c \leq a+b$.} for all $\bfx\in\mathbb R^d$. However, the number of rows required by \cite{SongWangYang:2019} is a large polynomial in $d$, and is also restricted to Orlicz norms of at most quadratic growth. We show that by applying our new sensitivity bounds, we can obtain subspace embeddings for Orlicz norms with $d^{2\lor(p/2+1)}\poly(\log n,\eps^{-1})$ rows, for any Orlicz norm with a polynomial growth bound of degree $p$.

\paragraph{Robust Subspace Approximation.}

The robust subspace approximation problem generalizes the classical low rank approximation problem of finding a rank $k$ projection $\bfX$ minimizing $\norm*{\bfA\bfX - \bfA}_F$ by replacing the Frobenius norm with an extension of $M$-estimators to matrix norms. \cite{ClarksonWoodruff:2015a} showed the first dimensionality reduction results for this problem for a general class of $M$-estimators of at most quadratic growth via a recursive sampling scheme using the sensitivity sampling framework. However, due to the use of looser sensitivity bounds, they suffer an undesirable factor of $(\log n)^{O(\log k)}$ in their sample complexities. Our new sensitivity bounds allow us to remove this factor, giving a dimension reduction result into a $\poly(k,\log n,\eps^{-1})\times \poly(k,\log n,\eps^{-1})$ instance. We also extend their method beyond quadratic growth, to any degree $p$ polynomial growth.

\paragraph{Active Regression for the Huber Loss.}
Our active regression result for general $M$-estimators discussed above is loose by a factor of $d$ in the sample complexity, compared to our $\ell_p$ active regression results. 
This is attributed to the use of our net argument for general $M$-estimators, whereas our $\ell_p$ active regression results can make use of more sophisticated chaining arguments of \cite{BourgainLindenstraussMilman:1989,LT2011,SchechtmanZvavitch:2001}. A natural question is if this gap can be improved.

We consider the important special case of the Huber loss, which is defined as follows:
\begin{definition}[Huber loss \cite{Huber:1992}]\label{dfn:huber-loss}
The Huber loss of width $\tau\geq 0$ is defined as
\begin{align*}
    H(x) &\coloneqq \begin{cases}
        x^2 / 2\tau  & \text{if $\abs{x}\leq \tau$} \\
        \abs{x} - \tau/2 & \text{otherwise}
    \end{cases}
\end{align*}
and the Huber norm\footnote{Again, this is a standard abuse of notation, and the Huber norm is not an actual norm.} is defined as $\norm*{\bfy}_H \coloneqq \sqrt{\sum_{i=1}^n H(\bfy(i))}$.
\end{definition}

The Huber loss is ``arguably one of the most widely used $M$-estimators'' \cite{ClarksonWoodruff:2015b}, owing its popularity to its convexity and differentiability properties of $\ell_2$, which allows for efficient algorithms (see, e.g., \cite{MangasarianMusicant:2000} for algorithms), in combination with its robustness properties of $\ell_1$ \cite{GuittonSymes:1999}. This makes it widely applicable in practical big data settings (see, e.g., \cite{BaldaufSilva:2012} for a list of popular software packages implementing Huber regression as well as references that make use of Huber regression). Variations on Huber regression have also recently been shown to hold theoretical guarantees in the robust statistics literature (see, e.g., \cite{Loh:2017, Loh:2018} and references therein).

For the Huber loss, we show that it is indeed possible to leverage the chaining techniques in order to obtain improved sample complexity bounds for active regression. We show that we can improve beyond the $d^2$ bound obtained by our general $M$-estimator algorithm as applied to the Huber loss, and obtain a sample complexity of $O(d^{4-2\sqrt 2}\poly(\log n,\eps^{-1}))$ queries to $\bfb$, where $4-2\sqrt2 \approx 1.17157$. For this result, we use the chaining techniques of \cite{BourgainLindenstraussMilman:1989}, which provides a more flexible alternative to \cite{LT2011}, but requires more technical effort to adapt to the active setting.

\begin{restatable}[Main Result for Huber Active Regression]{theorem}{ActiveHuber}\label{thm:active-huber}
    Let $\bfA \in \R^{n \times d}$, $\bfb \in \R^n$. Then, with probability at least $99/100$, Algorithm \ref{alg:huber-active} returns a $\tilde\bfx$ satisfying 
    \[
        \norm{\bfA\tilde \bfx - \bfb}_H \le (1+\epsilon) \cdot \min_\bfx\norm*{\bfA\bfx - \bfb}_H
    \]
    Furthermore, the algorithm reads at most $d^{4-2\sqrt2}\poly(\log n,\eps^{-1})$ entries of $\bfb$. 
\end{restatable}

Our techniques also yield a subspace embedding result, which constructs a weighted subset $\tilde\bfA$ of $O(d^{4-2\sqrt 2}\poly(\log n,\eps^{-1}))$ rows such that $\norm{\tilde\bfA\bfx}_H = (1\pm\eps)\norm{\tilde\bfA\bfx}_H$ for all $\bfx\in\mathbb R^d$, contained in Theorem \ref{thm:huber-subspace-embedding}. Previously, the best known dimension reduction bound for Huber regression, even in the non-active setting, was $d^4$ due to \cite{ClarksonWoodruff:2015b}. 

Furthermore, this is, to the best of our knowledge, the first example of a loss function other than $\ell_p$ which achieves a sensitivity sampling bound of better than $d^2$, despite the fact that such results have been sought in many works \cite{ClarksonWoodruff:2015b, ClarksonWoodruff:2015a, SongWangYang:2019, ClarksonWangWoodruff:2019, TukanMaaloufFeldman:2020, GPV2021}. The reason for this is that $d^2$ is a natural bound for sensitivity sampling, attributed to one $d$ factor from the sum of sensitivities and one $d$ factor from carrying out a union bound over a net of $\exp(d)$ vectors. For $\ell_p$ norms, the arguments of \cite{BourgainLindenstraussMilman:1989, SchechtmanZvavitch:2001} and their subsequent improvements avoid this problem by using a more sophisticated chaining argument. However, these arguments use the structure of $\ell_p$ spaces in crucial ways, such as isometric changes of density using Lewis weights \cite{JS2001}, and do not generalize easily to other loss functions.

It is an interesting open question to determine whether our dimension reduction bound for the Huber loss can be improved all the way down to $d$.
Our results for the Huber loss are found in Section \ref{sec:huber}.



\paragraph{Dimension Reduction for Gamma Functions for Faster $\ell_p$ Regression.}

One particularly important application of sampling-based dimension reduction for loss functions beyond $\ell_p$ losses is, perhaps surprisingly, in the design of algorithms for $\ell_p$ regression. The work of \cite{BCLL2018} introduces \emph{gamma functions} $\gamma_p$, which are generalizations of the Huber loss which behave quadratically near the origin and like $|x|^p$ away from the origin, in the context of algorithms for $\ell_p$ regression. Subsequently, \cite{AKPS2019} obtained even faster algorithms by using constant factor approximations of $\gamma_p$ regression as a subroutine, in which the $\gamma_p$ loss is minimized over a subspace. Dimension reduction for this loss function has been a crucial ingredient for recent results in fast algorithms for $\ell_p$ regression \cite{ABKS2021, GPV2021}. In particular, \cite{ABKS2021} highlighted the open question of designing sparsification methods for $\gamma_p$ functions for $p\in(1,2)$, and \cite{GPV2021} designed a sampling algorithm which samples $\tilde O(d^3)$ rows. By generalizing our dimension reduction techniques for the Huber loss, we obtain an algorithm which samples at most $O(d^{4-2\sqrt 2}\poly(\log n,\eps^{-1}))$ rows for any $p\in[1,2)$, and improves to $O(d\poly(\log n,\eps^{-1}))$ rows as $p\to 2$ (see Figure \ref{fig:l2-lq} for the trade-off curve). We give a further discussion in Section \ref{sec:gamma}.

\paragraph{Kronecker Product Regression.} Beyond applications in data-efficient regression, Theorem \ref{thm:main} implies the first sublinear time algorithm for Kronecker product regression in any $\ell_p$ norm, where explicitly constructing the vector $\bfb$ is a computational bottleneck. We detail this result in Section \ref{sec:kronecker}. 
In $q$-th order Kronecker product regression, one is 
given matrices $\bfA_1, \bfA_2,\ldots, \bfA_q$, where $\bfA_i \in \mathbb{R}^{n_i \times d_i}$,
as well as a vector $\bfb \in \mathbb{R}^{n_1 n_2 \cdots n_q}$, and 
the goal is to solve:
$\min_{\bfx \in \R^{d_1 d_2 \cdots d_q}} \|(\bfA_1 \otimes \bfA_2 \cdots \otimes \bfA_q)\bfx - \bfb\|_p,$
where $\otimes$ denotes the Kronecker product. Typically $\prod_{i=1}^q d_i$ is much less
than $\prod_{i=1}^q n_i$, and the goal is to obtain algorithms that do not explicitly form 
$\bfA_1 \otimes \bfA_2 \otimes \cdots \otimes \bfA_q$ or $\bfb$, which is too expensive. Our results yield the first algorithm for Kronecker product regression, 
for every $p \geq 1$, whose running time {\it does not depend on $\nnz(\bfb)$}, whereas previous results had a linear dependence on $\nnz(\bfb)$, which can be as large as $\prod_{i=1}^q n_i$ \cite{DJSSW19}.

\begin{theorem}
Let $q \geq 1$, $p \geq 1$ be constant, and $\epsilon > 0$. Kronecker product regression
can be solved up to a $(1+\epsilon)$-factor with constant probability in 
$\tilde{O}(\sum_{i=1}^q \nnz(\bfA_i) + \poly(\prod_{i=1}^q d_i / \epsilon))$ time.
\end{theorem}

\subsection{Technical Approach}

\subsubsection{\texorpdfstring{$\ell_p$}{lp} Active Regression}\label{sec:lp-active-techniques}
Our algorithm for solving Problem \ref{prob:informal_problem} uses a novel variation on the ``sample-and-solve'' approach. In particular, we randomly select a row sampling matrix $\bfS\in \R^{m\times n}$ and return $\tilde{\bfx} = \argmin_x\|\bfS \bfA\bfx - \bfS\bfb\|$, which only requires querying $m$ entries of $\bfb$ (those that appear in $\bfS\bfb$).
To get tight bounds for $\ell_p$ regression, we select $\bfS$ using $\ell_p$ Lewis weight sampling, a generalization of leverage score sampling for $\ell_2$. 

It can be shown (Lemma \ref{lem:lwBound}) that the $\ell_p$ Lewis weights upper bound the \emph{$\ell_p$ sensitivities} of $\bfA$, a measure of importance for the rows of $\bfA$. The $\ell_p$ sensitivity of the $i^\text{th}$ row of $\bfA$ is defined as
\[
	\bfs^p_i(\bfA) \coloneqq \max_{\bfx \in \R^d\setminus\{0\}} \frac{|[\bfA\bfx](i)|^p}{\norm{\bfA\bfx}_p^p},
\]
where $[\bfA\bfx](i)$ denotes the $i^\text{th}$ entry of the vector $\bfA\bfx$, and captures how large the $i^\text{th}$ entry of any $\bfA\bfx\in\colspan(\bfA)$ can be, relative to the $\ell_p$ norm. A standard scalar Bernstein bound shows that if $\bfS$ samples rows with probabilities that upper bound the sensitivities, then $\|\bfS \bfA\bfx\|_p^p = (1\pm\epsilon)\|\bfA\bfx\|_p^p$ with high probability, for each $\bfx\in\mathbb R^d$. An $\epsilon$-net argument can extend this to a \emph{for all} claim. 

\paragraph{Prior Approaches to $\ell_p$ Active Regression.}

While the above ideas give an approach for standard $\ell_p$ regression, this bound does not suffice for \emph{active} $\ell_p$ regression. To solve Problem \ref{prob:informal_problem}, we actually want that $\|\bfS (\bfA\bfx - \bfb)\|_p^p = (1\pm\epsilon)\|\bfA\bfx-\bfb\|_p^p$ for any $\bfx$. Will $\bfS$ provide such a guarantee? The main problem, as discussed in \cite{ChenDerezinski:2021, ParulekarParulekarPrice:2021} is that the translation by $\bfb$ may introduce outliers, i.e., entries with high sensitivity which are not captured by the sensitivity scores of $\bfA$. As shown by \cite{ChenDerezinski:2021, ParulekarParulekarPrice:2021}, in the case of $\ell_1$, the special structure of the loss function provides a solution. Indeed, by the triangle inequality,
\[
    \abs*{\parens*{\abs*{[\bfA\bfx-\bfb](i)} - \abs*{[\bfA\bfx^*-\bfb](i)}}} \leq \abs*{[\bfA\bfx-\bfA\bfx^*](i)} = \abs*{[\bfA(\bfx-\bfx^*)](i)}
\]
where $\bfx^*$ is the optimal solution. This fact can be used to show that sampling by the sensitivities of $\bfA$ preserves the \emph{differences} between the cost of any $\bfx$ and the optimal $\bfx^*$. However, such a proof cannot work for $p\neq 1$, in which case we do not have such a nice inequality. For $p\in(1,2)$, \cite{ChenDerezinski:2021} take the approach of bounding the residual error terms from the above approach by using a Taylor approximation, but this leads to a sample complexity of at least $d^2$.

\paragraph{Our Solution: Partitions by Sensitivity.}

Instead of relying on the technique of ``cancelling out the outliers'', we take a conceptually different approach. We proceed in two stages, where we (1) first find a constant factor solution $\bfx_c$ such that $\norm*{\bfA\bfx_c-\bfb}_p^p \leq O(1)\cdot\min_\bfx\norm*{\bfA\bfx-\bfb}_p^p$ using an idea of \cite{DasguptaDrineasHarb:2009} and replace $\bfb$ by the residual vector $\bfb - \bfA\bfx_c$, and then (2) conceptually partition the target vector $\bfb$ into two sets of coordinates, the coordinates $i\in[n]$ that are small enough to be comparable to the sensitivity $\bfs_i^p(\bfA)$ and those that are much larger. That is, we consider the coordinates $i\in[n]$ such that $|\bfb(i)|^p / \|\bfb\|_p^p \leq C\cdot \bfs^p_i(\bfA)$ for some $C>0$, and all other coordinates. For the former set of coordinates, one can check that the Bernstein bound still applies, and $\bfS$ \emph{does} preserve the norm of $\|\bfA\bfx - \bfb\|_p^p$, when restricted to these coordinates. On the other hand, for the latter set of coordinates, we show that \emph{no vector} of the form $\bfA\bfx$ can both be close to $\bfb(i)$ in its $i^\text{th}$ entry, and still close to the remainder of $\bfb$ -- the $i^\text{th}$ entry is simply too large in magnitude. In particular, to have $[\bfA\bfx](i)$ close to $\bfb(i)$, we would require $\norm{\bfA\bfx}_p^p$ to be much larger than $\norm{\bfb}_p^p$, which by our preprocessing step, is on the order of the optimal cost $\min_\bfx \|\bfA\bfx - \bfb\|_p$. Via the triangle inequality, this implies that $\bfA\bfx$ must be far from an optimal solution. Thus, we can argue that any near-optimal solution to $\min_\bfx \|\bfA\bfx - \bfb\|_p$ \emph{does not need to fit $\bfb(i)$ with $|\bfb(i)|^p / \|\bfb\|_p^p$ much larger than $\bfs^p_i(\bfA)$}. We can effectively ignore the contribution of these rows.

Another technical challenge remains: to obtain an optimal dimension dependence, we need a refined $\epsilon$-net argument to make \emph{for all} statements about $\bfx\in \R^d$. To do so, we adapt the chaining arguments of Bourgain, Lindenstrauss, and Milman \cite{BourgainLindenstraussMilman:1989} and Ledoux and Talagrand \cite{LT2011} to the active regression setting, avoiding the standard $\eps$-net and union bound argument used by, e.g., \cite{Schechtman:1987}. Although both \cite{BourgainLindenstraussMilman:1989} and \cite{LT2011} provide such approaches, we adapt the (slightly) more complex recursive Lewis weight sampling algorithm of \cite{LT2011} in order to obtain tighter dependencies on $\eps$. The streamlined proof of \cite{LT2011} also adapts nicely to the active regression setting with minimal changes to the original argument. We note here that we will later also need to adapt the much more involved \cite{BourgainLindenstraussMilman:1989} argument to handle the Huber loss, in which case the proof of \cite{BourgainLindenstraussMilman:1989} allows for more fine-grained control over bounding the sensitivity sampling algorithm, but requires a more complex argument based on carefully partitioning the coordinates of the target vector $\bfb$ based on sensitivity weight classes.
Aside from our new application of \cite{BourgainLindenstraussMilman:1989, LT2011}, we hope that by translating the arguments of \cite{BourgainLindenstraussMilman:1989,LT2011} to the language of theoretical computer science and matrix approximation, they will find further applications to randomized algorithm design. 

We note that our algorithm is quite a bit more involved than a simple scheme of sampling proportionally to Lewis weights and solving. This is for good reasons. Not only is it not clear that such an approach works at all, we show in Theorem \ref{thm:delta-lb} that for any $p>1$, any algorithm which simply samples reweighted rows and solves the system must have a polynomial dependence on $1/\delta$ in the query complexity, while our algorithm achieves a $\log\frac1\delta$ dependence, by solving residual problems of a constant factor solution. Thus, our two-stage approach is necessary to achieve our $\delta$ dependence. Furthermore, the best known analysis of a simple ``one-shot'' Lewis weight sampling scheme suffers in $\eps$ dependencies for $p>2$, where the one-shot approach is only known to give a $\tilde O(d^{p/2}/\eps^5)$ bound for subspace embeddings \cite{BourgainLindenstraussMilman:1989, CohenPeng:2015}, whose losses translate to losses for our active regression algorithms as well, while the recursive approach can achieve $\tilde O(d^{p/2}/\eps^2)$ \cite{LT2011}. While the \cite{LT2011} result is an existential result, we provide an analysis of the \cite{LT2011} proof in Theorem \ref{thm:lp-subspace-embedding} to turn it into a randomized algorithm with logarithmic dependencies on the failure rate $\delta$, which achieves the best known dependence on $d$, $\eps$, and $\delta$, up to logarithmic factors. We further modify this subspace embedding result for active regression, to optimize our $\eps$ dependence.

\paragraph{Optimized $\eps$ Dependence for $p\in(1, 2)$.}

For $p\in(1,2)$, the above argument gives a bound of $\tilde O(d/\eps^2)$. While the linear dependence on $d$ is optimal, confirming the conjecture of \cite{ChenDerezinski:2021}, it has a quadratic dependence on $\eps$, which is in fact \emph{not} optimal. We now show how to improve our bound to $\tilde O(d/\eps)$, which requires additional ideas. We first use strong convexity to show that a $(1+\gamma)$-approximate solution $\hat\bfx\in\mathbb R^d$ satisfying $\norm*{\bfA\hat\bfx-\bfb}_p^p \leq (1+\gamma)\norm*{\bfA\bfx^*-\bfb}_p^p$, for the optimal solution $\bfx^*\in\mathbb R^d$, in fact satisfies $\norm*{\bfA\hat\bfx-\bfA\bfx^*}_p \leq O(\sqrt\gamma)\norm*{\bfA\bfx^*-\bfb}_p$. Then, by using that $\hat\bfx$ is close to the optimal solution, we show an improved bound on the difference in the objective values of $\hat\bfx$ and $\bfx^*$, i.e., that $\hat\bfx$ actually has an approximation ratio better than $(1+\gamma)$. We then iterate this argument until we obtain a $(1+\eps)$-approximation using $\tilde O(d/\eps)$ queries, at which point we can no longer get improvements. The chaining argument used in this proof, while similar to the previous proofs, has a different geometry than the previous chaining arguments, and requires additional ideas.

Our upper bound is tight up to polylogarithmic factors due to a lower bound we show in Theorem \ref{thm:p-in-1-2-lb-d-dim}. This improves an $\ell_p$ power means lower bound of \cite{CSS2021}, who only showed a lower bound of $\Omega(\eps^{1-p})$ queries. Unfortunately, we are unable to port our algorithmic techniques to the $\ell_p$ power means problem in high dimensions, due to difficulties in adapting their chaining argument.


\subsubsection{Sensitivity Bounds}

The notion of $\ell_p$ sensitivities, as discussed above, naturally generalizes to loss functions that take the form of coordinate-wise sums. Consider a loss function $M$ and an $n\times d$ matrix $\bfA$. Then, the \emph{sensitivity} of the $i$th coordinate with respect to the loss function $M$ is defined as
\begin{equation}\label{eq:sensitivity-def}
	\bfs_i^M(\bfA) \coloneqq \sup_{\bfx\in\mathbb R^d\setminus\{0\}}\frac{M([\bfA\bfx](i))}{\sum_{j=1}^n M([\bfA\bfx](j))}.
\end{equation}
It is well-established that sensitivities provide a general framework for sampling rows of $\bfA$ that approximate $\bfA$ well under the loss function $M$ \cite{FeldmanLangberg:2011}. While a rich literature exists for $\ell_p$ \cite{DasguptaDrineasHarb:2009,SohlerWoodruff:2011, CohenPeng:2015}, little was known about the approximation of sensitivities for general loss functions until \cite{TukanMaaloufFeldman:2020}, which used L\"owner-John ellipsoids to obtain sensitivity bounds for a general family of near-convex losses. However, the computation of L\"owner-John ellipsoids has running time that is a large polynomial in $n$ and $d$, and is impractical for large datasets, and \cite{TukanMaaloufFeldman:2020} raise the open question of obtaining general sensitivity bounds without this expensive subroutine.

Our approach to new sensitivity bounds significantly generalizes the approach of \cite{ClarksonWangWoodruff:2019}, whose algorithm can be seen as a way to use hashing and Lewis weights to compute sensitivities for the Tukey loss, but heavily uses the properties of the Tukey loss in their analysis.

Suppose that a coordinate $i\in[n]$ has $M$-sensitivity $\alpha\in(0,1]$, that is,
\[
	\bfs_i^M(\bfA) = \sup_{\bfx\in\mathbb R^d\setminus\{0\}}\frac{M([\bfA\bfx](i))}{\sum_{j=1}^n M([\bfA\bfx](j))} = \alpha,
\]
and let $\bfy = \bfA\bfx$ witness this supremum, and assume for simplicity that $\sum_{i=1}^n M([\bfA\bfx](j)) = 1$. Note then that there can be at most $1/\alpha$ entries $j\in[n]$ of $\bfy$ that have coordinate value $M(\bfy_j)\geq M(\bfy_i) = \alpha$. Then, if we randomly hash the $n$ coordinates into $O(1/\alpha)$ buckets, then with constant probability, coordinate $i$ will be isolated from any other entry with $M(\bfy_j)\geq M(\bfy_i) = \alpha$. Now if $M$ is monotonic, then this means that $\bfy_i$ is the largest coordinate in its hash bucket. Furthermore, the sum of the $M$-mass of all of the other coordinates in $i$'s hash bucket is only an $\alpha$ fraction of the total $M$-mass, so entry $i$ carries a constant fraction of the $M$-mass in its bucket. In this case, it can be shown that entry $i$ must in fact carry a constant fraction of the \emph{$\ell_2$ mass} inside its hash bucket, if $M$ is a function of at most quadratic growth. This is because when we switch the error metric from $M$ to $\ell_2$, then the largest entry will have the largest increase in its normalized contribution (see Lemma \ref{lem:max-sensitivity-comparison}). This means that row $i$ must have an $\ell_2$ leverage score of $\Omega(1)$, in this hash bucket.

This leads to the following algorithm: (1) hash the $n$ coordinates into $O(1/\alpha)$ buckets (2) compute $\ell_2$ leverage scores for each bucket (3) assign an $M$-sensitivity of $\alpha$ for any coordinate that has $\ell_2$ leverage score $\Omega(1)$. In each of the $O(1/\alpha)$ buckets, we will find at most $O(d)$ coordinates with leverage score at least $\Omega(1)$, so we assign an $M$-sensitivity of $\alpha$ to at most $O(d/\alpha)$ coordinates, which has a total sensitivity contribution of $O(d)$. By repeating this for $O(\log n)$ guesses of $\alpha$ in powers of $2$, this gives a total sensitivity bound of $O(d\log n)$. The constant probability events in the hashing process can be boosted to probability $1-1/\poly(n)$ by repeating the procedure $O(\log n)$ times, which increases the total sensitivity to roughly $O(d\log^2 n)$.

By sampling according to these sensitivities and applying a union bound over a net, we obtain the first active regression algorithms for general loss functions in Theorem \ref{thm:m-active-regression}. Note that this result is made possible by a combination of both our new sensitivity bounds for $M$-estimators and our new active regression techniques as discussed in Section \ref{sec:lp-active-techniques}. Furthermore, we demonstrate other applications of our sensitivity bound result, showing how to improve Orlicz norm subspace embeddings in Section \ref{sec:orlicz} and robust subspace approximation in Section \ref{sec:robust-subspace-approx}.

\subsubsection{Subspace Embeddings and Active Regression for the Huber Loss}

As discussed previously, we tackle the question of leveraging the theory of \cite{BourgainLindenstraussMilman:1989} nets in order to obtain sample complexities for the Huber loss beyond $d^2$. Our algorithmic framework for active regression is based on the earlier idea of partitioning the entries of $\bfb$ by sensitivity and then applying sensitivity sampling, so we focus on the problem of preserving the Huber norm using an improved sensitivity sampling technique. Note that unlike the $\ell_p$ losses, the Huber loss is not scale-invariant. Furthermore, perhaps the largest obstacle in designing row sampling algorithms for the Huber loss going beyond standard $\eps$-net arguments is that there is no analogue of the chaining constructions of \cite{BourgainLindenstraussMilman:1989,SchechtmanZvavitch:2001} for the Huber loss. This can also be attributed to the fact that the Huber loss is not scale-invariant, which precludes an isometric change-of-density type theorem for the Huber loss as done in \cite{Lewis:1978,SchechtmanZvavitch:2001}. We show how to overcome these obstacles in the following discussion.

\paragraph{A Sharp Huber Inequality.}

Our algorithmic framework follows the Huber algorithm of \cite{ClarksonWoodruff:2015b}, which is a recursive sampling algorithm which reduces the number of rows from $n$ to roughly $n^{1/2}d^2$ in each recursive application of the algorithm. To show this result, \cite{ClarksonWoodruff:2015b} first show in their Lemma 2.1 that the Huber norm is within a factor of $O(n^{1/2})$ of the smaller of the $\ell_1$ and $\ell_2$ norms:

\begin{lemma}[Huber Inequality version 1 (\cite{ClarksonWoodruff:2015b}, Lemma 2.1)]
Let $\bfy\in\mathbb R^n$. Then,
\[
    \norm*{\bfy}_H^2 = \sum_{i=1}^n H(\bfy_i) \geq \Omega(n^{-1/2})\min\braces{\norm*{\bfy}_1, \norm*{\bfy}_2^2}
\]
\end{lemma}

It can be shown that the above lemma implies that the Huber sensitivities are within a factor of $O(n^{1/2})$ of the sum of the $\ell_1$ and $\ell_2$ sensitivities. 
This motivates the idea of sampling the rows of $\bfA$ with probability proportional to the sum of the $\ell_1$ and $\ell_2$ Lewis weights, oversampled by a factor of $O(n^{1/2})$. This is indeed how \cite{ClarksonWoodruff:2015b} proceeds.

The recursion $n\to n^{1/2}d^2$ solves to a final row count of around $d^4$, which is quadratically worse than our general loss function result of $d^2$ using our new sensitivity upper bounds and our general framework. To improve this further, first note that two improvements can be made to the above argument. First, by using the Huber inequality in a different way, we can use it in conjunction with the \cite{BourgainLindenstraussMilman:1989} net bounds, which reduces the row count in one recursive application to roughly $n^{1/2}d$ rather than $n^{1/2}d^2$ (see Lemma \ref{lem:blm-net-for-huber}). This reduces the overall row count to $d^2$ after solving for the recursion, but this still does not beat our general purpose sensitivity sampling algorithm, despite the use of the \cite{BourgainLindenstraussMilman:1989} nets. The second improvement is that the Huber inequality as proved in \cite{ClarksonWoodruff:2015b} is in fact loose by a polynomial factor in $n$, and can be improved to the following:

\begin{lemma}[Huber Inequality version 2]
Let $\bfy\in\mathbb R^n$. Then,
\[
    \norm*{\bfy}_H^2 = \sum_{i=1}^n H(\bfy_i) \geq \Omega(n^{-1/3})\min\braces{\norm*{\bfy}_1, \norm*{\bfy}_2^2}
\]
\end{lemma}

This lemma is tight up to constant factors\footnote{Consider the vector with one coordinate with $n^{1/3}$ and $(n-1)$ coordinates with $n^{-1/3}$.}, and gives a recursion of roughly $n\to n^{1/3}d$, giving
\[
    O(d^{3/2}\poly(\eps^{-1},\log n))
\]
rows, which shaves a factor of approximately $\sqrt d$ over the na\"ive Bernstein bound over a net. 

\paragraph{Storing Large Huber Sensitivities.}

In order to further improve upon this bound, we crucially make use of our improved sensitivity bounds from Section \ref{sec:sensitivity-bounds} and a generalized version of the above Huber inequality lemma that is parameterized by an upper bound on the size of the entries of $\bfy$. 
\begin{lemma}[Huber Inequality version 3]
    Let $\bfy\in\mathbb R^n$ and let $0 < \gamma \leq 1$. Let
    \[
        T \supseteq \braces*{i \in [n] :H(\bfy_i) \leq \gamma\norm*{\bfy}_H^2}.
    \]
    Then, for some constant $c>0$, at least one of the following bounds holds:
    \begin{align*}
        \norm*{\bfy\mid_T}_H^2 = \sum_{i\in T} H(\bfy_i) &\geq c\frac{1}{(\gamma n)^{1/3}}\min\braces{\norm*{\bfy\mid_T}_1, \norm*{\bfy\mid_T}_2^2} \\
        \norm*{\bfy}_H^2 = \sum_{i=1}^n H(\bfy_i) &\geq c\gamma\min\braces{\norm*{\bfy}_1, \norm*{\bfy}_2^2}.
    \end{align*}
\end{lemma}

By directly including the rows of $\bfA$ with Huber sensitivity at least $\gamma$, we exactly preserve the Huber norm inside $\overline T = [n]\setminus T$ for every $\bfy$. On the remaining coordinates inside $T$, we then have an improved Huber inequality, which implies an improved sampling bound. By balancing the number of rows which we directly include, which is roughly $d/\gamma$, and the sampling bound inside $T$, which is roughly $(1/\gamma + (\gamma n)^{1/3})d$, we obtain a bound of roughly $n^{1/4}d$ rows by choosing $\gamma = n^{-1/4}$ at each step. By recursively applying this result, we obtain an improved sampling bound of
\[
    O(d^{4/3}\poly(\eps^{-1},\log n)).
\]

\paragraph{Comparing to Every $p\in[1,2]$.}

Finally, to achieve our final optimization, we further drive down the ratio between $\norm*{\bfy}_H^2$ and $\norm*{\bfy}_p^p$ \emph{by choosing the best $p\in[1,2]$ for each $\bfy$}:

\begin{lemma}[Huber Inequality ver.\ 4]\label{lem:generalized-huber-inequality}
    Let $\bfy\in\mathbb R^n$, $\alpha \in [0,1/2]$, $\gamma = n^{-\alpha}$ with $2/n \leq \gamma \leq 1$. Let
    \[
        T \supseteq \braces*{i\in[n] : H(\bfy_i) \leq \gamma \norm*{\bfy}_H^2}.
    \]
    Then, for some $c>0$ and $\beta = 3-2\sqrt 2\approx 0.17157$, at least one of the following bounds holds:
    \begin{align*}
        \norm*{\bfy\mid_T}_H^2 &\geq c\frac1{(\gamma n)^{\beta}}\min_{p\in[1,2]}\norm*{\bfy\mid_T}_p^p \\
        \norm*{\bfy}_H^2 &\geq c\gamma\min_{p\in\{1,2\}}\norm*{\bfy}_p^p
    \end{align*}
\end{lemma}

In fact, we prove a generalized bound for the $\ell_2$-$\ell_q$ loss for any $q\in(0,2)$ in Lemma \ref{lem:l2-lq-inequality}. The interval $p\in[1,2]$ can be discretized in increments of $\frac1{\log n}$, so with $O(\log n)$ applications of \cite{BourgainLindenstraussMilman:1989} nets, we can always find a $p$ within an additive $\frac1{\log n}$ of the optimal $p$ for every net vector $\bfy$, which only affects Lemma \ref{lem:generalized-huber-inequality} by constant factors whenever $\bfy$ has entries bounded by $\poly(n)$. This is formalized in Corollary \ref{cor:generalized-huber-inequality}. By proceeding as previously discussed, we arrive at our final bound of
\[
    O(d^{4-2\sqrt 2}\poly(\eps^{-1},\log n)).
\]

\paragraph{Extensions to $\ell_2$-$\ell_q$ Loss.}

We generalize our results to the $\ell_2$-$\ell_q$ loss for $q\in(0,2)$. As $q$ ranges from $0$ to $1$ to $2$, the $\ell_2$-$\ell_q$ interpolates between the Tukey, Huber, and $\ell_2$ losses up to constant factors, and provides a natural generalization of these loss function. 

\begin{lemma}[$\ell_2$-$\ell_q$ Inequality]\label{lem:l2-lq-inequality}
    Let $q\in(0,2)$ and define
	\[
		M(x) = \begin{cases}
			\abs*{x}^2 & \text{if $\abs{x} \leq 1$} \\
			\abs*{x}^q & \text{if $\abs{x} > 1$}
		\end{cases}.
	\]
	Let $\bfy\in\mathbb R^n$ and let $\alpha \in [0,q/2]$ and $\gamma = n^{-\alpha}$ with $2/n \leq \gamma \leq 1$. Let
    \[
        T \supseteq \braces*{i\in[n] : M(\bfy_i) \leq \gamma \norm*{\bfy}_M^2}.
    \]
    Then, for some constant $c>0$, at least one of the following bounds holds:
    \begin{equation}\label{eq:l2-lq-whole-vector}
    \begin{aligned}
        \norm*{\bfy}_M^2 &\geq c\gamma^{2/q - 1}\min_{p\in\{q,2\}}\norm*{\bfy}_p^p \\
        \norm*{\bfy\mid_T}_M^2 &\geq c\frac1{(\gamma n)^{\beta}}\min_{p\in[q,2]}\norm*{\bfy\mid_T}_p^p
    \end{aligned}
	\end{equation}
	where
	\[
		\beta = \frac1{2/q - 1}\bracks*{(2/q + 1) - 2\sqrt{2/q}}.
	\]
\end{lemma}
\begin{proof}
    A proof of this bound can be found in Appendix \ref{sec:huber-appendix}.
\end{proof}

For $q\in[1,2)$, the $1/\gamma^{2/q-1}$ distortion in Equation \eqref{eq:l2-lq-whole-vector} is smaller than the $1/\gamma$ factor incurred from keeping $M$-sensitivities at least $\gamma$, so we can balance the parameters as $1/\gamma = (\gamma n)^\beta$ as before, which leads to a recursion that gives us a bound of $n = d^{1+\beta} \poly((\log n)/\eps)$. For $q\in(0,1)$, the $1/\gamma^{2/q-1}$ distortion in Equation \eqref{eq:l2-lq-whole-vector} is worse than $1/\gamma$, which means we must balance $1/\gamma^{2/q-1} = (\gamma n)^{\beta}$, or $\gamma = n^{-\beta/(2/q-1+\beta)}$, which gives a worse bound of $n = d^{\gamma}\poly((\log n)/\eps)$ for
\[
	\gamma = \frac{2/q-1+\beta}{2/q-1+\beta(2-2/q)}.
\]
This is better than a $d^2$ bound as long as $q \geq (\sqrt 5-1)^2/8 \approx 0.19098$.

\begin{figure}[ht]
	\centering
	\begin{tikzpicture}[scale=3]
		\newcommand{\offset}{0.6}

	  \draw[->] (-0.5,0) -- (2.2,0) node[right] {$q$};
	  \draw[->] (0,-0.5) -- (0,2.2 - \offset);
	  \draw[domain=1:1.995,smooth,variable=\q,thick,line width=0.9mm,samples=100,blue] plot ({\q},{1 + ((1/(2/\q-1) * ((2/\q+1)-2 * sqrt(2/\q)))) - \offset});
	  \draw[domain=0.191:1,smooth,variable=\q,thick,line width=0.9mm,samples=100,red] plot ({\q},{(2/\q - 1 + ((1/(2/\q-1) * ((2/\q+1)-2 * sqrt(2/\q))))) / (2/\q - 1 + ((1/(2/\q-1) * ((2/\q+1)-2 * sqrt(2/\q)))) * (2 - 2/\q)) - \offset});
	  \draw[thick,line width=0.9mm] (0,2- \offset)--(0.2,2- \offset);
	  
	  \node[below] (a) at (0.6,1.1) {\color{red}$d^{\gamma}$};
	  \node[below] (a) at (1.5,0.9) {\color{blue}$d^{1+\beta}$};

	\draw[gray, dashed] (-0.5,1 - \offset)--(2.2,1 - \offset);
	\node[left] (a) at (-0.5,1 - \offset) {$d$};

	\draw[gray, dashed] (-0.5,2 - \offset)--(2.2,2 - \offset);
	\node[left] (a) at (-0.5,2 - \offset) {$d^2$};

	\draw[gray, dashed] (-0.5,1.172 - \offset)--(2.2,1.172 - \offset);
	\node[left] (a) at (-0.5,1.172 - \offset) {$d^{4-2\sqrt 2}$};

	\draw[gray, dashed] (1,-0.5)--(1,1.2 - \offset);
	\node[below] (a) at (1,-0.5) {$1$};

	\draw[gray, dashed] (2,-0.5)--(2,1 - \offset);
	\node[below] (a) at (2,-0.5) {$2$};

	\draw[gray, dashed] (0.191,-0.5)--(0.191,2 - \offset);
	\node[below] (a) at (0.191,-0.5) {$\frac{(\sqrt 5-1)^2}{8}\approx 0.191$};

	\end{tikzpicture}
	\caption{Dependence on $d$ for the active regression sample complexity for the $\ell_2$-$\ell_q$ loss. Similar bounds apply to subspace embeddings as well.}
	\label{fig:l2-lq}
\end{figure}
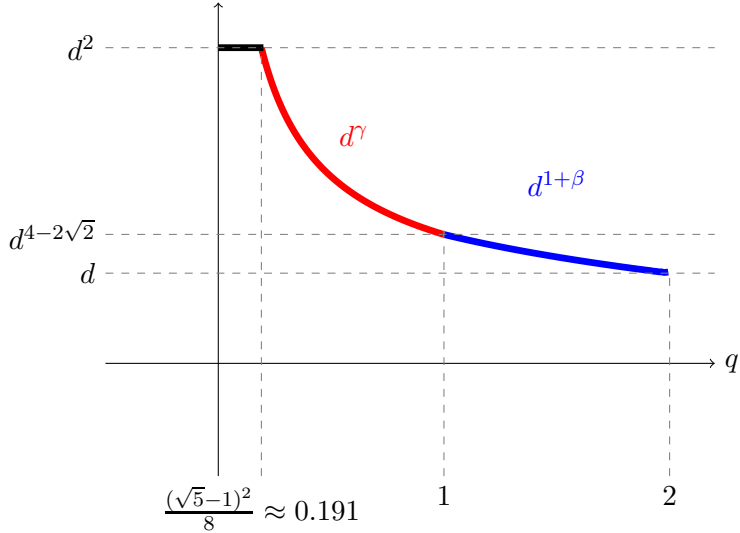

\subsection{Conclusions and Future Directions}

In this work, we study the sample complexity of active linear regression for both the $\ell_p$ norm as well as general $M$-estimator losses. 

For the $\ell_p$ norm, we provide optimal algorithms and lower bounds for $p\in(0, 2)$, with $\tilde\Theta(d/\eps^2)$ samples for $p\in(0,1)$ and $\tilde\Theta(d/\eps)$ samples for $p\in(1,2)$. For $p>2$, we provide an upper bound of $\tilde O(d^{p/2}/\eps^p)$, which is optimal in the $d$ dependence and off by a single $\eps$ factor in the $\eps$ dependence, up to polylogarithmic factors. Our algorithms provide the first nontrivial bounds, i.e., sample complexity less than $n$, for $p\in(0,1)\cup(2,\infty)$, while for $p\in(1,2)$, we significantly improve upon the $\tilde O(d^2/\eps^2)$ upper bound of \cite{ChenDerezinski:2021} and answer their main open question. We obtain these results via a two-stage algorithm and a novel sensitivity partitioning technique for every $p$, as well as an iterative improvement argument via strong convexity and Lewis bases to improve the $\eps$ dependence for $p\in(1,2)$. Our result is the first to achieve a linear dependence on $\eps$ for dimension reduction for $\ell_p$ regression for $p\in(1,2)$.

Next, we obtain a new sensitivity bound which achieves optimal total sensitivity bounds for $M$-estimators of at most polynomial growth, which runs in input sparsity time and avoids the use of L\"owner--John ellipsoids. This answers an open question of \cite{TukanMaaloufFeldman:2020} and makes significant progress in the general direction of matrix approximation beyond $\ell_p$ losses. By combining this with our new active regression techniques, we obtain active regression algorithms for general $M$-estimator losses, including the Tukey and Huber losses, which answers an open question of \cite{ChenDerezinski:2021}. 

For the important special case of the Huber loss, we introduce new techniques which bound Huber sensitivities by the sum of $\ell_p$ Lewis weights, which allows us to take advantage of chaining arguments for $\ell_p$ in order to obtain an active regression algorithm making at most $O(d^{4-2\sqrt 2}\poly(\log n,\eps^{-1}))$ queries. Our techniques also give subspace embeddings with the same number of rows. This is the first dimension reduction result for losses other than $\ell_p$ to approximate a $d$-dimensional subspace with fewer than $d^2$ dimensions. This improves over a previous bound of $d^4$ for the Huber loss, which held only for subspace embeddings, and not active regression, in \cite{ClarksonWoodruff:2015b}. 

Finally, our results and techniques give many applications in a wide variety of related problems. Our lower bounds for active regression give improved lower bounds for the sublinear power means problem \cite{CSS2021}; our new sensitivity bounding techniques sharpen and generalize previous results on Orlicz norm subspace embeddings \cite{SongWoodruffZhong:2019} and robust subspace approximation \cite{ClarksonWoodruff:2015a}; our techniques for dimension reduction for the Huber loss gives improved bounds for sparsification for $\gamma_p$ functions for applications in fast algorithms for $\ell_p$ regression \cite{GPV2021, ABKS2021}. We believe that our techniques will be applicable much further, and hope to see more uses in future work.

We conclude with questions that are still left open by our work. Perhaps the most pressing is to resolve the query complexity of active $\ell_p$ regression for $p>2$: our upper bound is $\tilde O(d^{p/2}/\eps^p)$, while the lower bound is $\Omega(d^{p/2}+\eps^{1-p})$. Closing this gap would be interesting. Our bounds are also loose by a factor of $\log\frac1\delta$ for all $p>0$, while we can get an optimal dependence on $\delta$ if we assume knowledge of $\OPT$ (Section \ref{sec:optimal-delta-dep}) and sacrifice a factor of $\eps$. A natural question if one can achieve a simultaneously optimal dependence on $d$, $\eps$, and $\delta$, up to logarithmic factors, and without assumptions. Another gap to close is the query complexity of Huber regression, or more generally $M$-estimator regression, even for just the $d$ dependence: our upper bound is $\tilde O(d^{4-2\sqrt 2}\poly\log n)$ for constant $\eps$, while only a trivial lower bound of $\Omega(d)$ is known.

\section{Preliminaries}\label{sec:prelim}

Throughout, we assume that $p$ is a fixed constant, and thus do not include constants depending on $p$ (e.g. $p$, $2^p$) in our big-O notation. We do include precise constants involving $p$ that appear in the exponents of other terms like $d$ and $1/\epsilon$. For simplicity, we also assume that our design matrix $\bfA\in \R^{n\times d}$ is overdetermined ($n > d$) and has full-rank (i.e., rank $d$). This is without loss of generality, as if $\bfA$ had rank $r < d$ we could replace it with an $n \times r$ basis for its column span and solve the regression problem with that basis as the design matrix. 

Throughout, we will use the following fact, which follows from the convexity of $x^p$ for $x \ge 0$.
\begin{fact}\label{fact:tri} 
For any $p \ge 1$ and any $a,b \in \R$, $|a+b|^p \le 2^{p-1} (|a|^p + |b|^p) = O(|a|^p + |b|^p)$.
\end{fact}

%
%

We denote the maximum of two real numbers $a,b$ by $a\lor b$ and the minimum as $a\land b$. We access the coordinates of vectors as $\bfv_i$ or $\bfv(i)$, which means the $i$th coordinate of $\bfv\in\mathbb R^n$.

\subsection{Sensitivities and Lewis Weights}
We use two central and closely related notions of matrix row importance for $\ell_p$ linear regression.

\begin{definition}[$\ell_p$ sensitivity]
For any $\bfA \in \R^{n \times d}$ and $0 < p < \infty$ define the $\ell_p$ sensitivity of row $i$ for $i\in[n]$ as:
\begin{align*}
\bfs^p_i(\bfA) \coloneqq \max_{\bfx \in \R^d\setminus\{0\}} \frac{|[\bfA\bfx](i)|^p}{\norm{\bfA\bfx}_p^p}.
\end{align*}
\end{definition}
$[\bfA\bfx](i)$ denotes the $i^\text{th}$ entry of the vector $\bfA\bfx$. The $\ell_p$ sensitivity measures how large an entry of a vector in the column span of $\bfA$ can be compared to the total $\ell_p$ norm of that vector. When $p = 2$, $\bfs^2_i(A)$ is equivalent to the \emph{statistical leverage score} of row $i$.
A related, but not equivalent, generalization of the leverage scores, are the Lewis weights:
\begin{definition}[$\ell_p$ Lewis weight \cite{CohenPeng:2015}]\label{def:lewis}
For any $\bfA \in \R^{n \times d}$ and $0<p<\infty$ the $\ell_p$ Lewis weights $\{\bfw^p_1(\bfA), \ldots, \bfw^p_n(\bfA) \}$ are the unique set of weights such that, if we let $\bfW \in \R^{n \times n}$ be the diagonal matrix with $\bfW(i,i) = \bfw_i^p(\bfA)$, then for all $i$,
\begin{align*}
 \bfw_i^p(\bfA) = \bfs^2_i (\bfW^{1/2-1/p} \bfA).
\end{align*}
Here $\bfs^2_i (\bfW^{1/2-1/p} \bfA)$ is the $i^\text{th}$ leverage score of $\bfW^{1/2-1/p} \bfA$. Note that the leverage scores of a full-rank matrix always sum to $d$, so we have $\sum_{i=1}^n \bfw^p_i(\bfA) = d$. 
\end{definition}

\begin{remark}
	Note that \cite{CohenPeng:2015} only (explicitly) define $\ell_p$ Lewis weights for $p\geq 1$. However, the definition in fact makes sense for the range $p\in(0,1)$ as well, as shown by \cite{SchechtmanZvavitch:2001}.
\end{remark}

From Definition \ref{def:lewis} we can see that for $p = 2$, we have $\bfs^2_i(\bfA) = \bfw^2_i(\bfA)$, so the Lewis weights also correspond to the statistical leverage scores. While not equivalent for other values of $p$, the Lewis weights do upper bound the sensitivities, up to a polynomial in $d$ factor. 

\begin{lemma}[Lewis weights bound sensitivities, Lemma 3.8 of \cite{ClarksonWangWoodruff:2019}]\label{lem:lwBound}
For $\bfA \in \R^{n \times d}$ and $0<p<\infty$, $$\bfs_i^p(\bfA) \le d^{\max(0,p/2-1)} \cdot \bfw_i^p(\bfA).$$
\end{lemma}
\begin{proof}
	The result is stated in Lemma 3.8 of \cite{ClarksonWangWoodruff:2019} for $p\geq 1$. See Lemma \ref{lem:lewis-sensitivity-bound} for a proof of the case when $0 < p < 1$.
\end{proof}

Lemma \ref{lem:lwBound} immediately yields a bound on the total sum of sensitivities.
\begin{lemma}[Sum of Sensitivities]\label{lem:sum}
For $\bfA \in \R^{n \times d}$ and $0<p<\infty$, we have $ \sum_{i=1}^n \bfs^p_i(\bfA) \le d^{\max(1,p/2)}$.
\end{lemma}
\begin{proof}
By Definition \ref{def:lewis}, the Lewis weights are the leverage scores of $\bfW^{1/2-1/p} \bfA$. It is well known that $\sum_{i=1}^n \bfw_i^p(\bfA) = \rank(\bfW^{1/2-1/p} \bfA) \le d$. The lemma then follows by applying Lemma \ref{lem:lwBound}.
\end{proof}

Finally, we recall that Lewis weights can be efficiently approximated:

\begin{theorem}[\cite{CohenLeeMusco:2015, CohenPeng:2015}]\label{thm:cohen-peng-fast-lewis-weights}
	Let $\bfA\in\mathbb R^{n\times d}$. There is an algorithm which computes upper bounds $\tilde\bfw_i^p(\bfA) \geq \bfw_i^p(\bfA)$ to the $\ell_p$ Lewis weights of $\bfA$ such that
	\[
		\sum_{i=1}^n \tilde\bfw_i^p(\bfA) = O(d)
	\]
	with the following running times:
	\begin{itemize}
	\item for $p = 2$: $O(\nnz(\bfA) + d^\omega\log^2 d)$
	\item for $0 < p < 4$: $O(\nnz(\bfA)\log n + d^{\omega})$
	\item for $p \geq 2$: $O(\nnz(\bfA)\log n + d^{p/2 + O(1)})$
	\end{itemize}
	Here, $\omega \approx 2.37286$ is the current exponent of fast matrix multiplication. 
	\end{theorem}



\subsection{Lewis Weight Sampling and \texorpdfstring{$\ell_p$}{lp} Subspace Embedding}\label{sec:lws}

Using Lemma \ref{lem:sum} and a standard Bernstein bound, one can show that sampling $m$ rows of $\bfA$ for 
\[
	m = \tilde O(d^{\max(1,p/2)}/\epsilon^2)
\]
according to their sensitivities (and appropriately reweighting to keep the expectation correct) preserves $\norm{\bfA\bfx}_p$ to $(1 \pm \epsilon)$ relative error for any fixed $\bfx \in \R^d$ with high probability. Via an $\epsilon$-net argument, this can be shown to hold for all $\bfx$ simultaneously, if $\tilde O(d^{\max(1,p/2)+1}/\epsilon^2)$ rows are sampled, giving a method for $\ell_p$ subspace embedding. This argument holds when sampling by any set of upper bounds on the sensitivities that sum to $O(d^{\max(1,p/2)})$, including, e.g., the Lewis weights scaled up by a $d^{\max(0,p/2-1)}$ factor by Lemma \ref{lem:lwBound}. Note that this result gives a dimension bound that is independent of $n$. 

Importantly however, the Lewis weights can be combined with more refined chaining arguments \cite{BourgainLindenstraussMilman:1989,CohenPeng:2015} to shave a $d$ factor from the above bound, giving $\ell_p$ subspace embedding bounds into $\tilde O(d^{\max(1,p/2)}/\poly(\eps)))$ dimensions, which achieves a tight dependence on $d$, as shown by \cite{LWW20}. These bounds have been subsequently sharpened by \cite{Tal1990,Tal1995,Zvavitch:2000,SchechtmanZvavitch:2001,LT2011}, which use Lewis weights in a slightly different manner to achieve a bound of $\tilde O(d^{\max(1,p/2)}/\eps^2)$ dimensions. For $p\leq 2$, this type of analysis can in fact be turned into an analysis for the algorithm for sampling directly by the Lewis weights, via a reduction found in Lemma 7.4 of \cite{CohenPeng:2015}. This proof crucially uses the monotonicity properties of Lewis weights for $p\leq 2$, and does not work for $p > 2$. To simplify the presentation and work with similar techniques simultaneously for all $0 < p < \infty$, we stick with analyzing the ``recursive'' form of Lewis weight sampling rather than the ``one shot'' form of Lewis weight sampling.

\subsubsection{Chaining and the Iteration Procedure}\label{sec:chaining}

We now describe the improved chaining arguments of \cite{Tal1990,Tal1995,Zvavitch:2000,SchechtmanZvavitch:2001,LT2011}, since we build on these to obtain our active regression algorithms. In all of these works, the algorithms proceed in multiple rounds of reducing the number of rows by a constant factor, rather than reducing the rows to $\tilde O(d^{\max(1,p/2)}/\eps^2)$ rows in one shot. Roughly, the idea is to sample each row $i\in[n]$ of $\bfA$ with probability $1/2$ and scale the result by $2$. This amounts to analyzing the quantity
\begin{align*}
	\Lambda &\coloneqq \sup_{\norm*{\bfA\bfx}_p\leq 1} \abs*{\norm*{\bfS\bfA\bfx}_p^p - \norm*{\bfA\bfx}_p^p} \\
	&= \sup_{\norm*{\bfA\bfx}_p \leq 1}\abs*{\sum_{i=1}^n (1+\sigma_i)\abs*{[\bfA\bfx](i)}^p - \sum_{i=1}^n \abs*{[\bfA\bfx](i)}^p} \\
	&= \sup_{\norm*{\bfA\bfx}_p \leq 1}\abs*{\sum_{i=1}^n \sigma_i\abs*{[\bfA\bfx](i)}^p}
\end{align*}
where $\sigma_i\in\{\pm1\}$ are independent Rademacher variables and $\bfS\in\mathbb R^{n\times n}$ is a diagonal sampling matrix with $\bfS_{i,i} = (1+\sigma_i)$. The following is known:

\begin{theorem}[Theorem 15.13, \cite{LT2011} and Proposition 4.1, \cite{SchechtmanZvavitch:2001}]\label{thm:lambda-bound}
	Suppose that $\bfA\in\mathbb R^{n\times d}$ has Lewis weights bounded by $cd/n$ for some constant $c > 0$. Then, there exists a constant $C>0$ such that the following holds:
	\begin{itemize}
		\item if $2 \leq p < \infty$, then
		\[
			\E_{\bfsigma}[\Lambda] \leq \bracks*{C p^2 \frac{d^{p/2}}{n}(\log d)^2\log n}^{1/2}
		\]
		\item if $1 < p \leq 2$, then
		\[
			\E_{\bfsigma}[\Lambda] \leq \bracks*{C\frac{d}{n}(\log d)^2 \max\braces*{\frac1{p-1}, \log n}}^{1/2}
		\]
		\item if $0 < p < 1$, then
		\[
			\E_{\bfsigma}[\Lambda] \leq \bracks*{\frac{C}{p}\frac{d}{n}(\log d)^3}^{1/2}
		\]
	\end{itemize}
\end{theorem}

If $p$ is a constant, then we may concisely write the result of Theorem \ref{thm:lambda-bound} (giving up some log factors) as
\[
	\E_{\bfsigma}[\Lambda] \leq \bracks*{C\frac{d^{\max(1,p/2)}}{n}(\log d)^2\log n}^{1/2}.
\]
The bounds of Theorem \ref{thm:lambda-bound} are proven using Dudley's entropy integral inequality, which comes with the following tail bound, which gives sub-Gaussian tails on the above quantity.
\begin{theorem}[Theorem 8.1.6, \cite{Ver2018}]\label{thm:dudley-tail}
	Let $(X_t)_{t\in T}$ be a Gaussian process with pseudo-metric $d_X(s,t)\coloneqq \norm*{X_s - X_t}_2$. Let $E(T, d_X, u)$ denote the minimal number of $d_X$-balls of radius $u$ required to cover $T$. Then, for every $u\geq 0$, we have that
	\[
		\Pr\braces*{\sup_{t\in T}X_t \geq C\bracks*{\int_0^\infty \sqrt{\log E(T, d_X, u)}~du + z\cdot \diam(T)}} \leq 2\exp(-z^2)
	\]
\end{theorem}
This results in the following:
\begin{equation}\label{eq:lt-sz-tail-bound}
	\Pr\braces*{\Lambda \geq C\bracks*{\frac{d^{\max(1,p/2)}}{n}}^{1/2}\bracks*{(\log d)\sqrt{\log n}+z}} \leq 2\exp(-z^2).
\end{equation}
The necessary diameter calculations are carried out in Lemma \ref{lem:dudley-tail-diameter-bound} for $1 < p < \infty$ and in \cite{SchechtmanZvavitch:2001}for $0 < p < 1$. 

Furthermore, the assumption that the Lewis weights are bounded by $O(d/n)$ can be enforced by a standard procedure of ``splitting rows'' (see, e.g., Remark 2.2 of \cite{SchechtmanZvavitch:2001}):

\begin{lemma}[Splitting Rows]\label{lem:splitting-rows}
Let $\bfA\in\mathbb R^{n\times d}$ and let $0 < p < \infty$. Let $\tilde\bfw_i^p(\bfA) \geq \bfw_i^p(\bfA)$ be Lewis weight upper bounds such that
\[
	\sum_{i=1}^n \tilde\bfw_i^p(\bfA) \leq C_1\cdot d.
\]
Let $C_2 > 0$. Then, there exists an $n'\times d$ matrix with $n \leq n' \leq (1+2C_1/C_2)n$ such that
\[
	\norm*{\bfA\bfx}_p = \norm*{\bfA'\bfx}_p
\]
for all $\bfx\in\mathbb R^d$, and furthermore, $\bfw_i^p(\bfA') \leq C_2d/n$ for all $i\in[n']$. 
\end{lemma}
\begin{proof}
Let $\bfA'$ be the $(n+k-1)\times d$ matrix obtained by replacing the first row $\bfa_1$ by $k$ copies of $\bfa_1 / k^{1/p}$. Let $\bfw_i$ denote the Lewis weights of $\bfA$. Then,
\[
\sum_{i=1}^n \bfw_i^{1-2/p}\bfa_i\bfa_i^\top = k\cdot (\bfw_1/k)^{1-2/p} \frac{\bfa_1}{k^{1/p}}\frac{\bfa_1^\top}{k^{1/p}} +  \sum_{i=2}^n \bfw_i^{1-2/p}\bfa_i\bfa_i^\top
\]
and
\[
\bracks*{\frac{\bfa_i^\top}{k^{1/p}}(\bfA^\top\bfW^{1-2/p}\bfA)^{-1}\frac{\bfa_i}{k^{1/p}}}^{p/2} = \frac{\bfw_i}{k}
\]
so $\bfw_1/k$ is the Lewis weight for all $k$ copies of $\bfa_1$ in $\bfA'$. Furthermore, $\norm*{\bfA\bfx}_p = \norm*{\bfA'\bfx}_p$ for every $\bfx\in\mathbb R^d$. Now suppose that for any row $i\in[n]$ with $\tilde\bfw_i^p(\bfA) \geq C_2 d/n$, we replace the row with $\ceil*{\tilde\bfw_i^p(\bfA)/(C_2 d/n)}$ scaled copies of $\bfa_i$. Then, we add at most
\begin{align*}
	\sum_{i : \tilde\bfw_i^p(\bfA) \geq C_2d/n} \ceil*{\frac{\tilde\bfw_i^p(\bfA)}{C_2d/n}} &\leq \sum_{i : \tilde\bfw_i^p(\bfA) \geq C_2d/n} \frac{\tilde\bfw_i^p(\bfA)}{C_2d/n} + 1 \\
	&\leq \frac{\sum_{i:\tilde\bfw_i^p(\bfA) \geq C_2d/n}\tilde\bfw_i^p(\bfA)}{C_2d/n} + \abs*{\braces*{i:\tilde\bfw_i^p(\bfA) \geq C_2d/n}} \\
	&\leq 2\frac{C_1 d}{C_2 d/n} = \frac{2C_1}{C_2}n
\end{align*}
new rows, and furthermore, all Lewis weights of the new matrix $\bfA'$ are at most $C_2 d/n$. 
\end{proof}

Given approximations to the Lewis weights that sum to $C_1 d$ (using, e.g., Theorem \ref{thm:cohen-peng-fast-lewis-weights}), we may choose $C_2$ to be a large enough constant so that $2C_1/C_2 \leq 1/3$. Then, after splitting rows and sampling half of the rows, we only have $(4/3) \cdot (1/2) = 2/3$ of the rows remaining in expectation, and thus at most $3/4$ of the rows with probability at least $1 - \delta$ as long as the expected number of rows is at least $\Theta(\log(1/\delta))$, by Chernoff bounds. We condition on this event. Now let $n_i$ denote the number of rows remaining after $i$ rounds of splitting and sampling, until we have at most $m \coloneqq O(\eps^{-2}d^{\max(1,p/2)}[(\log d)^2\log n + \log(1/\delta)])$ rows remaining. Furthermore, let $\Lambda_i$ denote the value of $\Lambda$ after the $i$th round. Then, by applying Equation \eqref{eq:lt-sz-tail-bound} with $z$ set to $\sqrt{\log((\log n)/\delta)}$, we can union bound over the at most $O(\log n)$ iterations to conclude that with probability at least $1-\delta$, we simultaneously have
\[
	\Lambda_i \leq C\bracks*{\frac{d^{\max(1,p/2)}}{n_i}\bracks*{(\log d)^2\log n+\log\frac1\delta}}^{1/2}
\]
for all iterations $i$. Then, the total distortion of the subspace embedding over all of the iterations is at most
\[
	\prod_{i} (1+\Lambda_i) \leq \exp\parens*{\sum_i \Lambda_i} \leq \exp\parens*{O(\eps) \cdot \sum_{j=0}^\infty (\sqrt{3/4})^i} = \exp(O(\eps)) \leq 1 + O(\eps).
\]
This gives the following theorem:

\begin{theorem}[$\ell_p$ Subspace Embedding]\label{thm:lp-subspace-embedding}
Let $0 < p < \infty$. There is a randomized algorithm which constructs a sampling matrix $\bfS\in\mathbb R^{r\times n}$ such that with probability at least $1-\delta$, we have
\[
	r = O\parens*{\frac{d^{\max(1,p/2)}}{\eps^2}\bracks*{(\log d)^2\log n + \log\frac1\delta}}
\]
and $\norm*{\bfS\bfA\bfx}_p = (1\pm \eps)\norm*{\bfA\bfx}_p$ for every $\bfx\in\mathbb R^d$.
\end{theorem}

\begin{remark}
	Although Theorem \ref{thm:lp-subspace-embedding} has a $\log n$ dependence, this can be replaced with a $\log(d/\eps)$ dependence by first applying the subspace embedding result using the simpler union bound over a net, which has no dependence on $n$. 
\end{remark}

We give the pseudocode for the splitting and sampling procedure below in Algorithm \ref{alg:split-sample}.

\begin{algorithm}
	\caption{$\ell_p$ sampling}
	\textbf{input:} Matrix $\bfA \in \R^{n \times d}$, constants $C_1, C_2 > 0$. \\
	\textbf{output:} Diagonal sampling matrix $\bfS$. 
	\begin{algorithmic}[1] 
		\State Let $\tilde\bfw_i^p(\bfA)\geq\bfw_i^p(\bfA)$ be $\ell_p$ Lewis weight upper bounds summing to $C_1 d$ (Theorem \ref{thm:cohen-peng-fast-lewis-weights}).
		\State Obtain $\bfA'$ by replacing each row $i\in[n]$ with $k = \ceil*{\tilde\bfw_i^p(\bfA)/(C_2d/n)}$ copies of $\bfa_i/k^{1/p}$.
		\State Sample each row of $\bfA'$ with probability $1/2$ and scale by $2^{1/p}$.
		\State \Return $\bfS$, the matrix corresponding to the above procedure.
	\end{algorithmic}\label{alg:split-sample}
\end{algorithm}



\section{Upper Bound for \texorpdfstring{$\ell_p$}{lp} Regression}\label{sec:main}
At a high level, our active regression algorithm follows the $\ell_p$ subspace embedding algorithm of Theorem \ref{thm:lp-subspace-embedding} by using Lewis weights to split rows and sample half of the rows. 

\subsection{Constant Factor Approximation}
\label{sec:constant_factor}

Analyzing the first stage is simple: as mentioned in Section \ref{sec:prior_work}, constant factor approximation algorithms for Problem \ref{prob:informal_problem} can be obtained from existing subspace embedding results. Formally, we study Algorithm \ref{alg:constant}, which is a simple ``sample-and-solve'' approach to Problem \ref{prob:informal_problem}. We analyze the method using the $\ell_p$ subspace embedding of Theorem \ref{thm:lp-subspace-embedding}, along with Markov's inequality. 

\begin{algorithm}
	\caption{Constant factor $\ell_p$ regression}
	\textbf{input:} Matrix $\bfA \in \R^{n \times d}$, measurement vector $\bfb \in \R^n$. \\
	\textbf{output:} Approximate solution $\tilde \bfx \in \R^d$ to $\min_\bfx\|\bfA\bfx - \bfb\|_p$. 
	\begin{algorithmic}[1] 
		\State Let $\bfS \in \R^{m \times n}$ be an $1/2$-approximate $\ell_p$ subspace embedding for $\bfA$ (Theorem \ref{thm:lp-subspace-embedding}).
		\State \Return $\tilde \bfx$ with $\norm{\bfS\bfA\tilde \bfx - \bfS\bfb}_p \le (1+\gamma) \cdot \min_{\bfx \in \R^d} \norm{\bfS\bfA\bfx - \bfS\bfb}_p$ for $\gamma \geq 0$.\label{line:const-lp-const-sol} 
	\end{algorithmic}\label{alg:constant} 
\end{algorithm}

\begin{remark}
Running Algorithm \ref{alg:constant} only requires querying $m$ entries of $\bfb$ in order to construct the vector $\bfS\bfb$. Also note that in Line \ref{line:const-lp-const-sol} of the algorithm, we would have $\gamma = 0$ if an exact minimizer of the subsampled regression problem $\min_\bfx\norm{\bfS\bfA \bfx - \bfS\bfb}_p^p$ was obtained. To allow for the use of approximation algorithms in implementing Line \ref{line:const-lp-const-sol}, we state the method for a general $\gamma \geq 0$. 
\end{remark}

We first give an algorithm which works with constant probability, and then show how to boost the probability to $1-\delta$ for any $\delta \in (0,1)$ incuring an $O(\log(1/\delta))$ factor overhead in our sample complexity.

\begin{theorem}[Constant factor approximation]\label{thm:const} For $\bfA \in \R^{n \times d}$, $\bfb \in \R^n$, and $0 < p < \infty$, let $\OPT = \min_{\bfx \in \R^d} \norm{\bfA\bfx - \bfb}_p$. For any $\delta \in (0,1]$, if $\tilde\bfx$ is the output of Algorithm \ref{alg:constant}, then with probability at least $1-\delta$,
	\[
		\norm{\bfA\tilde \bfx - \bfb}_p \le 2^{2\max\{0,1/p-1\}+1+1/p}(3+\gamma)/\delta^{1/p} \cdot \OPT.
	\]
	When $\delta$ is constant (e.g., $\delta = 1/100$) and $(1+\gamma)$ is constant (e.g., $\gamma = 0$) then $\norm{\bfA\tilde\bfx - \bfb}_p \le C\cdot \OPT$ for constant $C$.
	
\end{theorem}
\begin{proof}
Let $\bfx^* = \argmin_{\bfx \in \R^d} \norm{\bfA\bfx - \bfb}_p$. By triangle inequality for $p\geq1$ or subadditivity and approximate triangle inequality (Fact \ref{fact:tri}) for $p\in(0,1)$,
\[
	\norm{\bfA\tilde\bfx - \bfb}_p \le 2^{\max\{0,1/p - 1\}}\parens*{\norm{\bfA\bfx^* - \bfb}_p + \norm{\bfA\tilde\bfx - \bfA\bfx^*}_p} = 2^{\max\{0,1/p - 1\}}\parens*{\OPT + \norm{\bfA\tilde \bfx - \bfA \bfx^*}_p}.
\]
Applying the subspace embedding property of Theorem \ref{thm:lp-subspace-embedding} with $\epsilon =1/2$ and failure probability $\delta/2$, we conclude that, with probability at least $1-\delta/2$,
\begin{align*}
\norm{\bfA\tilde \bfx - \bfb}_p \le 2^{\max\{0,1/p - 1\}}\parens*{\OPT + 2 \norm{\bfS\bfA\tilde \bfx - \bfS\bfA \bfx^*}_p}.
\end{align*}
By similar reasoning, we have $\parens*{\norm{\bfS\bfA\tilde \bfx - \bfS\bfA \bfx^*}_p} \le 2^{\max\{0,1/p - 1\}}\parens*{\norm{\bfS\bfA\tilde \bfx - \bfS\bfb}_p + \norm{\bfS\bfA \bfx^* - \bfS\bfb}_p}$. We know that $\norm{\bfS\bfA\tilde \bfx - \bfS\bfb}_p \le (1+\gamma) \cdot \min_{\bfx \in \R^d} \norm{\bfS\bfA\bfx - \bfS\bfb}_p \le (1+\gamma) \cdot \norm{\bfS\bfA \bfx^* - \bfS\bfb}_p$, so we conclude that
\begin{align*}
\norm{\bfS\bfA\tilde \bfx - \bfS\bfA \bfx^*}_p \le 2^{\max\{0,1/p - 1\}}(2+\gamma)\norm{\bfS\bfA \bfx^* - \bfS\bfb}_p.
\end{align*}
Finally, by Markov's inequality, since $\E[\norm{\bfS\bfA \bfx^* - \bfS\bfb}_p^p] = \OPT^p$, with probability $\ge 1-\delta/2$, $\norm{\bfS\bfA \bfx^* - \bfS\bfb}_p^p \le \OPT^p/(\delta/2)$ and so $\norm{\bfS\bfA \bfx^* - \bfS\bfb}_p \le \OPT/(\delta/2)^{1/p}$.
Combining all these bounds we have that with probability $1-\delta$,
\begin{align*}
\norm{\bfA\tilde \bfx - \bfb}_p &\le 2^{\max\{0,1/p - 1\}}\parens*{\OPT + 2 \cdot 2^{\max\{0,1/p - 1\}}(2+\gamma) \cdot 2^{1/p} \OPT/\delta^{1/p}} \\
&\le 2^{2\max\{0,1/p-1\}+1+1/p}(3+\gamma)\OPT/\delta^{1/p}.\qedhere
\end{align*}
\end{proof}

\subsubsection{Probability Boosting for Constant Factor Approximation}\label{sec:boost}
We now show a boosting step for our constant factor approximation algorithm (Algorithm \ref{alg:constant}), described in Algorithm \ref{alg:boost}. If we repeat the constant factor approximation algorithm with success probability $99/100$ for a total of $O(\log(1/\delta))$ times, then via a standard Chernoff bound, with probability at least $1-\delta$, at least $9/10$ of the computed $\bfx_c$ will satisfy the guarantee of Theorem \ref{thm:const} -- i.e., that $\norm{\bfA\bfx_c-\bfb}_p = O(\OPT)$. Thus, we just need to identify one of these good solutions, which Algorithm \ref{alg:boost} does, deterministically, and without reading any entries of $\bfb$. The approach simply computes pairwise distances between solutions and returns any solution with a relatively low distance to at least $1/2$ of the other solutions. For later use, we state the result in terms of a general error measure $\norm*{\cdot}$ which satisfies an approximate triangle inequality (for example, $\norm*{\cdot}_p$ for $p\in(0,1)$ satisfies an approximate triangle inequality with constant $2^{1/p - 1}$ by Fact \ref{fact:tri}).

\begin{algorithm}
	\caption{Constant factor $\norm*{\cdot}$ regression -- Boosted Success Probability}
	\textbf{input:} $\ell$ candidate solutions $\bfx_1,\ldots,\bfx_\ell$ with at least $9/10\cdot \ell$ satisfying $\norm{\bfA\bfx_i-\bfb} \le \alpha \min_{\bfx} \norm{\bfA\bfx-\bfb}$. \\
	\textbf{output:} Approximate solution $\tilde \bfx \in \R^d$ to $\min_{\bfx} \norm{\bfA\bfx-\bfb}$. 
	\begin{algorithmic}[1]
		\State{Let $\bfd \in \R^{\ell^2}$ contain all pairwise distances $\norm{\bfA\bfx_{i}-\bfA\bfx _{j}}$ (over ordered pairs $(i,j)$) sorted in increasing order. Let $\tau = \bfd(\floor*{\ell^2 \cdot 8/10})$ be the $80^{th}$ percentile distance.}\label{line:thresh}
		\State{Return any $\bfx_{i}$ such that $\norm{\bfA\bfx _{i}-\bfA\bfx _{j}} \le \tau$ for at least $1/2 \cdot \ell$ vectors $\bfx_{j}$. }\label{line:return}
	\end{algorithmic}\label{alg:boost} 
\end{algorithm}

\begin{theorem}[Constant factor $\norm*{\cdot}$ regression -- Success Boosting]\label{thm:boost} Consider $\bfA \in \R^{n \times d}$, $\bfb \in \R^{n}$, and an error measure $\norm*{\cdot}$ which satisfies an approximate triangle inequality, that is, there exists a constant $\kappa\geq1$ such that $\norm*{\bfy_1+\bfy_2} \leq \kappa(\norm*{\bfy_1} + \norm*{\bfy_2})$ for any two vectors $\bfy_1,\bfy_2\in\R^n$. Let $\OPT = \min_{\bfx \in \R^d} \norm{\bfA\bfx-\bfb}$. Given a set of solution vectors $\bfx_1,\ldots,\bfx_\ell \in \R^d$ where $\norm{\bfA\bfx_i - \bfb} \le \alpha \cdot \OPT$ for at least $9/10 \cdot \ell $ of the vectors, Algorithm \ref{alg:boost} identifies $\bfx_i$ with $\norm{\bfA\bfx_i - \bfb} \le (\kappa\alpha + 2\kappa^3(\alpha+1)) \cdot \OPT$, without querying any entries of $\bfb$.
\end{theorem}
\begin{proof}
Let $\bfx^* = \argmin_{\bfx \in \R^d} \norm{\bfA\bfx-\bfb}$. Call $\bfx_i$ \emph{good} if $\norm{\bfA\bfx_i - \bfb} \le \alpha \cdot \OPT$. By approximate triangle inequality , for any good $\bfx_i$, 
\begin{align*}
\norm{\bfA\bfx_i-\bfA\bfx^*} \le \kappa(\norm{\bfA\bfx_i-\bfb}_p + \norm{\bfA\bfx^*-\bfb}_p) = \kappa(\alpha+1) \cdot \OPT.
\end{align*}
Thus, again via approximate triangle inequality, for any good $\bfx_i,\bfx_j$, 
\[
	\norm{\bfA\bfx_i-\bfA\bfx_j} \le \kappa(\kappa(\alpha+1) \cdot \OPT + \kappa(\alpha+1) \cdot \OPT) = 2\kappa^2(\alpha+1) \cdot \OPT.
\]
Thus, for the pairwise distance vector $\bfd \in \R^{\ell^2}$ computed in line 1 of Algorithm \ref{alg:boost}, at least $(9/10)^2 \cdot \ell^2 \ge 8/10 \cdot \ell^2$ of the distances will be upper bounded by $2\kappa^2(\alpha+1) \cdot \OPT$. Thus, the threshold $\tau$ computed in Line \ref{line:thresh}, which is the $80^{th}$ percentile of the distances, gives a lower bound $\tau \le 2\kappa^2(\alpha+1) \cdot \OPT$. In Line \ref{line:return}, we return any $\bfx_i$ with $\norm{\bfA\bfx_i - \bfA\bfx_j} \le \tau$ for at least $1/2 \cdot \ell$ vectors $\bfx_j$. First observe that at least one such $\bfx_i$ must exist. Otherwise, at most $1/2\cdot \ell^2$ of the pairwise distances would lie below $\tau$. 

Additionally, observe that since at least $9/10 \cdot \ell$ of the $\bfx_i$ are good, if $\bfx_i$ is returned, it must have $\norm{\bfA\bfx_i - \bfA\bfx_j} \le \tau \le 2\kappa^2(\alpha+1) \cdot \OPT$ for at least one good $\bfx_j$. Since this good $\bfx_j$ has $\norm{\bfA\bfx_j - \bfb} \le \alpha \cdot \OPT$, by approximate triangle inequality, the returned $\bfx_i$ must then satisfy 
\[
	\norm{\bfA\bfx_i -\bfb} \le \kappa(\alpha + 2\kappa^2(\alpha+1)) = (\kappa\alpha + 2\kappa^3(\alpha+1))\cdot \OPT.
\]
\end{proof} 

\subsection{Relative Error Approximation}\label{subsec:relative-error-lp}
We next show how to achieve a relative error $(1+\epsilon)$ solution to Problem \ref{prob:informal_problem}. The relative error algorithm simply solves a Lewis weight sampled regression problem multiple times and outputs a boosted solution as done in Theorem \ref{thm:boost}. Pseudocode is given in Algorithm \ref{alg:hp-clip} below.

\begin{algorithm}
	\caption{Recursive relative error $\ell_p$ regression}
	\textbf{input:} Matrix $\bfA \in \R^{n \times d}$, measurement vector $\bfb \in \R^n$, number of measurements $m$. \\
	\textbf{output:} Approximate solution $\tilde \bfx \in \R^d$ to $\min_\bfx\|\bfA\bfx - \bfb\|_p$. 
	\begin{algorithmic}[1] 
		\If{$n \leq m$}
			\State Compute $\bar\bfx$ with $\norm{\bfA\bar\bfx - \bfb}_p \le (1+\gamma) \cdot \min_{\bfx \in \R^d} \norm{\bfA\bfx - \bfb}_p$ for $\gamma \geq 0$. \label{line:solve-rel-lp}
			\State \Return $\bar\bfx$
		\EndIf
		\State Let $\bfS \in \R^{m \times n}$ be generated via Algorithm \ref{alg:split-sample}. \label{line:split-sample}
		\State Run Algorithm \ref{alg:clip} with inputs $\bfS\bfA$ and $\bfS\bfb$ to get $\tilde\bfx$.
		\State \Return $\tilde \bfx$
	\end{algorithmic}\label{alg:clip}
\end{algorithm}

\begin{algorithm}
	\caption{High probability relative error $\ell_p$ regression}
	\textbf{input:} Matrix $\bfA \in \R^{n \times d}$, measurement vector $\bfb \in \R^n$, number of measurements $m$. \\
	\textbf{output:} Approximate solution $\tilde \bfx \in \R^d$ to $\min_\bfx\|\bfA\bfx - \bfb\|_p$. 
	\begin{algorithmic}[1] 
		\State Run Algorithm \ref{alg:clip} with inputs $\bfA$ and $\bfb$ independently $\ell$ times for $\ell = O(\log\frac1\delta)$.
          \State Run Algorithm \ref{alg:constant} (and \ref{alg:boost}) to find $\bfx_c$ such that $\norm{\bfA\bfx_c-\bfb}_p \leq \alpha \min_\bfx\norm{\bfA\bfx-\bfb}_p$ for $\alpha > 0$. 
        \State Discard the $1/10$ fraction of trials with the largest $\norm*{\bfS(\bfb - \bfA\bfx_c)}_p^p$.\label{line:discard}
		\State Run Algorithm \ref{alg:boost} on the remaining candidate solutions.
	\end{algorithmic}\label{alg:hp-clip}
\end{algorithm}


\begin{theorem}[Main Result]\label{thm:eps}
	Let $\bfA \in \R^{n \times d}$, $\bfb \in \R^n$, and $0 < p < \infty$. Let
	\[
		m = O(1)\frac{d^{1\lor (p/2)}}{\eps^{2\lor p}}\bracks*{(\log d)^2\log n + \log\frac1\delta}\log\frac1\delta.
	\]
	Then, with probability at least $1-\delta$, Algorithm \ref{alg:hp-clip} returns $\tilde\bfx\in\mathbb R^d$ such that
	\[
		\norm*{\bfA\tilde\bfx-\bfb}_p \leq (1+\eps)\min_{\bfx\in\mathbb R^d}\norm*{\bfA\bfx-\bfb}_p
	\]
	and nonadaptively reads at most $m$ entries of $\bfb$.
\end{theorem}

Our main result, Theorem \ref{thm:main}, follows directly from Theorem \ref{thm:eps}, up to a $\log n$ factor that is replaced by a $\log(d/\eps)$ factor by results in Section \ref{sec:assump}, as well as an optimization in the $\eps$ dependence for $p\in(1,2)$ as done in Section \ref{sec:improved-eps-p-1-2}.
The proof of Theorem \ref{thm:eps} proceeds in 2 steps. We first show in Lemma \ref{lem:clip} that any entries of the residual $\bfz = \bfb-\bfA\bfx^*$ whose contribution to $\norm{\bfz}_p^p$ is significantly larger than the corresponding Lewis weight in $\bfA$ can be effectively ignored, since no $\bfx$ achieving small error can accurately fit such entries. We then show in Lemma \ref{lem:additive2} that if we ignore these entries, sampling by the $\ell_p$ Lewis weights approximately preserves $\norm{\bfA\bfx-\bfz}_p$ for all $\bfx \in \R^d$. The result follows by combining Lemmas \ref{lem:clip} and \ref{lem:additive2}.

\begin{lemma}\label{lem:clip} Let $\bfA\in\mathbb R^{n\times d}$ and $\bfb\in\mathbb R^n$, and let $0 < p < \infty$. Let $\bfz = \bfb - \bfA\bfx^*\in \R^n$. Let 
	\[
		\mathcal B \coloneqq \braces*{i\in[n] : \frac{|\bfz(i)|^p}{\OPT^p} \ge \frac{d^{\max(0,p/2-1)} \cdot \bfw_{i}^p(\bfA)}{\epsilon^p}}	
	\]
	Let $\bar \bfz \in \R^n$ be equal to $\bfz$ but with all entries in $\mathcal{B}$ set to $0$. Then for all $\bfx \in \R^d$ with $\norm{\bfA\bfx}_p = O(\OPT)$,
$$\left | \norm{\bfA\bfx-\bfz}_p^p - \norm{\bfA\bfx-\bar \bfz}_p^p - \norm*{\bfz-\bar\bfz}_p^p \right | = O( \epsilon) \cdot \OPT^p.$$
\end{lemma}
\begin{proof}
For any $\bfx \in \R^d$ and any $i \in \mathcal{B}$, we have that
\begin{align*}
|[\bfA\bfx](i)|^p &\le \norm{\bfA\bfx}_p^p \cdot \bfs_i^p(\bfA) \\
&\le \norm{\bfA\bfx}_p^p \cdot d^{\max(0,p/2-1)} \cdot \bfw^p_i(\bfA) && \text{Lemma \ref{lem:lwBound}} \\
&\le \norm{\bfA\bfx}_p^p \cdot \epsilon^p \cdot \frac{|\bfz(i)|^p}{\OPT^p} = O(\epsilon^p) \cdot |\bfz(i)|^p,
\end{align*}
where we use the assumption that $\norm{\bfA\bfx}_p^p = O(\OPT^p)$.
From the above, we have that
\begin{align*}
|[\bfA\bfx](i) - \bfz(i)|^p - |[\bfA\bfx](i) - \bar \bfz(i)|^p &= |[\bfA\bfx](i) - \bfz(i)|^p - |[\bfA\bfx](i)|^p = (1 \pm O(\epsilon)) \cdot |\bfz(i)|^p
\end{align*}
for $i\in\mathcal B$.
Since $\norm{\bfz-\bar \bfz}_p^p \le \norm{\bfz}_p^p = \OPT^p$,
\begin{align*}
\left | \norm{\bfA\bfx-\bfz}_p^p - \norm{\bfA\bfx-\bar \bfz}_p^p -\norm{\bfz}_p^p \right | = O( \epsilon) \cdot \norm{\bfz-\bar \bfz}_p^p = O(\epsilon) \cdot \OPT^p. 
\end{align*}
\end{proof}

\begin{lemma}\label{lem:clip2} Consider the setting of Lemma \ref{lem:clip}. Let $\bfS \in \R^{m \times n}$ be formed as in Line \ref{line:split-sample} of Algorithm \ref{alg:clip}. With probability at least $99/100$, $\norm{\bfS\bfz}_p^p = O(\OPT^p)$ and further, for all $\bfx \in \R^d$ with $\norm{\bfA\bfx}_p = O(\OPT)$,
\[
	\left | \norm{\bfS\bfA\bfx-\bfS\bfz}_p^p - \norm{\bfS\bfA\bfx-\bfS\bar \bfz}_p^p - \norm{\bfS(\bfz-\bar \bfz)}_p^p \right | = O(\eps) \cdot \OPT^p.
\]
\end{lemma}
\begin{proof}
The proof follows that of Lemma \ref{lem:clip}. We have, for all $\bfx \in \R^d$ and any $i \in \mathcal{B}$ sampled by $\bfS$, 
\begin{align*}
|[\bfS\bfA\bfx](i) - \bfS\bfz(i)|^p - |[\bfS\bfA\bfx](i) - \bfS\bar \bfz(i)|^p &= |[\bfS\bfA\bfx](i) - \bfS\bfz(i)|^p - |[\bfS\bfA\bfx](i)|^p \\
&= (1 \pm O(\epsilon)) \cdot |\bfS\bfz(i)|^p,
\end{align*}
where we have used that $\bfS$ is a subspace embedding so $\norm*{\bfS\bfA\bfx}_p = O(\norm*{\bfA\bfx}_p) = O(\OPT)$. Now, $\E[\norm{\bfS \bfz}_p^p] = \norm{\bfz}_p^p = O(\OPT^p)$. Then by Markov's inequality, with probability at least $99/100$, $\norm{\bfS(\bfz-\bar \bfz)}_p^p \le \norm{\bfS \bfz}_p^p = O(\OPT^p)$. We then have
\begin{align*}
\left | \norm{\bfS\bfA\bfx-\bfS\bfz}_p^p - \norm{\bfS\bfA\bfx-\bfS\bar \bfz}_p^p - \norm{\bfS(\bfz-\bar \bfz)}_p^p\right | = O(\epsilon) \cdot \norm{\bfS(\bfz-\bar \bfz)}_p^p = O(\eps)\cdot \OPT^p.
\end{align*}
\end{proof}

\begin{remark}
While Lemma \ref{lem:clip2} only succeeds with constant probability, we may boost it by repeating the procedure $\ell = O(\log\frac1\delta)$ times, sorting the trials by $\norm*{\bfS(\bfb-\bfA\bfx_c)}_p^p$ for a constant factor solution $\bfx_c$, and then discarding the top $1/10$ of trials (see Line \ref{line:discard}). By Chernoff bounds, we will then select a trial satisfying 
\[
    \norm*{\bfS\bfz}_p \leq \norm*{\bfS\bfA(\bfx^*-\bfx_c)}_p + \norm*{\bfS(\bfb-\bfA\bfx_c)}_p = O(\OPT)
\]
with probability at least $1-\delta$.
\end{remark}

We now state the following lemma, which shows that Lewis weight sampling as in Line \ref{line:split-sample} of Algorithm \ref{alg:clip} will approximately preserve the $\ell_p$ norms of $\norm*{\bfA\bfx-\bar\bfz}_p$, provided that $\norm*{\bfA\bfx}_p = O(\OPT)$. 

\begin{lemma}\label{lem:additive2} 
	Consider the setting of Lemma \ref{lem:clip}. Let $\bfS \in \R^{m \times n}$ be formed as in Line \ref{line:split-sample} of Algorithm \ref{alg:clip}. Let
	\[
		\Lambda \coloneqq \sup_{\norm*{\bfA\bfx}_p \leq O(\OPT)}\abs*{\norm*{\bfS\bfA\bfx-\bfS\bar\bfz}_p^p - \norm*{\bfA\bfx-\bar\bfz}_p^p}.
	\]
	Then,
	\[
		\Pr\braces*{\Lambda \geq \bracks*{C\frac1n \frac{d^{1\lor (p/2)}}{\eps^{0\lor(p-2)}}}^{1/2} \bracks*{(\log d)\sqrt{\log n} + z} \OPT^p} \leq 2\exp(-z^2)
	\]
\end{lemma}

We defer the proof of Lemma \ref{lem:additive2} to Section \ref{sec:proof-additive2} and first show how to use it to prove Theorem \ref{thm:eps}. 

\begin{proof}[\textbf{Proof of Theorem \ref{thm:eps}}]
Let $\bfS$ be the sketching matrix generated as a result of the $O(\log n)$ rounds of recursion. As in the proof of Theorem \ref{thm:lp-subspace-embedding}, the algorithm makes at most $O(\log n)$ levels of recursion with probability at least $1-\delta$, and generates a $\bfS$ such that
\[
    \Pr\braces*{\sup_{\norm*{\bfA\bfx}_p \leq O(\OPT)}\abs*{\norm*{\bfS\bfA\bfx-\bfS\bar\bfz}_p^p - \norm*{\bfA\bfx-\bar\bfz}_p^p} \geq \eps \OPT^p} \leq \frac{\delta}{\log\frac1\delta}
\]
Consider $\ell = O(\log\frac1\delta)$ independent repetitions of Algorithm \ref{alg:clip}. By a union bound over all $\ell$ repetitions, the above guarantee holds simultaneously for all repetitions. Note also that with probability at least $1-\delta$, we only keep trials such that $\norm*{\bfS\bfz}_p = O(\OPT)$ in Line \ref{line:discard}. Furthermore, with probability at least $1-\delta$, the boosting algorithm of Theorem \ref{thm:boost} selects a run outputting a constant factor approximation. Let $\bar\bfx$ be the solution that is output by the algorithm. Since $\bar\bfx$ is a constant factor approximation, we have the guarantee that
\[
	\norm*{\bfA\bar\bfx}_p \leq \norm*{\bfA\bar\bfx-\bfz}_p + \norm*{\bfz}_p = O(\OPT).
\]
Similarly, $\norm*{\bfA\bfx^*}_p = O(\OPT)$, and thus $\norm*{\bfA(\bar\bfx-\bfx^*)}_p = O(\OPT)$. Then, 
\begin{align*}
\norm*{\bfA\bar\bfx-\bfb}_p^p &= \norm*{\bfA(\bar\bfx-\bfx^*)-\bfz}_p^p \\
&\leq \norm*{\bfA(\bar\bfx-\bfx^*)-\bar\bfz}_p^p + \norm*{\bar\bfz - \bfz}_p^p + O(\eps)\cdot\OPT^p && \text{Lemma \ref{lem:clip}} \\
&\leq \norm*{\bfS\bfA(\bar\bfx-\bfx^*)-\bfS\bar\bfz}_p^p + \norm*{\bar\bfz - \bfz}_p^p + O(\eps)\cdot\OPT^p && \text{Lemma \ref{lem:additive2}} \\
&\leq \norm*{\bfS\bfA(\bar\bfx-\bfx^*)-\bfS\bfz}_p^p - \norm*{\bfS(\bfz-\bar\bfz)}_p^p + \norm*{\bar\bfz - \bfz}_p^p + O(\eps)\cdot\OPT^p && \text{Lemma \ref{lem:clip2}} \\
&\leq \norm*{\bfS\bfA\bfx-\bfS\bfb}_p^p - \norm*{\bfS(\bfz-\bar\bfz)}_p^p + \norm*{\bar\bfz - \bfz}_p^p + O(\eps)\cdot\OPT^p \\
&\leq \norm*{\bfS\bfA\bfx^*-\bfS\bfb}_p^p - \norm*{\bfS(\bfz-\bar\bfz)}_p^p + \norm*{\bar\bfz - \bfz}_p^p + O(\eps)\cdot\OPT^p && \text{near optimality of $\bar\bfx$} \\
&\leq \norm*{\bfS\bfA\bfx^*-\bfS\bar\bfz}_p^p + \norm*{\bar\bfz - \bfz}_p^p + O(\eps)\cdot\OPT^p && \text{Lemma \ref{lem:clip2}} \\
&\leq \norm*{\bfA\bfx^*-\bar\bfz}_p^p + \norm*{\bar\bfz - \bfz}_p^p + O(\eps)\cdot\OPT^p && \text{Lemma \ref{lem:additive2}} \\
&\leq \norm*{\bfA\bfx^*-\bfz}_p^p + O(\eps)\cdot\OPT^p && \text{Lemma \ref{lem:clip}} \qedhere
\end{align*}
\end{proof}

\subsection{Proof of Lemma \ref{lem:additive2}}\label{sec:proof-additive2}

We return to proving Lemma \ref{lem:additive2}, which shows that Lewis weight sampling using the Lewis weights of $\bfA$ can preserve the norm of $\bfA\bfx-\bar\bfz$. To do this, we will modify the proofs of the Lewis weight sampling subspace embeddings of \cite{LT2011,SchechtmanZvavitch:2001}. We return to the framework of Section \ref{sec:chaining} which considers sampling each row with probability $1/2$ under the assumption that $\bfA$ has $\ell_p$ Lewis weights bounded by $O(d/n)$. However, this time, we must preserve the norm of $\bfA\bfx - \bar\bfz$ rather than $\bfA\bfx$. We thus consider the quantity
\[
	\Lambda \coloneqq \sup_{\bfy\in T}\abs*{\sum_{i=1}^n \sigma_i\abs*{\bfy(i)}^p} = \sup_{\bfy\in T}\abs*{\sum_{i=1}^n \bfw_i \sigma_i\abs*{[\bfW^{-1/p}\bfy](i)}^p}
\]
where $T \coloneqq \braces*{\bfA\bfx-\bar\bfz : \norm*{\bfA\bfx}_p \leq O(\OPT)}$ and $\bfw_i$ are the $\ell_p$ Lewis weights of $\bfA$ which are bounded by $O(d/n)$. We now handle this by considering three different parameter regimes: $p\in (2, \infty)$, $p\in(1,2)$, and $p\in(0,1)$. Note that we exclude $p = 1$ and $p = 2$, as they have already been handled by previous work \cite{ChenPrice:2019a, ParulekarParulekarPrice:2021, ChenDerezinski:2021}. Throughout this section, we assume that $\OPT = O(1)$ by scaling. 

We first adapt the work of \cite{LT2011} to handle the cases of $p\in (1,2)$ and $p\in (2,\infty)$. The following is an adaptation of Lemma 15.17 of \cite{LT2011}, which is essentially the statement that Lewis weights of $\bfA$ uniformly bound the sensitivities over the set $T$, which is in turn a combination of Lemma \ref{lem:lwBound} and the definition of $\mathcal B$ in Lemma \ref{lem:clip}.

\begin{lemma}[Sensitivity Bounds]\label{lem:sensitivity-bound-T}
Let $p\geq 1$ and $\bfy\in T$. Then,
\[
    \norm*{\bfW^{-1/p}\bfy}_\infty \leq O\parens*{\frac{d^{0\lor(1/2-1/p)}}{\eps}}\OPT
\]
\end{lemma}
\begin{proof}
Write $\bfy = \bfA\bfx-\bar\bfz$ for $\bfx\in\mathbb R^d$. For any coordinate $i\in[n]$, we have
\begin{align*}
    \abs*{[\bfW^{-1/p}\bfy](i)}^p &= \bfw_i^{-1}\abs*{[\bfA\bfx-\bar\bfz](i)}^p \\
    &\leq O(2^p)\bfw_i^{-1}(\abs*{[\bfA\bfx](i)}^p+\abs*{[\bar\bfz](i)}^p) \\
    &\leq O(2^p)\bfw_i^{-1}\cdot \bfw_i d^{0\lor(p/2-1)}\bracks*{\norm*{\bfA\bfx}_p^p+\frac1{\eps^p}\norm*{\bar\bfz}_p^p} \\
    &= O(2^p) \frac{d^{0\lor(p/2-1)}}{\eps^p}\OPT^p
\end{align*}
Taking $p$th roots yields the desired result.
\end{proof}

We now separately handle indices $i\in[n]$ such that the Lewis weights $\bfw_i$ are less than some $\poly(\eps/n)$. 

\begin{definition}
Define the set $J\subseteq[n]$ by
\[
    J \coloneqq \braces*{i\in[n] : \bfw_i \geq \frac{\eps^p d}{n^2}}.
\]
\end{definition}

For indices not contained in $J$, we use the following simple bound:

\begin{lemma}\label{lem:not-J-bound}
	We have with probability $1$ that
\[
    \sup_{\bfy\in T}\abs*{\sum_{i\notin J} \bfw_i \sigma_i \abs*{[\bfW^{-1/p}\bfy](i)}^p} \leq O\parens*{\frac{d^{1\lor (p/2)}}{n}}\OPT^p
\]
\end{lemma}
\begin{proof}
For any $\bfy\in T$, we have that
\begin{align*}
    \abs*{\sum_{i\notin J} \bfw_i \sigma_i \abs*{[\bfW^{-1/p}\bfy](i)}^p} &\leq \sum_{i\notin J}\bfw_i\norm*{\bfW^{-1/p}\bfy}_\infty^p && \text{triangle inequality} \\
    &\leq \sum_{i\notin J}\bfw_i O\parens*{\frac{d^{0\lor(p/2-1)}}{\eps^p}}\OPT^p && \text{Lemma \ref{lem:sensitivity-bound-T}} \\
    &\leq \sum_{i\notin J}\frac{\eps^p d}{n^2} \cdot  O\parens*{\frac{d^{0\lor(p/2-1)}}{\eps^p}}\OPT^p && \text{$i\notin J$} \\
    &\leq O\parens*{\frac{d^{1\lor(p/2)}}{n}}\OPT^p.\qedhere
\end{align*}
\end{proof}

For the indices contained in $J$, we use a comparison theorem for Rademacher and Gaussian averages, and then use Dudley's entropy integral (Theorem \ref{thm:dudley}) to bound the corresponding Gaussian process. 

\begin{theorem}[Dudley's entropy integral (Theorem 11.17, \cite{LT2011})]\label{thm:dudley}
	Let $\{X_t\}_{t\in T}$ be a Gaussian process. Define the pseudo-metric on $T$ via
	\[
		d_X(s,t) \coloneqq \norm*{X_s - X_t}_2.
	\]
	Let $E(T, d_X, u)$ denote the minimal number of $d_X$-balls of radius $u$ required to cover $T$. Then,
	\[
		\E\sup_{t\in T} X_t \leq 24 \int_0^\infty \sqrt{\log E(T, B_{X}, u)} ~du .
	\]
\end{theorem}

We first use that $\bfw_i \leq O(d/n)$ to bound
\begin{equation}\label{eq:J-bound}
    \E\sup_{\bfy\in T}\abs*{\sum_{i\in J} \bfw_i \sigma_i \abs*{[\bfW^{-1/p}\bfy](i)}^p} \leq \frac{O(1)}{\sqrt{n}}\cdot \E\sup_{\bfy\in T}\abs*{\sum_{i\in J} (d\bfw_i)^{1/2} \sigma_i \abs*{[\bfW^{-1/p}\bfy](i)}^p}
\end{equation}
We now study the latter quantity, with Rademacher variables $\sigma_i$ replaced by standard Gaussian variables $g_i$.

\begin{definition}\label{def:gp-active-lewis}
Let $\{X_\bfy\}_{\bfy\in T}$ be the Gaussian process defined by
\[
    X_\bfy \coloneqq \sum_{i\in J}(d\bfw_i)^{1/2}g_i\abs*{[\bfW^{-1/p}\bfy](i)}^p.
\]
Note that the pseudo-metric $d_X$ associated with this Gaussian process (see Theorem \ref{thm:dudley}) is
\begin{align*}
    d_X(\bfy,\bfy') &= \norm*{X_\bfy - X_{\bfy'}}_2 \\
    &= \sqrt{\E\bracks*{\sum_{i\in J}(d\bfw_i)^{1/2}g_i\abs*{[\bfW^{-1/p}\bfy](i)}^p - \sum_{i\in J}(d\bfw_i)^{1/2}g_i\abs*{[\bfW^{-1/p}\bfy'](i)}^p}^2} \\
    &= \sqrt{\E\bracks*{\sum_{i\in J}(d\bfw_i)^{1/2}g_i(\abs*{[\bfW^{-1/p}\bfy](i)}^p - \abs*{[\bfW^{-1/p}\bfy'](i)}^p)}^2} \\
    &= O(1)\sqrt{\sum_{i\in J}d\bfw_i(\abs*{[\bfW^{-1/p}\bfy](i)}^p - \abs*{[\bfW^{-1/p}\bfy'](i)}^p)^2}
\end{align*}
\end{definition}

For tail bounds, we can use the same Rademacher contraction principle \cite[Theorem 4.12]{LT2011} and Rademacher average comparison \cite[Equation 4.8]{LT2011} theorems along with a standard comparison lemma due to Panchenko \cite[Lemma 1]{Pan2003}. 

\subsubsection{Entropy Bounds}\label{sec:entropy}

To apply Dudley's entropy integral theorem, we now require estimates on the metric entropy $E(T, d_X, u)$. For easier entropy calculations, we bound $d_X$ by a more convenient metric, adapting Equations (15.18) and (15.19) in \cite{LT2011} to encorporate the change in Lemma \ref{lem:sensitivity-bound-T}.

\begin{lemma}[Bounds on $d_X$]\label{lem:l-inf-dx-bound}
For $\bfy,\bfy'\in T$, we have the following:
\begin{itemize}
    \item if $2\leq p < \infty$, then
    \[
        d_X(\bfy,\bfy') \leq O(p\sqrt d) O\parens*{\frac{d^{1/2-1/p}}{\eps}}^{p/2-1} \norm*{(\bfW^{-1/p}\bfy-\bfW^{-1/p}\bfy')\vert_J}_\infty
    \]
    \item if $1 < p < 2$, then
    \[
        d_X(\bfy,\bfy') \leq O(\sqrt d)\norm*{(\bfW^{-1/p}\bfy - \bfW^{-1/p}\bfy')\vert_J}_\infty^{p/2}
    \]
\end{itemize}
\end{lemma}
\begin{proof}
	We first handle $p\geq 2$. For $a,b\geq 0$, we have by convexity for $p > 1$ that
	\[
		a^p - b^p \leq p(a^{p-1}+b^{p-1})\abs*{a-b}.
	\]
	Then,
	\begin{align*}
		d_X(\bfy,\bfy')^2 &\leq O(d)\sum_{i\in J}\bfw_i\parens*{\abs*{[\bfW^{-1/p}\bfy](i)}^p - \abs*{[\bfW^{-1/p}\bfy'](i)}^p}^2 \\
		&\leq O(p^2d)\norm*{(\bfW^{-1/p}\bfy - \bfW^{-1/p}\bfy')\vert_J}_\infty^2\sum_{i\in J}\bfw_i\max\braces*{\abs*{[\bfW^{-1/p}\bfy](i)}, \abs*{[\bfW^{-1/p}\bfy'](i)}}^{2(p-1)}.
	\end{align*}	
	We then use Lemma \ref{lem:sensitivity-bound-T} to bound
	\begin{align*}
		&\sum_{i\in J}\bfw_i\max\braces*{\abs*{[\bfW^{-1/p}\bfy](i)}, \abs*{[\bfW^{-1/p}\bfy'](i)}}^{2(p-1)} \\
		\leq ~ &\max\braces*{\norm*{\bfW^{-1/p}\bfy}_\infty, \norm*{\bfW^{-1/p}\bfy'}_\infty}^{p-2} \sum_{i\in J}\bfw_i \parens*{\abs*{[\bfW^{-1/p}\bfy](i)}^p + \abs*{[\bfW^{-1/p}\bfy'](i)}^p} \\
		\leq ~ &O\parens*{\frac{d^{1/2-1/p}}{\eps}\OPT}^{p-2}(\norm*{\bfy}_p^p + \norm*{\bfy'}_p^p)\leq O\parens*{\frac{d^{1/2-1/p}}{\eps}}^{p-2}
	\end{align*}
	so combining the bounds and taking square roots gives the result.

For $1 < p < 2$, we have using the subadditivity of $|\cdot|^{p/2}$ that
\begin{align*}
	d_X(\bfy,\bfy')^2 &\leq O(d)\sum_{i\in J}\bfw_i\parens*{\abs*{[\bfW^{-1/p}\bfy](i)}^p - \abs*{[\bfW^{-1/p}\bfy'](i)}^p}^2 \\
	&= O(d)\sum_{i\in J}\bfw_i\parens*{\abs*{[\bfW^{-1/p}\bfy](i)}^{p/2} - \abs*{[\bfW^{-1/p}\bfy'](i)}^{p/2}}^2 \\
	&\hspace{6.4em}\parens*{\abs*{[\bfW^{-1/p}\bfy](i)}^{p/2} + \abs*{[\bfW^{-1/p}\bfy'](i)}^{p/2}}^2 \\
	&\leq O(d)\sum_{i\in J}\bfw_i\abs*{[\bfW^{-1/p}\bfy - \bfW^{-1/p}\bfy'](i)}^p \parens*{\abs*{[\bfW^{-1/p}\bfy](i)}^{p} + \abs*{[\bfW^{-1/p}\bfy'](i)}^{p}} \\
	&\leq O(d)\norm*{(\bfW^{-1/p}\bfy-\bfW^{-1/p}\bfy')\vert_J}_\infty^p (\norm*{\bfy}_p^p + \norm*{\bfy'}_p^p) \\
	&\leq O(d)\norm*{(\bfW^{-1/p}\bfy-\bfW^{-1/p}\bfy')\vert_J}_\infty^p.\qedhere
\end{align*}
\end{proof}

Because $\bfw_i$ is bounded below for $i\in J$, we may further bound 
\begin{equation}\label{eq:dx-bound-wq}
	\norm*{(\bfW^{-1/p}\bfy-\bfW^{-1/p}\bfy')\vert_J}_\infty \leq O(1)\norm*{\bfW^{-1/p}(\bfy-\bfy')}_{\bfw,q}
\end{equation}
for $q = O(\log n)$, where
\[
	\norm*{\bfv}_{\bfw,q} \coloneqq \bracks*{\sum_{i=1}^n \frac{\bfw_i}{d}\abs*{\bfv(i)}^q}^{1/q}.
\]
Indeed, if $i\in J$, then
\[
	\frac1{\poly(n)}\sum_{i\in J}\abs*{\bfv(i)}^q \leq \sum_{i=1}^n\frac{\bfw_i}{d}\abs*{\bfv(i)}^q = \norm*{\bfv}_{\bfw,q}^q
\]
so $\norm*{\bfv\vert_J}_\infty \leq O(1)\norm*{\bfv}_{\bfw,q}$. 

Now by the results of Lemma \ref{lem:l-inf-dx-bound} and Equation \eqref{eq:dx-bound-wq}, we can replace $d_X$ balls by $\norm*{\cdot}_{\bfw,q}$ balls of an appropriate radius. Furthermore, we may note that the set $T$ is simply a translation of $\ell_p$ ball of radius $O(\OPT) = O(1)$ by $\bar\bfz$, so it suffices to cover the unit $\ell_p$ ball $B_p$ by $\norm*{\cdot}_{\bfw,q}$ balls. These are exactly the types of entropy estimates given by \cite{BourgainLindenstraussMilman:1989, LT2011}. Indeed, we have the following result:

\begin{lemma}[Propositions 15.18 and 15.19, \cite{LT2011}]\label{lem:lt2011-entropy-estimate}
	Let $1 < p < \infty$ and $q = O(\log n)$. Let $E$ be the subspace spanned by the columns of $\bfW^{-1/p}\bfA$ and let $B_{\bfw,p}\subseteq E$ denote the unit $\norm*{\cdot}_{\bfw,p}$-ball in this subspace. Then, for some universal constant $C>0$, the following holds:
	\begin{itemize}
	\item if $2 \leq p < \infty$, then
	\[
		\log E(B_{\bfw,p}, \norm*{\cdot}_{\bfw,q}, u) \leq C\frac{d\log n}{u^2}.
	\]
	\item if $1 < p \leq 2$, then
	\[
		\log E(B_{\bfw,p}, \norm*{\cdot}_{\bfw,q}, u) \leq C\frac{d\log n}{u^p}.
	\]
	\end{itemize}
\end{lemma}

For $u$ tending to $0$, a standard volume argument gives a better bound:
\begin{lemma}\label{lem:simple-entropy-estimate}
	Consider the setting of Lemma \ref{lem:lt2011-entropy-estimate}. Then, for some universal constant $C>0$ we have that
	\[
		\log E(B_{\bfw,p}, \norm*{\cdot}_{\bfw,q}, u) \leq C d\log \frac{d}{u}.
	\]
\end{lemma}

\paragraph{Diameter Bounds.}

Finally, we give diameter bounds for the application of the Dudley tail bound of Theorem \ref{thm:dudley-tail}. 
\begin{lemma}\label{lem:dudley-tail-diameter-bound}
	The $d_X$-diameter of $T$ is bounded by $[O(d^{p/2}/\eps^{p-2})]^{1/2}$ for $2 \leq p < \infty$ and $O(\sqrt d)$ for $1 < p < 2$.
\end{lemma}
\begin{proof}
Let $\bfy = \bfA\bfx-\bar\bfz$ and $\bfy' = \bfA\bfx'-\bar\bfz$ be two points in $T$. Then,
\begin{align*}
	\norm*{[\bfW^{-1/p}(\bfy-\bfy')]\vert_J}_\infty^p &= \norm*{[\bfW^{-1/p}\bfA(\bfx-\bfx')]\vert_J}_\infty^p \\
	&\leq \norm*{\bfW^{-1/p}\bfA(\bfx-\bfx')}_\infty^p \\
	&\leq \max_{i=1}^n O(d^{0\lor (p/2-1)})\bfw_i\cdot \bfw_i^{-1}\norm*{\bfA(\bfx-\bfx')}_p^p \\
	&= O(d^{0\lor (p/2-1)})\norm*{\bfy-\bfy'}_p^p.
\end{align*}
It follows from the above calculation and Lemma \ref{lem:l-inf-dx-bound} that the $d_X$-diameter of the set $T$ is bounded by
\begin{align*}
	\sup_{\bfy,\bfy'\in T}d_X(\bfy,\bfy') &\leq \sup_{\bfy,\bfy'\in T}O(p\sqrt d) O\parens*{\frac{d^{1/2-1/p}}{\eps}}^{p/2-1} \norm*{(\bfW^{-1/p}\bfy-\bfW^{-1/p}\bfy')\vert_J}_\infty \leq \bracks*{O(1)\frac{d^{p/2}}{\eps^{p-2}}}^{1/2}
\end{align*}
for $2\leq p < \infty$ and
\begin{align*}
	\sup_{\bfy,\bfy'\in T}d_X(\bfy,\bfy') &\leq \sup_{\bfy,\bfy'\in T}O(\sqrt d) \norm*{(\bfW^{-1/p}\bfy-\bfW^{-1/p}\bfy')\vert_J}_\infty^{p/2} \leq O(\sqrt d)
\end{align*}
for $1 < p < 2$.
\end{proof}

\subsubsection{Proof for \texorpdfstring{$2<p<\infty$}{2 < p < inf}}

\begin{proof}[Proof of Lemma \ref{lem:additive2}, $2 < p < \infty$]
We will calculate Dudley's entropy integral. As previously noted, we may replace $T$ by the set
\[
	B_p \coloneqq \braces*{\bfA\bfx : \norm*{\bfA\bfx}_p \leq \OPT} = \braces*{\bfA\bfx : \norm*{\bfW^{-1/p}\bfA\bfx}_{\bfw,p} \leq \frac{\OPT}{d^{1/p}}} = \frac{\Theta(1)}{d^{1/p}}\bfW^{1/p}(B_{\bfw,p})
\]
since translations do not change cover numbers. Furthermore, we have by Lemma \ref{lem:l-inf-dx-bound} and Equation \eqref{eq:dx-bound-wq} that for
\[
	\alpha = O(p\sqrt d) O\parens*{\frac{d^{1/2-1/p}}{\eps}}^{p/2-1},
\]
we have that
\begin{align*}
	\log E(B_p, d_X, t) &\leq \log E(\Theta(d^{-1/p})B_{\bfw,p}, \norm*{\cdot}_{\bfw,q}, t/\alpha) \\
	&= \log E(B_{\bfw,p}, \norm*{\cdot}_{\bfw,q}, \Theta(d^{1/p}t/\alpha))
\end{align*}
since linear transformations do not change cover numbers. Then, by Theorem \ref{thm:dudley}, we have that
\begin{align*}
	&\E\sup_{\bfy\in T}\abs*{(d\bfw_i)^{1/2} g_i \abs*{[\bfW^{-1/p}\bfy](i)}^p} \\
	\leq ~&24\int_0^\infty \sqrt{\log E(T, d_X, u)}~du \\
	\leq ~&\frac{O(\alpha)}{d^{1/p}}\int_0^\infty \sqrt{\log E(B_{\bfw,p}, \norm*{\cdot}_{\bfw,q}, u)}~du \\
	= ~&\frac{O(\alpha)}{d^{1/p}}\bracks*{\int_0^1 \sqrt{\log E(B_{\bfw,p}, \norm*{\cdot}_{\bfw,q}, u)}~du + \int_1^\infty \sqrt{\log E(B_{\bfw,p}, \norm*{\cdot}_{\bfw,q}, u)}~du} \\
	\leq ~&\frac{O(\alpha)}{d^{1/p}}\bracks*{\int_0^1 \sqrt{d\log\frac{d}{u}}~du + \int_1^{\poly(d)} \frac{\sqrt{d\log n}}{u}~du} \tag{Lemmas \ref{lem:lt2011-entropy-estimate} and \ref{lem:simple-entropy-estimate}} \\
	\leq ~&O(1)\frac{(d^{p/2})^{1/2}(\log d)\sqrt{\log n}}{\eps^{p/2-1}}.
\end{align*}
Combining the above bound with Equation \eqref{eq:J-bound} and Lemma \ref{lem:not-J-bound} yields the result that
\[
	\E\Lambda \leq \bracks*{O(1)\frac1n \frac{d^{p/2}(\log d)^2\log n}{\eps^{p-2}}}^{1/2}.
\]
Furthermore, applying Theorem \ref{thm:dudley-tail} and Lemma \ref{lem:dudley-tail-diameter-bound} yields the result that
\[
	\Pr\braces*{\Lambda \geq C\frac1{\sqrt n}\bracks*{\frac{d^{p/2}}{\eps^{p-2}}}^{1/2}\bracks*{(\log d)\sqrt{\log n}+z}} \leq 2\exp(-z^2).\qedhere
\]
\end{proof}

\subsubsection{Proof for \texorpdfstring{$1<p<2$}{1 < p < 2}}\label{sec:proof-p-1-2}

\begin{proof}[Proof of Lemma \ref{lem:additive2}, $1 < p < 2$]
Our proof is very similar to that of $2 < p < \infty$. By combining the reasoning of the proof for $2 < p < \infty$ with Lemma \ref{lem:l-inf-dx-bound} and Equation \eqref{eq:dx-bound-wq}, we obtain that
\begin{align*}
	\log E(T, d_X, t) &\leq \log E(\Theta(d^{-1/p})B_{\bfw,p}, \norm*{\cdot}_{\bfw,q}, (t/\sqrt d)^{2/p}) \\
	&= \log E(B_{\bfw,p}, \norm*{\cdot}_{\bfw,q}, \Theta(d^{1/p})(t/\sqrt d)^{2/p}) \\
	&= \log E(B_{\bfw,p}, \norm*{\cdot}_{\bfw,q}, \Theta(t^{2/p})).
\end{align*}
Then by Theorem \ref{thm:dudley}, we have that
\begin{align*}
	&\E\sup_{\bfy\in T}\abs*{(d\bfw_i)^{1/2} g_i \abs*{[\bfW^{-1/p}\bfy](i)}^p} \\
	\leq ~&24\int_0^\infty \sqrt{\log E(T, d_X, u)}~du \\
	\leq ~&O(1)\int_0^\infty \sqrt{\log E(B_{\bfw,p}, \norm*{\cdot}_{\bfw,q}, u^{2/p})}~du \\
	= ~&O(1)\bracks*{\int_0^1 \sqrt{\log E(B_{\bfw,p}, \norm*{\cdot}_{\bfw,q}, u^{2/p})}~du + \int_1^\infty \sqrt{\log E(B_{\bfw,p}, \norm*{\cdot}_{\bfw,q}, u^{2/p})}~du} \\
	\leq ~&O(1)\bracks*{\int_0^1 \sqrt{d\log\frac{d}{u}}~du + \int_1^{\poly(d)} \frac{\sqrt{d\log n}}{u}~du} \tag{Lemmas \ref{lem:lt2011-entropy-estimate} and \ref{lem:simple-entropy-estimate}} \\
	\leq ~&O(1)\sqrt d(\log d)\sqrt{\log n}.
\end{align*}
Combining the above bound with Equation \eqref{eq:J-bound} and Lemma \ref{lem:not-J-bound} yields the result that
\[
	\E\Lambda \leq \bracks*{O(1)\frac1n d(\log d)^2\log n}^{1/2}.
\]
Furthermore, applying Theorem \ref{thm:dudley-tail} and Lemma \ref{lem:dudley-tail-diameter-bound} yields the result that
\[
	\Pr\braces*{\Lambda \geq C\frac1{\sqrt n}\sqrt d\bracks*{(\log d)\sqrt{\log n}+z}} \leq 2\exp(-z^2).\qedhere
\]
\end{proof}

\subsubsection{Proof for \texorpdfstring{$0<p<1$}{0 < p < 1}}

Note that in the proof of $1 < p < 2$, essentially the only thing that changed between the active setting and the proof of the subspace embedding was the translation by a vector $\bar\bfz$ with norm at most $O(\OPT) = O(1)$, which did not affect the entropy integral bound. In the case of $0 < p < 1$, the proof of \cite{SchechtmanZvavitch:2001} is structured similarly and can be easily checked to have essentially the same straightforward modifications. 

\subsection{Sharp Dependence on \texorpdfstring{$\eps$}{eps} for \texorpdfstring{$1 < p < 2$}{1 < p < 2}}\label{sec:improved-eps-p-1-2}

We will now optimize our dependence on $\eps$ for $1 < p < 2$. 

\subsubsection{Closeness of Near-Optimal Solutions}

We first show that near-optimal solutions $\bfx\in\mathbb R^d$ are close to the optimal solution $\bfx^*$, using strong convexity with respect to the $\norm*{\cdot}_p$ norm. 

\begin{theorem}\label{thm:close-opt-strong-convex}
Let $1 < p \leq 2$. Let $\bfA\in\mathbb R^{n\times d}$ and $\bfb\in\mathbb R^n$. Then, for any $\bfx\in\mathbb R^d$ such that
\[
    \norm*{\bfA\bfx-\bfb}_p \leq (1+\gamma) \OPT
\]
with $\gamma\in(0,1)$, we have that
\[
    \norm*{\bfA\bfx-\bfA\bfx^*}_p \leq O(\sqrt\gamma) \OPT,
\]
where $\bfx^* \coloneqq \arg\min_{\bfx\in\mathbb R^d}\norm*{\bfA\bfx-\bfb}_p$
\end{theorem}
\begin{proof}
For $p = 2$, this follows from the Pythagorean theorem. Otherwise, note first that the KKT conditions require that
\[
    \angle*{(\bfA\bfx^*-\bfb)^{p-1},\bfA\bfx} = 0
\]
for all $\bfx\in\mathbb R^d$, where the power to $p-1$ denotes the signed entrywise power. Then by the strong convexity of $\norm*{\cdot}_p^2$ with respect to $\norm*{\cdot}_p$ \cite[Lemma 8.1]{BMN2001}, we have that
\begin{align*}
    \norm*{\bfA\bfx-\bfb}_p^2 &\geq \norm*{\bfA\bfx^*-\bfb}_p^2 + 2\norm*{\bfA\bfx^*-\bfb}_p^{2-p}\angle*{(\bfA\bfx^*-\bfb)^{p-1},\bfA\bfx-\bfA\bfx^*} + \frac{p-1}2\norm*{\bfA\bfx-\bfA\bfx^*}_p^2 \\
    &\geq \norm*{\bfA\bfx^*-\bfb}_p^2 + \frac{p-1}2\norm*{\bfA\bfx-\bfA\bfx^*}_p^2 
\end{align*}
which rearranges to
\[
    \norm*{\bfA\bfx-\bfA\bfx^*}_p \leq \sqrt{\frac{\norm*{\bfA\bfx-\bfb}_p^2 - \norm*{\bfA\bfx^*-\bfb}_p^2}{(p-1)/2}} \leq \sqrt{\frac{O(\gamma)\OPT^2}{(p-1)/2}} \leq O(\sqrt\gamma)\OPT.\qedhere
\]
\end{proof}

\subsubsection{Cost Difference}

By Theorem \ref{thm:close-opt-strong-convex}, if we first obtain a $(1+\gamma)$-approximate solution, then this solution must be within a distance of $O(\sqrt\gamma)\OPT$ from the optimum in the $\ell_p$ norm. Thus, we may in fact just analyze the distortion of the sampling process when restricted to a ball of radius $O(\sqrt\gamma)\OPT$. In this case, we will show that we obtain improved approximations.

We now consider sampling each row independently with probability $1/2$. This can be expressed as multiplying each row by $(1+\sigma_i)$ for independent Rademacher variables $\sigma_i\in\{\pm1\}$ for $i\in[n]$.w Suppose that we find $\hat \bfx\coloneqq \arg\min_{\bfx\in\mathbb R}\hat f(\bfx)$ for
\begin{equation}\label{eq:def-f-hat}
    \hat f(\bfx) \coloneqq \sum_{i=1}^n (1+\sigma_i) \abs*{[\bfA\bfx-\bfb](i)}^p.
\end{equation}
Then,
\begin{align*}
    f(\hat \bfx) - f(\bfx^*) &= [f(\hat \bfx) - \hat f(\hat \bfx)] + [\hat f(\hat \bfx) - \hat f(\bfx^*)] + [\hat f(\bfx^*) - f(\bfx^*)] \\
    &\leq [f(\hat \bfx) - \hat f(\hat \bfx)] - [f(\bfx^*) - \hat f(\bfx^*)]
\end{align*}
where we have used that $\hat f(\hat \bfx) \leq \hat f(\bfx^*)$. Thus, the difference in the quality of the two solutions $\hat \bfx$ and $\bfx^*$ depends on the \emph{difference in $f-\hat f$ between $\hat \bfx$ and $\bfx^*$}. We thus define the difference function
\[
    \bar f(\bfx) \coloneqq f(\bfx) - \hat f(\bfx)
\]
for our analyses. Note then that
\[
    \bar f(\bfx) = \sum_{i=1}^n \sigma_i \abs*{[\bfA\bfx-\bfb](i)}^p
\]
up to switching the signs on the Rademacher variables $\sigma_i$. The problem then is to bound $\abs*{\bar f(\bfx_1) - \bar f(\bfx_2)}$ for $\bfx_1, \bfx_2$ such that $\norm*{\bfA\bfx_1 - \bfA\bfx_2}_p \leq O(\sqrt\gamma)\OPT$. We will do so by bounding
\[
    \abs*{\bar f(\bfx)-\bar f(0)} = \abs*{\sum_{i=1}^n \sigma_i \parens*{\abs*{[\bfA\bfx-\bfb](i)}^p - \abs*{\bfb(i)}^p}}
\]
for $\bfx$ with $\norm*{\bfA\bfx}_p \leq O(\sqrt\gamma)\OPT$. We thus define
\[
\Lambda \coloneqq \sup_{\norm*{\bfA\bfx}_p \leq O(\sqrt\gamma)\OPT} \abs*{\sum_{i=1}^n \sigma_i \parens*{\abs*{\bfA\bfx-\bfb}^p - \abs*{\bfb(i)}^p}}.
\]
We will now bound this quantity.

\subsubsection{Chaining Argument}

For a subset $S\subseteq[n]$, define
\[
    \Lambda_S \coloneqq \sup_{\norm*{\bfA\bfx}_p \leq O(\sqrt\gamma)\OPT} \abs*{\sum_{i\in S} \sigma_i \parens*{\abs*{\bfA\bfx-\bfb}^p - \abs*{\bfb(i)}^p}}.
\]
Via a Gaussian comparison theorem (see, e.g., \cite{CohenPeng:2015, CSS2021}), one can show that
\begin{equation}\label{eq:sym-gp}
\begin{aligned}
\E_\Omega\Lambda_S &\leq \E_\Omega \E_\sigma \sup_{\norm*{\bfA\bfx}_p \leq O(\sqrt\gamma)\OPT}\abs*{\sum_{i\in S}\sigma_i \cdot\parens*{\abs*{[\bfA\bfx-\bfb](i)}^p - \abs*{\bfb(i)}^p}} \\
&= \sqrt{2\pi}\E_\Omega \E_g \sup_{\norm*{\bfA\bfx}_p \leq O(\sqrt\gamma)\OPT}\abs*{\sum_{i\in S}g_i \cdot \parens*{\abs*{[\bfA\bfx-\bfb](i)}^p - \abs*{\bfb(i)}^p}},
\end{aligned}
\end{equation}
where $\sigma_i$ are independent Rademacher variables and $g_i$ are independent standard Gaussian variables.

By Lemma \ref{lem:clip}, we may assume that $\norm*{\bfW^{-1/p}\bfb}_\infty \leq O(1/\eps)\OPT$. Now as done previously, we first handle the coordinates whose Lewis weights are very small:
\begin{lemma}\label{lem:bound-small-lw}
Suppose that $\norm*{\bfW^{-1/p}\bfb}_\infty \leq O(1/\eps)\OPT$. Let
\[
R \coloneqq \braces*{i\in[n] : \bfw_i^p(\bfA) < \sqrt{\frac{\gamma d}{n}}\frac{\eps^{p}}{n}}.
\]
Then, $\Lambda_R \leq O(\sqrt{\gamma d/n})\OPT^p$ almost surely.
\end{lemma}
\begin{proof}
Using \eqref{eq:sym-gp}, we have that
\begin{align*}
\Lambda_R &= \sup_{\norm*{\bfA\bfx}_p \leq O(\sqrt\gamma)\OPT}\abs*{\sum_{i\in R}\sigma_i \cdot\parens*{\abs*{\bfA\bfx-\bfb}^p - \abs*{\bfb(i)}^p}} \\
&\leq \sup_{\norm*{\bfA\bfx}_p \leq O(\sqrt\gamma)\OPT}\sum_{i\in R} \abs*{\bfA\bfx-\bfb}^p + \abs*{\bfb(i)}^p \\
&\leq O(1)\sup_{\norm*{\bfA\bfx}_p \leq O(\sqrt\gamma)\OPT}\sum_{i\in R} \bfw_i^p(\bfA)\bracks*{\norm*{\bfA\bfx}_p^p + \frac{\sqrt\gamma^p}{\eps^p}\OPT^p} \\
&\leq O(1)\cdot n \cdot \sqrt{\frac{\gamma d}{n}}\frac{\eps^{p}}{n} \cdot \frac{1}{\eps^p}\OPT^p \\
&\leq O(\sqrt{\gamma d/n})\OPT^p.\qedhere
\end{align*}
\end{proof}

For the remaining coordinates, we partition into two sets: those with large coordinates of $\bfb$ and those with small coordinates. Define the sets
\begin{align*}
    L_\bfb &\coloneqq \braces*{i\in [n] \setminus R : \abs*{\bfb(i)} \geq C\sqrt\gamma\cdot \bfw_i^p(\bfA)^{1/p}\OPT} \\
    S_\bfb &\coloneqq \braces*{i\in [n] \setminus R : \abs*{\bfb(i)} < C\sqrt\gamma\cdot \bfw_i^p(\bfA)^{1/p}\OPT}
\end{align*}
where $R$ is as defined in Lemma \ref{lem:bound-small-lw} and $C>0$ is a sufficiently large constant. We then bound $\E\Lambda_{[n]\setminus R}$ as
\begin{equation}\label{lem:lambda-n-r-bound}
\begin{aligned}
    \E\Lambda_{[n]\setminus R} &\leq \E\sup_{\norm*{\bfA\bfx}_p \leq O(\sqrt\gamma)\OPT}\abs*{\sum_{i\in [n]\setminus R}\sigma_i \cdot\parens*{\abs*{\bfA\bfx-\bfb}^p - \abs*{\bfb(i)}^p}} \\
    &\leq \E\sup_{\norm*{\bfA\bfx}_p \leq O(\sqrt\gamma)\OPT}\frac1{\sqrt n}\abs*{\sum_{i\in [n]\setminus R}(d\bfw_i)^{1/2}\sigma_i \cdot\parens*{\abs*{\bfW^{-1/p}(\bfA\bfx-\bfb)}^p - \abs*{\bfW^{-1/p}\bfb(i)}^p}}
\end{aligned}
\end{equation}
as done previously, and bound the latter term by bounding the corresponding terms for $L_\bfb$ and $S_\bfb$. We now bound the corresponding Gaussian processes given by
\[
    X_\bfy \coloneqq C_S\sum_{i\in S} (d\bfw_i)^{1/2} g_i \cdot \parens*{\abs*{[\bfW^{-1/p}(\bfy-\bfb)](i)}^p - \abs*{\bfW^{-1/p}\bfb(i)}^p}
\]
for $\bfy\in \braces*{\bfA\bfx : \bfx\in\mathbb R^d, \norm*{\bfA\bfx}_p \leq O(\sqrt\gamma)\OPT}$, where $S$ can be either $L_\bfb$ or $S_\bfb$, and $C_S$ is a normalizing constant associated with $S$. Our first task is to bound the associated metric $d_X$. 

\begin{lemma}[Metric for $S_\bfb$]\label{lem:gp-sb}
Let $C_S \coloneqq 1/ \sqrt\gamma^p\OPT^p$. Consider the metric
\begin{align*}
    d_X(\bfy,\bfy') &\coloneqq \norm*{X_\bfy-X_{\bfy'}}_2 \\
    &= \frac{1}{\sqrt\gamma^{p}\OPT^p}\bracks*{\sum_{i\in S_\bfb}(d\bfw_i)\parens*{\abs*{[\bfW^{-1/p}(\bfy-\bfb)](i)}^p - \abs*{\bfW^{-1/p}(\bfy'-\bfb(i))}^p}^2}^{1/2}
\end{align*}
for $\bfy,\bfy'\in \braces*{\bfA\bfx : \bfx\in\mathbb R^d, \norm*{\bfA\bfx}_p \leq O(\sqrt\gamma)\OPT}$. Then,
\[
    d_X(\bfy,\bfy') \leq O(\sqrt d)\frac{\norm*{(\bfW^{-1/p}\bfy-\bfW^{-1/p}\bfy')\vert_{S_\bfb}}_\infty^{p/2}}{\sqrt\gamma^{p/2}\OPT^{p/2}}
\]
\end{lemma}
\begin{proof}
The proof follows from a straightforward adaptation of Lemma \ref{lem:l-inf-dx-bound}.
\end{proof}

Bounding the metric for $L_\bfb$ requires a different approach than the metric used in \cite{LT2011}.

\begin{lemma}[Metric for $L_\bfb$]\label{lem:gp-lb}
Let $C_S \coloneqq 1/\sqrt\gamma\OPT^p$. Consider the metric
\begin{align*}
    d_X(\bfy,\bfy') &\coloneqq \norm*{X_\bfy-X_{\bfy'}}_2 \\
    &= \frac{1}{\sqrt\gamma\OPT^p}\bracks*{\sum_{i\in L_\bfb}(d\bfw_i)\parens*{\abs*{[\bfW^{-1/p}(\bfy-\bfb)](i)}^p - \abs*{\bfW^{-1/p}(\bfy'-\bfb(i))}^p}^2}^{1/2}
\end{align*}
for $\bfy,\bfy'\in \braces*{\bfA\bfx : \bfx\in\mathbb R^d, \norm*{\bfA\bfx}_p \leq O(\sqrt\gamma)\OPT}$. Then,
\[
    d_X(\bfy,\bfy') \leq O(\sqrt d)\frac{\norm*{(\bfW^{-1/p}(\bfy-\bfy'))\vert_{L_\bfb}}_\infty^{p/2}}{\sqrt\gamma^{p/2}\OPT^{p/2}}
\]
\end{lemma}
\begin{proof}
Note that
\begin{align*}
    \abs*{[\bfW^{-1/p}(\bfy-\bfb)](i)}^p &= \parens*{\abs*{[\bfW^{-1/p}(\bfy'-\bfb)](i)} \pm \abs*{[\bfW^{-1/p}(\bfy-\bfy')](i)}}^p \\
    &= \abs*{[\bfW^{-1/p}(\bfy'-\bfb)](i)}^p \pm O(1)\abs*{[\bfW^{-1/p}(\bfy-\bfy')](i)}\abs*{[\bfW^{-1/p}(\bfy'-\bfb)](i)}^{p-1}
\end{align*}
since $i\in L_\bfb$. Thus,
\begin{align*}
    \abs*{\abs*{[\bfW^{-1/p}(\bfy-\bfb)](i)}^p - \abs*{[\bfW^{-1/p}(\bfy'-\bfb)](i)}^p} &\leq O(1)\abs*{[\bfW^{-1/p}(\bfy-\bfy')](i)}\abs*{[\bfW^{-1/p}(\bfy'-\bfb)](i)}^{p-1} 
\end{align*}
so we have that
\begin{align*}
    &\sum_{i\in L_\bfb}(d\bfw_i)\parens*{\abs*{[\bfW^{-1/p}(\bfy-\bfb)](i)}^p - \abs*{\bfW^{-1/p}(\bfy'-\bfb(i))}^p}^2 \\
    \leq~&O\parens*{d}\norm*{(\bfW^{-1/p}(\bfy-\bfy'))\vert_{L_\bfb}}_\infty^p \sum_{i\in L_\bfb}\bfw_i \abs*{[\bfW^{-1/p}(\bfy-\bfy')](i)}^{2-p}\abs*{[\bfW^{-1/p}(\bfy'-\bfb)](i)}^{2p-2}.
\end{align*}
Now note that $\frac{p}{2-p} > 1$ and $\frac{p}{2p-2} > 1$ are H\"older conjugates for $1 < p < 2$. Then, H\"older's inequality on the sum yields
\begin{align*}
    &\sum_{i\in L_\bfb}\bfw_i \abs*{[\bfW^{-1/p}(\bfy-\bfy')](i)}^{2-p}\abs*{[\bfW^{-1/p}(\bfy'-\bfb)](i)}^{2p-2} \\
    \leq~&\parens*{\sum_{i\in L_\bfb}\bfw_i \abs*{[\bfW^{-1/p}(\bfy-\bfy')](i)}^{p}}^{\frac{2-p}{p}}\parens*{\sum_{i\in L_\bfb}\bfw_i\abs*{[\bfW^{-1/p}(\bfy'-\bfb)](i)}^{p}}^{1 - \frac{2-p}{p}} \\
    \leq~&\norm*{\bfy-\bfy'}_p^{2-p} \norm*{\bfy'-\bfb}_p^{2p-2} \leq O(\sqrt\gamma^{2-p})\OPT^p.
\end{align*}
Combining these bounds yields the claim.
\end{proof}

Now given the bounds on the metrics, we can finish by using the entropy bounds from Section \ref{sec:entropy} and then calculating Dudley's entropy integral as done in Section \ref{sec:proof-p-1-2}. 

\begin{lemma}\label{lem:cost-diff-lambda}
We have that
\[
    \E\Lambda \leq O(1) \bracks*{\gamma\frac{d}{n}(\log d)^2(\log n)}^{1/2}\OPT^p
\]
and
\[
	\Pr\braces*{\Lambda \geq C\sqrt{\frac{\gamma d}{n}}\bracks*{(\log d)\sqrt{\log n}+z}\OPT^p} \leq 2\exp(-z^2).
\]
\end{lemma}
\begin{proof}
Our proof is very similar to that of Lemma \ref{lem:additive2} for $1 < p < 2$. Let
\[
    T = \braces*{\bfA\bfx : \bfx\in\mathbb R^d, \norm*{\bfA\bfx}_p \leq O(\sqrt\gamma)\OPT}.
\]
Then, as reasoned in Lemma \ref{lem:additive2} for $1 < p < 2$, for the Gaussian process $X_\bfy$ associated with $S_\bfb$ (see Lemma \ref{lem:gp-sb}), 
\begin{align*}
	\log E(T, d_X, t) &\leq \log E(\Theta(\sqrt\gamma d^{-1/p}\OPT)B_{\bfw,p}, \norm*{\cdot}_{\bfw,q}, (t/\sqrt d)^{2/p}\sqrt\gamma\OPT) \\
	&= \log E(B_{\bfw,p}, \norm*{\cdot}_{\bfw,q}, \Theta(t^{2/p})).
\end{align*}
Similarly, for the Gaussian process $X_\bfy$ associated with $L_\bfb$ (see Lemma \ref{lem:gp-lb}),
\[
    \log E(T, d_X, t) \leq \log E(B_{\bfw,p}, \norm*{\cdot}_{\bfw,q}, \Theta(t^{2/p})).
\]
Then by Theorem \ref{thm:dudley}, following the calculation in Lemma \ref{lem:additive2} for $1 < p < 2$, we have that
\begin{align*}
	\E\sup_{\bfy\in T}\abs*{X_\bfy} \leq O(1)\sqrt d(\log d)\sqrt{\log n}
\end{align*}
for both of the Gaussian processes. We combine this with \eqref{lem:lambda-n-r-bound} to obtain
\begin{align*}
    \E\Lambda_{[n]\setminus R} &\leq O(1)\frac1{\sqrt n}\bracks*{\frac1{C_{L_\bfb}}\sqrt d(\log d)\sqrt{\log n} + \frac1{C_{S_\bfb}}\sqrt d(\log d)\sqrt{\log n}} \\
    &\leq O(1) \frac1n\sqrt d(\log d)\sqrt{\log n}[\sqrt\gamma\OPT^p + \sqrt\gamma^p\OPT^p] \\
    &\leq O(1) \bracks*{\gamma\frac{d}{n}(\log d)^2(\log n)}^{1/2}\OPT^p
\end{align*}
where $C_{L_\bfb} = 1/\sqrt\gamma\OPT^p$ and $C_{S_\bfb} = 1/\sqrt\gamma^p\OPT^p$ are the $C_S$ constants for $S = L_\bfb$ and $S = S_\bfb$ (see Lemmas \ref{lem:gp-lb} and \ref{lem:gp-sb}). Finally, together with Lemma \ref{lem:bound-small-lw}, we conclude that
\[
    \E\Lambda \leq O(1) \bracks*{\gamma\frac{d}{n}(\log d)^2(\log n)}^{1/2}\OPT^p.
\]
In a similar manner as Lemma \ref{lem:dudley-tail-diameter-bound}, the $d_X$-diameter of $T$ is $O(\sqrt d)$ so Theorem \ref{thm:dudley-tail} gives us that
\[
	\Pr\braces*{\Lambda \geq C\sqrt{\frac{\gamma d}{n}}\bracks*{(\log d)\sqrt{\log n}+z}\OPT^p} \leq 2\exp(-z^2).\qedhere
\]
\end{proof}

\subsubsection{Iterative Size Reduction}

We will now assemble the previous lemmas into an iterative size reduction argument. 

Because Lemma \ref{lem:cost-diff-lambda} differs from Lemma \ref{lem:additive2} only by a factor of $\sqrt\gamma$, one can show as done in Theorem \ref{thm:eps} that $m$ entries need to be read to produce a $(1+\eps)$-approximation, for
\[
	m = O(1)\frac{\gamma d}{\eps^{2}}\bracks*{(\log d)^2\log n + \log\frac1\delta},
\]
as long as it can be shown that reading $m$ entries gave a $(1+\gamma)$-approximation. Now let
\[
    C = O(1)\bracks*{(\log d)^2\log n + \log\frac1\delta}
\]
and suppose that reading $Cd/\eps^\beta$ is sufficient to obtain a $(1+\eps)$-approximation (for instance, Theorem \ref{thm:eps} applied directly gives that $\beta\leq 2$). Then, if we set $\gamma = \eps^{2/(1+\beta)}$, then making $Cd/\gamma^\beta$ gives a $(1+\gamma)$-approximation. Furthermore, note that
\[
    C\frac{\gamma d}{\eps^2} = Cd \eps^{\frac{2}{1+\beta} - 2} = Cd \eps^{\frac{2-2(1+\beta)}{1+\beta}} = Cd \eps^{-\frac{2\beta}{1+\beta}} = C\frac{d}{\gamma^\beta}
\]
so reading $Cd/\gamma^\beta = C\gamma d/\eps^2$ entries is in fact sufficient to obtain a $(1+\eps)$-approximation for this setting of $\gamma$. We may now iterate this argument. Consider the sequence $\beta_i$ given by
\[
    \beta_1 = 2, \qquad \beta_{i+1} \to \frac{2\beta_i}{1+\beta_i}.
\]
It can be checked that the solution is
\[
    \beta_i = \frac{2^i}{2^i - 1} = 1 + \frac1{2^i-1},
\]
so applying this argument $O(\log\log\frac1\eps)$ times yields that $\beta_i \leq 1 + O(1/\log(\frac1\eps))$ which means that reading only $O(1)Cd/\eps$ entries suffices.

We have shown the following:
\begin{theorem}[Main Result for $1 < p < 2$]\label{thm:eps-1-p-2}
	Let $\bfA \in \R^{n \times d}$, $\bfb \in \R^n$, and $1 < p < 2$. Let
	\[
		m = O(1)\frac{d}{\eps}\bracks*{(\log d)^2\log n + \log\frac1\delta}\log\frac1\delta.
	\]
	Then, with probability at least $1-\delta$, Algorithm \ref{alg:clip} returns $\tilde\bfx\in\mathbb R^d$ such that
	\[
		\norm*{\bfA\tilde\bfx-\bfb}_p \leq (1+\eps)\min_{\bfx\in\mathbb R^d}\norm*{\bfA\bfx-\bfb}_p
	\]
	and reads at most $m$ entries of $\bfb$.
\end{theorem}

\subsection{Removing the \texorpdfstring{$\log n$}{log n}}\label{sec:assump}

As shown in Section \ref{sec:main}, Algorithm \ref{alg:clip} gives a relative error approximation to $\min \norm{\bfA\bfx-\bfb}_p$ using $\eps^{-2} d^{\max(1,p/2)} \poly(\log d, \log n)$ samples from $\bfb$. However, the dependence on $\log n$ is undesirable, and we now show how to replace the $\log n$ dependence with a $\log(d/\eps)$ dependence.

\subsubsection{Bounding \texorpdfstring{$n$}{n}}\label{sec:nBound}

Due to the use of the sophisticated chaining argument, Theorem \ref{thm:eps} has an undesirable dependence on $\log n$. By replacing this net construction with a standard $\epsilon$-net, we can still show that Algorithm \ref{alg:clip} solves $\ell_p$ regression to $(1+\epsilon)$ error, albeit using $\tilde{O}(d^{\max(1,p/2)+1}/\poly(\epsilon))$ samples (without dependence on $\log n$). This bound is loose by a $d$ factor as compared to Theorem \ref{thm:eps} due to the simpler net construction.
It can, however, be used as a preprocessing step for Theorem \ref{thm:eps}, to ensure that $n$ is only $\poly(d/\eps)$ and so that the $\log n$ dependence is removed --- after sampling by Lewis weights, we are left with a problem where $n = \poly(d,1/\epsilon)$. We can then apply Theorem \ref{thm:eps} to further subsample this problem, achieving $\tilde{O}(d^{\max(1,p/2)}/\poly(\epsilon))$ sample complexity. See Algorithm \ref{alg:clip2} for pseudocode for the full approach.

We use a more standard Lewis weight sampling algorithm for this result:
\begin{definition}\label{def:lewis-weight-sampling}
	Let $\bfA\in\mathbb R^{n\times d}$. Let $\tilde\bfw_i^p(\bfA) \geq \bfw_i^p(\bfA)$ be Lewis weight upper bounds, and let $m$ be a parameter. Define probabilities
	\[
		\bfp_i \coloneqq \min\{1, m\cdot d^{0\lor (p/2-1)}\tilde\bfw_i^p(\bfA)\}
	\]
	Then, $\bfS$ is a Lewis weight sampling matrix if for each $i\in[n]$, $\bfS(i,i) = \bfp_i^{-1/p}$ independently with probability $\bfp_i$ and $0$ otherwise. 
\end{definition}

\begin{algorithm}
	\caption{Relative error $\ell_p$ regression -- No Assumptions}
	\textbf{input:} Matrix $\bfA \in \R^{n \times d}$, measurement vector $\bfb \in \R^n$. \\
	\textbf{output:} Approximate solution $\tilde \bfx \in \R^d$ to $\min_\bfx\|\bfA\bfx - \bfb\|_p$. 
	\begin{algorithmic}[1]
		\State Let $\bfS_1 \in \R^{m_1 \times n}$ be generated via Def.~\ref{def:lewis-weight-sampling} for $m_1=O \left (d^{\max(1,p/2)+1} \cdot \frac{\log(1/\epsilon\delta)}{\epsilon^{2+p}}\right )$. \label{line:sample-m1-rows}
		\State Run Algorithm \ref{alg:clip} with inputs $\bfS_1\bfA$ and $\bfS_1\bfz$ to obtain $\tilde\bfx$. \label{line:solve-small-prob}
		\State \Return $\tilde \bfx$
		\label{line:return-x}
	\end{algorithmic}\label{alg:clip2}
\end{algorithm}

\begin{theorem}[Main Result]\label{thm:epsL}
Let $\bfA \in \R^{n \times d}$, $\bfb \in \R^n$, and $0 < p < \infty$. Let
\[
	m = \begin{dcases}
        O(1)\frac{d}{\eps^{2}}\bracks*{(\log d)^2\log(d/\eps) + \log\frac1\delta}\log\frac1\delta & p\in(0,1) \\
        O(1)\frac{d}{\eps}\bracks*{(\log d)^2\log(d/\eps) + \log\frac1\delta}\log\frac1\delta & p\in(1,2) \\
        O(1)\frac{d^{p/2}}{\eps^{p}}\bracks*{(\log d)^2\log(d/\eps) + \log\frac1\delta}\log\frac1\delta & p\in(2,\infty)
        \end{dcases}
\]
Then, with probability at least $1-\delta$, Algorithm \ref{alg:clip2} returns $\tilde\bfx\in\mathbb R^d$ such that
\[
	\norm*{\bfA\tilde\bfx-\bfb}_p \leq (1+\eps)\min_{\bfx\in\mathbb R^d}\norm*{\bfA\bfx-\bfb}_p
\]
and nonadaptively reads at most $m$ entries of $\bfb$.
\end{theorem}

Theorem \ref{thm:epsL} will follow from an analog of Lemma \ref{lem:additive2}, when $m=O \left (d^{\max(1,p/2)+1} \cdot \frac{\log(1/\epsilon\delta)}{\epsilon^{2+p}}\right )$. In particular:
\begin{lemma}\label{lem:analog}
Let $\bfA \in \R^{n \times d}$, $\bfb \in \R^n$, and $0 < p < \infty$. Let $\bfz = \bfb - \bfA\bfx^*\in \R^n$, with $\norm{\bfz}_p = \OPT$. Let 
\[
	\mathcal B\coloneqq \braces*{i\in[n] : \frac{|\bfz(i)|^p}{\OPT^p} \ge \frac{d^{\max(0,p/2-1)} \cdot \bfw_i^p(\bfA)}{\epsilon^p}}
\]
Let $\bar \bfz \in \R^n$ be equal to $\bfz$ but with all entries in $\mathcal{B}$ set to $0$. Let $\bfS_1 \in \R^{m_1 \times n}$ be formed as in Line \ref{line:sample-m1-rows} of Algorithm \ref{alg:clip2}. With probability at least $1-\delta$, for all $\bfx \in \R^d$ with $\norm{\bfA\bfx}_p = O(\OPT)$, 
\begin{align*}
 \left | \left \| \bfS_1\bfA \bfx -\bfS_1 \bar \bfz \right \|_p^p - \norm{\bfA\bfx - \bar \bfz}_p^p \right | = O(\epsilon) \cdot \OPT^p.
\end{align*}
\end{lemma}
We first use Lemma \ref{lem:analog} to prove Theorem \ref{thm:epsL}. Then we prove the lemma.

\begin{proof}[\textbf{Proof of Theorem \ref{thm:epsL}}]
By restricting to the $m_1$ rows sampled by Line \ref{line:sample-m1-rows} of Algorithm \ref{alg:clip2}, we have a $\poly(d/\eps)\times d$ matrix $\bfS_1\bfA$ such that
\begin{align*}
	\left | \left \| \bfS_1\bfA \bfx -\bfS_1 \bar \bfz \right \|_p^p - \norm{\bfA\bfx - \bar \bfz}_p^p \right | = O(\epsilon) \cdot \OPT^p
\end{align*}
by Lemma \ref{lem:analog}. In combination with Lemmas \ref{lem:clip} and \ref{lem:clip2}, Theorem \ref{thm:const} and Lemma \ref{lem:analog} give that with probability at least $1-O(\delta)$, there is some fixed $C\geq0$ with $C = O(\OPT^p)$ such that for all $\bfx \in \R^d$ with $\norm{\bfA\bfx}_p = O(\OPT)$,
\begin{align*}
|\norm{\bfS_1\bfA\bfx-\bfS_1\bfz}_p^p - \norm{\bfA\bfx-\bfz}_p^p - C| = O(\epsilon) \cdot \OPT^p.
\end{align*}
Following the proof of Theorem \ref{thm:eps}, we then have that if $\bar \bfx \in \R^d$ satisfies $\norm{\bfS_1 \bfA \bar \bfx -\bfS_1 \bfz}_p \le (1+\epsilon/2) \cdot \min_{\bfx \in \R^d} \norm{\bfS_1 \bfA \bfx -\bfS_1 \bfz}_p$ then $\norm{\bfA \bar \bfx - \bfz}_{p} \le (1+\epsilon) \cdot \OPT$. We can verify that $\bar \bfx$ as computed in Line \ref{line:solve-small-prob} of Algorithm \ref{alg:clip2} does indeed satisfy this near optimality condition with probability $1-O(\delta)$ by Theorem \ref{thm:eps}.  Overall, with probability at least $1-O(\delta)$, we have that $\bar \bfx$ as computed in Line \ref{line:solve-small-prob} of Algorithm \ref{alg:clip2} satisfies $\norm{\bfA \bar \bfx - \bfz}_{p} \le (1+\epsilon) \cdot \OPT$. 


\end{proof}

\begin{proof}[\textbf{Proof of Lemma \ref{lem:analog}}]
For simplicity, we assume via scaling that $\norm{\bfz}_p \le 1$ and $\OPT = \Theta(1)$. By a standard argument (see, e.g., Lemma 2.5 of \cite{BourgainLindenstraussMilman:1989}), it suffices to show that with high probability, for all $\bfy' \in \mathcal{N}_\epsilon$, 
\begin{align}\label{eq:netError}
\left | \left \|\bfS_1\bfy' -\bfS_1 \bar \bfz \right \|_p^p - \norm{\bfy ' - \bar \bfz}_p^p \right | \le \epsilon,
\end{align}
 where $\mathcal{N}_\epsilon$ is an $\epsilon$-net in the $p$-norm over $\{\bfA\bfx: \norm{\bfA\bfx}_p \le 1\}$. By a standard volume argument (see, e.g., Lemma 2.4 of \cite{BourgainLindenstraussMilman:1989}), it is known that one can construct this net such that $\log |\mathcal{N}_\epsilon| = O(d\log(1/\epsilon))$.
 
 We prove \eqref{eq:netError} via a Bernstein inequality and union bound. We have $\E[\norm{\bfS _1\bfy' -\bfS_1 \bar \bfz}_p^p] = \norm{\bfy ' - \bar \bfz}_p^p = O(1)$. Additionally, by definition, $|\bar \bfz(i)|^p \le \frac{d^{\max(0,p/2-1)}\cdot \bfw_i^p(\bfA)}{\epsilon^p}$ for all $i$. Similarly, by Lemma \ref{lem:lwBound} which bounds the $\ell_p$ sensitivities by the Lewis weights, $ |\bfy '(i)|^p \le d^{\max(0,p/2-1)}\cdot \bfw_i^p(\bfA) $. Overall,
 \begin{align*}
 |\bfy'(i) - \bar \bfz(i)|^p = O \left ( \frac{d^{\max(0,p/2-1)}\cdot \bfw_i^p(\bfA)}{\epsilon^p}\right)
 \end{align*}
 By the construction of $\bfS_1$ (Definition \ref{def:lewis-weight-sampling}), this gives
 $$|[\bfS_1\bfy'-\bfS_1\bar \bfz](i)|^p \le \frac{d}{m \cdot \bfw_i^p(\bfA)} \cdot \frac{d^{\max(0,p/2-1)}\cdot \bfw_i^p(\bfA)}{\epsilon^p} =O \left (\frac{d^{\max(1,p/2)}}{m \cdot \epsilon^p} \right ),$$
 and so applying a Bernstein bound
 \begin{align*}
 \Pr \left [|\norm{\bfS _1\bfy' -\bfS_1\bar \bfz}_p^p - \norm{\bfy '-\bar \bfz}_p^p |\le \epsilon \right ] & \le 2 \exp \left (- \Omega \left ( \frac{\epsilon^2}{{d^{\max(0,p/2-1)}} / { m \cdot \epsilon^p}} \right ) \right ) \\
 &= 2 \exp \left (- \Omega \left ( \frac{\epsilon^{2+p} \cdot m}{d^{\max(0,p/2-1)}} \right ) \right ).
 \end{align*}
Setting
\[
	m = O \left (\frac{d^{\max(1,p/2)}}{\epsilon^{2+p}} \cdot \log \frac{|\mathcal{N}_\epsilon|}{\delta} \right ) \leq O \parens*{\frac{d^{\max(1,p/2)+1}}{\epsilon^{2+p}} \log\frac1{\eps\delta}}
\]
and applying a union bound gives that this bound holds with high probability for all $\bfy ' \in \mathcal{N}_\epsilon$. This completes the proof.
\end{proof}

\subsection{Optimal Dependence on \texorpdfstring{$\delta$}{delta} with Additive Error}\label{sec:optimal-delta-dep}

In this section, we show that with a knowledge of some overestimate $E$ of $\OPT$, we can obtain an optimal dependence on $\delta$ with an additive error of $\eps E$. Our approach is as follows. We first use Theorem \ref{thm:boost} to compute a constant factor approximation $\bfx_c$. Now with the knowledge of $E$ as well as $\bfx_c$, we can \emph{explicitly} clip the large entries of $\bfb$ as done in Lemma \ref{lem:clip}, where $\bfz$ can be replaced by $\bfz = \bfb - \bfA\bfx_c$ for a $\bfx_c$ which can be found with high probability by Theorem \ref{thm:boost}. In this case, we can explicitly approximate the cost of this clipped vector by sampling, which means we can collect $O(\log\frac1\delta)$ runs of a constant probability solution, and output the best solution among $O(\log\frac1\delta)$ many repetitions. This achieves an optimal dependence on $\delta$ up to constant factors due to a lower bound of \cite[Theorem 3.5]{ParulekarParulekarPrice:2021}, as this hard instance has an $\OPT$ that is fixed up to a $\Theta(1)$ factor, with high probability.

Note first that in Lemma \ref{lem:clip}, we may replace $\OPT$ by $E$, $\bfz = \bfb-\bfA\bfx^*$ by $\bfz = \bfb-\bfA\bfx_c$, and $\bar\bfz$ by the explicitly clipped $\bar\bfz$ using $E$ instead of $\OPT$, to obtain the following:
\begin{lemma}[Lemma \ref{lem:clip} -- Explicit Version]\label{lem:clip--explicit} Let $\bfA\in\mathbb R^{n\times d}$ and $\bfb\in\mathbb R^n$, and let $0 < p < \infty$. Let $\bfz = \bfb - \bfA\bfx_c\in \R^n$ satisfy $\norm*{\bfz}_p = O(\OPT)$. Let $E\geq\OPT$ and let
\[
    \mathcal B \coloneqq \braces*{i\in[n] : \frac{|\bfz(i)|^p}{E^p} \ge \frac{d^{\max(0,p/2-1)} \cdot \bfw_{i}^p(\bfA)}{\epsilon^p}}	
\]
Let $\bar \bfz \in \R^n$ be equal to $\bfz$ but with all entries in $\mathcal{B}$ set to $0$. Then for all $\bfx \in \R^d$ with $\norm{\bfA\bfx}_p = O(E)$,
$$\left | \norm{\bfA\bfx-\bfz}_p^p - \norm{\bfA\bfx-\bar \bfz}_p^p - \norm*{\bfz-\bar\bfz}_p^p \right | = O( \epsilon) \cdot E^p.$$
\end{lemma}

With this modification, the following version of Lemma \ref{lem:additive2} holds with the same proof:
\begin{lemma}[Lemma \ref{lem:additive2} -- Explicit Version]\label{lem:additive2-explicit} 
Consider the setting of Lemma \ref{lem:clip--explicit}. Let $\bfS \in \R^{m \times n}$ be formed as in Line \ref{line:split-sample} of Algorithm \ref{alg:clip}. Let
\[
    \Lambda \coloneqq \sup_{\norm*{\bfA\bfx}_p \leq O(E)}\abs*{\norm*{\bfS\bfA\bfx-\bfS\bar\bfz}_p^p - \norm*{\bfA\bfx-\bar\bfz}_p^p}.
\]
Then,
\[
    \Pr\braces*{\Lambda \geq \bracks*{C\frac1n \frac{d^{1\lor (p/2)}}{\eps^{0\lor(p-2)}}}^{1/2} \bracks*{(\log d)\sqrt{\log n} + z} E^p} \leq 2\exp(-z^2)
\]
\end{lemma}

We then obtain the following theorem, which has an improved dependence on $\delta$. 
\begin{theorem}
Let $E\geq \OPT = \min_\bfx\norm*{\bfA\bfx-\bfb}_p$ be known. There is an algorithm which, with probability at least $1-\delta$, reads only $M$ entries of $\bfb$ and outputs $\tilde\bfx$ satisfying
\[
    \norm*{\bfA\tilde\bfx-\bfb}_p \leq \OPT + \eps E
\]
where
\[
	m = \begin{dcases}
        O(1)\frac{d}{\eps^{2}}(\log d)^2\log(d/\eps) \log\frac1\delta & p\in(0,2) \\
        O(1)\frac{d^{p/2}}{\eps^{p}}(\log d)^2\log(d/\eps) \log\frac1\delta & p\in(2,\infty)
        \end{dcases}
\]
\end{theorem}
\begin{proof}
We first prepare $\bfz$ as required in Lemma \ref{lem:clip--explicit} using Theorem \ref{thm:boost}. This requires reading only $O(d(\log d)^3\log\frac1\delta)$ entries. Then by using our knowledge of $E$ and $\bfz$, we may replace $\bfb$ by the clipped residual $\bar\bfz$ at a cost of only an additive $\eps\cdot E$ error by Lemma \ref{lem:additive2-explicit}. Next, we sample $\ell = O(\log\frac1\delta)$ independent trials of $\bfS^{(i)}$ for $i\in[\ell]$ each with a constant probability success guarantee and solve $\min_\bfx\norm{\bfS^{(i)}\bfA\bfx-\bfS^{(i)}\bar\bfz}_p$ to obtain $\ell$ candidate solutions $\tilde\bfx^{(i)}$. Note that each $\bfS^{(i)}$ only requires reading 
\[
    O(1)\frac{d^{1\lor(p/2)}}{\eps^{2\lor p}}(\log d)^2\log(d/\eps)
\]
entries of $\bfb$, so we only require $O(m)$ entries altogether. Then by Chernoff bounds, with probability at least $1-\delta$, at least $9/10$ of these solutions satisfy
\[
    \norm*{\bfA\tilde\bfx^{(i)}-\bfb}_p \leq \OPT + \eps\cdot E
\]
and satisfy $\norm{\bfA\tilde\bfx^{(i)}}_p \leq O(E)$, and we may discard any solution which fails to have $\norm{\bfA\tilde\bfx^{(i)}}_p \leq O(E)$ by using our knowledge of $E$. Finally, we sample a single high probability sketch $\bfS$ which satisfies the guarantee of Lemma \ref{lem:additive2-explicit} with probability $1-\delta$. This single sketch requires only reading $O(m)$ entries of $\bfb$. Then, up to an $\eps E$ additive error we may approximately evaluate the cost of each of the $\ell$ solutions $\tilde\bfx$ such that $\norm{\bfA\tilde\bfx^{(i)}}_p \leq O(E)$ and select the best candidate using $\bfS$, by using $\norm*{\bfS\bfA\tilde\bfx^{(i)}-\bfS\bar\bfz}_p$ as the proxy cost. We conclude by adjusting $\delta$ and $\eps$ by constant factors.
\end{proof}

\section{Sensitivity Bounds for \texorpdfstring{$M$}{M}-Estimators}\label{sec:sensitivity-bounds}

In this section, we present our new sensitivity bounds. In Section \ref{subsec:m-norm-geometry}, we collect basic definitions and lemmas concerning $M$-estimators. Section \ref{subsec:m-sensitivity-sampling} develops basic notions for sensitivity sampling for $M$-estimators. In Section \ref{subsec:efficient-sensitivity}, we describe our efficient algorithm for computing sensitivities for a broad class of $M$-estimators. In Section \ref{subsec:sharper-sensitivity}, we show that a variation on our efficient algorithm can be used to show an existential bound of $O(d^{\max\{1,p_M/2\}}\log n)$ total sensitivity for the same class of $M$-estimators. Finally, in Section \ref{subsec:sensitivity-lower-bounds}, we show that the Tukey loss has a total sensitivity of $\Omega(d\log n)$, and that the Huber loss has a total sensitivity of $\Omega(d\log\log n)$. 

\subsection{Geometry of \texorpdfstring{$M$}{M}-norms}\label{subsec:m-norm-geometry}

In this section, we define $M$-norms and collect some of their geometric properties. This is a slight generalization of Section 4.1 of \cite{ClarksonWoodruff:2015a} which allows for a broader class of $M$-norms (namely with a relaxed polynomial lower bound condition). With applications to active regression in mind, we also slightly generalize the results to handle translations by a single vector $\bfb$, which can be taken to be $0$ to retrieve the original results. 

\begin{definition}\label{dfn:polynomially-bounded}
Let $M:\mathbb R_{\geq0}\to\mathbb R_{\geq0}$ be increasing. If there exist constants $p>0$ and $c_U\geq 1$ such that for all $y>x$, 
\[
    \frac{M(y)}{M(x)} \leq c_U \parens*{\frac{y}{x}}^p,
\]
then we say that $M$ is \emph{polynomially bounded above with degree $p$ and constant $c_U$}. Similarly, if there exists constants $q>0$ and $c_L\geq 1$ such that for all $y>x$,
\[
    \frac{M(y)}{M(x)} \geq c_L \parens*{\frac{y}{x}}^q,
\]
then we say that $M$ is \emph{polynomially bounded below with degree $q$ and constant $c_L$}.
\end{definition}

\begin{remark}
As noted in \cite{ClarksonWoodruff:2015a}, it can be shown that convex functions are polynomially bounded below with degree $1$. 
\end{remark}

\begin{remark}
Throughout this work, we will consider the constants $p, q, c_U, c_L$ in Definition \ref{dfn:polynomially-bounded} to be absolute constants that don't depend on other parameters under consideration. 
\end{remark}

We define the $M$-norm as follows. Note that despite our abuse of notation and terminology, the $M$-norm need not be an actual norm.

\begin{definition}[$M$-norm]\label{dfn:m-norm}
Let $M:\mathbb R_{\geq0}\to\mathbb R_{\geq0}$ be such that
\begin{itemize}
    \item $M(0) = 0$
    \item $M$ is nondecreasing
    \item $M$ is polynomially bounded above with degree $p_M$ and constant $c_U$ (see Definition \ref{dfn:polynomially-bounded})
\end{itemize}
Let $\bfw\in\mathbb R^n$ be a set of weights such that
\[
    \bfw_i \geq 1
\]
for all $i\in[n]$. Then, we define the $M$-norm of a vector $\bfx\in\mathbb R^n$ as
\[
    \norm*{\bfx}_{M,\bfw} \coloneqq \bracks*{\sum_{i=1}^n \bfw_i M(\abs*{\bfx_i})}^{1/p_M}.
\]
If $\bfw$ is the vector of all ones, we simple write $\norm*{\bfx}_{M}$ for $\norm*{\bfx}_{M,\bfw}$. If $M(x) = \abs*{x}^p$ for some $p>0$, then we write $\norm*{\bfx}_{p,\bfw}$ for $\norm*{\bfx}_{M,\bfw}$.
\end{definition}

\begin{definition}[$M$ balls and spheres]\label{dfn:m-ball-sphere}
    Let $\bfA\in\mathbb R^{n\times d}$ and let $\mathcal V = \Span(\bfA)$. Let $M:\mathbb R_{\geq0}\to\mathbb R_{\geq0}$ satisfy the conditions of Definition \ref{dfn:m-norm}, and let $\bfw\geq\mathbf{1}_n$ be a set of weights. Define the ball $\mathcal B^{M,\bfw}_\rho$ of radius $\rho>0$ as
    \[
        \mathcal B^{M,\bfw}_\rho \coloneqq \braces*{\bfy \in \mathcal V: \norm*{\bfy}_{M,\bfw} \leq \rho}.
    \]
    Similarly define the sphere $\mathcal S^{M,\bfw}_\rho$ of radius $\rho>0$ as
    \[
        \mathcal S^{M,\bfw}_\rho \coloneqq \braces*{\bfy \in \mathcal V: \norm*{\bfy}_{M,\bfw} = \rho}.
    \]
    If $\bfw = \mathbf{1}_n$, then we simply write $\mathcal B_\rho^{M}$ and $\mathcal S_\rho^{M}$, respectively.
    \end{definition}

Additional useful properties that we will need are included in Appendix \ref{sec:m-norm-appendix}, including sufficient conditions for triangle inequality and net constructions.

\subsection{Sensitivities for \texorpdfstring{$M$}{M}-Estimators}\label{subsec:m-sensitivity-sampling}

Because $M$-estimators are defined as coordinate-wise sums, one can naturally define analogues of sensitivities, just as was done for $\ell_p$ norms. 

\begin{definition}[$M$-sensitivity]\label{dfn:m-sensitivity}
Let $\bfA\in\mathbb R^{n\times d}$ and let $\norm*{\cdot}_M$ be an $M$-norm. Then, the \emph{$i$th $M$-sensitivity} is defined as
\[
    \bfs_i^M(\bfA) \coloneqq \sup_{\bfx\in\mathbb R^d, \bfA\bfx\neq 0} \frac{M(\abs*{[\bfA\bfx](i)})}{\norm*{\bfA\bfx}_M^{p_M}}
\]
and the \emph{total $M$-sensitivity} is defined as
\[
    \mathcal T^M(\bfA)\coloneqq \sum_{i=1}^n \bfs_i^M(\bfA).
\]
Let $\bfw\geq\mathbf{1}_n$ be a set of weights. Then, the \emph{$i$th weighted $M$-sensitivity} is defined as
\[
    \bfs_i^{M,\bfw}(\bfA) \coloneqq \sup_{\bfx\in\mathbb R^d, \bfA\bfx\neq 0} \frac{\bfw_i M(\abs*{[\bfA\bfx](i)})}{\norm*{\bfA\bfx}_{M,\bfw}^{p_M}}
\]
and the \emph{total weighted $M$-sensitivity} is defined as
\[
    \mathcal T^{M,\bfw}(\bfA)\coloneqq \sum_{i=1}^n \bfs_i^{M,\bfw}(\bfA).
\]
\end{definition}

When $M(x) = \abs*{x}^p$, i.e. for the case of $\ell_p$ norms, it is known that sampling with probabilities proportional to upper bounds on sensitivities yields subspace embeddings \cite{BourgainLindenstraussMilman:1989,DasguptaDrineasHarb:2009, CohenPeng:2015}. Analogous results are known as well for $M$-estimators \cite{ClarksonWoodruff:2015b,ClarksonWoodruff:2015a,ClarksonWangWoodruff:2019} and Orlicz norms \cite{SongWangYang:2019}.

\begin{definition}[Sensitivity Sampling for $M$-Estimators]\label{def:m-sensitivity-sampling}
Let $\bfA\in\mathbb R^{n\times d}$, let $\norm*{\cdot}_M$ be an $M$-norm, and let $\bfw\geq\mathbf{1}_n$ be a set of weights. Let $m$ be an oversampling parameter. Then, a random set of weights $\bfw'$ is sampled according to sensitivity upper bounds $\tilde\bfs_i^{M,\bfw}(\bfA) \geq \bfs_i^{M,\bfw}(\bfA)$ (see Definition \ref{dfn:m-sensitivity}) if 
\[
    \bfw_i' \coloneqq \begin{cases}
        \bfw_i/\bfp_i & \text{w.p. $\bfp_i$} \\
        0 & \text{otherwise}
    \end{cases}
\]
where $\bfp_i \coloneqq \min\{1,m\cdot \tilde\bfs_i^{M,\bfw}(\bfA)\}$.
\end{definition}

Note that in the case of $M$-estimators, the lack of scale invariance means that we get norm preservation guarantees for spheres rather than for entire subspaces. That is, we can get the following lemma, similar Lemma 43 of \cite{ClarksonWoodruff:2015a}:

\begin{lemma}\label{lem:m-sensitivity-sampling}
    Let $\bfA\in\mathbb R^{n\times d}$. Let $\eps\in(0,1)$, $\delta>0$, and let $\rho \geq 1$. Let $M:\mathbb R_{\geq0}\to\mathbb R_{\geq0}$ satisfy the conditions of Definition \ref{dfn:m-norm}, and furthermore that
    \begin{itemize}
        \item $M^{1/p_M}$ is subadditive
        \item $M$ is polynomially bounded below with degree $q_M$ and constant $c_L$ (see Definition \ref{dfn:polynomially-bounded})
    \end{itemize}
    Let $\bfw\geq\mathbf{1}_n$ be a set of weights. Let $\tilde\bfs_i^{M,\bfw}(\bfA)\geq\bfs_i^{M,\bfw}(\bfA)$ be sensitivity upper bounds. Let $m\geq m_0$ be an oversampling parameter larger than some
    \[
        m_0 = O\parens*{\frac{d}{\eps^2}\parens*{\log\frac1\eps}\parens*{\log\frac1\delta}}.
    \]
    Let $\bfw'\geq \mathbf{1}_n$ be sampled according to Definition \ref{def:m-sensitivity-sampling}. Then with probability at least $1 - \delta$, 
    \[
        \norm*{\bfy}_{M,\bfw'}^{p_M} = (1\pm\eps)\norm*{\bfy}_{M,\bfw}^{p_M}
    \]
    for all $\bfy\in\mathcal S_\rho^M$. Furthermore,
    \[
        \E\nnz(\bfw') \leq m\sum_{i=1}^n\tilde\bfs_i^{M,\bfw} = m\tilde{\mathcal T}^{M,\bfw}(\bfA)
    \]
\end{lemma}
\begin{proof}
    The proof is by a standard Bernstein bound, and is included in Appendix \ref{sec:m-norm-appendix}.
\end{proof}

\subsection{Efficient Sensitivity Bounds}\label{subsec:efficient-sensitivity}

We first show that algorithmically, one can compute upper bounds to the $M$-estimator sensitivities that sum to at most $O(d^{\max\{1,p_M/2\}}\log^2 n + \tau)$ in time
\[
    O\parens*{\nnz(\bfA)\log^3 n + \frac{nT}{\tau}\log n},
\]
where $T = T(n,d)$ is such that constant factor $\ell_{p_M}$ Lewis weight approximation for an $n\times d$ matrix $\bfB$ takes $O(\nnz(\bfB)\log n + T)$ time. For example, it is known that $\ell_p$ Lewis weights for $0 < p < 4$ can be approximated up to constant factors in $O(\nnz(\bfA)\log n + d^{\omega})$ time, so for $\tau = T = d^{\omega}$, we obtain a nearly input sparsity time algorithm that computes upper bounds to $M$-estimator sensitivities that sum to at most $O(d^{\max\{1,p_M/2\}}\log^2 n + d^{\omega})$. Note that in applications (see applications to active regression (Section \ref{sec:active-m}), Orlicz norm regression (Section \ref{sec:orlicz}), and robust subspace approximation (Section \ref{sec:robust-subspace-approx}) in this work), this is enough to compute a set of $\poly(d)\log^2 n$ rows that approximates the original matrix well, at which point we can compute sensitivities that sum to only $O(d\log^2 n)$ in an additional $\poly(d\log n)$ time. 

The algorithm draws ideas from a theorem of \cite[Theorem 3.4]{ClarksonWangWoodruff:2019}, which shows an input sparsity time algorithm for locating ``heavy entries'' for the Tukey loss, which is equivalent to finding coordinates with high Tukey sensitivity.

\begin{algorithm}
	\caption{Sensitivity upper bounds}
	\textbf{input:} Matrix $\bfA \in \mathbb R^{n \times d}$, $M$-norm $M$, parameter $1\leq\tau\leq n$. \\
	\textbf{output:} Upper bounds $\tilde\bfs_i^M(\bfA)$ on $\bfs_i^M(\bfA)$. 
	\begin{algorithmic}[1] 
		\State Initialize $\tilde\bfs_i^M(\bfA) \leftarrow 2\tau/n$.
        \For{$r\in[\ceil{\log_2(n/\tau)}]$}
            \For{$t\in[O(\log n)]$}
                \State Hash the rows of $\bfA$ into $B = 10\cdot 2^r$ buckets $S_1, S_2, \dots, S_B$. \label{line:hash-rows}
                \State Compute $O(1)$-approximate $\ell_p$ Lewis weights of each of the $B$ buckets $\bfA\mid_{S_1}, \bfA\mid_{S_2}, \dots, \bfA\mid_{S_B}$. \label{line:approx-lewis-weights}
                \State For any row $i$ with an $\ell_{p_M}$ Lewis weight of at least $\Omega(1)$, set $\tilde\bfs_i^M(\bfA) \leftarrow \max\{2/2^r, \tilde\bfs_i^M(\bfA)\}$. \label{line:assign-m-sensitivity}
            \EndFor
        \EndFor
	\end{algorithmic}\label{alg:m-estimator-alg-sensitivity}
\end{algorithm}

\begin{restatable}[Main Result for Sensitivity Bounds]{theorem}{SensitivityBound}\label{thm:m-estimator-alg}
Let $\norm*{\cdot}_M$ be an $M$-norm. Let $1\leq \tau\leq n$ be a parameter. Then, with probability at least $99/100$, Algorithm \ref{alg:m-estimator-alg-sensitivity} computes sensitivity upper bounds $\tilde\bfs_i^M(\bfA) \geq \bfs_i^M(\bfA)$ that sum to at most
\[
    \tilde{\mathcal T}^M(\bfA) \coloneqq \sum_{i=1}^n \tilde\bfs_i^M(\bfA) = O(d^{\max\{1,p_M/2\}} \log^2 n + \tau).
\]
If constant factor Lewis weight approximation takes for an $n\times d$ matrix $\bfB$ takes $O(\nnz(\bfB)\log n + T)$ time (see Theorem \ref{thm:cohen-peng-fast-lewis-weights}), then the total running time is
\[
    O\parens*{\nnz(\bfA)\log^3 n + \frac{nT}{\tau}\log n}.
\]
\end{restatable}
\begin{proof}
We first show correctness of the algorithm, then show the sensitivity bound, and finally the running time guarantee.

\paragraph{Correctness.}

Let $r\in[\ceil{\log_2 (n/\tau)}]$. Consider a coordinate $i\in[n]$ that has $M$-sensitivity between $1/2^r$ and $2/2^r$ and let $\bfy = \bfA\bfx$ be a corresponding vector which satisfies
\[
	\frac{M(|\bfy_i|)}{\norm*{\bfy}_M^{p_M}} \in \left[\frac1{2^r},\frac2{2^r}\right].
\]
Note then that there are at most $2^r-1$ entries $j$ of $\bfy$ such that $\abs*{\bfy_j} > \abs*{\bfy_i}$. In our algorithm (line \ref{line:hash-rows}), we randomly hash the $n$ rows of $\bfA$ into $B = 10\cdot 2^r$ buckets. Then, the probability that any one of these $2^r-1$ entries is hashed to the same bucket as $i$ is $1/B$, so by a union bound, the probability that $\bfy_i$ has the largest absolute value in its hash bucket is at least $1 - 2^r / B \geq 9/10$. Call this event $\mathcal E$. 

Now let $S$ be the set of indices which hash to the same bucket as row $i$ and let $S' = S\setminus\{i\}$. By Markov's inequality, with probability at least $9/10$, $i$ is hashed to a bucket such that the $M$-norm of all other entries is most $\norm*{\bfy\mid_{S'}}_M^{p_M} \leq \norm*{\bfy}_M^{p_M}/2^r$, where $\bfy\mid_{S'}$ is the restriction of $\bfy$ to the indices in $S'$. Call this event $\mathcal F$.

Condition on $\mathcal E$ and $\mathcal F$. We have by Lemma \ref{lem:max-sensitivity-comparison} that
\[
    \frac{|\bfy_i|^{p_M}}{\norm*{\bfy\mid_{S}}_{p_M}^{p_M}} \geq \frac{M(|\bfy_i|^{p_M})}{\norm*{\bfy\mid_{S}}_M^{p_M}} \geq \frac{M(|\bfy_i|^{p_M})}{M(|\bfy_i|^{p_M}) + \norm*{\bfy\mid_{S'}}_M^{p_M}} \geq \frac{1/2^r}{2/2^r + 1/2^r} = \frac13.
\]
The above holds with probability at least $4/5$. Thus, by repeating the hashing process $O(\log n)$ times, with probability at least $1 - 1/(100 n\log_2 n) = 1 - 1/\poly(n)$, there exists some trial where the $\ell_{p_M}$ sensitivity of the $i$th row \emph{in the matrix $\bfA\mid_S$} is at least $1/3$. In this trial, our algorithm will correctly set $\tilde\bfs_i^M(\bfA) \geq 2/2^r$ (line \ref{line:assign-m-sensitivity}). By a union bound over $O(\log (n/\tau))$ levels $r$ and the $n$ rows, our algorithm succeeds with probability at least $99/100$.

\paragraph{Sensitivity Bound.}

By Lemma \ref{lem:lewis-sensitivity-bound}, the $\ell_{p_M}$ sensitivities sum to at most $d^{\max\{1,{p_M}/2\}}$. Thus, each time we compute $O(1)$-approximate $\ell_{p_M}$ Lewis weights (line \ref{line:approx-lewis-weights}), we find at most $O(d^{\max\{1,{p_M}/2\}})$ entries with $M$-sensitivity at least $2/2^r$. Thus, for each $r$ and each iteration, we increase the sum of our upper bounds on $M$-sensitivities by a total of at most
\[
    B \cdot O(d^{\max\{1,{p_M}/2\}}) \cdot \frac2{2^r} = O(d^{\max\{1,{p_M}/2\}}). 
\]
This occurs at most $O((\log n)(\log(n/\tau)))$ times, and we start at a sensitivity bound of 
\[
    \frac{2\tau}{n}\cdot n = O(\tau)
\]
so our upper bounds on the sensitivities sum to at most
\[
    O(d^{\max\{1,{p_M}/2\}}(\log n)(\log(n/\tau)) + \tau) \leq O(d^{\max\{1,{p_M}/2\}}\log^2 n + \tau).
\] 

\paragraph{Running Time.}

For a given $r$, the dominating running time cost of the inner-most loop of Algorithm \ref{alg:m-estimator-alg-sensitivity} is the computation of $\ell_{p_M}$ Lewis weights for $O(2^r)$ matrices whose sparsities sum to $\nnz(\bfA)$. Thus, if Lewis weight computation for an $n\times d$ matrix $\bfB$ takes $O(\nnz(\bfB)\log n + T)$ time, then the total running time is
\[
    \sum_{r=1}^{\ceil{\log_2(n/\tau)}}O(\log n)\cdot O\parens*{\nnz(\bfA)\log n + 2^r T} = O\parens*{\nnz(\bfA)\log^3 n + \frac{nT}{\tau}\log n}.\qedhere
\]
\end{proof}

\begin{remark}
As noted by \cite{TukanMaaloufFeldman:2020}, if we can control the sensitivities of functions $M_1$ and $M_2$, then it is straightforward to control the sensitivities of the sum of these two functions, i.e., $M = M_1 + M_2$. This applies to our algorithm as well. Suppose that $\bfs_i^M(\bfA) \in [1/2^r, 2/2^r]$ and let $\bfy = \bfA\bfx$ be such that 
\[
    \frac{M_1(\abs*{\bfy_i}) + M_2(\abs*{\bfy_i})}{\norm*{\bfy}_{M_1}^{p_{M_1}} + \norm*{\bfy}_{M_2}^{p_{M_2}}} \in \bracks*{\frac1{2^r}, \frac2{2^r}}.
\]
Then,
\[
    \frac{M_1(\abs*{\bfy_i})}{\norm*{\bfy}_{M_1}^{p_{M_1}}} + \frac{M_2(\abs*{\bfy_i})}{\norm*{\bfy}_{M_2}^{p_{M_2}}} \geq \frac1{2^r}
\]
so then there is some $j\in\{1,2\}$ such that
\[
    \frac{M_j(\abs*{\bfy_i})}{\norm*{\bfy}_{M_j}^{p_{M_j}}} \geq \frac12\cdot \frac1{2^r},
\]
so the $\ell_{p_j}$ Lewis weight of the $i$th coordinate must be large by using a similar proof as Theorem \ref{thm:m-estimator-alg}. Thus, we can obtain similar sensitivity upper bounds up to a constant factor loss. Similarly, if $M_2$ is a ``flat'' sensitivity function in the sense of \cite{TukanMaaloufFeldman:2020}, that is, if 
\[
    \sup_{\bfx\in\mathbb R^d}\frac{M_2(\abs*{[\bfA\bfx](i)})}{\norm*{\bfA\bfx}_{M_2}^{p_{M_2}}} = O\parens*{\frac1n}
\]
for all $i\in[n]$, then this just means that either the $M_1$ sensitivity is large, or the $M$ sensitivity is at most $O(1/n)$, in which case we still get the same bounds. 
\end{remark}

Although Theorem \ref{thm:m-estimator-alg} only handles unweighted $M$-estimators, this result can be generalized to weighted $M$-estimators by splitting into level sets, similarly to Lemma 39 of \cite{ClarksonWoodruff:2015a}.

\begin{lemma}\label{lem:weighted-m-sensitivity-bound}
    Let $\norm*{\cdot}_M$ be an $M$-norm. Let $\bfw\geq\mathbf{1}_n$ be a set of weights. Let $N\coloneqq \ceil*{\log_2(1+\norm*{\bfw}_\infty)}$. For $j\in[N]$, let
\[
    T_j \coloneqq \braces*{i\in[n] : 2^{j-1}\leq \bfw_i < 2^j},
\]
and let $\bfA\mid_{T_j}$ denote the restriction of $\bfA$ to the rows of $T_j$. Then,
\[
    \bfs_i^{M,\bfw}(\bfA) \leq 2\cdot \bfs_i^M(\bfA\mid_{T_j})
\]
for $i\in T_j$. 
\end{lemma}
\begin{proof}
Let $i\in T_j$ for some $j\in[N]$. We have that
\begin{align*}
    \bfs_i^{M,\bfw}(\bfA) &= \sup_{\bfx\in\mathbb R^n, \bfA\bfx\neq 0} \frac{\bfw_i M([\bfA\bfx](i))}{\norm*{\bfA\bfx}_{M,\bfw}^{p_M}} \\
    &\leq \sup_{\bfx\in\mathbb R^n, \bfA\bfx\neq 0} \frac{\bfw_i M([\bfA\mid_{T_j}\bfx](i))}{\norm*{\bfA\mid_{T_j}\bfx}_{M,\bfw}^{p_M}} \\
    &\leq \sup_{\bfx\in\mathbb R^n, \bfA\bfx\neq 0} \frac{2^j M([\bfA\mid_{T_j}\bfx](i))}{2^{j-1}\norm*{\bfA\mid_{T_j}\bfx}_{M}^{p_M}} \\
    &= 2\cdot \bfs_i^{M}(\bfA\mid_{T_j})
\end{align*}
as desired.
\end{proof}

This leads to an algorithm that achieves guarantees similar to Theorem \ref{thm:m-estimator-alg} for weighted $M$-sensitivities, up to a loss of a factor of $N$ in the running time and sensitivity bound. 

\begin{corollary}\label{cor:weighted-m-sensitivity-alg}
    Let $\norm*{\cdot}_M$ be an $M$-norm. Let $\bfw\geq\mathbf{1}_n$ be a set of weights. Define $N$ as in Lemma \ref{lem:weighted-m-sensitivity-bound}. There is an algorithm that computes weighted $M$-estimator sensitivities that sum to at most
    \[
        O(Nd^{\max\{1,p_M/2\}}\log^2 n + N\tau)
    \]
    in time
    \[
        O\parens*{\nnz(\bfA)\log^3 n + N\frac{nT}{\tau}\log n},
    \]
    where $T$ is such that constant factor Lewis weight approximation for an $n\times d$ matrix $\bfB$ takes $O(\nnz(\bfB) + T)$ time (see Theorem \ref{thm:cohen-peng-fast-lewis-weights}). 
\end{corollary}
\begin{proof}
    This is simply the result of applying Theorem \ref{thm:m-estimator-alg} on the $N$ matrices $\bfA\mid_{T_j}$ as defined in Lemma \ref{lem:weighted-m-sensitivity-bound}. Note that the $\nnz(\bfA\mid_{T_j})$ terms add up to $\nnz(\bfA)$ in the running time. 
\end{proof}

\subsection{Sharper Sensitivity Bounds}\label{subsec:sharper-sensitivity}

We show that we may modify the proof of our input sparsity time algorithm to show that the sum of sensitivities is at most $O(d^{\max\{1,p_M/2\}}\log n)$, if we do not need to efficient algorithms for constructing these sensitivities.

\begin{theorem}
    Let $\norm*{\cdot}_M$ be an $M$-norm. Then, the total $M$-sensitivity of $\bfA$ is at most
\[
    O(d^{\max\{1,p_M/2\}}\log n).
\]
\end{theorem}
\begin{proof}
Our idea is essentially to run Algorithm \ref{alg:m-estimator-alg-sensitivity} with $\tau = d$ without the $O(\log n)$ repetitions of the hashing process. 

Let $r\in[\ceil{\log_2 n}]$ and let $I_r$ be the set of coordinates with $M$-sensitivity in $[1/2^r, 2/2^r]$. Suppose we hash the rows of $\bfA$ into $B = 10\cdot 2^r$ buckets. Then, as in the proof of Theorem \ref{thm:m-estimator-alg}, for each $i\in I_r$, there is at least a $9/10$ probability that $i$ has $\ell_{p_M}$ Lewis weight at least $1/3$ in its hash bucket. Thus, the number of such $i$ is $(9/10) \abs*{I_r}$ in expectation, so there exists some hashing such that at least $(9/10) \abs*{I_r}$ of the indices $i\in I_r$ have $\ell_p$ Lewis weight at least $1/3$ in its hash bucket. However, there can be at most $B\cdot d^{\max\{1,p/2\}}$ such indices, so we must have that
\[
    \frac{9}{10}\abs*{I_r} \leq B \cdot d^{\max\{1,p/2\}}
\]
so
\[
    \abs*{I_r} = O(B \cdot d^{\max\{1,p/2\}}) = O(2^r\cdot d^{\max\{1,p/2\}}).
\]
By summing over the $r$, we obtain a bound of
\[
    \sum_{r=1}^{\ceil{\log_2 n}} \frac2{2^r}\abs*{I_r} \leq \sum_{r=1}^{\ceil{\log_2 n}}\frac2{2^r}O(2^r \cdot d^{\max\{1,p/2\}}) = O(d^{\max\{1,p/2\}}\log n) = O(d^{\max\{1,p/2\}}\log n)
\]
on the total $M$-sensitivity, as claimed.
\end{proof}

\subsection{Sensitivity Lower Bounds}\label{subsec:sensitivity-lower-bounds}

Finally, we show that our sensitivity upper bounds are tight by showing that the Tukey loss can have a total sensitivity as large as $\Omega(d\log(n/d))$. We also show a weaker lower bound of $\Omega(d\log\log(n/d))$ for the Huber loss. This is in contrast to sensitivities for the $\ell_p$ loss for $0 < p < \infty$, which is always at most $d^{\max\{1,p/2\}}$ due to the existence of Lewis bases \cite{Lewis:1978, SchechtmanZvavitch:2001}, and thus has no dependence on $n$. The necessity for a dependence on $n$ can be attributed to the lack of scale invariance for these $M$-estimator losses. A similar observation has been made previously in \cite[Theorem 1.3]{SongWoodruffZhong:2019}, which shows that the column subset selection problem with the entrywise Huber loss exhibits a lower bound of $\Omega(\sqrt{\log n})$ columns, also attributed to the lack of scale invariance. 

We simultaneously handle the Tukey and Huber losses by analyzing the $\ell_2$-$\ell_p$ loss for $p\in[0,1]$, which grows quadratically near the origin and as $\ell_p$ away from the origin, and is polynomially bounded above with degree $2$. 

\begin{lemma}[Sensitivity Lower Bound for the $\ell_2$-$\ell_p$ Loss]
Define the $\ell_2$-$\ell_p$ loss of width $\tau$ to be
\[
    M(x) = \begin{cases}
        x^2 & \abs*{x} \leq \tau \\
        (\tau^2 / \tau^p) \cdot x^p & \abs*{x} > \tau
    \end{cases}.
\]
For $d \geq 1$ and $n\geq d$, there exists an $n\times d$ matrix $\bfA$ with total $M$-sensitivity that is at least
\[
    \mathcal T^T(\bfA) \geq \begin{cases}
        \Omega\parens*{d\log\frac{n}{d}} & \text{if $p \in [0, 1)$} \\
        \Omega\parens*{d\log\log\frac{n}{d}} & \text{if $p = 1$}.
    \end{cases}
\]
\end{lemma}
\begin{proof}
Let $\ell = \floor*{\log_2 n}$ and let $\bfx\in\mathbb R^n$ be a vector with $2^i$ coordinates of value $\tau/2^i$ for $i\in[\ell]$. We will show a sensitivity lower bound of $\Omega(\ell) = \Omega(\log n)$ for the $n\times 1$ matrix formed by the vector $\bfx$. By considering $d$ disjoint copies of this vector, each on $n/d$ coordinates, this implies a lower bound of $\Omega(d\log(n/d))$. 

Let $j\in[\ell]$. Then,
\begin{align*}
    \norm*{2^j\cdot \bfx}_M^2 &= \sum_{i=1}^\ell 2^i \cdot M\parens*{\tau\frac{2^j}{2^i}} \\
    &\leq \frac{\tau^2}{\tau^p}\sum_{i=1}^j 2^i\cdot\parens*{\tau\frac{2^j}{2^i}}^p + \sum_{i=j+1}^\ell 2^i\cdot  \parens*{\tau\frac{2^j}{2^i}}^2 \\
    &= \tau^2 2^{pj}\sum_{i=1}^j 2^{(1-p)i} + \tau^2 2^{2j}\sum_{i=j+1}^\ell \frac1{2^i} \\
    &= \begin{cases}
        O(\tau^2 \cdot2^j) & \text{if $p\in[0,1)$} \\
        O(\tau^2 \cdot j2^j) & \text{if $p = 1$} \\
    \end{cases}
\end{align*}
so for each $j\in[\ell]$, there are $2^j$ coordinates $i$ such that
\begin{align*}
    \frac{M(2^j\cdot\bfx_i)}{\norm*{2^j\cdot \bfx}_M^2} &= \begin{cases}\Omega\parens*{\frac{\tau^2}{\tau^2 \cdot 2^j}} & \text{if $p\in[0,1)$} \\
        O\parens*{\frac{\tau^2}{\tau^2 \cdot j2^j}} & \text{if $p = 1$} \\
    \end{cases} \\
    &= \begin{cases}\Omega\parens*{\frac{1}{2^j}} & \text{if $p\in[0,1)$} \\
        O\parens*{\frac{1}{j2^j}} & \text{if $p = 1$} \\
    \end{cases}.
\end{align*}
Thus, the sum of sensitivities for the Tukey loss for this matrix is at least
\[
    \sum_{j=1}^\ell 2^j \cdot \Omega\parens*{\frac1{2^j}} = \Omega\parens*{\ell} = \Omega\parens*{\log n}
\]
for $p\in[0,1)$ and
\[
    \sum_{j=1}^\ell 2^j \cdot \Omega\parens*{\frac1{j2^j}} = \Omega\parens*{\log\ell} = \Omega\parens*{\log\log n}
\]
for $p=1$.
\end{proof}
\section{Applications: Active Regression for \texorpdfstring{$M$}{M}-Estimators}\label{sec:active-m}

We now show that our algorithmic ideas for $\ell_p$ active regression may be applied to handle active regression for a general class of $M$-estimators, using our sensitivity results in Section \ref{sec:sensitivity-bounds}.

Throughout this section, we assume the following. Let $\bfA\in\mathbb R^{n\times d}$ and let $\mathcal V = \Span(\bfA)$. Let $M:\mathbb R_{\geq0}\to\mathbb R_{\geq0}$ satisfy the conditions of Definition \ref{dfn:m-norm}, and furthermore that
\begin{itemize}
    \item $M^{1/p_M}$ is subadditive
    \item $M$ is polynomially bounded below with degree $q_M$ and constant $c_L$ (see Definition \ref{dfn:polynomially-bounded})
\end{itemize}

In the linear regression problem with the $M$-loss, we are given a matrix $\bfA\in\mathbb R^{n\times d}$ and a vector $\bfb\in\mathbb R^n$ we must solve the following optimization problem:
\[
    \min_{\bfx\in\mathbb R^d}\norm*{\bfA\bfx - \bfb}_M.
\]
In the active setting, we wish to solve the above problem while querying as few entries of $\bfb$ as possible. We will also consider the weighted version, i.e.
\begin{equation}\label{eqn:m-linear-regression}
    \min_{\bfx\in\mathbb R^d}\norm*{\bfA\bfx - \bfb}_{M,\bfw}
\end{equation}
for a set of weights $\bfw\geq\mathbf{1}_n$. Throughout this section, let for a given weighted $M$-norm $\norm*{\cdot}_{M,\bfw}$, let
\[
    \OPT \coloneqq \min_\bfx\norm*{\bfA\bfx-\bfb}_{M,\bfw}.
\]

\subsection{Constant Factor Approximation}

We adapt Lemmas 7 and 8 of \cite{DasguptaDrineasHarb:2009} to first obtain a constant factor solution $\bfx_c$ to the regression problem. 
\begin{lemma}[Constant Factor Approximation]\label{lem:m-sensitivity-constant-factor-approx}
    Let $\bfw\geq\mathbf{1}_n$ be a set of weights. Let $\bfw'\geq\mathbf{1}_n$ be a random set of weights that satisfies the following:
    \begin{itemize}
        \item $\E_{\bfw'}\norm*{\bfy}_{M,\bfw'} = \norm*{\bfy}_{M,\bfw}$ for any fixed $\bfy\in\mathbb R^n$
        \item $\bfw'$ is a $1/8$-subspace embedding for any $\mathcal S_\rho^M$ for any fixed $\rho>0$ with probability at least $9/10$, that is,
        \[
            \Pr\braces*{\norm*{\bfA\bfx}_{M,\bfw'} = \parens*{1\pm\frac18}\norm*{\bfA\bfx}_{M,\bfw}, \forall \bfA\bfx\in\mathcal S_\rho^M} \geq \frac{9}{10}
        \]
    \end{itemize}
    Let $\tilde\bfx$ satisfy
    \[
        \norm*{\bfA\tilde\bfx - \bfb}_{M,\bfw'} \leq \kappa\min_{\bfx\in\mathbb R^d}\norm*{\bfA\bfx - \bfb}_{M,\bfw'}.
    \]
    Then, with constant probability,
    \[
        \norm*{\bfA\tilde\bfx - \bfb}_{M,\bfw} \leq O(\kappa)\min_{\bfx\in\mathbb R^d} \norm*{\bfA\bfx - \bfb}_{M,\bfw}.
    \]
\end{lemma}
\begin{proof}
    Let
    \[
        \OPT \coloneqq \min_{\bfx\in\mathbb R^d} \norm*{\bfA\bfx - \bfb}_{M,\bfw}
    \]
    and let $\bfx^*$ be the minimizer achieving this value. Then by assumption, we have that
    \[
        \E_{\bfw'} \norm*{\bfA\bfx^* - \bfb}_{M,\bfw'} = \OPT
    \]
    so by Markov's inequality, with probability at least $9/10$, $\norm*{\bfA\bfx^* - \bfb}_{M,\bfw'}\leq 10\cdot\OPT$. Condition on this event. Note then that
    \begin{equation}\label{eqn:m-regression-opt-value}
        \norm*{\bfA\tilde\bfx - \bfb}_{M,\bfw'} \leq \kappa\norm*{\bfA\bfx^* - \bfb}_{M,\bfw'} \leq 10\kappa\cdot\OPT.
    \end{equation}
    
    Now suppose for contradiction that
    \[
        \norm*{\bfA\tilde\bfx - \bfb}_{M,\bfw} > \frac{25\kappa}{c_L^{1/p_M}}\cdot\OPT.
    \]
    Then,
    \begin{align*}
        \norm*{\bfA\tilde\bfx - \bfA\bfx^*}_{M,\bfw} &\geq \norm*{\bfA\tilde\bfx - \bfb}_{M,\bfw} - \norm*{\bfA\bfx^* - \bfb}_{M,\bfw} && \text{triangle inequality} \\
        &> \frac{25\kappa}{c_L^{1/p_M}}\cdot\OPT - \OPT = \frac{24\kappa}{c_L^{1/p_M}}\cdot\OPT.
    \end{align*}
    We now apply the subspace embedding property, with $\rho = (24\kappa/c_L^{1/p_M})\cdot\OPT$ so that
    \[
        \norm*{\bfA\bfx}_{M,\bfw'} \geq \parens*{1\pm\frac18} \norm*{\bfA\bfx}_{M,\bfw}
    \]
    for all $\bfA\bfx$ with $\norm*{\bfA\bfx}_{M,\bfw} = (24\kappa/c_L^{1/p_M})\cdot\OPT$. We then apply Lemma \ref{lem:net-to-ball} to see that
    \[
        \norm*{\bfA\bfx}_{M,\bfw'} \geq c_L^{1/p_M} \frac78 \cdot \frac{24\kappa}{c_L^{1/p_M}}\cdot\OPT = 21\kappa\cdot\OPT
    \]
    for all $\bfA\bfx$ with $\norm*{\bfA\bfx}_{M,\bfw} \geq (24\kappa/c_L^{1/p_M})\cdot\OPT$. Thus,
    \begin{align*}
        \norm*{\bfA\tilde\bfx - \bfb}_{M,\bfw'} &\geq \norm*{\bfA\tilde\bfx - \bfA\bfx^*}_{M,\bfw'} - \norm*{\bfA\bfx^* - \bfb}_{M,\bfw'} && \text{triangle inequality} \\
        &> 21\kappa\cdot\OPT - 10\kappa\cdot\OPT  = 11\kappa\cdot\OPT.
    \end{align*}
    This contradicts Equation \eqref{eqn:m-regression-opt-value}, so we conclude that
    \[
        \norm*{\bfA\tilde\bfx - \bfb}_{M,\bfw} \leq \frac{25\kappa}{c_L^{1/p_M}}\cdot\OPT = O(\kappa)\OPT.\qedhere
    \]
\end{proof}

\subsection{Relative Error Approximation}

After an initial constant factor approximation, we show that sensitivity sampling on the residual yields a relative error approximation. 

We first reduce to considering only $\bfA\bfx$ with $\norm*{\bfA\bfx}_{M,\bfw} = O(\OPT)$. 

\begin{lemma}\label{lem:opt-norm-bound}
    Let $\bfA\in\mathbb R^{n\times d}$ and let $\bfb\in\R^n$. Let $\bfw\geq\mathbf{1}_n$ and let $\norm*{\cdot}_M$ be an $M$-norm that satisfies an approximate triangle inequality, that is,
    \[
        M^{1/p_M}(a+b) \leq O(1)(M^{1/p_M}(a) + M^{1/p_M}(b)).
    \]
    Suppose that $\bfx_c\in\R^d$ satisfies
    \[
        \norm*{\bfA\bfx_c - \bfb}_{M,\bfw} \leq O(1)\min_\bfx\norm*{\bfA\bfx - \bfb}_{M,\bfw} = O(\OPT)
    \]
    and let $\bfb' = \bfb - \bfA\bfx_c$. Let $\bar\bfx$ satisfy
    \[
        \norm*{\bfA\bar\bfx-\bfb'}_{M,\bfw} \leq O(1)\min_\bfx\norm*{\bfA\bfx-\bfb'}_{M,\bfw} = O(\OPT).
    \]
    Then, $\norm*{\bfA\bar\bfx}_{M,\bfw} = O(\OPT)$.
\end{lemma}
\begin{proof}
    We have by approximate triangle inequality that
    \[
        \norm*{\bfA\bar\bfx}_{M,\bfw} \leq O(1)(\norm*{\bfA\bar\bfx-\bfb'}_{M,\bfw} + \norm*{\bfb'}_{M,\bfw}) \leq O(1)(O(\OPT) + O(\OPT)) = O(\OPT).
    \]
\end{proof}

Next, we show that we only need to preserve the cost on a certain subset of coordinates where the entries of $\bfb'$ are relatively small.

\begin{lemma}\label{lem:reduction-to-good-b-coordinates}
    Consider the setting of Lemma \ref{lem:opt-norm-bound}. Further suppose that either
    \begin{enumerate}[label={(\arabic*)}]
    \item for any $a,b\in\R$ with $M(a)\leq\eps^{p_M}M(b)$,
    \[
        M(a+b) = (1\pm O(\eps))M(b)
    \]
    \item $M^{1/p_M}$ is subadditive
    \end{enumerate}
    Let $\tilde\bfs_i^{M,\bfw}(\bfA) \geq \bfs_i^{M,\bfw}(\bfA)$ be upper bounds on the weighted $M$-sensitivities of $\bfA$. Let 
    \[
        \mathcal B \coloneqq \braces*{i\in[n] : \bfw_i M(\bfb_i') > \frac{\tilde\bfs_i^{M,\bfw}(\bfA)\OPT^{p_M}}{\eps^{p_M}}}
    \]
    Let $\bfw'$ be a set of random weights such that for any fixed $\bfy\in\R^n$, $\E_{\bfw'}\norm*{\bfy}_{M,\bfw'}^{p_M} = \norm*{\bfy}_{M,\bfw}^{p_M}$, and such that with probability at least $1-\delta$ satisfies
    \begin{equation}\label{eq:good-coordinates-guarantee}
        \norm*{(\bfA\bfx-\bfb')\mid_{\overline{\mathcal B}}}_{M,\bfw'}^{p_M} = \norm*{(\bfA\bfx-\bfb')\mid_{\overline{\mathcal B}}}_{M,\bfw}^{p_M} \pm O(\eps)\OPT^{p_M}
    \end{equation}
    for every $\bfA\bfx$ with $\norm*{\bfA\bfx}_{M,\bfw} = O(\OPT)$. Then with probability at least $1-\delta$, there is a $C$ with $\abs*{C} \leq O(\OPT^{p_M}/\delta)$ such that
    \begin{align*}
        \abs*{\norm*{\bfy - \bfb'}_{M,\bfw'}^{p_M} - \norm*{\bfy - \bfb'}_{M,\bfw'}^{p_M} - C} = O(\eps/\delta)\OPT^{p_M}.
    \end{align*}
\end{lemma}
\begin{proof}
    First note that assumption (2) implies assumption (1), since that implies that
    \[
        M^{1/p_M}(a+b) = M^{1/p_M}(b)\pm M^{1/p_M}(a) = M^{1/p_M}(b) \pm \eps M^{1/p_M}(b)
    \]
    and raising both sides to the $p_M$th power gives assumption (1). 

    Now note that for any $i\in \mathcal B$ and $\bfy = \bfA\bfx$ with $\norm*{\bfA\bfx}_{M,\bfw} = O(\OPT)$, we have that
    \[
        \bfw_i M([\bfA\bfx](i)) \leq \tilde\bfs_i^{M,\bfw}(\bfA)\cdot\norm*{\bfA\bfx}_{M,\bfw}^{p_M} = O(1)\tilde\bfs_i^{M,\bfw}(\bfA)\cdot\OPT^{p_M} \leq O(\eps^{p_M})\bfw_i M(\bfb'(i)).
    \]
    It follows by assumption (1) that
    \[
        M(\abs*{\bfy_i - \bfb'_i}) = (1\pm O(\eps))M(\abs*{\bfb'_i})
    \]
    Then, we have that
    \begin{align*}
        \norm*{\bfy - \bfb'}_{M,\bfw}^{p_M} &= \norm*{(\bfy - \bfb')\mid_{\overline{\mathcal B}}}_{M,\bfw}^{p_M} + \norm*{(\bfy - \bfb')\mid_{\mathcal B}}_{M,\bfw}^{p_M} \\
        &= \norm*{(\bfy - \bfb')\mid_{\overline{\mathcal B}}}_{M,\bfw}^{p_M} + \norm*{\bfb'\mid_{\mathcal B}}_{M,\bfw}^{p_M} \pm O(\eps)\norm*{\bfb'\mid_{\mathcal B}}_{M,\bfw}^{p_M}
    \end{align*}
    and similarly,
    \begin{align*}
        \norm*{\bfy - \bfb'}_{M,\bfw'}^{p_M} &= \norm*{(\bfy - \bfb')\mid_{\overline{\mathcal B}}}_{M,\bfw'}^{p_M} + \norm*{(\bfy - \bfb')\mid_{\mathcal B}}_{M,\bfw'}^{p_M} \\
        &= \norm*{(\bfy - \bfb')\mid_{\overline{\mathcal B}}}_{M,\bfw'}^{p_M} + \norm*{\bfb'\mid_{\mathcal B}}_{M,\bfw'}^{p_M} \pm O(\eps)\norm*{\bfb'\mid_{\mathcal B}}_{M,\bfw'}^{p_M}
    \end{align*}
    Now using that $\norm*{\bfb'}_{M,\bfw'}^{p_M} = O(1/\delta)\norm*{\bfb'}_{M,\bfw}^{p_M}$ with probability at least $1-\delta$ by Markov's inequality, we set $C = \norm*{\bfb'\mid_{\mathcal B}}_{M,\bfw}^{p_M} - \norm*{\bfb'\mid_{\mathcal B}}_{M,\bfw'}^{p_M}$ to see that
    \begin{align*}
        \abs*{\norm*{\bfy - \bfb'}_{M,\bfw'}^{p_M} - \norm*{\bfy - \bfb'}_{M,\bfw'}^{p_M} - C} &= \abs*{\norm*{(\bfy - \bfb')\mid_{\overline{\mathcal B}}}_{M,\bfw}^{p_M} - \norm*{(\bfy - \bfb')\mid_{\overline{\mathcal B}}}_{M,\bfw'}^{p_M}} \pm O(\eps/\delta)\norm*{\bfb'\mid_{\mathcal B}}_{M,\bfw}^{p_M} \\
        &\leq O(\eps/\delta)\norm*{\bfb'\mid_{\mathcal B}}_{M,\bfw}^{p_M} = O(\eps/\delta)\OPT^{p_M}
    \end{align*}
    as desired.
\end{proof}

We next show how to guarantee Equation \eqref{eq:good-coordinates-guarantee} using Bernstein's inequality. The proof closely follows that of Lemma \ref{lem:m-sensitivity-sampling}, but needs a slight modification to handle the coordinates of $\bfb'$ in $\overline{\mathcal B}$.

\begin{lemma}\label{lem:m-bernstein-good-coordinates}
    Consider the setting of Lemma \ref{lem:reduction-to-good-b-coordinates}. Let $\bfA\bfx$ have $\norm*{\bfA\bfx}_{M,\bfw} = O(\OPT)$ and let $\bfy = \bfA\bfx- \bfb'$. Let $m\geq 1$ be a parameter and let $\bfw'$ be obtained from weights $\bfw$ according to Definition \ref{def:m-sensitivity-sampling}. Then,
    \[
        \Pr\braces*{\abs*{\norm*{(\bfA\bfx-\bfb')\mid_{\overline{\mathcal B}}}_{M,\bfw'}^{p_M} - \norm*{(\bfA\bfx-\bfb')\mid_{\overline{\mathcal B}}}_{M,\bfw}^{p_M}} \geq \eps\OPT^{p_M}} \leq \exp\parens*{-\Theta(1)m\eps^{2+p_M}}
    \]
\end{lemma}
\begin{proof}
    Let $G\subseteq \overline{\mathcal B}$ be the subset of coordinates of $\overline{\mathcal B}$ such that $\bfp_i < 1$. Define the random variable
    \[
        W_i \coloneqq \bfw_i' M(\bfy(i))
    \]
    for each $i\in G$. Then,
    \[
        \E\bracks*{\sum_{i\in G} W_i} = \sum_{i\in G} \frac{\bfw_i}{\bfp_i}M(\bfy(i))\cdot \bfp_i = \norm*{\bfy\mid_G}_{M,\bfw}^{p_M}.
    \]
    Note that
    \begin{align*}
        \frac{\bfw_i}{\bfp_i}M(\bfy(i)) &\leq O(1)\frac{1}{m\cdot \tilde\bfs_i^{M,\bfw}(\bfA)} (\bfw_i M([\bfA\bfx](i)) + \bfw_i M(\bfb'(i))) \\
        &\leq O(1)\frac{1}{m\cdot \tilde\bfs_i^{M,\bfw}(\bfA)} \parens*{\tilde\bfs_i^{M,\bfw}(\bfA)\norm*{\bfA\bfx}_{M,\bfw}^{p_M} + \frac{\tilde\bfs_i^{M,\bfw}\OPT^{p_M}}{\eps^{p_M}}} \\
        &\leq O(1)\frac{1}{m\cdot \tilde\bfs_i^{M,\bfw}(\bfA)} \parens*{\tilde\bfs_i^{M,\bfw}(\bfA)\OPT^{p_M} + \frac{\tilde\bfs_i^{M,\bfw}\OPT^{p_M}}{\eps^{p_M}}} \\
        &\leq O(1)\frac{\OPT^{p_M}}{m\cdot\eps^{p_M}} \\
    \end{align*}
    We next bound the variance:
    \begin{align*}
        \Var\bracks*{\sum_{i\in G} W_i} &= \sum_{i\in G} \Var[W_i] \leq \sum_{i\in G} \frac{\bfw_i^2}{\bfp_i^2}M(\bfy(i))^2\cdot \bfp_i = \sum_{i\in G} \frac{\bfw_i}{\bfp_i}M(\bfy(i))\cdot \bfw_i M(\bfy(i)) \\
        &\leq O(1)\frac{\OPT^{p_M}}{m\cdot\eps^{p_M}} \sum_{i\in G} \bfw_i M(\bfy(i)) = O(1)\frac{\OPT^{2p_M}}{m\cdot\eps^{p_M}}
    \end{align*}
    Then by Bernstein's inequality,
    \begin{align*}
        \Pr\braces*{\abs*{\sum_{i\in G} W_i - \norm*{\bfy\mid_G}_{M,\bfw}^{p_M}} > t} &\leq 2\exp\parens*{-\Theta(1)\frac{t^2}{\frac{1}{m\eps^{p_M}}\OPT^{2p_M} + \frac{1}{m\eps^{p_M}} \OPT^{p_M} t}} \\
        &= 2\exp\parens*{-\Theta(1)\frac{mt^2 \eps^{p_M}}{\OPT^{p_M}(\OPT^{p_M} + t)}}.
    \end{align*}
    For $t = \eps\OPT^{p_M}$, this gives a bound of
    \[
        2\exp\parens*{-\Theta(1)\frac{m(\eps\OPT^{p_M})^2\eps^{p_M}}{\OPT^{2p_M}}} =  2\exp\parens*{-\Theta(1)m\eps^{2+p_M}}.
    \]
\end{proof}

\begin{lemma}\label{lem:m-active-regression-relative-error-approximation}
Consider the setting of Lemma \ref{lem:m-bernstein-good-coordinates}. Recall that $\bfx_c$ satisfies $\norm*{\bfA\bfx_c-\bfb}_{M,\bfw} = O(\OPT)$ and $\bfb' = \bfb-\bfA\bfx_c$. Let
\[
    m = O\parens*{\frac{d}{\eps^{2+p_M}}\parens*{\log\frac1\eps}\parens*{\log\frac1\delta}}
\]
and let $\bfw'$ be as defined in Lemma \ref{lem:m-bernstein-good-coordinates} with the above choice of $m$. Let $\bar\bfx$ satisfy $\norm*{\bfA\bar\bfx - \bfb'}_{M,\bfw'}^{p_M} \leq (1+\eps)\norm*{\bfA\bfx - \bfb'}_{M,\bfw'}^{p_M}$. Then, with probability at least $1-\delta$,
\[
    \norm*{\bfA(\bfx_c + \bar\bfx) - \bfb}_{M,\bfw} \leq (1+O(\eps/\delta))\min_{\bfx\in\mathbb R^d}\norm*{\bfA\bfx - \bfb}_{M,\bfw}.
\]
\end{lemma}

\begin{proof}
By Lemma \ref{lem:opt-norm-bound}, we have that $\norm*{\bfA\tilde\bfx}_{M,\bfw} \leq O(\OPT)$. We may thus restrict our attention to the ball $\mathcal B_{\rho}^{M,\bfw} \subseteq\mathcal V$ of radius $\rho = O(\OPT)$, since
\[
    \min_{\bfx\in\mathbb R^d} \norm*{\bfA\bfx - \bfb'}_{M,\bfw'} = \min_{\bfy\in\mathcal B_{\rho}^{M,\bfw}} \norm*{\bfy - \bfb'}_{M,\bfw'}
\]
by the above.

Next, we show Equation \eqref{eq:good-coordinates-guarantee} for every $\bfA\bfx\in \mathcal B_{\rho}^{M,\bfw}$. By Lemma \ref{lem:m-norm-net-to-sphere}, it suffices to show the approximation guarantee for every $\bfy\in\mathcal N$ and $\mathcal N-\bfb'$, for an $O(\eps^{p_M/q_m})\rho$-cover $\mathcal N$ of $\mathcal B_\rho^{M,\bfw}$. Let $\mathcal N$ be such a net over $\mathcal B_\rho^{M,\bfw}$, which has size at most $\abs*{\mathcal N} \leq O(d\log\frac1\eps)$ by Lemma \ref{lem:m-norm-net-size}. We may then use Lemma \ref{lem:m-bernstein-good-coordinates} and a union bound over the net $\mathcal N$ to get Equation \eqref{eq:good-coordinates-guarantee} for every $\bfA\bfx\in \mathcal B_{\rho}^{M,\bfw}$. 

Finally, note that we now have the conclusion of Lemma \ref{lem:reduction-to-good-b-coordinates}, since we have satisfied its assumption. Then,
\begin{align*}
    \norm{\bfA \bar \bfx - \bfb'}_{M,\bfw}^{p_M} &\le \norm{\bfA \bar \bfx - \bfb'}_{M,\bfw'}^{p_M} - C + O(\epsilon/\delta) \cdot \OPT^{p_M}\nonumber\\
    &\le (1+\epsilon) \cdot \min_\bfx\norm{\bfA \bfx - \bfb'}_{M,\bfw'}^{p_M} - C + O(\epsilon/\delta) \cdot \OPT^p\nonumber\\
    &\le (1+\epsilon) \min_\bfx(\norm{\bfA \bfx - \bfb'}_{M,\bfw}^{p_M} + C) - C + O(\epsilon/\delta) \cdot \OPT^p \nonumber \\
    &\le \min_\bfx\norm{\bfA \bfx - \bfb'}_{M,\bfw}^{p_M} + O(\epsilon/\delta) \cdot \OPT^p
\end{align*}
\end{proof}

\subsection{Nearly Input Sparsity Time Algorithm}

We now combine the active regression algorithms obtained in this section with the sensitivity bounds of Section \ref{sec:sensitivity-bounds}. 

\begin{algorithm}
	\caption{Input sparsity time, constant factor $M$-estimator active regression}
	\textbf{input:} Matrix $\bfA \in \R^{n \times d}$, measurement vector $\bfb\in\R^n$. \\
	\textbf{output:} Constant factor approximate solution $\bfx_c\in\R^d$ to $\min_\bfx\norm*{\bfA\bfx-\bfb}_M$.
	\begin{algorithmic}[1] 
        \State Run Algorithm \ref{thm:m-estimator-alg} to obtain approximate sensitvities $\tilde\bfs_i^M(\bfA)$ with $\tau = T$\footnote{See Theorem \ref{thm:m-active-regression} for the definition of this setting.}
        \State Given weights $\mathbf{1}_n$, obtain $\bfw$ according to Definition \ref{def:m-sensitivity-sampling}
        \State Run Algorithm \ref{thm:m-estimator-alg} to obtain approximate sensitvities $\tilde\bfs_i^{M,\bfw}(\bfA)$ with $\tau = d$
        \State Given weights $\bfw$, obtain $\bfw'$ according to Definition \ref{def:m-sensitivity-sampling}
        \State Let $\bfx_c$ be an approximate solution $\norm*{\bfA\bfx_c-\bfb}_{M,\bfw'} \leq O(1)\min_\bfx\norm*{\bfA\bfx-\bfb}_{M,\bfw'}$
        \State \Return $\bfx_c$
	\end{algorithmic}\label{alg:constant-m-active-regression}
\end{algorithm}

\begin{algorithm}
	\caption{Input sparsity time, relative error $M$-estimator active regression}
	\textbf{input:} Matrix $\bfA \in \R^{n \times d}$, measurement vector $\bfb\in\R^n$. \\
	\textbf{output:} Approximate solution $\tilde\bfx\in\R^d$ to $\min_\bfx\norm*{\bfA\bfx-\bfb}_M$.
	\begin{algorithmic}[1] 
        \State Run Algorithm \ref{alg:constant-m-active-regression} to obtain a constant factor solution $\bfx_c$
        \State Set $\bfb' \gets \bfb - \bfA\bfx_c$
        \State Run Algorithm \ref{thm:m-estimator-alg} to obtain approximate sensitvities $\tilde\bfs_i^M(\bfA)$ with $\tau = T$
        \State Given weights $\mathbf{1}_n$, obtain $\bfv$ according to Definition \ref{def:m-sensitivity-sampling}
        \State Run Algorithm \ref{thm:m-estimator-alg} to obtain approximate sensitvities $\tilde\bfs_i^{M,\bfv}(\bfA)$ with $\tau = d$
        \State Given weights $\bfv$, obtain $\bfv'$ according to Definition \ref{def:m-sensitivity-sampling}
        \State Let $\bar\bfx$ be an approximate solution $\norm*{\bfA\bar\bfx-\bfb}_{M,\bfv'} \leq (1+\eps)\min_\bfx\norm*{\bfA\bfx-\bfb}_{M,\bfv'}$
        \State Let $\tilde\bfx \gets \bfx_c + \bar\bfx$
        \State \Return $\tilde\bfx$
	\end{algorithmic}\label{alg:relative-m-active-regression}
\end{algorithm}

\begin{theorem}\label{thm:m-active-regression}
    There is an algorithm, Algorithm \ref{alg:relative-m-active-regression}, which with probability at least $99/100$ computes $\tilde\bfx$ such that
    \[
        \norm*{\bfA\tilde\bfx - \bfb}_M \leq (1+\eps)\min_{\bfx\in\mathbb R^d}\norm*{\bfA\bfx - \bfb}_M
    \]
    in time
    \[
        O\parens*{\nnz(\bfA)\log^3 n + \frac{(Td^{\max\{1,p_M/2\}}+T^2)\log^4 n}{\eps^{2+p_M}}\log\frac1\eps}
    \]
    where constant factor $\ell_{p_M}$ Lewis weight approximation for an $n\times d$ matrix $\bfB$ takes time $O(\nnz(\bfB) + T))$ time. Furthermore, the algorithm only queries
    \[
        O\parens*{\frac{d^{\max\{2,p_M/2+1\}}\log^3 n}{\eps^{2+p_M}}\log\frac1\eps}
    \]
    entries of $\bfb$. 
\end{theorem}
\begin{proof}
In order to obtain nearly input sparsity time algorithms, we will need to apply the sensitivity sampling algorithms of Theorem \ref{thm:m-estimator-alg} twice, once to reduce the number of rows to $O(\poly(d,\log n))$ rows in nearly input sparsity time, and once to reduce the number of rows all the way down to $O(d^{\max\{2,p/2+1\}}\poly\log(n))$ rows in an additional $O(\poly(d,\log n))$ time.

\paragraph{Constant Factor Approximation.} We will first obtain a constant factor solution $\bfx_c$ such that
\[
    \norm*{\bfA\bfx_c - \bfb}_M \leq O(1)\min_{\bfx\in\mathbb R^d}\norm*{\bfA\bfx - \bfb}_M.
\]
To do this, we first obtain sensitivity upper bounds $\tilde\bfs_i^M(\bfA)$ by Theorem \ref{thm:m-estimator-alg} with $\tau = T$, so that they sum to
\[
    \tilde{\mathcal T}^M(\bfA) = \sum_{i=1}^n \tilde\bfs_i^M(\bfA) = O\parens*{d^{\max\{1,p_M/2\}}\log^2 n + T}
\]
and can be computed in time
\[
    O\parens*{\nnz(\bfA)\log^2 n + n\log^2 n} = O\parens*{\nnz(\bfA)\log^2 n}.
\]
By Lemma \ref{lem:m-sensitivity-constant-factor-approx}, we can perform sensitivity sampling with these sensitivity estimates to obtain a set of weights $\bfw$ such that if $\bfx_c$ satisfies
\[
    \norm*{\bfA\bfx_c - \bfb}_{M,\bfw} \leq O(1)\min_{\bfx\in\mathbb R^d}\norm*{\bfA\bfx - \bfb}_{M,\bfw},
\]
then
\[
    \norm*{\bfA\bfx_c - \bfb}_{M} \leq O(1)\min_{\bfx\in\mathbb R^d}\norm*{\bfA\bfx - \bfb}_{M}. 
\]
Next, we obtain \emph{weighted} sensivitity upper bounds $\tilde\bfs_i^{M,\bfw}(\bfA)$ by Corollary \ref{cor:weighted-m-sensitivity-alg} with $\tau = d$. Note that $\norm*{\bfw}_\infty = O(n)$ since the sensitivity upper bounds are all at least $1/n$, so the weighted sensitivities sum to at most
\[
    \sum_{i=1}^n \tilde\bfs_i^{M,\bfw}(\bfA) = O(d^{\max\{1,p_M/2\}}\log^3 n). 
\]
Furthermore, the weights $\bfw$ only have
\[
    \nnz(\bfw) = O\parens*{d\cdot \tilde{\mathcal T}^M(\bfA)} = O(d^{\max\{2,p_M/2+1\}}\log^2 n + dT)
\]
nonzero entries with constant probability by Markov's inequality, so the running time is
\begin{align*}
    &O\parens*{\nnz(\bfA)\log^3(\nnz(\bfw)) + \log n\frac{\nnz(\bfw)T}{d}\log (\nnz(\bfw))} \\
    \leq~&O\parens*{\nnz(\bfA)\log^3 n + (Td^{\max\{1,p_M/2\}}+T^2)\log^4 n}.
\end{align*}
By Lemma \ref{lem:m-sensitivity-constant-factor-approx}, we can perform sensitivity sampling with these sensitivity estimates to obtain a set of weights $\bfw'$. We can now define our $\bfx_c$ to be the minimizer of
\[
    \min_{\bfx\in\mathbb R^d}\norm*{\bfA\bfx - \bfb}_{M,\bfw'},
\]
which means that
\[
    \norm*{\bfA\bfx_c - \bfb}_{M,\bfw} \leq O(1)\min_{\bfx\in\mathbb R^d}\norm*{\bfA\bfx - \bfb}_{M,\bfw},
\]
which in turn implies that
\[
    \norm*{\bfA\bfx_c - \bfb}_{M} \leq O(1)\min_{\bfx\in\mathbb R^d}\norm*{\bfA\bfx - \bfb}_{M},
\]
as desired. Computing $\bfx_c$ only required $\nnz(\bfw')$ entries of $\bfb$ to sample, which is only $O(d^{\max\{2,p_M/2+1\}}\log^3 n)$ entries. 

\paragraph{Relative Error Approximation.} With the constant factor solution $\bfx_c$ in hand, we can now apply Lemma \ref{lem:m-active-regression-relative-error-approximation} to refine this to a relative error approximation. Following Lemma \ref{lem:m-active-regression-relative-error-approximation}, let
\[
    \bfb' = \bfb - \bfA\bfx_c
\]

Recall the sensitivity upper bounds $\tilde\bfs_i^{M}(\bfA)$ constructed earlier. This time, we obtain a different set of weights $\bfv\geq\mathbf{1}_n$ by following Lemma \ref{lem:m-active-regression-relative-error-approximation}, so that the expected number of rows sampled is now
\begin{align*}
    \E[\nnz(\bfv)] \leq O\parens*{\frac{d}{\eps^{2+p_M}}\parens*{\log\frac1\eps}\cdot \tilde{\mathcal T}^M(\bfA)} &= O\parens*{\frac{d(d^{\max\{1,p_M/2\}}\log^2 n + T)}{\eps^{2+p_M}}\parens*{\log\frac1\eps}} \\
    &= O\parens*{\frac{d^{\max\{2,p_M/2+1\}}\log^2 n + dT}{\eps^{2+p_M}}\parens*{\log\frac1\eps}}.
\end{align*}
By a Markov bound, with constant probability, the number of rows sampled has the same bound, up to constant factors. We are then guaranteed by Lemma \ref{lem:m-active-regression-relative-error-approximation} that if $\bar\bfx$ satisfies
\[
    \norm*{\bfA\bar\bfx - \bfb'}_{M,\bfv} \leq (1+\eps)\min_{\bfx\in\mathbb R^d}\norm*{\bfA\bfx - \bfb'}_{M,\bfv},
\]
then
\[
    \norm*{\bfA\bar\bfx - \bfb'}_{M} \leq (1+O(\eps))\min_{\bfx\in\mathbb R^d}\norm*{\bfA\bfx - \bfb'}_{M}.
\]
Then, as done in the previous constant factor approximation step, we perform another round of weighted sensitivity sampling. Let $\tilde\bfs_i^{M,\bfv}(\bfA)$ be weighted sensitivity upper bounds given by Corollary \ref{cor:weighted-m-sensitivity-alg} with $\tau = d$, which sum to at most $O(d\log^3 n)$ and be computed in time at most
\begin{align*}
    &O\parens*{\nnz(\bfA)\log^3(\nnz(\bfv)) + \log n\frac{\nnz(\bfv)T}{d}\log (\nnz(\bfv))} \\
    \leq~&O\parens*{\nnz(\bfA)\log^3 n + \frac{(Td^{\max\{1,p_M/2\}}+T^2)\log^4 n}{\eps^{2+p_M}}\log\frac1\eps}.
\end{align*}
We can then again do weighted sensitivity sampling according to the $\tilde\bfs_i^{M,\bfv}(\bfA)$ as in Lemma \ref{lem:m-active-regression-relative-error-approximation} to obtain a set of weights $\bfv'$ such that if $\bar\bfx$ satisfies
\[
    \norm*{\bfA\bar\bfx - \bfb'}_{M,\bfv'} \leq (1+\eps)\min_{\bfx\in\mathbb R^d}\norm*{\bfA\bfx - \bfb'}_{M,\bfv'},
\]
then
\[
    \norm*{\bfA\bar\bfx - \bfb'}_{M,\bfv} \leq (1+O(\eps))\min_{\bfx\in\mathbb R^d}\norm*{\bfA\bfx - \bfb'}_{M,\bfv}.
\]
It then follows from the previous discussion that
\[
    \norm*{\bfA\bar\bfx - \bfb'}_{M} \leq (1+O(\eps))\min_{\bfx\in\mathbb R^d}\norm*{\bfA\bfx - \bfb'}_{M}.
\]
Furthermore, computing $\bar\bfx$ only required $\nnz(\bfv')$ entries of $\bfb'$ to sample, which is only 
\[
    O\parens*{\frac{d\cdot\tilde{\mathcal T}^{M,\bfv}(\bfA)}{\eps^{2+p_M}}\log\frac1\eps} = O\parens*{\frac{d^{\max\{2,p_M/2+1\}}\log^3 n}{\eps^{2+p_M}}\log\frac1\eps}
\]
samples. 
\end{proof}

\section{Applications: Active Regression for the Huber Loss}\label{sec:huber}

Note that our sampling bound for general $M$-estimators from Section \ref{sec:active-m} is loose by a factor of $d$ compared to the corresponding bounds for $\ell_p$. This is due to the reliance on a na\"ive net argument rather than using the more efficient construction of \cite{BourgainLindenstraussMilman:1989,SchechtmanZvavitch:2001,LT2011}. A natural question then is whether this can be improved or not, for loss functions other than the $\ell_p$ loss. We answer this in the affirmative for the important special case of the Huber loss.

More specificially, in this section, we obtain a sampling bound of
\[
    O(d^{4-2\sqrt 2}\poly(\eps^{-1},\log n))
\]
for subspace embeddings as well as active regression for the Huber loss, where $4 - 2\sqrt 2\approx 1.17157$. This bound is substantially better compared to the $O(d^{2}\poly(\eps^{-1},\log n))$ bound for other losses, which is obtained by a na\"ive Bernstein bound over a net, as well as a previous subspace embedding result of \cite{ClarksonWoodruff:2015b} which achieved roughly $O(d^4)$ rows. 

In what follows, we will first develop our recursive sampling algorithm for obtaining subspace embeddings for the Huber loss with the above sampling complexity, which can be used to obtain constant factor solutions in the active setting by Lemma \ref{lem:m-sensitivity-constant-factor-approx}. In Section \ref{sec:active-huber}, we then show how to modify our subspace embedding construction in order to get relative error bounds, again following Section \ref{sec:active-m}.

\subsection{Properties of the Huber Loss}

We first collect properties of the Huber loss needed to apply previous results for $M$-estimators. It is known that the Huber loss is polynomially bounded above with degree $2$ and constant $1$, and polynomially bounded below with degree $1$ and constant $1$. 

\begin{lemma}[Lemma 2.2 of \cite{ClarksonWoodruff:2015b}]
The Huber loss satisfies
\[
    \frac{y}{x} \leq \frac{H(y)}{H(x)} \leq \frac{y^2}{x^2}
\]
for all $y > x\geq 0$. 
\end{lemma}

This in turn implies that $H^{1/2}$ is subadditive, since it is known that polynomial growth with degree $1$ and constant $1$ implies subadditivity:

\begin{lemma}[Theorem 103 of \cite{HardyLittlewoodPolya:1973}]
Let $f:\mathbb R\to\mathbb R$. Suppose that $x\mapsto x^{-1}f(x)$ is decreasing, that is,
\[
    \frac{f(y)}{f(x)} \leq \frac{y}{x}
\]
for all $y > x$. Then, $f$ is subadditive. 
\end{lemma}

The following is a corollary that discretizes Lemma \ref{lem:generalized-huber-inequality}:

\begin{corollary}[Huber Inequality ver.\ 4]\label{cor:generalized-huber-inequality}
    Let $\bfy\in\mathbb R^n$ and let $\alpha \in [0,1/2]$ and $\gamma = n^{-\alpha}$ with $2/n \leq \gamma \leq 1$. Let
    \[
        T \supseteq \braces*{i\in[n] : H(\bfy_i) \leq \gamma \norm*{\bfy}_H^2},
    \]
    $\eps = c/\log n$ for some sufficiently small constant $c$, and let
    \[
        \mathcal I = \braces*{1, 1+\eps, 1+2\eps, \dots, 2 - \eps, 2} = \braces*{1 + \eps\cdot i : i\in \{0\}\cup[1/\eps]}. 
    \]
    Then, for some constant $c>0$, either
    \[
        \norm*{\bfy\mid_T}_H^2 \geq c\frac1{(\gamma n)^{3-2\sqrt 2}}\min_{p\in \mathcal I}\norm*{\bfy\mid_T}_p^p
    \]
    or
    \[
        \norm*{\bfy}_H^2 \geq c\gamma\min_{p\in\{1,2\}}\norm*{\bfy}_p^p
    \]
    where $3 - 2\sqrt 2 \approx 0.17157$.
\end{corollary}
\begin{proof}
    It was shown in the proof of Lemma \ref{lem:generalized-huber-inequality} that if $\norm*{\bfy}_1\geq 2n$, then $\norm*{\bfy}_H^2 = \Theta(\norm*{\bfy}_1)$, and if $\norm*{\bfy}_\infty \leq 1$, then $\norm*{\bfy}_H^2 = \norm*{\bfy}_2^2$. Thus, we may assume that all the entries of $\bfy$ are bounded by $O(n)$, and $\norm*{\bfy}_\infty \geq 1$. Note then that the entries that are at most $1/2n$ can make up at most an $\ell_p$ mass of $1/2$, and thus only affects the claimed inequality by a constant factor for any $p$. For all other entries, the magnitude is bounded below by $1/\poly(n)$ and above by $\poly(n)$, so an additive $O(1/\log n)$ difference in the exponent can only affect the bounds by a constant factor. We thus conclude the desired result. 
\end{proof}

\subsection{Net Arguments for the Huber Loss}\label{sec:huber-nets}

When constructing subspace embeddings for norms, it suffices to preserve the lengths of vectors on the unit normed sphere, by the scale invariance/homogeneity of norms. On the other hand, general $M$-estimators, including the Huber norm, do not necessarily satisfy scale invariance, which means that this strategy, as is, does not work. For the Huber norm, we show that for sufficiently small and sufficiently large scales, we have at least an $(1\pm\eps)$-approximate scale invariance, which allows us to obtain subspace embeddings by union bounding over finitely many scales. For all scales near $0$, this is because the Huber norm coincides with the $\ell_2$ norm. For larger scales, we show that the Huber norm is the $\ell_1$ norm up to $(1\pm\eps)$ factors. Then, the following result then shows that in order to prove subspace embeddings, it suffices to preserve the Huber norms of Huber balls at approximately $\log n$ radii.

The following lemma shows that for extremely large radii, the Huber norm essentially coincides with the $\ell_1$ norm.
\begin{lemma}\label{lem:big-huber-is-l1}
    Let $\bfw\in\mathbb R^n$ be a set of weights. Suppose $\bfy\in\mathbb R^n$ has weighted Huber norm at least $\norm*{\bfy}_{H,\bfw}^2 \geq n\norm*{\bfw}_\infty/\eps$. Then,
    \[
        \norm*{\bfy}_{H,\bfw}^2 = (1\pm\eps)\norm*{\bfy}_{1,\bfw}.
    \]
\end{lemma}
\begin{proof}
    Suppose that $\bfy\in\mathbb R^n$ has Huber norm at least $\norm*{\bfy}_H^2 \geq n\norm*{\bfw}_\infty/\eps$. Now define the set
    \[
        S \coloneqq \braces*{i\in[n] : \abs*{\bfy_i} \leq 1}.
    \]
    Then, $\norm*{\bfy\mid_S}_{1,\bfw} \leq n\norm*{\bfw}_\infty$ so
    \[
        \norm*{\bfy}_{1,\bfw} = \norm*{\bfy\mid_{\overline S}}_{1,\bfw} + \norm*{\bfy\mid_{S}}_{1,\bfw} \leq \norm*{\bfy\mid_{\overline S}}_{H,\bfw}^2 + n\norm*{\bfw}_\infty \leq (1+\eps)\norm*{\bfy\mid_{\overline S}}_{H,\bfw}^2.
    \]
    On the other hand, we have $\norm*{\bfy}_{1,\bfw} \geq \norm*{\bfy}_{H,\bfw}^2$, so we conclude. 
\end{proof}

We now show the nets we need to obtain subspace embeddings for the Huber norm.

\begin{lemma}\label{lem:huber-multiple-nets}
Let $\bfA\in\mathbb R^{n\times d}$ and let $\mathcal V = \Span(\bfA)$. Let $\bfw,\bfw'\geq\mathbf{1}_n$ be two sets of weights. Let $\bfw'\geq\mathbf{1}_n$ be another set of weights. Let $\eps>0$. Let 
\[
    \ell \coloneqq \log_2\sqrt{\frac{n\norm*{\bfw'}_\infty}{\eps}} = O\parens*{\log\frac{n\norm*{\bfw'}_\infty}{\eps}}.
\]
For each $k\in[\ell]$, let $\rho = 2^k$ and suppose that
\[
    \norm*{\bfy}_{H,\bfw'} = \norm*{\bfy}_{H,\bfw} \pm \eps \rho
\]
for every $\bfy\in \mathcal B_\rho^H$, for every $k\in[\ell]$. Then,
\[
    \norm*{\bfy}_{H,\bfw'} = (1\pm O(\eps))\norm*{\bfy}_{H,\bfw}
\]
for all $\bfy\in\mathcal V$. 
\end{lemma}
    
\begin{proof}
    We handle three cases: scales near $0$, scales near infinity, and all scales in between. 
    
    \paragraph{Scales near $0$.} Note that if the weighted Huber norm of $\bfy\in\mathbb R^n$ is at most $1$, then all entries of $\bfy$ are at most $1$ in absolute value, so $\norm*{\bfy}_{H,\bfw}^2 = \norm*{\bfy}_{2,\bfw}^2$, and thus the Huber norm is scale invariant at these scales. Thus, $(1\pm\eps)$-approximations for the unit Huber sphere implies $(1\pm\eps)$-approximations for all vectors with $\norm*{\bfy}_{H,\bfw}^2 \leq 1$. 
    
    \paragraph{Scales near infinity.} 
    
    Suppose that we have the guarantee that 
    \[
        \norm*{\bfx}_{H,\bfw'}^2 = (1\pm\eps)\norm*{\bfx}_{H,\bfw}^2
    \]
    for every $\bfx$ in the Huber sphere $\mathcal S_{n\norm*{\bfw'}_\infty/\eps}^{H,\bfw}$ of radius $n\norm*{\bfw'}_\infty/\eps$. Then, if $\bfx$ has any Huber norm at least $n\norm*{\bfw'}_\infty/\eps$, then 
    \begin{align*}
        \norm*{\bfx}_{H,\bfw}^2 &= (1\pm\eps)\norm*{\bfx}_{1,\bfw} && \text{Lemma \ref{lem:big-huber-is-l1}} \\
        &= (1\pm\eps)\frac{\eps}{n\norm*{\bfw'}_\infty}\norm*{\bfx}_{1,\bfw} \cdot \norm*{\frac{n\norm*{\bfw'}_\infty}{\eps}\frac{\bfx}{\norm*{\bfx}_{1,\bfw}}}_{1,\bfw} && \text{Scale invariance} \\
        &= (1\pm\eps)^2\frac{\eps}{n\norm*{\bfw'}_\infty}\norm*{\bfx}_{1,\bfw} \cdot \norm*{\frac{n\norm*{\bfw'}_\infty}{\eps}\frac{\bfx}{\norm*{\bfx}_{1,\bfw}}}_{H,\bfw}^2 && \text{Lemma \ref{lem:big-huber-is-l1}} \\
        &= (1\pm\eps)^3\frac{\eps}{n\norm*{\bfw'}_\infty}\norm*{\bfx}_{1,\bfw} \cdot \norm*{\frac{n\norm*{\bfw'}_\infty}{\eps}\frac{\bfx}{\norm*{\bfx}_{1,\bfw}}}_{H,\bfw'}^2 && \text{Approximation guarantee} \\
        &= (1\pm\eps)^4\frac{\eps}{n\norm*{\bfw'}_\infty}\norm*{\bfx}_{1,\bfw} \cdot \norm*{\frac{n\norm*{\bfw'}_\infty}{\eps}\frac{\bfx}{\norm*{\bfx}_{1,\bfw}}}_{1,\bfw'} && \text{Lemma \ref{lem:big-huber-is-l1}} \\
        &= (1\pm\eps)^4 \cdot \norm*{\bfx}_{1,\bfw'}^2 && \text{Scale invariance} \\
        &= (1\pm\eps)^5 \cdot \norm*{\bfx}_{H,\bfw'}^2 && \text{Lemma \ref{lem:big-huber-is-l1}}
    \end{align*}
    Thus, it suffices to prove approximation guarantees for the Huber sphere $\mathcal S_{n\norm*{\bfw'}_\infty/\eps}^{H,\bfw}$ to handle all $\bfx\in\mathbb R^n$ with Huber norms at least $n\norm*{\bfw'}_\infty/\eps$. 
    
    \paragraph{Scales in between.}
    
    For the remaining scales between $1$ and $n\norm*{\bfw'}_\infty/\eps$, we work in powers of $2$ so that there are approximately
    \[
        \ell \coloneqq \log_{2}\sqrt{\frac{n\norm*{\bfw}_\infty}{\eps}} = O\parens*{\log\frac{n\norm*{\bfw}_\infty}{\eps}}
    \]
    total scales. Then for any $\rho \in [1, \sqrt{n\norm*{\bfw}_\infty/\eps}]$, there exists a $\rho \leq \rho_k \leq 2\rho$ such that we can get an additive error of $O(\eps)\rho_k = O(\eps)\rho$ for every $\bfy\in \mathcal B_{\rho_k}$, which in particular contains $\mathcal B_\rho^H$.
\end{proof}

\subsection{Compact Rounding for the Huber Loss}

In Section \ref{sec:huber-nets}, we have reduced our task to proving approximation guarantees for a small number of Huber spheres. We now focus on showing small sampling bounds for a single Huber sphere. 

The following rounding lemma is implicit in the proof of Theorems 7.3 and 7.4 of \cite{BourgainLindenstraussMilman:1989}, which uses these net constructions. 

\begin{lemma}\label{lem:blm-compact-rounding}
Let $1\leq p < \infty$ and let
\[
    \mathcal B_p \coloneqq \braces*{\bfA\bfx : \norm*{\bfA\bfx}_p \leq 1}.
\]
Let $\eps \geq \gamma > 0$ and let $\mathcal N_\gamma$ be a standard $\gamma$-net on $\mathcal B_p$ with the $\ell_p$ norm, which has
\[
    \log \abs*{\mathcal N_\gamma} \leq O\parens*{d\log\frac{1}{\gamma}}
\]
by Lemma \ref{lem:standard-net}. Let $\bfy\in \mathcal N_\gamma$. Let $\ell = \ceil*{\log_{1+\eps}(2^{1/p}d^{1/p\lor 1/2})}$. Then, for $k\in[\ell]$, there exists a rounding
\[
    \tilde\bfy = \sum_{k=0}^\ell \tilde\bfy^{(k)}
\]
such that:
\begin{itemize}
    \item $\abs*{\tilde\bfy(i) - \bfy(i)} \leq 3\eps \abs*{\bfy(i)}$ for all $i\in[n]$
    \item $\abs*{\tilde\bfy^{(k)}(i)} \leq \bfW_{i,i}^{1/p}(1+\eps)^k$
    \item $\tilde\bfy^{(k)}$ for $0\leq k\leq \ell$ have disjoint supports
    \item each $\tilde\bfy^{(k)}$ is drawn from a set of vectors $\mathcal D_k$ with size at most
    \[
        \log\abs*{\mathcal D_k} \leq c(p)\frac{d}{\eps^{1+\beta}(1+\eps)^{\beta k}}\parens*{\log\frac{n}{\eps} + \log\frac1\gamma}
    \]
    where $\beta = (2\land p)$ and $c(p) = O\parens*{\max\braces*{p, \frac1{p-1}}}$. 
\end{itemize}
\end{lemma}

The lemma above is used to prove the following similar result for the Huber loss.

\begin{lemma}\label{lem:blm-net-for-huber}
    Let $\eps\in(0,1)$ and $\rho\geq 1$. Let $\bfA\in\mathbb R^{n\times d}$ and let $\bfy\in \mathcal B_\rho^H \subseteq \Span(\bfA)$. Let $\mathcal I\subset[1,2]$ be the set in the statement of Corollary \ref{cor:generalized-huber-inequality} and suppose that $\bfy$ satisfies
    \[
        \frac{\min_{p\in\mathcal I}\norm*{\bfy}_p^p}{\norm*{\bfy}_H^2} \leq \kappa.
    \]
    Then, there exists a $p\in\mathcal I$, $\ell = O((\log d)/\eps^2)$, and a rounding
    \[
        \tilde\bfy = \sum_{k=0}^\ell \tilde\bfy^{(k)}
    \]
    such that:
    \begin{itemize}
        \item $\norm*{\tilde\bfy - \bfy}_H \leq O(\eps)\norm*{\bfy}_H$
        \item $H\parens*{\abs*{\tilde\bfy^{(k)}(i)}} \leq O(\kappa)\parens*{\frac{\bfw_i^p(\bfA)}{d} + \frac1n} (1+\eps^2)^{pk}\norm*{\bfy}_H^2$
        \item $\tilde\bfy^{(k)}$ for $0\leq k\leq \ell$ have disjoint supports
        \item each $\tilde\bfy^{(k)}$ is drawn from a set of vectors $\mathcal D_k$ with size at most
        \[
            \log\abs*{\mathcal D_k} \leq O(1)\frac{d}{\eps^{2(1+p)}(1+\eps^2)^{pk}}(\log n)\parens*{\log\frac{n}{\eps}}.
        \]
    \end{itemize}
\end{lemma}
\begin{proof}
    Let $\gamma = \poly(\eps/n)$ and let $\mathcal N_\gamma^p$ be a $\gamma$-net for the $\ell_p$ unit ball for $p\in\mathcal I$. We apply the results of Lemma \ref{lem:blm-compact-rounding} to each of these nets to obtain sets of vectors $\mathcal D_k^p$, corresponding to the set of vectors $\mathcal D_k$ in the lemma statement when we apply the result with $p$. 

    Now let $\bfy\in\mathcal B_\rho^H$. By assumption, we may find a $p\in\mathcal I$ so that
    \[
        \rho^2 \leq \norm*{\bfy}_p^p \leq \kappa\rho^2.
    \]
    We then let $L \leq O(\log \kappa)\leq O(\log n)$ be such that
    \[
        \rho^2 \cdot 2^L \leq \norm*{\bfy}_p^p \leq \rho^2 \cdot 2^{L+1}.
    \]
    Let
    \[
        \beta\coloneqq(\rho^2 \cdot 2^{L+1})^{1/p} = \Theta(\norm*{\bfy}_p).
    \]
    Then, $\norm*{\bfy/\beta}_p \leq 1$ so there exists a $\bfy'\in\mathcal N_\gamma^p$ such that
    \[
        \norm*{\bfy' - \bfy/\beta}_p \leq \gamma
    \]
    so
    \[
        \norm*{\beta\bfy' - \bfy}_H^2 \leq \norm*{\beta\bfy' - \bfy}_p^p \leq \gamma^p \beta^p \leq 2\gamma^p \norm*{\bfy}_p^p \leq 2\gamma^p \kappa \norm*{\bfy}_H^2 \leq \eps^2\norm*{\bfy}_H^2
    \]
    for an appropriate choice of $\gamma$. 

    We now let
    \[
        \tilde\bfy = \sum_{k=0}^\ell \tilde\bfy^{(k)}
    \]
    be the rounding given in Lemma \ref{lem:blm-compact-rounding} for $\bfy'$, which is in the $\ell_p$ unit ball, with the $\eps$ in the lemma set to our $\eps^2$. Then,
    \[
        \abs*{\beta\bfy'(i) - \beta\tilde\bfy(i)} \leq 3\eps^2 \abs*{\beta\bfy'(i)}
    \]
    for all $i\in[n]$ so
    \begin{align*}
        \norm*{\beta\bfy' - \beta\tilde\bfy}_H^2 &= \sum_{i=1}^n H(\abs*{\beta\bfy'(i) - \beta\tilde\bfy(i)}) \\
        &\leq \sum_{i=1}^n H(3\eps^2 \abs*{\beta\bfy'(i)}) \\
        &\leq 3\eps^2 \sum_{i=1}^n H(\abs*{\beta\bfy'(i)}) \\
        &= 3\eps^2\norm*{\beta\bfy'}_H^2 \\
        &= O(\eps^2)\norm*{\bfy}_H^2.
    \end{align*}
    Thus by the triangle inequality for the Huber norm,
    \[
        \norm*{\beta\tilde\bfy - \bfy}_H \leq O(\eps)\norm*{\bfy}_H.
    \]
    Futhermore,
    \[
        \abs*{\beta\tilde\bfy^{(k)}(i)} \leq \bfW_{i,i}^{1/p}(1+\eps^2)^k \beta
    \]
    so
    \begin{align*}
        H\parens*{\abs*{\beta\tilde\bfy^{(k)}(i)}} &\leq O(1)\bfW_{i,i}(1+\eps^2)^{pk}\beta^p && \text{$H(y) \leq O\parens*{\min_{p\in[1,2]} y^p}$} \\
        &\leq O(1)\parens*{\frac{\bfw_i^p(\bfA)}{d} + \frac1n} (1+\eps^2)^{pk} \cdot 2\norm*{\bfy}_p^p \\
        &\leq O(\kappa)\parens*{\frac{\bfw_i^p(\bfA)}{d} + \frac1n} (1+\eps^2)^{pk}\norm*{\bfy}_H^2.
    \end{align*}
    We conclude as desired by rescaling the $\tilde\bfy$ and $\tilde\bfy^{(k)}$ by $\beta$. 
\end{proof}

\subsection{Sampling Bounds}\label{subsec:huber-sampling-bounds}

With the above net results in hand, we can now complete the argument for one step of the sampling recursion. In order to implement a recursive sampling scheme, note that we need to handle weighted Huber norms. These weighted Huber norms will be handled by handling the weights in a small number of groups such that the weights are within constant factors of each other.

\subsubsection{Sampling Guarantees for a Single Weight Class and a Single Radius}

We start with a Bernstein bound using the net results, for a single group of weights and at a single radius.

\begin{lemma}[Bernstein bounds for Huber sampling]\label{lem:huber-bernstein}
    Let $\bfA\in\mathbb R^{n\times d}$ and let $\mathcal V = \Span(\bfA)$. Let $\eps\in(0,1)$ be an accuracy parameter, $\delta>0$ a failure rate parameter, $\rho \geq 1$ a Huber radius, and $\kappa\geq 1$ a distortion parameter. Let 
    \[
        m \coloneqq \frac{d}{\poly(\eps)}(\log^2 d)\parens*{\log^2\frac{n}{\eps}}\parens*{\log\frac{\log d}{\delta\eps}}.
    \]
    Let $\mathcal I$ be as in the statement of Corollary \ref{cor:generalized-huber-inequality}. Let $\bfw\geq \mathbf{1}_n$ be a set of weights such that  
    \begin{equation}\label{eq:bounded-weight-ratio}
        \frac{\max_{i\in[n]} \bfw_i}{\min_{i\in[n]} \bfw_i}\leq 2
    \end{equation}
    and let $w = \min_{i\in[n]}\bfw_i$. Let $\bfw'$ be chosen randomly so that
    \[
        \bfw_i' \coloneqq \begin{cases}
            \bfw_i/\bfp_i & \text{w.p. $\bfp_i$} \\
            0 & \text{otherwise}
        \end{cases}
    \]
    where
    \[
        \bfp_i \geq \min\braces*{1, m\cdot \kappa\bracks*{\frac1n + \sum_{p\in\mathcal I}\frac{\bfw_i^p(\bfA)}{d}}}
    \]
    Let
    \[
        S_\kappa \coloneqq \braces*{\bfy\in\mathcal V : \frac{\min_{p\in\mathcal I}\norm*{\bfy}_p^p}{\norm*{\bfy}_H^2}\leq \kappa}
    \]
    Then with probability at least $1 - \delta$, 
    \[
        \norm*{\bfy}_{H,\bfw'} = \norm*{\bfy}_{H,\bfw} \pm \eps w\rho
    \]
    for all $\bfy\in \mathcal N\cap S_\kappa$, where $\mathcal N$ is an $\eps^2\rho$-net over $\mathcal B_\rho^H$ given by Lemma \ref{lem:blm-net-for-huber}, by setting $\eps$ in the lemma to $\eps^2$. 
\end{lemma}
\begin{proof}
    Let $\tilde\bfy\in\mathcal N\cap S_\kappa$ be a net vector in the result of Lemma \ref{lem:blm-net-for-huber} and let $p$ be the corresponding $p\in\mathcal I$. Note that
    \[
        \abs*{\norm*{\tilde\bfy}_{H,\bfw'}^2 - \norm*{\tilde\bfy}_{H,\bfw}^2} \leq \sum_{k=0}^\ell \abs*{\norm*{\tilde\bfy^{(k)}}_{H,\bfw'}^2 - \norm*{\tilde\bfy^{(k)}}_{H,\bfw}^2} 
    \]
    by the disjointness of the supports of the $\tilde\bfy^{(k)}$. Thus, it suffices to bound each term in the sum by $\eps\rho^2/(\ell+1)$. 

    Fix a $0\leq k\leq \ell$ and define the random variable
    \[
        W_i \coloneqq \bfw_i' H(\tilde\bfy^{(k)}(i))
    \]
    for each $i\in[n]$. Then,
    \[
        \E\bracks*{\sum_{i=1}^n W_i} = \sum_{i=1}^n \frac{\bfw_i}{\bfp_i}H(\tilde\bfy^{(k)}(i))\cdot \bfp_i = \norm*{\tilde\bfy^{(k)}}_{H,\bfw}^2.
    \]
    We next bound the variance:
    \begin{align*}
        \Var\bracks*{\sum_{i=1}^n W_i} &= \sum_{i=1}^n \Var[W_i] \\
        &\leq \sum_{i=1}^n \frac{\bfw_i^2}{\bfp_i^2}H(\tilde\bfy^{(k)}(i))^2\cdot \bfp_i \\
        &= \sum_{i=1}^n \frac1{\bfp_i}\bfw_i H(\tilde\bfy^{(k)}(i))\cdot \bfw_i H(\tilde\bfy^{(k)}(i))
    \end{align*}
    Note that $W_i$ is almost surely bounded by
    \begin{align*}
        \frac1{\bfp_i}\bfw_i H(\tilde\bfy^{(k)}(i)) &\leq \bfw_i\frac{O(\kappa)\parens*{\frac{\bfw_i^p(\bfA)}{d} + \frac1n}(1+\eps^4)^{pk}\norm*{\bfy}_H^2}{m\cdot O(\kappa)\parens*{\frac{\bfw_i^p(\bfA)}{d} + \frac1n}} && \text{Lemma \ref{lem:blm-net-for-huber}} \\
        &\leq \frac2m (1+\eps^4)^{pk}\norm*{\bfy}_{H,\bfw}^2 && \text{Equation \eqref{eq:bounded-weight-ratio}}
    \end{align*}
    so the variance is bounded by
    \[
        \Var\bracks*{\sum_{i=1}^n W_i} \leq \frac2m (1+\eps^4)^{pk}\norm*{\bfy}_{H,\bfw}^2\sum_{i=1}^n \bfw_iH(\tilde\bfy^{(k)}(i)) = \frac{2}{m}(1+\eps^4)^{pk}\norm*{\bfy}_{H,\bfw}^4.
    \]
    Then by Bernstein's inequality,
    \begin{align*}
        \Pr\braces*{\abs*{\sum_{i=1}^n W_i - \norm*{\tilde\bfy^{(k)}}_{H,\bfw}^2} > t} &\leq 2\exp\parens*{-\Theta(1)\frac{t^2}{\frac{1}{m}(1+\eps^4)^{pk}\norm*{\bfy}_{H,\bfw}^4 + \frac1m (1+\eps^4)^{pk}\norm*{\bfy}_{H,\bfw}^2 t}} \\
        &= 2\exp\parens*{-\Theta(1)\frac{mt^2}{(1+\eps^4)^{pk}\norm*{\bfy}_{H,\bfw}^2(\norm*{\bfy}_{H,\bfw}^2 + t)}}.
    \end{align*}
    For $t = \eps w^2\rho^2/(\ell+1)$, this gives a bound of
    \[
        2\exp\parens*{-\Theta(1)\frac{m(\eps w^2\rho^2/(\ell+1))^2}{(1+\eps^4)^{pk}w^4\rho^4}} =  2\exp\parens*{-\Theta(1)\frac{m\eps^2}{(\ell+1)^2(1+\eps^4)^{pk}}}.
    \]
    We then set
    \begin{align*}
        m &= O(1)\frac{d}{\eps^{4(1+p)+2}}(\ell+1)^2\log\frac{n}{\eps}\log\frac{\ell}\delta \\
        &= O(1)\frac{d}{\eps^{4p+10}}(\log^2 d)\parens*{\log^2\frac{n}{\eps}}\parens*{\log\frac{\log d}{\delta\eps}}
    \end{align*}
    which is enough to union bound over the set $\mathcal D_k$ (see Lemma \ref{lem:blm-net-for-huber}) of log size at most
    \[
        O(1)\frac{d}{\eps^{4(1+p)}(1+\eps^4)^{pk}}\log^2\frac{n}{\eps}
    \]
    with failure probability at most $\delta/(\ell+1)$. We union bound over the $\ell+1$ choices of $k$ to obtain that
    \[
        \abs*{\norm*{\tilde\bfy}_{H,\bfw'}^2 - \norm*{\tilde\bfy}_{H,\bfw}^2} \leq \eps w^2\rho^2.
    \]
    Setting $\eps$ in the above result to $\eps^2$, we get
    \[
        \abs*{\norm*{\tilde\bfy}_{H,\bfw'}^2 - \norm*{\tilde\bfy}_{H,\bfw}^2} \leq \eps^2 w^2\rho^2.
    \]
    Now note that if both $\norm*{\tilde\bfy}_{H,\bfw'}$ and $\norm*{\tilde\bfy}_{H,\bfw}$ are at most $\eps w\rho/2$, then we have that
    \[
        \abs*{\norm*{\tilde\bfy}_{H,\bfw'} - \norm*{\tilde\bfy}_{H,\bfw}} \leq \eps w\rho
    \]
    by the triangle inequality. Otherwise, we have that
    \[
        \abs*{\norm*{\tilde\bfy}_{H,\bfw'} - \norm*{\tilde\bfy}_{H,\bfw}} = \frac{\abs*{\norm*{\tilde\bfy}_{H,\bfw'}^2 - \norm*{\tilde\bfy}_{H,\bfw}^2}}{\abs*{\norm*{\tilde\bfy}_{H,\bfw'} + \norm*{\tilde\bfy}_{H,\bfw}}}\leq \frac{\eps^2 w^2\rho^2}{\eps w\rho} = \eps w\rho.\qedhere
    \]
\end{proof}

We can now apply the result twice, once on all of the rows and once on only the set of rows with small Huber sensitivity, in order to obtain a sampling bound for preserving a Huber sphere at a single scale. 

\begin{lemma}\label{lem:huber-single-scale}
    Let $\bfA\in\mathbb R^{n\times d}$ and let $\mathcal V = \Span(\bfA)$. Let $\eps\in(0,1)$ be an accuracy parameter, $\delta>0$ a failure rate parameter, and $\rho \geq 1$ a Huber radius. Let $\beta = 3 - 2\sqrt 2$, $\gamma = n^{-\beta/(\beta+1)}$, and let $T$ be a subset
    \[
        T \subseteq \braces*{i\in[n] : \bfs_i^H(\bfA) \leq \gamma}
    \]
    Let $\kappa \coloneqq O((\gamma n)^{\beta}) = O(\gamma^{-1})$. Let $\bfw\geq \mathbf{1}_n$ be a set of weights such that  
    \[
        \frac{\max_{i\in[n]} \bfw_i}{\min_{i\in[n]} \bfw_i}\leq 2
    \]
    and let $w = \min_{i\in[n]}\bfw_i$. Let $\bfp$ be the sampling probabilities given by Lemma \ref{lem:huber-bernstein} for $\bfA$ and let $\bfq$ be the sampling probabilities given by Lemma \ref{lem:huber-bernstein} for $\bfA\mid_T$. We then consider sampling probabilities $\bfr$ such that
    \[
        \bfr_i \coloneqq \begin{cases}
            1 & \text{if $i\in \overline T$} \\
            \min\braces*{\bfp_i + \bfq_i, 1} & \text{if $i\in T$}
        \end{cases}
    \]
    Let $\bfw'\geq \mathbf{1}_n$ be a set of weights chosen randomly so that
    \[
        \bfw_i' \coloneqq \begin{cases}
            \bfw_i/\bfr_i & \text{w.p. $\bfr_i$} \\
            0 & \text{otherwise}
        \end{cases}
    \]
    Then,
    \[
        \norm*{\bfy}_{H,\bfw'} = \norm*{\bfy}_{H,\bfw} \pm \eps w\rho.
    \]
    for all $\bfy\in\mathcal B_\rho^H$ with probability at least $1 - \delta$. Furthermore,
    \begin{align*}
        \E\nnz(\bfw') &= O(\kappa m\log n) = O\parens*{\kappa d\poly\parens*{\frac{\log n}{\eps}}\log\frac1\delta} \\
        \nnz(\bfw') &= O\parens*{\E\nnz(\bfw')}
    \end{align*}
    with probability at least $1 - \delta$, and $\norm*{\bfw'}_\infty \leq O(n)\norm*{\bfw}_\infty$.
\end{lemma}
\begin{proof}
    Note that for our choice of $\gamma$, we have that $\gamma = (\gamma n)^{-\beta}$. Then by our choice of $\kappa$ and Corollary \ref{cor:generalized-huber-inequality}, any vector $\bfy$ either has at least one of
    \begin{equation}\label{eq:huber-distortion-choice}
        \min_{p\in\mathcal I}\frac{\norm*{\bfy}_p^p}{\norm*{\bfy}_H^2} \leq \kappa \qquad \mbox{ or } \qquad \min_{p\in\mathcal I}\frac{\norm*{\bfy\vert_T}_p^p}{\norm*{\bfy\vert_T}_H^2} \leq \kappa.
    \end{equation}
    Now let $\mathcal N$ and $\mathcal N_T$ denote the nets obtained by applying Lemma \ref{lem:huber-bernstein} on $\bfA$ and $\bfA\vert_T$, respectively, with radius $\rho$. Let $\bfy\in\mathcal B_\rho^H$. Then, clearly, $\bfy$ is within $\eps^2\rho$ Huber distance of some $\tilde\bfy$ belonging to either $\mathcal N$ or $\mathcal N_T$, so $\mathcal N\cup\mathcal N_T$ is an $\eps^2\rho$-cover of $\mathcal B_\rho^H$. Furthermore, by Lemma \ref{lem:huber-bernstein}, 
    \[
        \norm*{\bfy}_{H,\bfw'} = \norm*{\bfy}_{H,\bfw} \pm \eps w\rho
    \]
    for every $\bfy\in \mathcal N\cup\mathcal N_T$ with probability at least $1 - 2\delta$. Then by Lemma \ref{lem:m-norm-net-to-sphere} we have that
    \[
        \norm*{\bfy}_{H,\bfw'} = \norm*{\bfy}_{H,\bfw} \pm O(\eps)w\rho
    \]
    for all $\bfy\in\mathcal B_\rho^H$, with probability at least $1 - 2\delta$. 
    
    Finally, letting $m$ be as in the statement of Lemma \ref{lem:huber-bernstein},
    \begin{align*}
        \E\nnz(\bfw') &\leq \sum_{i=1}^n \bfp_i + \bfq_i \\
        &\leq O(\kappa m)\parens*{\sum_{p\in\mathcal I}\bracks*{\frac1n + \sum_{i=1}^n \frac{\bfw_i^p(\bfA)}{d}}+\sum_{p\in\mathcal I}\bracks*{\frac1n + \sum_{i\in T} \frac{\bfw_i^p(\bfA\mid_T)}{d}}} \\
        &= O(\kappa m\abs*{\mathcal I}) = O(\kappa m\log n)
    \end{align*}
    since Lewis weights sum to $d$. Furthermore, by Bernstein's inequality,
    \[
        \nnz(\bfw') = \Theta(\E\nnz(\bfw'))
    \]
    with probability at least $1 - \delta$. By rescaling $\epsilon$ and $\delta$ by constant factors, we conclude.
\end{proof}

\subsubsection{Sampling Guarantees for a Single Step}

We now remove the assumption of bounded weights and fixed radius by union bounding over their various levels.

\begin{lemma}\label{lem:huber-single-step}
    Let $\bfA\in\mathbb R^{n\times d}$ and let $\mathcal V = \Span(\bfA)$. Let $\eps\in(0,1)$ be an accuracy parameter, $\delta>0$ a failure rate parameter. Let $\bfw\geq\mathbf{1}_n$ be a set of weights and for each $j\in[\ceil*{\log_2(\norm*{\bfw}_\infty)}+1]$ define the sets
    \[
        T_j \coloneqq \braces*{i\in[n] : \bfw_i \in [2^{j-1}, 2^j)}.
    \]
    Let $\bfw'$ be obtained by applying Lemma \ref{lem:huber-single-scale} on each $\bfA\vert_{T_j}$ with weights $\bfw\vert_{T_j}$, with $\delta$ set to
    \[
        \Theta\parens*{\frac{\delta}{\log(n\norm*{\bfw}_\infty/\eps}}.
    \]
    Then, for all $\bfy\in\mathcal V$,
    \[
        \norm*{\bfy}_{H,\bfw'} = (1\pm\eps)\norm*{\bfy}_{H,\bfw}.
    \]
    Furthermore, for $\beta = 3-2\sqrt 2$ and $\kappa = n^{\beta/(\beta+1)}$,
    \begin{align*}
        \E\nnz(\bfw') &= O\parens*{\kappa d\poly\parens*{\frac{\log n}{\eps}}\log\frac{\norm*{\bfw}_\infty}\delta} \\
        \nnz(\bfw') &= O\parens*{\E\nnz(\bfw')}
    \end{align*}
    with probability at least $1 - \delta$. 
\end{lemma}
\begin{proof}
    For any $\rho\geq 1$, by a union bound over the $\ceil*{\log_2\norm*{\bfw}_\infty} + 1$ weight classes $T_j$, the approximation guarantee from Lemma \ref{lem:huber-single-scale} holds for all classes with probability at least $\Theta(\delta/\log(n\norm*{\bfw}_\infty/\eps))$. We then union bound over the $\Theta(\log(n\norm*{\bfw'}_\infty/\eps)) = \Theta(\log(n\norm*{\bfw}_\infty/\eps))$ radius levels $\rho$ required by Lemma \ref{lem:huber-multiple-nets} so that
    \[
        \norm*{\bfy}_{H,\bfw'} = (1\pm\eps)\norm*{\bfy}_{H,\bfw}
    \]
    for all $\bfy\in\mathcal V$. 
\end{proof}

We provide the algorithm for Lemma \ref{lem:huber-single-step} Algorithm \ref{alg:unweighted-huber-subspace-embedding}, whose guarantee is proven in Theorem \ref{thm:n13d-huber-sampling}. 

\begin{algorithm}
	\caption{Huber subspace embedding}
	\textbf{input:} Matrix $\bfA \in \R^{n \times d}$, weights $\bfw$. \\
	\textbf{output:} Weights $\bfw'$.
	\begin{algorithmic}[1] 
        \State $\beta\gets 3 - 2\sqrt 2$, $\gamma\gets n^{-\beta/(1+\beta)}$, $m\gets d\poly((\log(n\norm*{\bfw}_\infty/\delta))/\eps), W\gets \ceil*{\log_2\norm*{\bfw}_\infty}+1$
        \State $T_j \gets \braces*{i\in[n]:\bfw_i\in[2^{j-1},2^j)}$ for $j\in[W]$
        \State $\mathcal I\gets\{1+(c/\log n)\cdot j : j\in[(\log n)/c]\cup\{0\}\}$
        \For{$j\in[W]$}
		    \State Compute approximate Huber sensitivities $\tilde\bfs_i^H(\bfA\vert_{T_j})$ with Algorithm \ref{alg:m-estimator-alg-sensitivity} with $\tau = \gamma n$
            \State $S_j\gets \{i\in[n] : \tilde\bfs_i^H(\bfA\vert_{T_j})\leq \gamma\}$
            \State Set $\bfw_i' = \bfw_i$ for every $i\notin S_j$
            \State Compute approximate Lewis weights $\tilde\bfw_i^p(\bfA)$ with Theorem \ref{thm:cohen-peng-fast-lewis-weights} for every $p\in\mathcal I$
            \State Compute approximate Lewis weights $\tilde\bfw_i^p(\bfA\vert_{S_j\cap T_j})$ with Theorem \ref{thm:cohen-peng-fast-lewis-weights} for every $p\in\mathcal I$
            \For{$i\in S_j$}
                \State Let $\bfp_i = \min\{1,\gamma^{-1}m(1/n + \sum_{p\in\mathcal I}\tilde\bfw_i^p(\bfA)/d) + \sum_{p\in\mathcal I}\tilde\bfw_i^p(\bfA\mid_T)/d)\}$
                \State Set $\bfw_i' = \bfw_i/\bfp_i$ with probability $\bfp_i$ and $0$ otherwise.
            \EndFor
        \EndFor
        \State \Return $\bfw'$
	\end{algorithmic}\label{alg:unweighted-huber-subspace-embedding}
\end{algorithm}

\begin{theorem}\label{thm:n13d-huber-sampling}
    Let $\bfA\in\mathbb R^{n\times d}$ and let $\mathcal V = \Span(\bfA)$. Let $\eps>0$ and $\delta > 0$. Let $\bfw\geq\mathbf{1}_n$ be a set of weights. Let $\bfw'$ be the weights returned by Algorithm \ref{alg:unweighted-huber-subspace-embedding}. Then,
    \[
        \norm*{\bfy}_{H,\bfw}^2 = (1\pm O(\eps))\norm*{\bfy}_H^2
    \]
    for all $\bfy\in\mathcal V$ and
    \begin{align*}
        \E\nnz(\bfw') &= O\parens*{\gamma^{-1}d\poly\parens*{\frac{\log (n\norm*{\bfw}_\infty)}{\delta\eps}}} \\
        \nnz(\bfw') &= O\parens*{\E\nnz(\bfw')}
    \end{align*}
    with probability at least $1 - \delta$. Furthermore, Algorithm \ref{alg:unweighted-huber-subspace-embedding} runs in time
    \[
        O\parens*{(\nnz(\bfA) + d^{(1+1/\beta)\omega})(\log^2 n)\log\frac1\delta}.
    \]
\end{theorem}
\begin{proof}
    We first identify the rows with Huber sensitivity at least $\gamma$. By running Theorem \ref{thm:m-estimator-alg} $\log(1/\delta)$ times with $\tau = \gamma n = n^{1/(\beta+1)}$, we can boost the success probability to $1-\delta$ to find sensitivity upper bounds that sum to
    \[
        O\parens*{d(\log^2 n)\log\frac1\delta}
    \]
    in time
    \[
        O\parens*{\bracks*{\nnz(\bfA)\log^2 n + n^{1/(\beta+1)}T\log n}\log\frac1\delta}.
    \]
    for $T = O(d^{\omega}\log^2 d)$ by Theorem \ref{thm:cohen-peng-fast-lewis-weights}. Note that in the above running time, the term inside the square brackets is $O(\nnz(\bfA)\log^2 n)$ time for $T \leq n^{\beta/(\beta+1)}$ and $O((\nnz(\bfA)+T^{1+1/\beta})\log^2 n)$ time for $T \geq n^{\beta/(\beta+1)}$. Then, we may find a superset of rows with Huber sensitivity at least $\gamma$ of size at most $O\parens*{\gamma^{-1}d(\log^2 n)\log\frac1\delta}$. We may also approximate the $\ell_p$ Lewis weights in time $O([\nnz(\bfA) + T](\log n)\log\frac1\delta)$ by Theorem \ref{thm:cohen-peng-fast-lewis-weights}, for each $p\in\mathcal I$. The guarantees on $\nnz(\bfw)$ follow from Lemma \ref{lem:huber-single-scale}.
\end{proof}

\subsection{Subspace Embeddings via Recursive Sampling}

For large $n$, a bound of $\tilde O(n^{\beta/(\beta+1)}d)$ with $\beta = 3-2\sqrt 2$ of Theorem \ref{thm:n13d-huber-sampling} is much worse than the previous $O(d^2\poly\log(n))$ bound. However, by applying Theorem \ref{thm:n13d-huber-sampling} \emph{after} reducing the number of rows to $O(d^2\poly\log(n))$, we can obtain a sampling bound of roughly $O(d^{2\beta/(\beta+1)}\poly\log(n))$, where $1 + 2\beta/(\beta+1)\approx 1.29289$. Furthermore, by recursively applying this procedure, we can improve the dependence on $d$ to $d^{1+\beta} = d^{4-2\sqrt 2}$ after only $O(\log\log(d\log n))$ many iterations, where $4-2\sqrt 2\approx 1.17157$.

Note that as long as we apply this recursion at most $\log n$ times, the number of rows decreases as $m \to O(m^{\beta/(\beta+1)}d\poly((\log n) / \eps))$ at each iteration, since $\log\norm*{\bfw}_\infty$ increases by at most $O(1)$ at each iteration. The closed form solution to this can be found via the following recurrence:

\begin{lemma}\label{lem:affine-recurrence}
    Suppose that $(a_i)_{i=0}^\infty$ satisfies the recurrence
    \[
        a_{i+1} = \lambda a_i + b
    \]
    for some $b > 0$ and $\lambda\in(0,1)$. Then,
    \[
        a_i = \frac{1}{1-\lambda} \parens*{b - \lambda^{i}\parens*{b - (1-\lambda) a_0}}
    \]
\end{lemma}
\begin{proof}
    Note that
    \[
        \frac{1}{1-\lambda} \parens*{b - \lambda^{0}\parens*{b - (1-\lambda) a_0}} = \frac{1}{1-\lambda} \parens*{b -b + (1-\lambda) a_0} = a_0
    \]
    and
    \begin{align*}
        \frac{1}{1-\lambda} \parens*{b - \lambda^{i+1}\parens*{b - (1-\lambda) a_0}} &= \lambda \cdot \frac{1}{1-\lambda} \parens*{\frac{1-\lambda}{\lambda}b + b - \lambda^{i}\parens*{b - (1-\lambda) a_0}} \\
        &= \lambda \frac{1}{1-\lambda} \parens*{b - \lambda^{i}\parens*{b - (1-\lambda) a_0}} + b.
    \end{align*}
\end{proof}

By applying the above lemma on the logarithm of the number of rows after the $i$th recursive application of row sampling, where $b$ corresponds to $\log(d\poly((\log n) / \eps)))$, we obtain the following algorithm and theorem:

\begin{algorithm}
	\caption{Recursive Huber subspace embedding}
	\textbf{input:} Matrix $\bfA \in \R^{n \times d}$, weights $\bfw$ (defaulted to $\mathbf{1}_n$). \\
	\textbf{output:} Weights $\bfw''$.
	\begin{algorithmic}[1] 
        \State $\beta\gets3-2\sqrt 2$
        \If{$\nnz(\bfw) \leq O(d^{1+\beta}\poly(\log (n/\delta)/\eps))$}
            \State \Return $\bfw$
        \EndIf
        \State Obtain weights $\bfw'$ by running Algorithm \ref{alg:unweighted-huber-subspace-embedding} on $\bfA$ with weights $\bfw$
        \State Obtain weights $\bfw''$ by recursively running Algorithm \ref{alg:huber-subspace-embedding} on $\bfA$ with weights $\bfw'$
        \State \Return $\bfw''$
	\end{algorithmic}\label{alg:huber-subspace-embedding}
\end{algorithm}

\begin{theorem}[Huber Subspace Embedding]\label{thm:huber-subspace-embedding}
    There is an algorithm (Algorithm \ref{alg:huber-subspace-embedding}) which, with probability $1-\delta$, computes a set of weights $\bfw\geq\mathbf{1}_n$ with
    \[
        \nnz(\bfw) \leq O(d^{4-2\sqrt 2}\poly(\log \frac{n}{\delta},\eps^{-1})
    \]
    such that $\norm*{\bfA\bfx}_{H,\bfw}^2 = (1\pm \eps) \norm*{\bfA\bfx}_{H}^2$ for all $\bfx\in\mathbb R^d$, and runs in time
    \[
        O\parens*{\bracks*{\nnz(\bfA)(\log^2 n)(\log\log n) + d^{\omega\cdot(1+1/\beta)}\poly\parens*{\log n,\eps^{-1}}}\log\frac1\delta}.
    \]
\end{theorem}
\begin{proof}
We apply Lemma \ref{thm:n13d-huber-sampling} for $O(\log\log n)$ recursive steps with $\delta = O(1/\log\log n)$ so that we have a sequence of upper bounds $m_i$ on the number of rows such that, for $\beta = 3-2\sqrt 2$,
\[
    m_i = O\parens*{m_{i-1}^{\beta/(\beta+1)}d \poly\parens*{\frac{\log (n/\delta)}{\eps}}}.
\]
Note then that $a_i\coloneqq \log m_i$ satisfies
\[
    a_i = \frac{\beta}{\beta+1} a_{i-1} + b
\]
for
\[
    b = \log\bracks*{O\parens*{d\poly\parens*{\frac{\log (n/\delta)}{\eps}}}}.
\]
The closed form solution for this is given in Lemma \ref{lem:affine-recurrence}. Note that if $i$ satisfies
\[
    \parens*{\frac{\beta}{\beta+1}}^{i}\parens*{b + \frac{1}{\beta+1} a_0} \leq \frac1{\log b}
\]
then
\[
    a_i = (1+\beta) b \pm \frac1{\log b}
\]
and thus the number of rows is at most 
\[
    m_i = \exp(a_i) = O(\exp(b)^{1+\beta}) = O(d^{1+\beta}\poly((\log n)/\eps)) =  O(d^{4-2\sqrt2}\poly((\log n)/\eps)).
\]
This only requires $i$ at most $O(\log(b + a_0)) = O(\log\log n)$ iterations. 
\end{proof}

\subsubsection{Running Time Trade-offs}\label{sec:huber-running-time-opt}

The above running time is quite a large polynomial in $d$; for the current matrix multiplication exponent of $\omega\approx 2.37286$, the exponent of $d$ is
\[
    \omega\cdot \parens*{1+\frac1\beta} \approx  16.20290.
\]
However, by a further multi-step algorithm, we can further reduce the running time significantly, in a similar manner to Theorem \ref{thm:m-active-regression}. First note that by always only computing sensitivities which sum to at most $d^\omega\poly\log n$, we can achieve roughly $\nnz(\bfA) + d^\omega$ time to reduce the number of rows to roughly $n' = d^{\omega(1+\beta)}$ with the same proof, where we solve the recursion of reducing the dimension from $n$ to $n^{\beta/(1+\beta)}d^\omega$ at each step. We can then spend roughly $\nnz(\bfA) + n' d^\omega / d = \nnz(\bfA) + d^{\omega(2+\beta) - 1}$ in Theorem \ref{thm:m-estimator-alg} to fully reduce the dimension to $d^{1+\beta}$, where the exponent in $d$ is now only
\[
    \omega\cdot\parens*{2+\beta} - 1 \approx 4.14663.
\]
This running time may still be undesirable, if one is willing to sacrifice in the dimension reduction bound; for example, this does not beat the running time $\nnz(\bfA) + d^4$ in \cite{GPV2021}, even though the dimension reduction bound is significantly better. If we wish for a running time that is strictly $\nnz(\bfA) + d^\omega$, up to polylogarithmic factors, then the previous argument shows that we obtain a dimension reduction bound of
\[
    \omega\cdot(1+\beta) \approx 2.77998
\]
rows, which indeed slightly improves over their bound of $d^3$ rows in the same running time.

To find an intermediate trade-off by balancing the running time and row count, let $C \in [1, \omega]$ be a parameter. Assume we have already reduced to $n = d^{\omega(1+\beta)}$ rows in $\nnz(\bfA)+d^\omega$ time. Then, we can spend an additional
\[
    \frac{nd^\omega}{d^C} = d^{\omega(2+\beta) - C}
\]
time to reduce to $d^{C(1+\beta)}$ rows. One interesting choice is to balance the running time to be equal to $d$ times the row count, which gives $C = \frac{\omega(2+\beta)-1}{2+\beta} \approx 1.91236$ which allows for reduction to $d^C$ rows in $d^{C+1}$ time.

\subsection{Active Regression Algorithms}\label{sec:active-huber}

As noted before, Theorem \ref{thm:huber-subspace-embedding} gives constant factor solutions in the active regression setting, by Lemma \ref{lem:m-sensitivity-constant-factor-approx}. We now discuss how to build on this result to obtain a relative error solution, using the more refined techniques of Section \ref{subsec:relative-error-lp}.

Our main result of this section is the following:

\ActiveHuber*

\begin{algorithm}
	\caption{Active regression with the Huber loss}
	\textbf{input:} Matrix $\bfA \in \R^{n \times d}$, $\bfb\in\mathbb R^n$, weights $\bfw\geq\mathbf{1}_n$. \\
	\textbf{output:} Approximate solution $\tilde\bfx\in\R$ to $\min_\bfx\norm*{\bfA\bfx-\bfb}_H$.
	\begin{algorithmic}[1] 
        \If{$\nnz(\bfw) \leq d^{4-2\sqrt2}\poly(\log n,\eps^{-1})$}
            \State Compute an approximate solution $\tilde\bfx$ such that $\norm*{\bfA\tilde\bfx-\bfb}_{H,\bfw} \leq (1+\eps)\min_\bfx\norm*{\bfA\bfx-\bfb}_{H,\bfw}$
            \State \Return $\tilde\bfx$
        \EndIf
        \State Let $\bfw'$ be generated by Algorithm \ref{alg:huber-subspace-embedding}
        \State Let $\bfx_c$ be an approximate minimizer $\norm*{\bfA\bfx_c-\bfb}_{H,\bfw'} \leq O(1)\min_\bfx\norm*{\bfA\bfx-\bfb}_{H,\bfw'}$
        \State Set $\bfz\gets\bfb - \bfA\bfx_c$ \Comment{Implicit, for analysis}
        \State Let $\bfw''$ be generated by Algorithm \ref{alg:unweighted-huber-subspace-embedding}\label{line:eps-huber-subspace-embedding}, with possibly larger $\poly(\log(n/\delta)/\eps)$ terms
        \State Recursively compute an approximate solution $\tilde\bfx$ with inputs $(\bfA,\bfz,\bfw'')$ with Algorithm \ref{alg:huber-active}
        \State \Return $\tilde\bfx + \bfx_c$
	\end{algorithmic}\label{alg:huber-active}
\end{algorithm}

We first obtain an analog of Lemmas \ref{lem:clip} and \ref{lem:clip2}, mirroring some of the logic from Lemma \ref{lem:m-active-regression-relative-error-approximation}.

\begin{lemma}\label{lem:huber-clip}
Let $\bfw\geq\mathbf{1}_n$ be a set of weights. Let $\bfA\in\R^{n\times d}$. Let $\bfz\in\mathbb R^n$ be such that $\norm*{\bfz}_{H,\bfw} \leq O(\OPT)$, where
\[
    \OPT \coloneqq \min_{\bfx\in\mathbb R^d}\norm*{\bfA\bfx-\bfz}_{H,\bfw}.
\]
Let $\tilde\bfs_i^{H,\bfw}(\bfA)\geq\bfs_i^{H,\bfw}(\bfA)$ be an upper bound on the Huber sensitivities. Let $\eps>0$ and consider a subset
\[
    \mathcal B\subseteq\braces*{i\in[n] : \bfw_i H(\bfz_i) \geq \frac{\tilde\bfs_i^{H,\bfw}(\bfA)}{\eps^2}\OPT^2}
\]
and let $\bar\bfz\in\mathbb R^n$ be equal to $\bfz$, but with all entries in $\mathcal B$ set to $0$. Then, for all $\bfx\in\mathbb R^d$ with $\norm*{\bfA\bfx}_{H,\bfw} \leq O(\OPT)$,
\[
    \abs*{\norm*{\bfA\bfx-\bar\bfz}_{H,\bfw}^2-\norm*{\bfA\bfx-\bfz}_{H,\bfw}^2 - \norm*{\bfz-\bar\bfz}_{H,\bfw}^2} = O(\eps)\OPT^2.
\]
\end{lemma}
\begin{proof}
    For any $\bfx\in\R^d$ such that $\norm*{\bfA\bfx}_{H,\bfw} \leq O(\OPT)$. Then, for any $i\in\mathcal B$,
    \begin{equation}\label{eq:huber-small-entries}
    \begin{aligned}
        \bfw_i H([\bfA\bfx](i)) &\leq \tilde\bfs_i^{H,\bfw}(\bfA)\norm*{\bfA\bfx}_{H,\bfw}^2 \\
        &\leq \tilde\bfs_i^{H,\bfw}(\bfA)O(\OPT^2) \\
        &\leq O(\eps^2)\cdot\bfw_i H(\bfz_i).
    \end{aligned}
    \end{equation}
    so
    \begin{align*}
        &\bfw_i H([\bfA\bfx](i)-\bfz(i)) - \bfw_i H([\bfA\bfx](i)-\bar\bfz(i)) \\
        =~&\bfw_i H([\bfA\bfx](i)-\bfz(i)) - \bfw_i H([\bfA\bfx](i)) && i\in\mathcal B \\
        =~&\bfw_i\parens*{H^{1/2}(\bfz(i))\pm H^{1/2}([\bfA\bfx](i))}^2 - \bfw_i H([\bfA\bfx](i)) && \text{see Lemma \ref{lem:triangle-inequality-m}} \\
        =~&(1\pm O(\eps))\bfw_i H(\bfz_i) && \text{Equation \eqref{eq:huber-small-entries}}
    \end{align*}
    Thus,
    \begin{align*}
        &\abs*{\norm*{\bfA\bfx-\bfz}_{H,\bfw}^2 - \norm*{\bfA\bfx-\bar\bfz}_{H,\bfw}^2 - \norm*{\bfz-\bar\bfz}_{H,\bfw}^2} \\
        =~&\abs*{\sum_{i\in\mathcal B}(1\pm O(\eps))\bfw_i H(\bfz_i) - \bfw_i H(\bfz_i)} \\
        =~&\abs*{\sum_{i\in\mathcal B} O(\eps)\bfw_i H(\bfz_i)} \\
        =~&O(\eps)\norm*{\bfz-\bar\bfz}_{H,\bfw}^2 \leq O(\eps)\norm*{\bfz}_{H,\bfw}^2 \leq O(\eps)\OPT^2.\qedhere
    \end{align*}
\end{proof}

Next, we show that our Huber subspace embedding algorithm approximately preserves Huber norms $\norm*{\bfA\bfx-\bar\bfz}_{H,\bfw}^2$ for all $\bfx\in\mathbb R^d$ such that $\norm*{\bfA\bfx}_{H,\bfw} = O(\OPT)$. 

\begin{lemma}\label{lem:huber-additive2}
    Consider the setting of Lemma \ref{lem:huber-clip}. Suppose that $\bfw'$ is obtained by performing one step of Algorithm \ref{alg:unweighted-huber-subspace-embedding} (with the $\poly(\log(n\norm*{\bfw}_\infty/\delta)/\eps)$ term in $m$ possibly being a larger polynomial). Then, there exists a $\mathcal B$ satisfying the requirements of Lemma \ref{lem:huber-clip} (see Equation \eqref{eq:huber-large-entries-active}) such that, with probability at least $1-\delta$, for all $\bfx\in\mathbb R^{d}$ with $\norm*{\bfA\bfx}_{H,\bfw} = O(\OPT)$,
    \begin{equation}\label{eq:huber-additive2-guarantee}
        \abs*{\norm*{\bfA\bfx-\bar\bfz}_{H,\bfw'}^2 - \norm*{\bfA\bfx-\bar\bfz}_{H,\bfw}^2} = O(\eps)\cdot\OPT^2.
    \end{equation}
\end{lemma}
\begin{proof}
    We will make a careful modification of our Huber rounding technique of Section \ref{subsec:huber-sampling-bounds}. To adapt the rounding to the active setting, we will need to modify the proof for the \cite{BourgainLindenstraussMilman:1989} net construction rather than using them as a black box. 

    Note that we only need to provide Huber norm preservation guarantees for a single scale of Huber radius $\rho = O(\OPT)$ for this lemma, as opposed to $O(\log n)$ scales for the Huber subspace embedding result. Furthermore, by Lemma \ref{lem:m-norm-net-to-sphere}, it suffices to show the approximation guarantee for every $\bfy\in\mathcal N\cup(\mathcal N-\bar\bfz)$, where $\mathcal N$ is an $O(\eps^2\rho)$-cover of $\mathcal B_\rho^{H,\bfw}$. We have already shown in Theorem \ref{thm:huber-subspace-embedding} the preservation of $\norm*{\bfy}_{H,\bfw}^2$ for every $\bfy$ in $\mathcal N$, so we focus on showing the preservation of $\norm*{\bfy}_{H,\bfw}^2$ for every $\bfy\in\mathcal N-\bar\bfz$, which is similar but requires a slightly more involved argument to handle the $\bar\bfz$ translation.

    We start off with the following claim, which is the active version of Lemma \ref{lem:blm-net-for-huber}.

    \begin{claim}\label{clm:huber-compact-rounding}
        Let $\mathcal I\subset[1,2]$ be the set in the statement of Corollary \ref{cor:generalized-huber-inequality}. Let $\kappa \geq 1$ be a distortion parameter. Suppose that $\tilde\bfs_i^{H}(\bfA)\geq\bfs_i^{H}(\bfA)$ is an upper bound on the Huber sensitivities such that
        \begin{equation}\label{eq:huber-lp-sensitivity-bound}
            \kappa\bfw_i^p(\bfA) \leq \tilde\bfs_i^{H}(\bfA)
        \end{equation}
        for every $p\in\mathcal I$. Let $\Delta>0$ be a scale parameter and let $\bfv\in\mathbb R^n$ be a target vector satisfying
        \[
            H(\bfv(i)) \leq \frac{\tilde\bfs_i^{H}(\bfA)}{\eps^4}\Delta^2
        \]
        for each $i\in[n]$.
        Let $\bfr = \bfy - \bfv\in (\mathcal N - \bfv)$ for for $\mathcal N$ an $O(\eps^2\rho)$-cover of $\mathcal B_\rho^H$ for $\rho = O(\Delta)$. Suppose that $\bfy$ satisfies
        \[
            \frac{\min_{p\in\mathcal I}\norm*{\bfy}_p^p}{\norm*{\bfy}_H^2} \leq \kappa.
        \]
        Then, there exists a $p\in\mathcal I$, $\ell = O((\log d)/\eps^2)$, and a rounding $\bfr' = \bfe + \sum_{k=0}^\ell \bfd_k$ such that:
        \begin{itemize}
            \item $\norm*{\bfr' - \bfr}_H \leq O(\eps)\Delta$
            \item $H\parens*{\abs*{\bfd_k(i)}} \leq O(1)\frac{\tilde\bfs_i^H(\bfA)\cdot(1+\eps^2)^{pk}}{\eps^4\cdot d}\Delta^2$
            \item $\bfe, \bfd_0, \bfd_1, \dots, \bfd_\ell$ have disjoint supports
            \item $\bfe$ is a single fixed vector with $H(\abs{\bfe(i)}) \leq O(1)\frac{\tilde\bfs_i^H(\bfA)}{\eps^4\cdot d}\Delta^2$
            \item each $\bfd_k$ is drawn from a set of vectors $\mathcal D_k$ with size at most
            \[
                \log\abs*{\mathcal D_k} \leq O(1)\frac{d}{\eps^{2(1+p)}(1+\eps^2)^{pk}}\log\frac{n}{\eps}.
            \]
        \end{itemize}
    \end{claim}
    \begin{proof}
        Let $p\in\mathcal I$ be the $p$ achieving the minimum in Corollary \ref{cor:generalized-huber-inequality} for $\bfy$. We will now define $\ell$ nets $\mathcal N_k$ for $k\in[\ell]$ as follows. Recall the following result from \cite{BourgainLindenstraussMilman:1989} and \cite{SchechtmanZvavitch:2001}:

        \begin{lemma}\label{lem:blm-sz-nets}
            Let $\bfA\in\mathbb R^{n\times d}$ and $1 \leq p < \infty$. Let $\bfw\in\mathbb R^{n\times n}$ be the diagonal matrix with
            \[
                \bfW_{i,i} \coloneqq \frac12\parens*{\frac{\bfw_i^p(\bfA)}{d} + \frac1n.}
            \]
            Let $\mathcal B_p \coloneqq \braces*{\bfA\bfx : \norm*{\bfA\bfx}_p\leq 1}$. Then for any $t > 0$, there is a net $\mathcal N_\infty\subseteq\mathbb R^n$ such that, for any $\bfy\in\mathcal B_p$, there exists a $\bfy\in\mathcal N_\infty$ with $\norm*{\bfW^{-1/p}(\bfy-\bfy')}_\infty \leq t$ and
            \[
                \log\abs*{\mathcal N_\infty} \leq c(p) \frac{d}{t^{2\land p}}\log n
            \]
            where $c(p)$ is a constant depending only on $p$ \cite[Corollary 4.7 and Proposition 7.2]{BourgainLindenstraussMilman:1989}.
        \end{lemma} 

        We use the following definition of a sequence of nets from \cite{BourgainLindenstraussMilman:1989}.

        \begin{claim}[Index sets \cite{BourgainLindenstraussMilman:1989}]\label{clm:blm-index-sets}
            Let $\eps' > 0$. For $k\in[\ell]$, let $\mathcal N_k$ be the net from Lemma \ref{lem:blm-sz-nets} with $t = \eps'(1+\eps')^k / 3$. If $t \leq 1$, we take $\mathcal N_k = \mathcal N$. Now for each $\bfy\in\mathcal N$ and $k\in[\ell]$, let $\bff_{k,\bfy}\in\mathcal N_k$ satisfy
            \[
                \norm*{\bfW^{-1/p}(\bff_{k,\bfy}-\bfy)}_\infty \leq \eps'(1+\eps')^k / 3
            \]
            as given by Lemma \ref{lem:blm-sz-nets}. Define the index sets
            \begin{align*}
                C_{k,\bfy} &\coloneqq \braces*{i\in[n] : |\bfW^{-1/p} \bff_{k,\bfy}(i)| \ge (1+\epsilon')^{k-1}} \\
                D_{k,\bfy} &\coloneqq C_{k,\bfy} \setminus \bigcup_{k' > k} C_{k',\bfy} \\
                D_{0,\bfy} &\coloneqq [n] \setminus \bigcup_{k \ge 1} C_{k,\bfy}
            \end{align*}
            Then, for each $k$ we have $\log |\mathcal{N}_k| = O \left (\frac{d\log n}{(\epsilon'(1+\epsilon')^{k})^{2\land p}} \right )$, and for every $i\in D_{k,\bfy}$, we have
            \begin{align}\label{yi_bound}
            \frac{\bfw_i^p(\bfA)^{1/p} \cdot (1+\epsilon')^{k-2}}{d^{1/p}} \le |\bfy(i)| \le \frac{\bfw_i^p(\bfA)^{1/p} \cdot (1+\epsilon')^{k+2}}{d^{1/p}}.
            \end{align}
        \end{claim}
        \begin{proof}
            Note that since $\norm{\bfW^{-1/p}(\bff_{k,\bfy}-\bfy)}_\infty \le \epsilon'(1+\epsilon')^k/3$, if $i \in C_{k,\bfy}$ for $1 \le k \le \ell$,
            \begin{align*}
            |\bfy(i)| &\ge |\bff_{k,\bfy}(i)| - \frac{\bfw^p_i(\bfA)^{1/p} \cdot \epsilon' (1+\epsilon')^k/3}{d^{1/p}}\\
            &\ge \frac{\bfw^p_i(\bfA)^{1/p}}{d^{1/p}} \cdot \left [(1+\epsilon')^{k-1} - \epsilon' (1+\epsilon')^k/3 \right ]\tag{By definition, if $i \in C_{k,\bfy}$, $|\bff_{k,\bfy}(i)| \ge \frac{\bfw^p_i(\bfA)^{1/p} \cdot (1+\epsilon')^{k-1}}{d^{1/p}}$.}\\
            &\ge \frac{\bfw^p_i(\bfA)^{1/p} \cdot (1+\epsilon')^{k-2}}{d^{1/p}}.
            \end{align*}
            If $i \notin C_{k,\bfy}$ we analogously have $|\bfy(i)| \le \frac{\bfw_i^p(\bfA)^{1/p} \cdot (1+\epsilon')^{k+2}}{d^{1/p}}$.
            Thus, for $i \in D_{k,\bfy} = C_{k,\bfy} \setminus \bigcup_{k' > k} C_{k',\bfy}$ for $1 \le k < \ell$, Equation \eqref{yi_bound} holds. Equation \eqref{yi_bound} also holds for $k = \ell$ via Lemma \ref{lem:lwBound} -- since $\norm{\bfy}_p \le 1$, $|\bfy(i)| \le d^{\max(0,1/2-1/p)} \cdot \bfw_i^p(\bfA)^{1/p}$ for all $i$. And so vacuously, since for $\ell = \log_{1+\epsilon'} d^{\max(1/p,1/2 )}$, $(1+\epsilon')^{\ell+2} \ge d^{\max(1/p,1/2)}$, we have $|\bfy(i)| \le \frac{\bfw_i^p(\bfA)^{1/p} \cdot (1+\epsilon')^{\ell + 2}}{d^{1/p}}$.
        \end{proof}

        We now apply Claim \ref{clm:blm-index-sets} with $\eps' = \eps^2$ to obtain the nets $\mathcal N_k$. We then define index sets for the target vector $\bfv$ as follows:

        \begin{claim}[Huber index sets for $\bfv$]\label{clm:huber-partition}
            Define the following sets:
        \begin{align*}
            B_{k,\bfy} &\coloneqq \braces*{i \in D_{k,\bfy}: H(\bfv(i)) \le \frac{\tilde\bfs_i^H(\bfA) \cdot (1+\epsilon^2)^{p(k+2)}}{\epsilon^4 \cdot d}\Delta^2} \tag{$k\in[\ell]\cup\{0\}$} \\
            H_k &\coloneqq \braces*{i\in[n]: \frac{\tilde\bfs_i^H(\bfA)\cdot (1+\epsilon^2)^{p(k+1)}}{\epsilon^4 \cdot d}\Delta^2 < H(\bfv(i))\le \frac{\tilde\bfs_i^H(\bfA) \cdot (1+\epsilon^2)^{p(k+2)}}{\epsilon^4 \cdot d}\Delta^2 } \tag{$k\in[\ell]$} \\
            H_B &\coloneqq \braces*{ i\in[n]: H(\bfv(i)) > \frac{\tilde\bfs_i^H(\bfA) \cdot (1+\epsilon^2)^{p(\ell+2)}}{\epsilon^4 \cdot d}\Delta^2 } \\
            G_{k,\bfy} &\coloneqq H_{k} \setminus \bigcup_{k' \ge k} C_{k',\bfy} \tag{$k\in[\ell]$}
        \end{align*}
        Then, $B_{0,\bfy},\ldots,B_{\ell,\bfy},G_{1,\bfy},\ldots,G_{\ell,\bfy}, H_B$ form a partition of $[n]$.
        \end{claim}
        \begin{proof}
        Define $H_{0} = \{i: H(\bfv(i)) \le \frac{\tilde\bfs_i^H(\bfA) \cdot (1+\epsilon^2)^{2p}}{\epsilon^4 \cdot d}\Delta^2 \}$, so that $H_B \cup \bigcup_{k=0}^\ell H_{k} = [n]$ is a partition the range of all possible $H(\bfv(i))$ values. Furthermore, every $i\in H_0$ lies in some $B_{k,\bfy}$. 
        
        At a high level, our argument will be as follows. $H_k$ can be furthered partitioned into the indices that lie in some $\bigcup_{k'\geq k}C_{k,\bfy}$, or those that are not in any $\bigcup_{k'\geq k} C_{k',\bfy}$. Those that lie in some $\bigcup_{k'\geq k}C_{k,\bfy}$ are then regrouped into the $B_{k,\bfy}$, which is also a valid partition by definition of the $D_{k,\bfy}$.

        Any $i \in H_{k} \cap \left (\bigcup_{k' \ge k} C_{k',\bfy} \right)$ has $H(\bfv(i))\le \frac{\tilde\bfs_i^H(\bfA) \cdot (1+\epsilon^2)^{p(k+2)}}{\epsilon^4 \cdot d}\Delta^2$ and thus must lie in $B_{k',\bfy}$ for some $k' \ge k$. 
        Thus, $G_{k,\bfy} \cup (\bigcup_{k' \ge k} B_{k',\bfy} ) \supseteq H_{k}$ and so overall $(\bigcup_{k=1}^\ell G_{k,\bfy} \cup B_{k,\bfy}) \cup B_{0,\bfy} \cup H_B = [n]$.

        Now clearly, the $H_{k}$ and in turn the $G_{k,\bfy}$ are disjoint. $H_B$ is disjoint from all $H_k$ and hence all $G_{k,\bfy}$. It is also disjoint from all $B_{k,\bfy}$ since the $H(|\bfv(i)|)$ values in $H_B$ are too large to be assigned to $B_{k,\bfy}$, even for $k = \ell$. The $D_{k,\bfy}$, and in turn the $B_{k,\bfy}$, are also disjoint by construction. Further, $\bigcup_{k' \ge k} B_{k',\bfy} \subseteq \bigcup_{k' \ge k} D_{k',\bfy} = \bigcup_{k' \ge k} C_{k',\bfy}$, and thus $G_{k,\bfy}$ is disjoint from $B_{k',\bfy}$ for all $k' \ge k$. Also, for $k \ge 1$, being in $G_{k,\bfy}$ requires $H(|\bfv(i)|) > \frac{\tilde\bfs_i^H(\bfA)\cdot (1+\epsilon^2)^{p(k+1)}}{\epsilon^4 \cdot d}\Delta^2$, while being in $B_{k',\bfy}$ for $k' < k$ requires $H(\bfv(i)) \le \frac{\tilde\bfs_i^H(\bfA) \cdot (1+\epsilon^2)^{p(k+1)}}{\epsilon^4 \cdot d}\Delta^2$. Thus $G_{k,\bfy}$ is disjoint from $B_{k',\bfy}$ for all $k' < k$. Overall, $G_{k,\bfy}$ is disjoint from all $B_{k',\bfy}$. Thus, $B_{0,\bfy},\ldots,B_{\ell,\bfy},G_{1,\bfy},\ldots,G_{\ell,\bfy}, H_B$ are all mutually disjoint and partition $[n]$.
        \end{proof}

        Using the index sets of Claim \ref{clm:huber-partition}, we next show how to round $\bfr = \bfy-\bfv$ to a nearby $\bfr'$ using the partition of $[n]$ defined above.

        \begin{claim}[$\ell_\infty$ Error Bound]\label{clm:infty}
            Define the vectors $\bfr'$, $\bfe$, and $\bfd_k$ with
            \[
            \bfr' = \bfe + \sum_{k=0}^\ell \bfd_k
            \]
            as follows:
            \begin{equation}\label{eq:active-rounding-definition}
            \begin{aligned}
            \bfd_0(i) &\coloneqq \bfy(i)-\bfv(i) && i\in B_{0,\bfy} \\
            \bfd_k(i) &\coloneqq \frac{(1+\epsilon^2)^k \cdot \bfw_i^p(\bfA)^{1/p}}{d^{1/p}} - \bfv(i) && k\in[\ell], i\in B_{{k,\bfy}} \\
            \bfd_k(i) &\coloneqq \bfv(i) && i\in G_{k,\bfy} \\
            \bfd_k(i) &\coloneqq 0 && \text{otherwise} \\
            \bfe_i(i) &\coloneqq \bfv(i) && i\in H_B \\
            \bfe_i(i) &\coloneqq 0 && \text{otherwise}
            \end{aligned}
            \end{equation}
            Then
            $$|\bfr(i) - \bfr'(i)| \le \epsilon^2( |\bfy(i)| + |\bfv(i)|).$$
            \end{claim}
            \begin{proof}
            
            For $i \in B_{0,\bfy}$ we have $\bfr(i) = \bfr'(i)$ so the claim trivially holds. For $i \in B_{k,\bfy}$ for $k \ge 1$, by \eqref{yi_bound},
            \[
                \left |\bfy(i) - \frac{(1+\epsilon^2)^k \cdot \bfw_i^p(\bfA)^{1/p}}{d^{1/p}} \right | \le O(\epsilon^2) \cdot |\bfy(i)|.
            \]
            Thus by triangle inequality, $|\bfr(i) - \bfr'(i)| \le O(\epsilon^2) \cdot |\bfy(i)|$. So the claim holds after adjusting constants on $\epsilon$. 
            Finally, for $i\in G_{k,\bfy}$ we have
            \begin{align*}
                H(\abs*{\bfy(i)}) &\leq \abs*{\bfy(i)}^p \tag{$p\in[1,2]$} \\
                &\leq \frac{\bfw_i^p(\bfA)\cdot (1+\eps^2)^{p(k+2)}}{d}\norm*{\bfy}_p^p \tag{$i\notin \bigcup_{k'\geq k}C_{k',\bfy}$} \\
                &\leq \kappa\frac{\bfw_i^p(\bfA)\cdot (1+\eps^2)^{p(k+2)}}{d}\norm*{\bfy}_H^2\\
                &\leq \frac{\tilde\bfs_i^H(\bfA)\cdot (1+\eps^2)^{p(k+2)}}{d}\cdot O(\Delta^2) \tag{Equation \eqref{eq:huber-lp-sensitivity-bound}} \\
                &\leq O(\eps^4)H(\bfv(i)) \tag{$i\in H_k$}
            \end{align*}
            and similarly for $i\in H_B$.
            \end{proof}

        We then have that
        \begin{align*}
            \sum_{i=1}^n H(\abs*{\bfr'(i)-\bfr(i)}) &\leq \sum_{i=1}^n  H(\eps^2(\abs{\bfy(i)} + \abs*{\bfv(i)})) \tag{monotonicity} \\
            &\leq O(\eps^2)\sum_{i=1}^n H(\abs{\bfy(i)} + \abs*{\bfv(i)}) \tag{at least linear growth} \\
            &= O(\eps^2)(\norm*{\bfy}_{H} + \norm*{\bfv}_{H})^2 = O(\eps^2)\Delta^2 \tag{triangle inequality}
        \end{align*}
        so $\norm*{\bfr'-\bfr}_{H} \leq O(\eps)\Delta$. The disjointness of the $\bfe$ and $\bfd_k$ as well as the net size bounds carry over directly. Finally, the bounds on the coordinates of $\bfd_k$ follow from
        \begin{align*}
            H(\abs{\bfd_k(i)}) &\leq \abs*{\bfd_k(i)}^p \tag{$p\in[1,2]$} \\
            &\leq O(1)(\abs{\bfy(i)}^p + \abs{\bfv(i)}^p) \\
            &\leq O(1)\bracks*{\frac{\bfw_i^p(\bfA)}{d}(1+\eps^2)^{pk}\norm*{\bfy}_p^p + \frac{\tilde\bfs_i^H(\bfA)\cdot(1+\eps^2)^{p(k+2)}}{\eps^4\cdot d}\Delta^2} \tag{$i\in B_{k,\bfy}$} \\
            &\leq O(1)\bracks*{\kappa\frac{\bfw_i^p(\bfA)}{d}(1+\eps^2)^{pk}\norm*{\bfy}_H^2 + \frac{\tilde\bfs_i^H(\bfA)\cdot(1+\eps^2)^{p(k+2)}}{\eps^4\cdot d}\Delta^2} \\
            &= O(1)\frac{\tilde\bfs_i^H(\bfA)\cdot(1+\eps^2)^{pk}}{\eps^4\cdot d}\Delta^2
        \end{align*}
        for $i\in B_{k,\bfy}$, and similarly for other cases as well as for $\bfe$.

        This completes the proof of Claim \ref{clm:huber-compact-rounding}.
    \end{proof}

    We return to proving Lemma \ref{lem:huber-additive2}. Recall that we reduced our task to proving the approximation guarantee Equation \eqref{eq:huber-additive2-guarantee} for all vectors $\bfr'$ belonging to an $O(\eps^2\rho)$-cover over the set $\mathcal B_\rho^H - \bar\bfz$. We will do so using the cover given by Claim \ref{clm:huber-compact-rounding}, with $\eps$ set to $\eps^2$. Furthermore, we will set the Huber sensitivity upper bounds $\tilde\bfs_i^H(\bfA)$ to be at least $\kappa\sum_{p\in\mathcal I} \tilde\bfw_i^p(\bfA)$. This indeed satisfies the hypothesis \eqref{eq:huber-lp-sensitivity-bound} of Claim \ref{clm:huber-compact-rounding}. Then, performing a Bernstein bound gives the active version of Lemma \ref{lem:huber-bernstein}. 

    \begin{claim}\label{clm:huber-bernstein-active}
        Consider the setting of Claim \ref{clm:huber-compact-rounding}. Let $m = d\poly((\log(n/\delta))/\eps)$. Let $\bfw\geq \mathbf{1}_n$ be a set of weights such that  
        \[
            \frac{\max_{i\in[n]} \bfw_i}{\min_{i\in[n]} \bfw_i}\leq 2
        \]
        and let $w = \min_{i\in[n]}\bfw_i$. Let $\bfw'$ be chosen randomly so that
        \[
            \bfw_i' \coloneqq \begin{cases}
                \bfw_i/\bfp_i & \text{w.p. $\bfp_i$} \\
                0 & \text{otherwise}
            \end{cases}
        \]
        where
        \[
            \bfp_i \geq \min\braces*{1, m\cdot \frac{\tilde\bfs_i^H(\bfA)}{d}}.
        \]
        Let
        \[
            S_\kappa \coloneqq \braces*{\bfy\in \mathcal B_\rho^H-\bfv : \frac{\min_{p\in\mathcal I}\norm*{\bfy}_p^p}{\norm*{\bfy}_H^2}\leq \kappa}
        \]
        Then with probability at least $1 - \delta$, 
        \[
            \norm*{\bfy}_{H,\bfw} = (1\pm\eps)\norm*{\bfy}_H
        \]
        for all $\bfy\in \mathcal N\cap S_\kappa$, where $\mathcal N$ is an $\eps^2\rho$-net over $\mathcal B_\rho^H-\bfv$ given by Claim \ref{clm:huber-compact-rounding}, by setting $\eps$ in the lemma to $\eps^2$.
    \end{claim}
    \begin{proof}
        The proof follows that of Lemma \ref{lem:huber-bernstein} almost exactly, with minimal changes. The only necessary change is that the coordinate upper bounds of Claim \ref{clm:huber-compact-rounding} are now proportional to $\tilde\bfs_i^H(\bfA)$ rather than $\bfw_i^p(\bfA)$. However, this is not a problem, since we are sampling with these probabilities, and thus does not change the proof. Note also that the sensitivity upper bounds incur extra $\poly(\eps)$ factors, which changes the $m$ by $\poly(\eps)$ factors.
    \end{proof}

    Next, we use Claim \ref{clm:huber-bernstein-active} to obtain an active version of Lemma \ref{lem:huber-single-scale}.

    \begin{claim}\label{clm:huber-single-scale-active}
        Consider the setting of Claim \ref{clm:huber-bernstein-active}. Let $\bfz\in\mathbb R^n$ with $\norm*{\bfz}_H = O(\Delta)$. Let $\gamma$ be as defined in Lemma \ref{lem:huber-single-scale} and let $\kappa = 1/\gamma$. Let $\tilde\bfs_i^H(\bfA)$ be a Huber sensitivity upper bound and let
        \[
            T \subseteq \braces*{i\in[n] : \tilde\bfs_i^H(\bfA) \leq \gamma}.
        \]
        We now define a second set of Huber sensitivity upper bounds $\overline{\bfs_i}^H(\bfA)$ via
        \[
            \overline{\bfs_i}^H(\bfA) \coloneqq \begin{cases}
                1 & \text{if $i\in \overline T$} \\
                \kappa\sum_{p\in\mathcal I} \tilde\bfw_i^p(\bfA) + \tilde\bfw_i^p(\bfA\vert_T) & \text{if $i\in T$}
            \end{cases}
        \]
        Then consider a superset
        \[
            \mathcal B \supseteq \braces*{i\in[n] : H(\bfz_i) \geq \frac{\overline{\bfs_i}^H(\bfA)}{\eps^4}\Delta^2}
        \]
        and let $\bar\bfz\in\mathbb R^n$ be equal to $\bfz$, but with all entries in $\mathcal B$ set to $0$. Let $\bfp$ be the sampling probabilities given by Lemma \ref{lem:huber-bernstein} for $\bfA$ and let $\bfq$ be the sampling probabilities given by Claim \ref{clm:huber-bernstein-active} for $\bfA\mid_T$. We then consider sampling probabilities $\bfr$ such that
        \[
            \bfr_i \coloneqq \begin{cases}
                1 & \text{if $i\in \overline T$} \\
                \min\braces*{\bfp_i + \bfq_i, 1} & \text{if $i\in T$}
            \end{cases}
        \]
        Let $\bfw\geq \mathbf{1}_n$ be a set of weights such that  
        \[
            \frac{\max_{i\in[n]} \bfw_i}{\min_{i\in[n]} \bfw_i}\leq 2
        \]
        and let $w = \min_{i\in[n]}\bfw_i$. Let $\bfw'$ be chosen randomly so that
        \[
            \bfw_i' \coloneqq \begin{cases}
                \bfw_i/\bfr_i & \text{w.p. $\bfr_i$} \\
                0 & \text{otherwise}
            \end{cases}
        \]
        Then,
        \[
            \norm*{\bfy}_{H,\bfw'} = \norm*{\bfy}_{H,\bfw} \pm \eps w\rho.
        \]
        for all $\bfy\in\mathcal B_\rho^H - \bar\bfz$ with probability at least $1 - \delta$. Furthermore,
        \begin{align*}
            \E\nnz(\bfw') &= O(\kappa m\log n) = O\parens*{\kappa d\poly\parens*{\frac{\log n}{\eps}}\log\frac1\delta} \\
            \nnz(\bfw') &= O\parens*{\E\nnz(\bfw')}
        \end{align*}
        with probability at least $1 - \delta$, and $\norm*{\bfw'}_\infty \leq O(n)\norm*{\bfw}_\infty$.
    \end{claim}
    \begin{proof}
        We may check that $\bar\bfz$ satisfies the hypothesis for $\bfv$ in Claim \ref{clm:huber-compact-rounding}, both for $\bfA$ and for $\bfA\vert_T$. After noticing this fact, we simply follow the proof of Lemma \ref{lem:huber-single-scale}, combined with Claim \ref{clm:huber-bernstein-active}. 
    \end{proof}
    
    Finally, we adapt Lemma \ref{lem:huber-single-step} to the following claim:

    \begin{claim}\label{clm:huber-single-step-active}
        Let $\bfw\geq\mathbf{1}_n$ be a set of weights and for each $j\in[\ceil*{\log_2(\norm*{\bfw}_\infty)}+1]$ define the sets
        \[
            T_j \coloneqq \braces*{i\in[n] : \bfw_i \in [2^{j-1}, 2^j)}.
        \]
        Let $\bfz\in\mathbb R^n$ with $\norm*{\bfz}_{H,\bfw} = O(\OPT)$. Now define a set of weighted Huber sensitivity upper bounds $\tilde\bfs_i^{H,\bfw}(\bfA)$ via
        \[
            \tilde\bfs_i^{H,\bfw}(\bfA) \coloneqq 2\cdot\overline{\bfs_i}^H(\bfA\vert_{T_j})
        \]
        for each $i\in T_j$, where $\overline{\bfs_i}^H(\bfA\vert_{T_j})$ is as defined in Claim \ref{clm:huber-single-scale-active}. Then consider the set
        \begin{equation}\label{eq:huber-large-entries-active}
            \mathcal B = \braces*{i\in[n] : \bfw_i H(\bfz_i) \geq \frac{\tilde\bfs_i^{H,\bfw}(\bfA)}{\eps^4}\OPT^2}
        \end{equation}
        and let $\bar\bfz\in\mathbb R^n$ be equal to $\bfz$, but with all entries in $\mathcal B$ set to $0$. Let $\bfw'$ be obtained by applying Claim \ref{clm:huber-single-scale-active} on each $\bfA\vert_{T_j}$ with weights $\bfw\vert_{T_j}$. Then,
        \[
            \norm*{\bfy}_{H,\bfw'} = \norm*{\bfy}_{H,\bfw} \pm \eps\cdot\OPT
        \]
        for all $\bfy\in\mathcal B_\rho^H - \bar\bfz$ with probability at least $1 - \delta$. Furthermore, for $\beta = 3-2\sqrt 2$ and $\kappa = n^{\beta/(\beta+1)}$,
        \begin{align*}
            \E\nnz(\bfw') &= O\parens*{\kappa d\poly\parens*{\frac{\log(n\norm*{\bfw}_\infty/\delta)}{\eps}}} \\
            \nnz(\bfw') &= O\parens*{\E\nnz(\bfw')}
        \end{align*}
        with probability at least $1 - \delta$. 
    \end{claim}
    \begin{proof}
        First note that $\tilde\bfs_i^{H,\bfw}(\bfA)$ are valid sensitivity upper bounds, since for any $\bfx\in\mathbb R^d$,
        \[
            \frac{\bfw_i H([\bfA\bfx](i))}{\norm*{\bfA\bfx}_{H,\bfw}^2} \leq \frac{\bfw_i H([\bfA\bfx](i))}{2^{j-1}\norm*{\bfA\vert_{T_j}\bfx}_{H}^2} \leq 2\frac{H([\bfA\bfx](i))}{\norm*{\bfA\vert_{T_j}\bfx}_{H}^2} \leq 2\cdot\overline{\bfs_i}^H(\bfA\vert_{T_j}).
        \]
        Since $\bfw_i \geq1$, we have that
        \[
            \braces*{i\in T_j : H(\bfz_i) \geq \frac{\overline{\bfs_i}^{H}(\bfA\vert_{T_j})}{\eps^4}\OPT^2} \subseteq \mathcal B
        \]
        for each $j$. Thus, restricting $\bar\bfz$ to $T_j$ satisfies the hypotheses of Claim \ref{clm:huber-single-scale-active}. We then set $\Delta = \OPT/2^{j-1}$ and apply Claim \ref{clm:huber-single-scale-active} so that for all $\bfy\in\mathcal B_\rho^{H,\bfw}-\bar\bfz$ with $\rho = O(\OPT)$,
        \[
            \norm*{\bfy\vert_{T_j}}_{H,\bfw'} = \norm*{\bfy\vert_{T_j}}_{H,\bfw} \pm \eps \cdot\OPT
        \]
        or
        \[
            \norm*{\bfy\vert_{T_j}}_{H,\bfw'}^2 = \norm*{\bfy\vert_{T_j}}_{H,\bfw}^2 + \eps^2\cdot\OPT^2\pm 2\eps\cdot\OPT\norm*{\bfy\vert_{T_j}}_{H,\bfw}.
        \]
        Union bounding and summing over the $j$ gives us that
        \[
            \norm*{\bfy}_{H,\bfw'}^2 = \norm*{\bfy}_{H,\bfw}^2 \pm \Theta(\eps\log\norm*{\bfw}_\infty)\OPT^2
        \]
        with probability at least $1 - \Theta(\delta\log\norm*{\bfw}_\infty)$. Rescaling $\eps$ and $\delta$ by $\Theta(\log\norm*{\bfw}_\infty)$ gives the claim.
    \end{proof}
    
    This completes the proof of Lemma \ref{lem:huber-additive2}.
\end{proof}

Given Lemma \ref{lem:huber-additive2}, we now prove the active regression guarantees for a single step of the recursion.

\begin{lemma}\label{lem:active-huber-single-step}
    Consider the setting of Lemma \ref{lem:huber-additive2}. Let $\bar\bfx$ satisfy
    \[
        \norm*{\bfA\bar\bfx-\bfz}_{H,\bfw'} \leq (1+\eps)\min_{\bfx}\norm*{\bfA\bfx-\bfz}_{H,\bfw'}.
    \]
    Then, with probability at least $1-\delta$,
    \[
        \norm*{\bfA\bar\bfx-\bfz}_{H,\bfw} \leq (1+O(\eps/\delta))\min_{\bfx}\norm*{\bfA\bfx-\bfz}_{H,\bfw}.
    \]
\end{lemma}
\begin{proof}
    We follow the proof of Theorem \ref{thm:eps}. By Markov's inequality, we have that
    \[
        \norm*{\bfz}_{H,\bfw'}^2 = O(1/\delta)\norm*{\bfz}_{H,\bfw}^2
    \]
    with probability at least $1-\delta$. Let $\bfA\bfx$ have $\norm*{\bfA\bfx}_{H,\bfw} = O(\OPT)$. Let $\bar\bfz$ be as defined in Lemma \ref{lem:huber-additive2}. 
    Lemma \ref{lem:huber-clip}, we then have that 
    \begin{align*}
        &\abs*{\norm*{\bfA\bfx-\bar\bfz}_{H,\bfw}^2-\norm*{\bfA\bfx-\bfz}_{H,\bfw}^2 - \norm*{\bfz-\bar\bfz}_{H,\bfw}^2} = O(\eps)\OPT^2 \\
        &\abs*{\norm*{\bfA\bfx-\bar\bfz}_{H,\bfw'}^2-\norm*{\bfA\bfx-\bfz}_{H,\bfw'}^2 - \norm*{\bfz-\bar\bfz}_{H,\bfw'}^2} = O(\eps/\delta)\OPT^2.
    \end{align*}
    Furthermore, Lemma \ref{lem:huber-additive2} guarantees that
    \[
        \abs*{\norm*{\bfA\bfx-\bar\bfz}_{H,\bfw'}^2 - \norm*{\bfA\bfx-\bar\bfz}_{H,\bfw}^2} = O(\eps)\cdot\OPT^2.
    \]
    Then by the triangle inequality,
    \[
        \abs*{\norm*{\bfA\bfx-\bfz}_{H,\bfw'}^2-\norm*{\bfA\bfx-\bfz}_{H,\bfw}^2 + (\norm*{\bfz-\bar\bfz}_{H,\bfw'}^2 - \norm*{\bfz-\bar\bfz}_{H,\bfw}^2)} \leq O(\eps/\delta)\OPT^2.
    \]
    Let $C\coloneqq \norm*{\bfz-\bar\bfz}_{H,\bfw'}^2 - \norm*{\bfz-\bar\bfz}_{H,\bfw}^2$. Note that $\abs*{C} \leq O(1/\delta)\OPT^2$. Let $\bfx^*\coloneqq \arg\min_\bfx\norm*{\bfA\bfx-\bfz}_{H,\bfw}$. Then,
    \begin{align*}
        \norm*{\bfA\bar\bfx-\bfz}_{H,\bfw}^2 &\leq \norm*{\bfA\bar\bfx-\bfz}_{H,\bfw'}^2+C + O(\eps/\delta)\OPT^2 \\
        &\leq (1+\eps)\norm*{\bfA\bfx^*-\bfz}_{H,\bfw}^2 + C + O(\eps/\delta)\OPT^2 \\
        &\leq (1+\eps)(\norm*{\bfA\bfx^*-\bfz}_{H,\bfw}^2 - C) + C + O(\eps/\delta)\OPT^2 \\
        &\leq \norm*{\bfA\bfx^*-\bfz}_{H,\bfw}^2+ O(\eps/\delta)\OPT^2.\qedhere
    \end{align*}
\end{proof}

We can now prove Theorem \ref{thm:active-huber}.

\begin{proof}[\textbf{Proof of Theorem \ref{thm:active-huber}}]
By an analysis identical to Theorem \ref{thm:huber-subspace-embedding}, the algorithm makes at most $O(\log\log n)$ recursive calls. We then scale the failure rate $\delta$ and accuracy $\eps$ in \ref{lem:active-huber-single-step} by $\Theta(1/\log\log n)$ so that we can union bound over $O(\log\log n)$ recursive calls, and so that the error incurred at each recursive call is at most $(1+O(\eps/\log\log n))$, with probability at least $99/100$. The total error over the $O(\log\log n)$ recursive calls is then at most
\[
    (1+O(\eps/\log\log n))^{O(\log\log n)} \leq 1+O(\eps)
\]
so we conclude as desired.
\end{proof}

\section{Applications: Dimension Reduction for Gamma Functions}\label{sec:gamma}

The work of \cite{BCLL2018} introduced the $\gamma_p$ functions, which are also known as smoothed $p$-norms and mixed $\ell_2$-$\ell_p$ norms:

\begin{definition}[$\gamma_p$ functions, Definition 3.1, \cite{AKPS2019}]
For $t\geq 0$ and $p\geq 1$, define
\[
    \gamma_p(t, x) = \begin{cases}
        \frac{p}{2} t^{p-2} x^2 & \text{if $\abs{x}\leq t$} \\
        \abs*{x}^p + \parens*{\frac{p}{2}-1}t^p & \text{if $\abs{x}>t$}
    \end{cases}
\]
\end{definition}

The $\gamma_p$ functions have found many applications in fast algorithms for $\ell_p$ regression and related problems \cite{BCLL2018, AKPS2019, ABKS2021, GPV2021}. Recall also the simpler $\ell_2$-$\ell_p$ loss defined in Lemma \ref{lem:l2-lq-inequality}, which is simply $M_p(x) = \min\{\abs{x}^2,\abs{x}^p\}$. Now note that $\gamma_p(t, x) = \Theta(t^{p}) M_p(x/t)$. Thus, for the purpose of bounding the associated sensitivities up to constant factors, we may replace our discussion of $\gamma_p$ by $M_p$, for simplicity.

We now state our result for dimension reduction for $\gamma_p$ functions, which is analogous to our Huber subspace embedding result in Theorem \ref{thm:huber-subspace-embedding}.

\begin{definition}
For $p\geq 1$, thresholds $\bft\geq \mathbf{1}_n$, and weights $\bfw\geq\mathbf{1}_n$, define the weighted $\gamma_p$-norm of $\bfy\in\mathbb R^n$ as
\[
    \norm*{\bfy}_{\gamma_p(\bft,\cdot), \bfw} \coloneqq \bracks*{\sum_{i=1}^n \bfw_i \cdot \gamma_p(\bft_i, \bfy_i)}^{1/p}.
\]
If $\bfw = \mathbf{1}_n$. We simply write $\norm*{\bfy}_{\gamma_p(\bft,\cdot)}$. Note that while we refer to this as a norm by abuse of notation, it does not satisfy the properties of a norm.
\end{definition}

\begin{theorem}[$\gamma_p$ Subspace Embedding]
    Let $p\in(1,2)$ and let $\beta$ be as defined in Lemma \ref{lem:l2-lq-inequality}. Let $\bft\geq\mathbf{1}_n$. There is an algorithm which, with probability $1-\delta$, computes a set of weights $\bfw\geq\mathbf{1}_n$ with
    \[
        \nnz(\bfw) \leq O\parens*{d^{1+\beta}\poly\parens*{\log \frac{n}{\delta},\eps^{-1},\log\norm*{\bft}_\infty}}
    \]
    such that $\norm*{\bfA\bfx}_{\gamma_p(\bft,\cdot), \bfw}^2 = (1\pm \eps) \norm*{\bfA\bfx}_{\gamma_p(\bft,\cdot)}^2$ for all $\bfx\in\mathbb R^d$, and runs in time
    \[
        O\parens*{\bracks*{\nnz(\bfA) + d^{\omega\cdot(1+1/\beta)}}\poly\parens*{\log n,\log\frac1\delta,\eps^{-1},\log\norm*{\bft}_\infty}}.
    \]
\end{theorem}
\begin{proof}
Because our proof is a simple extension of our Huber subspace embedding proof (Theorem \ref{thm:huber-subspace-embedding}), we only outline the proof and necessary changes from the Huber result. 

We first bucket the coordinates of $\bft$ into $O(\log\norm*{\bft}_\infty)$ buckets $\braces*{i\in[n] : 2^{j-1} \leq \bft_i < 2^j}$. Note that in typical applications, $\norm*{\bft}_\infty$ is roughly $\poly(n)$ (see, e.g., the paragraph following Lemma 5.5 of \cite{AKPS2019}). A similar bucketing trick is used in \cite{GPV2021}. Once we have done this, the $\gamma_p$ functions are within constant factors of each other, and thus we may focus on sensitivity sampling with a fixed threshold $t$, which we WLOG assume is $1$, and just refer to the corresponding loss as $\gamma_p$. Note that even with a constant factor variation in the $\gamma_p(t_i,\cdot)$ function, we can still get $(1+\eps)$ approximations around the mean by sampling (see, e.g., how we handle weights differing by constant factors in Lemma \ref{lem:huber-bernstein}).

We first go over one step of the recursive sampling procedure, which reduces the number of rows from roughly $n$ to $n^{\beta/(1+\beta)}d$. Let $\gamma$ satisfy $1/\gamma = (\gamma n)^\beta$, or $\gamma = n^{-\beta/(1+\beta)}$. Now consider the coordinates $i\in[n]$ with $\gamma_p$-sensitivities at least $\gamma$, that is,
\[
    S\coloneqq \braces*{i\in[n] : \bfs_i^{\gamma_p}(\bfA) \geq \gamma}
\]
where
\[
    \bfs_i^{\gamma_p}(\bfA) \coloneqq \sup_{\bfx\in\mathbb R^d} \frac{\gamma_p([\bfA\bfx](i))}{\norm*{\bfA\bfx}_{\gamma_p}^p} = \sup_{\bfx\in\mathbb R^d} \frac{\gamma_p([\bfA\bfx](i))}{\sum_{i'=1}^n \gamma_p([\bfA\bfx](i'))} = \sup_{\bfx\in\mathbb R^d} \frac{\gamma_p(1, [\bfA\bfx](i))}{\sum_{i'=1}^n \gamma_p(1, [\bfA\bfx](i'))}. 
\]
As done in Lemma \ref{thm:n13d-huber-sampling}, we may find a superset $S'\supseteq S$ of such coordinates in time at most
\[
    O\parens*{\bracks*{\nnz(\bfA) + n^{1/(1+\beta)}d^\omega}\poly\parens*{\log n,\log\frac1\delta}} \leq O\parens*{\bracks*{\nnz(\bfA) + d^{\omega\cdot(1+1/\beta)}}\poly\parens*{\log n,\log\frac1\delta}}
\]
Since the sensitivity upper bounds sum to at most $O(d\poly\log(n/\delta))$, we have that \[
    \abs*{S'} \leq O(\gamma^{-1}d\poly\log(n/\delta)).
\]
Now by Lemma \ref{lem:l2-lq-inequality}, any $\bfy\in\mathbb R^n$ either has distortion at most $\gamma$ in the ratio between $\norm*{\bfy}_{\gamma_p}^p$ and $\min_{q\in[p,2]}\norm*{\bfy}_{q}^q$, or the same ratio when restricted to coordinates outside of $S'$. This means that the sum of Lewis weights in the range $[p, 2]$ (appropriately discretized by increments of $1/\log n$) bounds the $\gamma_p$-sensitivities up to a factor of $\gamma$ either on all the coordinates or restricted outside of the set $S'$, which means that the net arguments for Lewis weights via \cite{BourgainLindenstraussMilman:1989} yields a net bound of shows that we can preserve all $\gamma_p$ norms up to $(1\pm\eps)$ factors with $O(\gamma^{-1}d\poly(\log(n/\delta), \eps^{-1})) = O(n^{\beta/(1+\beta)}d\poly(\log(n/\delta), \eps^{-1}))$ samples. 

Finally, via a recursive analysis similar to that done in Theorem \ref{thm:huber-subspace-embedding}, we can repeat until we reach the desired number of rows in $O(\log\log n)$ iterations.
\end{proof}

Although the running time as stated is quite large, one can further optimize this in a similar manner as discussed in Section \ref{sec:huber-running-time-opt}. For example, we can reduce to $d^C$ dimensions in time $d^{C+1}$, where $C\approx 1.91236$, even for $p$ near $1$, and the trade-offs only improve as $p\to 2$. 
\section{Applications: Active Regression for the Tukey Loss}
\label{sec:general_loss}

An important loss function in practice is the \emph{Tukey loss}:

\begin{definition}[Tukey Loss Function]\label{def:tukey} For any $\tau \ge 0$, the Tukey loss function is given by:
    \begin{align*}
    T(x) = \begin{cases} \frac{\tau^2}{6} [1 - (1-x ^2/\tau^2)^3] \quad |x| \le \tau \\ \frac{\tau^2}{6} \quad \mathrm{otherwise} \end{cases}
    \end{align*}
    For simplicity throughout, we assume that $\tau = 1$. 
    For a vector $\bfy \in \R^n$ we denote $\norm{\bfy}_M = \sum_{i=1}^n M(\bfy(i))$. For a set of non-negative weights $\bfw \in \R^n$, we let $\norm{\bfy}_{M,\bfw} = \sum_{i=1}^n \bfw(i) \cdot M(\bfy(i))$.
\end{definition}

Note that our theorem for general $M$-estimators does not capture Tukey loss, due to our assumption of at least polynomial growth, whereas the Tukey loss is flat for $\abs*{x}\geq\tau$. In this section, we show that nonetheless, our active regression techniques apply to this loss as well.


Our results will in fact hold for a more general class of Tukey-like loss functions defined in Assumption 1 of \cite{ClarksonWangWoodruff:2019}. This class includes the $\ell_p$ Tukey losses, which we focus on in our proofs for simplicity.
\begin{definition}[$\ell_p$ Tukey Loss Function]\label{def:tukey2} For any $\tau \ge 0$, the $\ell_p$ Tukey loss function is given by:
\begin{align*}
M(x) = \begin{cases} |x|^p \quad |x| \le \tau \\ \tau^p \quad \mathrm{otherwise} \end{cases}
\end{align*}
For simplicity throughout, we assume that $\tau = 1$.
\end{definition}
We can check that for $\tau = 1$ and $p = 2$, the loss in Definition \ref{def:tukey} is equal to the loss in Definition \ref{def:tukey2} up to constants, so identical bounds will hold for it. 

\subsection{Preliminaries} 
We next state a few basic results regarding the $\ell_p$ Tukey loss function, shown in \cite{ClarksonWangWoodruff:2019}. Mostly, these are analogs of the triangle inequality, which must be shown since $\norm{\cdot }_M$ is not a norm.

\begin{lemma}[Tukey Triangle Inequality 1]\label{lem:tukTri1}
There is some fixed constant $c$ (depending on $p$) such that, for any $a, b \in \R^n$, and any non-negative weight vector $w \in \R^n$,
$$\norm{\bfa+\bfb}_{M,\bfw} \le c \cdot (\norm{\bfa}_{M,\bfw} + \norm{\bfb}_{M,\bfw}).$$
\end{lemma}
\begin{proof}
This simply uses the fact that $M(a) = |a|^p$ when $|a| \le 1$ and $M(a) = 1$ otherwise.
\end{proof}

\begin{lemma}[Tukey Triangle Inequality 2]\label{lem:tukTri2}
There is some fixed constant $c$ (depending on $p$) such that, for any $a, b \in \R$ with $|a| \le \epsilon |b|$,
$$ (1- c \epsilon) M(b) \le M(a+b) \le (1+ c \epsilon) M(b).$$
\end{lemma}
From Lemma \ref{lem:tukTri2} we can also prove the following, which is used in Lemma \ref{lem:reduction-to-good-b-coordinates}:
\begin{lemma}\label{lem:tukTri3}
There is some fixed constant $c$ such that, for any $a, b \in \R$ with $M(a) \le \epsilon^p M(b)$,
$$ (1- c \epsilon) M(b) \le M(a+b) \le (1+ c \epsilon) M(b).$$
\end{lemma}
\begin{proof}
If $M(a) \le \epsilon^p M(b)$, then since $M(a) = |a|^p$ when $|a|\le 1$ and $M(a) = 1$ otherwise, we must have $|a| = O(\epsilon) |b|$. The bound then follows from Lemma \ref{lem:tukTri2}.
\end{proof}
Relatedly we have,
\begin{lemma}[Tukey Triangle Inequality 3 -- Lemma 3.3. of \cite{ClarksonWangWoodruff:2019}]\label{lem:tukTri4}
There is some fixed constant $c$ (depending on $p$) such that, for any $\bfa, \bfb \in \R^n$ and non-negative weight vector $\bfw \in \R^n$ with $\norm{\bfa}_{M,\bfw} \le \epsilon^{2p+1} \norm{\bfb}_{M,\bfw}$,
$$ (1- c \epsilon) \norm{\bfb}_{M,\bfw} \le \norm{\bfa+\bfb}_{M,\bfw} \le (1+ c \epsilon) \norm{\bfb}_{M,\bfw} .$$
\end{lemma}

\subsection{Constant Factor Approximation}

As with $\ell_p$ regression, we start by using a subspace embedding result for Tukey regression to give a constant factor approximate regression algorithm. Unlike Theorem \ref{thm:lp-subspace-embedding} for the $\ell_p$ norm, our embedding will only hold for a finite net of vectors rather than the full column span of the sampled matrix. As in \cite{ClarksonWangWoodruff:2019}, we will later handle this assumption by assuming finite bit complexity of our inputs and outputs. See Section \ref{sec:bit}.

\begin{theorem}[Tukey Loss Subspace Embedding -- Lemma 6.7 of \cite{ClarksonWangWoodruff:2019}]\label{thm:tuksub} For any $\bfA \in \R^{n \times d},\bfb \in \R^n$, there is an algorithm that outputs a weight vector $\bfw \in \R^n$ such that for any set finite set of vectors $\mathcal{N}$ in the span of $[\bfA,\bfb]$, $\bfw$ only has $\tilde O(d^{\max(1,p/2)} \log(|\mathcal{N}|/\delta)/\epsilon^2)$ nonzero entries and with probability $\ge 1-\delta$, for all $\bfy \in \mathcal{N}$,
$$(1-\epsilon) \norm{\bfy}_{M} \le \norm{\bfy}_{M,\bfw} \le (1+\epsilon) \norm{\bfy}_{M}.$$
\end{theorem}

We observe that the sampling algorithm used to achieve Lemma 6.7 preserves the norm of any vector in expectation. Thus via Markov's inequality we have:
\begin{lemma}\label{lem:stillE}
Let $\bfw$ be the weight vector produced by Theorem \ref{thm:tuksub}. Let $\bfx \in \mathbb{R}^d$ be fixed. 
Then with arbitrarily large constant probability, $\norm{\bfA\bfx-\bfb}_{M,\bfw} = O(1) \norm{\bfA\bfx-\bfb}_M$.
\end{lemma}
\begin{proof}
We use the notation in Section 6.1 of \cite{ClarksonWangWoodruff:2019}.
Given a weight vector $\bfw$, one step of the recursive sampling algorithm of \cite{ClarksonWangWoodruff:2019} is to
assign probabilities $p_1, \ldots, p_n$ to each of the $n$ rows of $\bfA$, respectively. For the purposes of this lemma, it does not
matter what these sampling probabilities are.
One then chooses a sampling and rescaling matrix $\bfS$ according to this distribution, where if row $i$ is sampled,
then its new weight $\bfw_{i}'$ is set to $\bfw_i/p_i$, otherwise its weight is set to $0$. Thus, for any fixed $\bfx \in \R^d$ which does not depend on $\bfw$:
$$\E [\|\bfA\bfx-\bfb\|_{M, \bfw'}] = \|\bfA\bfx-\bfb\|_{M, \bfw},$$
When the algorithm begins, $\bfw = \mathbf{1}_n$. It follows inductively that for the final weight vector $\bfw$, that for any fixed $\bfx$, we have:
$$\E[\|\bfA\bfx-\bfb\|_{M, \bfw}] = \|\bfA\bfx-\bfb\|_M.$$
The lemma follows by Markov's inequality.
\end{proof}

Combining Theorem \ref{thm:tuksub} and Lemma \ref{lem:stillE}, we have an analog to Theorem \ref{thm:const} for Tukey regression.
\begin{theorem}[Constant Factor Approximation for Tukey Regression]\label{thm:tukconst}
Let $\bfw$ be a weight vector, and
suppose $\norm{\bfA (\bfx_1-\bfx _2)}_{M,\bfw} = (1 \pm \epsilon_0) \norm{\bfA (\bfx_1-\bfx _2)}_{M}$ for
all $\bfx_1,\bfx_2 \in \mathcal{N}$, where $\mathcal{N}$ is a fixed set of vectors. Let
$\tilde \bfx = \argmin_{\bfx \in \mathcal{N}} \norm{\bfA\bfx-\bfb}_{M,\bfw}$, and let
$\bfx^* = \argmin_{\bfx \in \mathcal{N}} \norm{\bfA\bfx-\bfb}_{M}$.
Suppose also that $\norm{\bfA\bfx^*-\bfb}_{M,\bfw} = O(1) \norm{\bfA\bfx^*-\bfb}_M$. 
Then $\norm{\bfA \tilde \bfx-\bfb}_M = O(1) \norm{\bfA\bfx^*-\bfb}_M$. 

By Theorem \ref{thm:tuksub} and Lemma \ref{lem:stillE} there is an algorithm producing $\bfw$ satisfying the required subspace embedding and norm preservation bounds with just $\tilde O(d^{\max(1,p/2)} \log(|\mathcal{N}|/\delta)/\epsilon^2)$ non-zero entries. Thus, computing $\tilde \bfx$ requires reading just this number of entries from $\bfb$. 
\end{theorem}
\begin{proof}
We have
\begin{align*}
\norm{\bfA \tilde \bfx-\bfb}_M &= O(1) \cdot (\norm{\bfA \tilde \bfx-\bfA\bfx^*}_M + \norm{\bfA\bfx^*-\bfb}_M)\tag{Approximate triangle inequality, Lem. \ref{lem:tukTri1}} \\
&\leq O(1) \cdot \left (\frac{1}{1-\epsilon_0} \cdot \norm{\bfA (\tilde \bfx-\bfx ^*)}_{M,\bfw} + \norm{\bfA\bfx^*-\bfb}_M \right )\tag{Subspace embedding assumption, since $\tilde \bfx,\bfx^* \in \mathcal{N}$}\\
&\leq O(1) \cdot \left (\frac{1}{1-\epsilon_0} \cdot \left (\norm{\bfA \tilde \bfx-\bfb}_{M,\bfw} + \norm{\bfA\bfx^*-\bfb}_{M,\bfw} \right ) + \norm{\bfA\bfx^*-\bfb}_M \right )\tag{Approximate triangle inequality, Lem. \ref{lem:tukTri1}} \\
&\leq O(1) \cdot \left (\frac{2}{1-\epsilon_0} \norm{\bfA\bfx^*-\bfb}_{M,\bfw} + \norm{\bfA\bfx^*-\bfb}_M \right )\tag{Optimality of $\tilde \bfx$ for the sampled problem}\\
& \leq O(1) \cdot \norm{\bfA\bfx^*-\bfb}_M\tag{Assumption that $\norm{\bfA\bfx^*-\bfb}_{M,\bfw} = O(1) \norm{\bfA\bfx^*-\bfb}_M$}.
\end{align*}
\end{proof}

\subsection{Net Argument}\label{sec:bit}
The constant factor approximation result of Theorem \ref{thm:tukconst} allows us to approximately minimize the Tukey loss regression problem over any finite net of vectors $\mathcal{N}$. 
As in \cite{ClarksonWangWoodruff:2019},
we need to make an assumption for the net $\mathcal{N}$ to have a small size. For simplicity, we adopt the
assumption stated after Assumption 2 in \cite{ClarksonWangWoodruff:2019}, that the entries of $\bfA$ and $\bfb$ are
integers of absolute values at most $n^{\poly(d)}$. 
We can also assume
$\bfA$ has full rank, since removing linearly dependent columns can be done in a preprocessing stage.

\begin{lemma}\label{lem:precision}
Assuming that $\bfA,\bfb$ have integer entries bounded by $n^{\poly(d)}$ and $\bfA$ has full column rank, if $\bfx \in \R^d$ is such that $\norm{\bfA\bfx-\bfb}_M \le C \cdot \min_{\bfz \in \R^d} \norm{\bfA \bfz-\bfb}_M$ for some $C \leq \poly(n)$,
and we round the entries of $\bfx$ to the nearest integer multiple of $1/n^{\poly(d)}$, obtaining $\bfx'$,
then $\norm{\bfA\bfx' -\bfb}_M \le C \cdot (1+1/n^{\poly(d)}) \cdot \min_{\bfz \in \R^d} \norm{\bfA \bfz-\bfb}_M$. Further, all entries
of $\bfx$ are at most $n^{d^{\alpha}}$ for a fixed constant $\alpha > 0$. 

Further, if $\norm{\bfA\bfx-\bfb}_M = 0$, then necessarily $\bfx' =\bfx$. Also, if $\norm{\bfA \bfz-\bfb}_M > 0$
for all $\bfz$, then $\min_\bfz \norm{\bfA \bfz-\bfb}_M \geq 1/n^{\poly(d)}$.
\end{lemma}
\begin{proof}
First suppose that $\min_\bfz \norm{\bfA \bfz-\bfb}_M = 0$. This implies that $\bfA\bfx = \bfb$ and thus $\bfx = (\bfA^\top \bfA)^{-1} \bfA^\top \bfb$. By Cramer's rule
and our integrality assumptions, it follows that the entries of $\bfx$ are integer multiples of $1/n^{\poly(d)}$
and bounded in absolute value by $n^{\poly(d)}$. Thus, $\bfx$ is rounded to itself, showing that $\bfx' =\bfx$
and the lemma holds.

Otherwise, suppose that $\min_\bfz \norm{\bfA\bfx-\bfb}_M > 0$. Then the claim
is that $\min_\bfz \norm{\bfA \bfz-\bfb}_M \ge 1/n^{\poly(d)}$. To see this,
note that the matrix $[\bfA, \bfb]$ has rank-$(d+1)$, since the columns of $\bfA$ are linearly independent.
Since $[\bfA, \bfb]$ is an integer matrix, it follows from
Lemma 4.1 of \cite{CW09}
that its minimum singular value is at least $n^{-\textrm{poly}(d)}$.
Notice that $\min_\bfz \|\bfA\bfz-\bfb\|_2^2$ is at least the squared minimum singular value of $[\bfA,\bfb]$, since:
\begin{align*}
\min_{\bfz \in \R^d} \|\bfA\bfz-\bfb\|_2^2 = \min_{\bfz' \in \R^{d+1}: \bfz'(d+1) = -1} \norm{[\bfA,\bfb] \bfz'}_2^2 \ge \min_{\bfz' \in \R^{d+1}: \norm{\bfz'} \ge 1} \norm{[\bfA,\bfb] \bfz'}_2^2.
\end{align*}
The last equation is exactly the minimum squared singular value of $[\bfA,\bfb]$.
Consequently, the $\ell_2$-regression
cost is lower bounded by $n^{-\textrm{poly}(d)}$, which implies the $\ell_p$-regression cost
is lower bounded by $n^{-\textrm{poly}(d)}$, and the Tukey loss is lower bounded by
$\min(1, n^{-\textrm{poly}(d)}) = n^{-\textrm{poly}(d)}$ (since we assume $\tau = 1$). In this
case, if we round $\bfx$ to the nearest vector with entries that are integer multiples of $1/n^{\poly(d)}$,
then using that the entries of $\bfA$ are bounded by $n^{\poly(d)}$, we still obtain a
$C \cdot (1+1/n^{\poly(d)})$-approximate solution.

Finally, note that the entries in any $\poly(n)$-approximate solution $\bfx$ cannot be larger than
$n^{\poly(d)}$. Indeed, by Lemma 5.2 of \cite{ClarksonWangWoodruff:2019} and since parts 2 and 3 of
Assumption 2 of \cite{ClarksonWangWoodruff:2019} hold by our integrality assumption, it follows that part 1
of Assumption 2 of \cite{ClarksonWangWoodruff:2019} also holds, and so $\|\bfx\|_2 \leq n^{d^{\alpha}}$ for a fixed
constant $\alpha > 0$. 
\end{proof}

Consider the set $\mathcal{N}$ of all vectors $\bfx \in \mathbb{R}^d$ that are integer multiples of
$1/n^{\poly(d)}$ and bounded in absolute value by $n^{\poly(d)}$, for a sufficiently large $\poly(d)$. 
Then $|\mathcal{N}| \leq n^{\poly(d)}$ and we have

\begin{theorem}\label{thm:net}
Suppose $\norm{\bfA\bfx'-\bfb}_{M,\bfw} = (1 \pm \epsilon)\norm{\bfA\bfx'-\bfb}_M$ for all $\bfx' \in \mathcal{N}$
for a weight vector $\bfw$ with $\|\bfw\|_{\infty} \leq n^{\poly(d)}$.
Then $\norm{\bfA\bfx-\bfb}_{M,\bfw} = (1 \pm \epsilon)\norm{\bfA\bfx-\bfb}_M$ for all $\bfx$ for which
$\|\bfx\|_2 \leq n^{d^\alpha}$ for a fixed constant $\alpha > 0$. Further, 
any $C$-approximate minimizer to $\norm{\bfA\bfx-\bfb}_M$, for some $C \leq \poly(n)$,
necessarily satisfies $\|\bfx\|_2 \leq n^{d^{\alpha}}$. 
\end{theorem}
\begin{proof}
Let $\bfx'$ be the rounding in $\mathcal{N}$ of such an $\bfx$, as in Lemma \ref{lem:precision}.
If $\min_\bfx \norm{\bfA\bfx-\bfb}_M = 0$, then necessarily $\bfx' =\bfx$ by Lemma \ref{lem:precision}.
Note that since the columns of $\bfA$ are linearly independent, $\bfx$ is unique, and as in 
the proof of Lemma \ref{lem:precision}, we have $\|\bfx\|_2 \leq n^{d^{\alpha}}$ for a fixed
constant $\alpha > 0$. 

Otherwise, we can observe that $\|\bfw\|_{\infty} \cdot \|\bfA(\bfx-\bfx ')\|_{\infty} = 1/n^{\poly(d)}$ by our assumption that $\norm{\bfw}_\infty \le n^{\poly(d)}$ and the fact that that the granularity of our net is $1/n^{\poly(d)}$ for a sufficiently large $\poly(d)$. This gives $\norm{\bfA (\bfx-\bfx ')}_{M,\bfw} \le 1/n^{\poly(d)}$. Further, $\norm{\bfA\bfx'-\bfb}_M \ge 1/n^{\poly(d)}$ and $\norm{\bfA\bfx-\bfb}_M \ge 1/n^{\poly(d)}$ by Lemma \ref{lem:precision}. In turn, $\norm{\bfA\bfx'-\bfb}_{M,\bfw} \ge 1/n^{\poly(d)}$ by our subspace embedding assumption on $\bfw$. Combining these bounds with the approximate triangle inequality of Lemma \ref{lem:tukTri4} gives:
%
\begin{align*}
\norm{\bfA\bfx-\bfb}_{M,\bfw} & = \norm{\bfA\bfx'-\bfb + \bfA(\bfx-\bfx ')}_{M,\bfw}\\
&= (1 \pm 1/n^{\poly(d)}) \norm{\bfA\bfx'-\bfb}_{M,\bfw}\tag{Lemma \ref{lem:tukTri4}}\\
&= (1 \pm O(\epsilon)) \norm{\bfA\bfx'-\bfb}_M\tag{Subspace embedding assumption on $\bfw$ for $\bfx' \in \mathcal{N}$}\\
&= (1 \pm O(\epsilon)) \norm{\bfA\bfx-\bfb  + \bfA(\bfx'-\bfx )}_M\\
&= (1 \pm O(\epsilon)) \norm{\bfA\bfx-\bfb}_M\tag{Lemma \ref{lem:tukTri4} again}.
\end{align*}
\end{proof}

Theorem \ref{thm:net} implies that we can use Theorem \ref{thm:tukconst} with our net $\mathcal{N}$,
so that if we obtain a constant factor solution to the sampled regression problem $\min_\bfx \norm{\bfA\bfx-\bfb}_{M,\bfw}$ over $\bfx$ with $\|\bfx\|_2 \leq n^{d^{\alpha}}$, then we
obtain an initial constant factor solution to the original regression problem.

\subsection{Relative Error Approximation} 

We now build on our constant factor approximation result and the net arguments above to give a relative error approximation result for Tukey regression, akin to Theorem \ref{thm:eps}. We assume that $\norm{\bfb}_M = O(\OPT)$, which we can ensure by computing an initial constant factor approximation as in Algorithm \ref{alg:clip}, using Theorem \ref{thm:net} combined with Theorem \ref{thm:tukconst}.

\begin{theorem}[Relative Error Tukey Regression]\label{thm:tuk} Consider $\bfA \in \R^{n \times d}$, $\bfb \in \R^n$, and $p \ge 1$. Let $\OPT = \min_{\bfx \in \R^d} \norm{\bfA\bfx - \bfb}_M$ where $M$ is the $\ell_p$ Tukey loss (Def. \ref{def:tukey2}) and assume that $\norm{\bfb}_M = O(\OPT)$. There is an algorithm that reads $\tilde O \left (\frac{d^{\max(1,p/2) + O(1)}}{\epsilon^{2+p}} \poly\log(n)\right )$ entries of $\bfb$ and outputs $\tilde \bfx$ satisfying $\norm{\bfA \tilde \bfx - \bfb}_M \le (1+\epsilon) \OPT$ with probability $99/100$.
\end{theorem}

\begin{proof}
Let $\tilde\bfs_i^{M}(\bfA)$ be as computed by the $\ell_p$ Tukey sensitivities as computed by Algorithm \ref{alg:m-estimator-alg-sensitivity} and Theorem \ref{thm:m-estimator-alg}. Note then that
\[
    \sum_{i=1}^n \tilde\bfs_i^{M}(\bfA) \leq O(d^{\max\{1,p/2\}}\log^2 n).
\]

Note that Lemmas \ref{lem:opt-norm-bound}, \ref{lem:reduction-to-good-b-coordinates}, and \ref{lem:m-bernstein-good-coordinates} all hold for the Tukey loss, where assumption (1) is satisfied in Lemma \ref{lem:reduction-to-good-b-coordinates} due to Lemma \ref{lem:tukTri3}. Thus, as reasoned in Lemma \ref{lem:m-active-regression-relative-error-approximation}, we just need to choose an $m$ that is sufficient to union bound over a net. For the Tukey loss, we set
\[
    m = \frac{\poly(d) \log n}{\epsilon^{2+p}}
\]
and apply a union bound over all vectors $\bfA\bfx$ in a net $\mathcal{N}$ with $|\mathcal{N}| \le n^{\poly(d)}$ gives that this bound holds with high probability simultaneously for all vectors in the net. 
Letting this net be as in Theorem \ref{thm:net}, this is enough to give that the bound holds for all $\bfx$ of interest.

Note that the number of entries of $\bfb$ that are read is, in expectation,
\[
    \sum_{i=1}^n \bfp_i \leq \sum_{i=1}^n m\cdot \tilde\bfs_i^M(\bfA) = O\parens*{d^{\max\{1+p/2\}+O(1)} \frac{\log^3 n}{\eps^{2+p}}}.
\]
Thus, this is, up to constant factors, the number of entries read with probability at least $99/100$ as well. 
\end{proof}

\section{Applications: Kronecker Product Regression}
\label{sec:kronecker}
In the $q$-th order Kronecker product regression problem, one is 
given matrices $\bfA_1, \bfA_2,\ldots, \bfA_q$, where $\bfA_i \in \mathbb{R}^{n_i \times d_i}$,
as well as a vector $\bfb \in \mathbb{R}^{n_1 n_2 \cdots n_q}$, and 
the goal is to obtain a solution to the problem:
\[\min_{\bfx \in \R^{d_1 d_2 \cdots d_q}} \|( \bfA_1 \otimes \bfA_2 \cdots \otimes \bfA_q)x - \bfb\|_p, \]
where $\bfA_1 \otimes \bfA_2 \cdots \otimes \bfA_q$ is the 
$\mathbb{R}^{(n_1 \times n_2 \times \cdots \times n_q)\bfx (d_1 \times d_2 \times \cdot \times d_q)}$ matrix
whose $((i_1, i_2, \ldots, i_q), (j_1, \ldots, j_q))$-th entry is equal to
$(\bfA_1)_{i_1, j_1} \cdots (\bfA_q)_{i_q, j_q}$. Typically $\prod_{i=1}^q d_i$ is much less
than $\prod_{i=1}^q n_i$, and the goal is to obtain algorithms that do not explicitly form 
$\bfA_1 \otimes \bfA_2 \otimes \cdots \otimes \bfA_q$, which is too expensive.

Kronecker product regression is a special case of $\ell_p$-regression 
when the design matrix is highly structured. Such matrices
arise in spline regression, 
signal processing, and multivariate data fitting 
\cite{vanbook,vanLoanKron,golubVanLoan2013Book}. 
Recent work \cite{dssw18,DJSSW19} designed algorithms for $p \in [1,2]$ and constant $q$ 
which solve this
problem with constant probability up to a $(1 \pm \epsilon)$-factor in 
$\tilde{O}(\sum_{i=1}^q \nnz(\bfA_i) + \nnz(\bfb) + \poly(d_i/\epsilon))$ time; 
see Theorem 1.2 of \cite{DJSSW19}. While $\sum_{i=1}^q \nnz(\bfA_i)$ is bounded by 
$O(\sum_i n_i d_i)$, the main drawback of these works is that $\nnz(\bfb)$ may
be as large as $n_1 \times n_2 \times \cdots \times n_q$, which is prohibitive. 
Theorem \ref{thm:main} shows that for every real $p \geq 1$, only $\tilde O \left (\prod_{i=1}^q d_i^{\max(1,p/2)} \right )$
samples are needed, and our runtime is polynomial in this quantity. This should
be contrasted with prior work which required $\nnz(\bfb)$ time for every $p \neq 2$, 
where $\nnz(\bfb)$ can be as large as $\prod_{i=1}^q n_i$. 

Before stating our theorem, we need to understand Lewis weights of the Kronecker product
of two matrices $\bfA$ and $\bfB$. Recall from Definition \ref{def:lewis} that the Lewis weights of an $n \times d$ matrix $\bfA$ are 
the unique numbers $\bfw^p_1(\bfA), \ldots, \bfw^p_n(\bfA)$ which satisfy
$\bfw_i^p(\bfA) = \bfs_i^2 (\bfW^{1/2-1/p} \bfA)$, where $\bfs_i^2(\bfW^{1/2-1/p} \bfA)$ is the $i^{th}$ leverage score of $\bfW^{1/2-1/p} \bfA$, and $\bfW$ is a diagonal matrix with $i^{th}$
diagonal entry equal to $\bfw_i^p(\bfA)$. We can show that the Lewis weights of $\bfA \otimes \bfB$ are just the products of those of $\bfA$ and $\bfB$.

\begin{lemma}\label{lem:kronLewis}
Given matrices $\bfA,\bfB$, the $(i,j)^{th}$ Lewis weight of $\bfA \otimes \bfB$ is given by 
$$\bfw_{i,j}^p(\bfA \otimes \bfB) = \bfw_i^p(\bfA) \cdot \bfw_j^p(\bfB).$$
\end{lemma}
\begin{proof}
let $\bfW_\bfA$ and $\bfW_\bfB$ denote the diagonal matrices with the Lewis weights of $\bfA$ and $\bfB$ respectively on their diagonals. To prove the lemma, we just need to verify that $\bfw_i^p(\bfA) \cdot \bfw_j^p(\bfA) = \bfs_{i,j}^2 ((\bfW_\bfA^{1/2 - 1/p} \bfA) \otimes (\bfW_\bfB^{1/2-1/p} \bfB))$.
By Proposition 3.2 of \cite{DJSSW19}, the $(i,j)$-th row leverage score of the 
Kronecker product of two matrices $\bfC$ and $\bfD$ is $\bfs_i^2(\bfC) \cdot \bfs_j^2(\bfD)$. For completeness: to see this,
recall the leverage scores are basis-invariant and thus a property of the column span of $\bfC \otimes \bfD$.
If $\bfR$ and $\bfS$ are square matrices for which $\bfC\bfR$ has orthonormal columns and $\bfD\bfS$ has orthonormal columns, then
one can verify that $(\bfC \otimes \bfD) \cdot (\bfR \otimes \bfS)$ has orthonormal columns, which implies the statement.
Consequently,
$$\bfs_{i,j}^2 ((\bfW_\bfA^{1/2 - 1/p} \bfA) \otimes (\bfW_\bfB^{1/2-1/p} \bfB)) = \bfs_i^2(\bfW_\bfA^{1/2 - 1/p} \bfA) \cdot \bfs_j^2(\bfW_\bfB^{1/2 - 1/p} \bfB) = \bfw_i^p(\bfA) \cdot \bfw_i^p(\bfB).$$ 
\end{proof}
\begin{corollary}\label{cor:lewis}
Given matrices $\bfA_1, \ldots \bfA_q$, 
the $(i_1, \ldots, i_q)$-th Lewis weight of $\bfA_1 \otimes \cdots \otimes \bfA_q$ is
$$\bfw_{i_1, \ldots, i_q}^p(\bfA_1 \otimes \cdots \otimes \bfA_q) = \bfw^p_{i_1}(\bfA_1) \cdot \bfw^p_{i_2}(\bfA_2) \cdot \ldots \cdot \bfw^p_{i_q}(\bfA_q).$$
\end{corollary}
\begin{proof}
Follows from Lemma \ref{lem:kronLewis} and induction. 
\end{proof}

The following theorem is notable, in that it is the first algorithm for Kronecker product regression, 
for every $p \geq 1$, whose running time {\it does not depend on $\nnz(\bfb)$}. 
\begin{theorem}\label{thm:kron}
Let $q \geq 1$ and $p \geq 1$ be constant, and $\epsilon > 0$. The Kronecker product regression
problem can be solved up to a $(1+\epsilon)$-factor with constant probability in 
$\tilde{O}(\sum_{i=1}^q \nnz(\bfA_i) + \poly(\prod_{i=1}^q d_i / \epsilon))$ time.
\end{theorem}
\begin{proof}
By Theorem 6.1 of \cite{CohenPeng:2015}, in $\tilde{O}(\sum_{i=1}^q \nnz(\bfA_i))$ time 
we can compute $2$-approximations $\hat{u^{\ell}_i}$ to all of the Lewis weights 
of all of the input matrices $\bfA^{\ell}$, where $\{u^{\ell}_i\}_i$ is the set of Lewis weights of $\bfA^{\ell}$. By Corollary
\ref{cor:lewis}, for each $(i_1, \ldots, i_q)$, we have that $\hat{u^{1}_{i_1}} \cdots \hat{u^q_{i_q}}$ is a $2^q = O(1)$-approximation
to the $(i_1, \ldots, i_q)$-th Lewis weight of $\bfA^1 \otimes \cdots \otimes \bfA^q$. By Theorem \ref{thm:eps}, it follows
that for constant $p$ if we sample $\poly(\prod_{i=1}^q d_i / \epsilon)$ rows of $\bfA^1 \otimes \cdot \otimes \bfA^q$ and
corresponding entries of $\bfb$, then we can solve for a $(1+\epsilon)$-approximate solution $\bfx$ to the $\ell_p$-Kronecker
product regression problem in $\poly(\prod_{i=1}^q d_i / \epsilon)$ time. To sample proportional to the
 $\hat{u^{1}_{i_1}} \cdots \hat{u^q_{i_q}}$ values, we can first sample an $i_1$ proportional to $\hat{u^1_{i_1}}$, then independently
sample an $i_2$ proportional to $\hat{u^2_{i_2}}$, and so on. It is possible to build a data structure in $O(n)$ time, one for
each $\bfA^{\ell}$, so that a single sample can be extracted in $O(1)$ time \cite{walker1974new,walker1977efficient}. 
So the total time for the sampling is
$O(n) + \poly(\prod_{i=1}^q d_i / \epsilon)$, and we can assume $\tilde{O}(\sum_{i=1}^q \nnz(\bfA_i))$ is at least $n$ (otherwise we
can throw away zero rows), and so the overall time is $\tilde{O}(\sum_{i=1}^q \nnz(\bfA_i) + \poly(\prod_{i=1}^q d_i / \epsilon))$. 
\end{proof}
\section{Applications: Subspace Embeddings for Orlicz Norms}\label{sec:orlicz}

In this section, we show that our sensitivity sampling techniques immediately carry over to \emph{Orlicz norms}, which can be thought of as scale-invariant versions of $M$-estimators. Subspace embeddings, regression, and row sampling algorithms for Orlicz norms have recently been studied by \cite{AndoniLinSheng:2018} and \cite{SongWangYang:2019}, respectively. 

The row sampling algorithms for Orlicz norm regression in \cite{SongWangYang:2019} proceeds by constructing a generalization of well-conditioned bases for Orlicz norms using an oblivious subspace embedding, and then sampling the rows proportionally to an importance distribution defined using the well-conditioned basis. Due to this indirect construction of importance distributions, the algorithm of \cite{SongWangYang:2019} incurs large polynomial factors in $d$ on the sample size (i.e., the number of rows in the subspace embedding). Specifically, their total sensitivity bound from Lemma 4 is $O(d\kappa^2)$ using a well-conditioned basis of distortion $\kappa$, where $\kappa$ given by their Corollary 12 is $O(d^{5/2}\poly\log n)$. Then, the subspace embedding they obtain in Lemma 7 has $O(d^7\poly(\log n,\eps^{-1}))$ rows with constant probability. Furthermore, small oblivious subspace embeddings do not exist when the $M$ estimator function $M$ grows faster than quadratic \cite{BravermanOstrovsky:2010}, and this constraint carries over to their construction as well (see Assumption 1 of \cite{SongWangYang:2019}). 

We overcome both of these problems by directly constructing an importance sampling distribution according to the Orlicz sensitivities (see Definition \ref{dfn:orlicz-sensitivity}) using our Theorem \ref{thm:m-estimator-alg}, which removes the quadratic growth assumption, allowing for a polynomial growth bound of degree $p_G$, and significantly improves the dependence on $d$ in the dimension of the subspace embedding. More specificially, in Theorem \ref{thm:input-sparsity-time-orlicz-subspace-embedding}, we obtain a $\tilde O(\nnz(\bfA) + \poly(d/\eps))$ time algorithm that computes a weighted subset of $O(d^2\poly(\log n, \eps^{-1}))$ rows that achieves a $(1\pm\eps)$ distortion subspace embedding for $p_G\leq 2$, as well as $O(d^{p_G/2+1}\poly(\log n, \eps^{-1}))$ rows for $p_G > 2$.

\begin{theorem}\label{thm:input-sparsity-time-orlicz-subspace-embedding}
    Let $G$ satisfy the hypotheses of Corollary \ref{cor:orlicz-sensitivity-upper-bound-alg} and let $p_G$ be the polynomial upper bound exponent for $G$. There is an algorithm which, with constant probability, computes a set of weights $\bfw'$ such that
    \[
        \norm*{\bfA\bfx}_{G,\bfw'} = (1\pm\eps)\norm*{\bfA\bfx}_G
    \]
    for all $\bfx\in\mathbb R^d$ in time
    \[
        O\parens*{\nnz(\bfA)\log^3 n + \frac{(Td^{\max\{1,p_G/2\}} + T^2)\log^4 n}{\eps^2}\log\frac1\eps},
    \]
    where constant factor $\ell_{p_G}$ Lewis weight approximation for an $n\times d$ matrix $\bfB$ takes time $O(\nnz(\bfb) + T))$ time (see Theorem \ref{thm:cohen-peng-fast-lewis-weights}). Furthermore, 
    \[
        \nnz(\bfw') =O\parens*{\frac{d^{\max\{2,p_G/2+1\}} \log^3 n}{\eps^2}\log\frac1\eps} 
    \]
    with constant probability.
\end{theorem}

\subsection{Fast Approximation of Orlicz Sensitivities}

First recall the definition of Orlicz norms. 

\begin{definition}[Orlicz Norms and Weighted Orlicz Norms \cite{SongWangYang:2019}]
Let $G:\mathbb R_{\geq 0}\to\mathbb R_{\geq 0}$ be a nonnegative convex function with $G(0) = 0$. Then for $\bfy\in\mathbb R^n$, the Orlicz norm of $\bfy$ is defined by
\[
    \norm*{\bfx}_G = \inf\braces*{t > 0 : \sum_{i=1}^n G(\abs*{\bfy_i}/t) \leq 1}.
\]
Given a set of nonnegative weights $\bfw\in\mathbb R^n$, the weighted Orlicz norm of $\bfy$ is defined by
\[
    \norm*{\bfx}_{G,\bfw} = \inf\braces*{t > 0 : \sum_{i=1}^n \bfw_i G(\abs*{\bfy_i}/t) \leq 1}.
\]
\end{definition}

We can naturally define the sensitivty of Orlicz norms as follows. 

\begin{definition}[Orlicz Sensitivities and Weighted Orlicz Sensitivities]\label{dfn:orlicz-sensitivity}
Let $\bfA\in\mathbb R^{n\times d}$. We define the \emph{$i$th Orlicz sensitivity} by
\[
    \bfs_i^G(\bfA) \coloneqq \sup_{\bfx\in\mathbb R^d, \norm*{\bfA\bfx}_G \leq 1} G(\abs*{[\bfA\bfx](i)}) = \sup_{\bfx\in\mathbb R^d, \bfA\bfx\neq 0} G\parens*{\frac{\abs*{[\bfA\bfx](i)}}{\norm*{\bfA\bfx}_G}}
\]
Let $\bfw\geq\mathbf{1}_n$ be a set of weights. We define the \emph{$i$th weighted Orlicz sensitivity} by
\[
    \bfs_i^{G,\bfw}(\bfA) \coloneqq \sup_{\bfx\in\mathbb R^d, \norm*{\bfA\bfx}_G \leq 1} \bfw_i G(\abs*{[\bfA\bfx](i)}) = \sup_{\bfx\in\mathbb R^d, \bfA\bfx\neq 0} \bfw_i G\parens*{\frac{\abs*{[\bfA\bfx](i)}}{\norm*{\bfA\bfx}_G}}
\]
\end{definition}

The Orlicz norm can be thought of as a scale invariant version of $M$-estimators, and is the gauge norm of the $M$-ball of radius $1$. By homogeneity of the Orlicz norm, we can then restrict our attention to radius $1$, where the Orlicz norm is equivalent to the corresponding $M$-norm. Thus, our results for $M$-estimators from earlier easily carry forth to Orlicz norms. For instance, the following corollary is a simple consequence of Theorem \ref{thm:m-estimator-alg}:

\begin{corollary}[Fast Approximation of Orlicz Sensitivities]\label{cor:orlicz-sensitivity-upper-bound-alg}
    Let $\bfA\in\mathbb R^{n\times d}$. Let $G:\mathbb R_{\geq0}\to\mathbb R_{\geq0}$ be an Orlicz function that is polynomially bounded above with degree $p_G$, and satisfies the requirements analogous to those for $M$ in Theorem \ref{thm:m-estimator-alg}. Let $1\leq \tau\leq n$ be a parameter. Then, with probability at least $99/100$, Algorithm \ref{alg:m-estimator-alg-sensitivity} computes upper bounds $\tilde\bfs_i^G(\bfA)$ on Orlicz sensitivities $\bfs_i^G(\bfA)$ that sum to at most
    \[
        \tilde{\mathcal T}^G(\bfA) \coloneqq \sum_{i=1}^n \tilde\bfs_i^G(\bfA) = O(d^{\max\{1,p_G/2\}}\log^2 n + \tau).
    \]
    If constant factor Lewis weight approximation takes for an $n\times d$ matrix $\bfB$ takes $O(\nnz(\bfB)\log n + T)$ time (see Theorem \ref{thm:cohen-peng-fast-lewis-weights}), then the total running time is
    \[
        O\parens*{\nnz(\bfA)\log^3 n + \frac{nT}{\tau}\log n}.\qedhere
    \]
\end{corollary}

As with $M$-estimators, we can also extend these results to the weighted case. The following lemma is a version of Lemma \ref{lem:weighted-m-sensitivity-bound} (or Lemma 39 in \cite{ClarksonWoodruff:2015a}), but requires a slightly different proof due to the fact that Orlicz norms are not ``entrywise'' norms. 

\begin{lemma}[Sensitivity Bounds for Weighted Orlicz Norms]\label{lem:sensitivity-bound-weighted-orlicz}
Let $\bfw\in\mathbb R^n$ with $\bfw \geq \mathbf{1}^n$. Let $N\coloneqq \ceil*{\log_2(1+\norm*{\bfw}_\infty)}$. For $j\in[N]$, let
\[
    T_j\coloneqq \braces*{i\in[n] : 2^{j-1} \leq \bfw_i < 2^j},
\]
let $\bfA\mid_{T_j}$ denote the restriction of $\bfA$ to the rows of $T_j$, and let $G_j \coloneqq 2^{j-1}\cdot G$. Then,
\[
    \bfs_i^{G,\bfw}(\bfA) \leq 2\cdot \bfs_i^{G_j}(\bfA\mid_{T_j})
\]
and
\[
    \sum_{i=1}^n \bfs_i^{G,\bfw}(\bfA) \leq O\parens*{N\cdot d^{\max\{1,p_G/2\}}\log n}.
\]
\end{lemma}
\begin{proof}
    The proof can be found in Appendix \ref{sec:orlicz-appendix}.
\end{proof}

The following is the resulting corollary for efficient computation of upper bounds on weighted Orlicz sensitivities, analogous to Corollary \ref{cor:weighted-m-sensitivity-alg} for $M$-estimators. 

\begin{corollary}[Fast Approximation of Weighted Orlicz Sensitivities]\label{cor:weighted-orlicz-sensitivity-upper-bound-alg}
    Let $G$ satisfy the hypotheses of Corollary \ref{cor:orlicz-sensitivity-upper-bound-alg} and let $p_G$ be the polynomial upper bound exponent for $G$. Let $\bfw\geq\mathbf{1}_n$ be a set of weights. Define $N$ as in Lemma \ref{lem:sensitivity-bound-weighted-orlicz}. There is an algorithm that computes weighted Orlicz norm sensitivities that sum to at most $O(Nd^{\max\{1,p_G/2\}}\log^2 n + N\tau)$ in time
    \[
        O\parens*{\nnz(\bfA)\log^3 n + N\frac{nT}{\tau}\log n},
    \]
    where $T$ is such that constant factor $\ell_{p_G}$ Lewis weight approximation for an $n\times d$ matrix $\bfB$ takes $O(\nnz(\bfB)\log n + T)$ time (see Theorem \ref{thm:cohen-peng-fast-lewis-weights}). 
\end{corollary}
\begin{proof}
    This is simply the result of applying Corollary \ref{cor:orlicz-sensitivity-upper-bound-alg} on the $N$ matrices $\bfA\mid_{T_j}$ as defined in Lemma \ref{lem:sensitivity-bound-weighted-orlicz}. 
\end{proof}

With upper bounds on Orlicz sensitivities in hand, we immediately obtain subspace embeddings by sampling from them. 

\begin{lemma}[Subspace Embeddings for Orlicz Norms]\label{lem:orlicz-subspace-embedding}
    Let $\bfA\in\mathbb R^{n\times d}$ and $\eps>0$. Let $G:\mathbb R_{\geq0}\to\mathbb R_{\geq0}$ be a convex Orlicz function satisfying the requirements analogous to those for $M$ in Theorem \ref{thm:m-estimator-alg}. Let $\bfw\geq\mathbf{1}_n$ be a set of weights. Let $\tilde\bfs_i^G(\bfA)\geq \bfs_i^{G,\bfw}(\bfA)$ be upper bounds on the weighted Orlicz sensitivities of $\bfA$ (see Definition \ref{dfn:orlicz-sensitivity}) and let
    \[
        \tilde{\mathcal T}^{G,\bfw}(\bfA) \coloneqq \sum_{i=1}^n \tilde\bfs_i^{G,\bfw}(\bfA).
    \]
    Let $m\geq m_0$ be a parameter larger than some
    \[
        m_0 = O\parens*{\frac{d}{\eps^2}\parens*{\log\frac1\eps}\parens*{\log\frac1\delta}}.
    \]
    Let $\bfw'\geq \mathbf{1}_n$ be a set of weights chosen randomly so that
    \[
        \bfw_i' \coloneqq \begin{cases}
            \bfw_i/\bfp_i & \text{w.p. $\bfp_i$} \\
            0 & \text{otherwise}
        \end{cases}
    \]
    where
    \[
        \bfp_i \coloneqq \min\braces*{1, m\cdot \tilde\bfs_i^{G,\bfw}(\bfA)}
    \]
    Then with probability at least $1 - \delta$, 
    \[
        \norm*{\bfy}_{G,\bfw'} = (1\pm\eps)\norm*{\bfy}_{G,\bfw}
    \]
    for all $\bfy\in\Span(\bfA)$. Furthermore,
    \[
        \E\nnz(\bfw') \leq m\tilde{\mathcal T}^{G,\bfw}(\bfA).
    \]
\end{lemma}
\begin{proof}
The proof follows from the proof of Lemma 7 of \cite{SongWangYang:2019}, replacing their sampling process with our sensitivity sampling. We only briefly outline the necessary changes, since the proof is almost identical to their proof, as well as other sensitivity Bernstein bounds done in this paper, such as Lemma \ref{lem:m-sensitivity-sampling}. 

By Lemma \ref{lem:norm-net-to-sphere}, it suffices to show that $\norm*{\bfy}_{G,\bfw'} = (1\pm\eps)\norm*{\bfy}_{G,\bfw}$ for all $\bfy\in\mathcal N$ where $\mathcal N$ is an $\eps$-net on the unit Orlicz sphere. Such a net can be chosen to have size at most $\log\abs*{\mathcal N} \leq O(d\log\frac1\eps)$ by Lemma \ref{lem:standard-net}. 

It is shown in the proof of Lemma 7 of \cite{SongWangYang:2019} that if
\[
    \sum_{i=1}^n \bfw_i' G(\abs*{\bfy_i}) \in \bracks*{1-\frac\eps4, 1+\frac\eps4},
\]
then $\norm*{\bfy}_{G,\bfw'} \in [1-\eps/4,1+\eps/4]$ as long as $G(0) = 0$ and $G$ is convex. Then, the result follows from exactly the same sensitivity sampling and Bernstein bound calculation as done in Lemma \ref{lem:m-sensitivity-sampling}. 
\end{proof}

\subsection{Nearly Input Sparsity Time Algorithms}

Lemma \ref{lem:orlicz-subspace-embedding} immediately yields subspace embeddings of size $O(d^{\max\{2,p_G/2+1\}}\poly(\log n, \eps^{-1}))$ by computing Orlicz sensitivity upper bounds as in Corollary \ref{cor:orlicz-sensitivity-upper-bound-alg} with $\tau = 1$. However, this does not yield input sparsity time algorithms, and as done in previous sections (see, e.g., Theorem \ref{thm:m-active-regression}), we can apply Theorem \ref{thm:m-estimator-alg} in a two step process in order to obtain these subspace embeddings in nearly input sparsity time.

\begin{proof}[Proof of Theorem \ref{thm:input-sparsity-time-orlicz-subspace-embedding}]
    We first obtain sensitivity upper bounds $\tilde\bfs_i^G(\bfA)$ by Corollary \ref{cor:orlicz-sensitivity-upper-bound-alg} with $\tau = T$, so that they sum to
    \[
        \tilde{\mathcal T}^G(\bfA) = \sum_{i=1}^n \tilde\bfs_i^G(\bfA) = O\parens*{d^{\max\{1,p_G/2\}}\log^2 n + T}
    \]
    and can be computed in time
    \[
        O\parens*{\nnz(\bfA)\log^3 n + n\log n} = O\parens*{\nnz(\bfA)\log^2 n}.
    \]
    Then by Lemma \ref{lem:orlicz-subspace-embedding}, we can obtain a set of weights $\bfw$ such that
    \[
        \nnz(\bfw) \leq O\parens*{(d^{\max\{1,p_G/2\}}\log^2 n + T) \cdot \frac{d}{\eps^2}\log\frac1\eps}
    \]
    with constant probability, and
    \[
        \norm*{\bfy}_{G,\bfw} = (1\pm\eps)\norm*{\bfy}_G
    \]
    for all $\bfy\in\Span(\bfA)$. We then compute weighted Orlicz sensitivities $\tilde\bfs_i^{G,\bfw}(\bfA)$ for $\bfw$ via Corollary \ref{cor:weighted-orlicz-sensitivity-upper-bound-alg} with $\tau = d$. Since the sensitivity upper bounds are at least $1/n$, $\norm*{\bfw}_\infty = O(n)$ so that the weighted Orlicz sensitivity upper bounds sum to at most
    \[
        \tilde{\mathcal T}^G(\bfA) = \sum_{i=1}^n \tilde\bfs_i^G(\bfA) = O\parens*{d^{\max\{1,p_G/2\}}\log^3 n}.
    \]
    Furthermore, the number of nonzero entries rows is at most $\nnz(\bfw)$, so these can be computed in time at most
    \[
        O\parens*{\nnz(\bfA)\log^3 n + \frac{(Td^{\max\{1,p_G/2\}} + T^2)\log^4 n}{\eps^2}\log\frac1\eps},
    \]
    Again applying Lemma \ref{lem:orlicz-subspace-embedding}, we can obtain a random set of weights $\bfw'$ such that, with constant probability,
    \[
        \nnz(\bfw') \leq O\parens*{d^{\max\{1,p_G/2\}}\log^3 n \cdot \frac{d}{\eps^2}\log\frac1\eps} = O\parens*{\frac{d^{\max\{2,p_G/2+1\}} \log^3 n}{\eps^2}\log\frac1\eps} 
    \]
    and 
    \[
        \norm*{\bfy}_{G,\bfw'} = (1\pm\eps)\norm*{\bfy}_{G,\bfw} = (1\pm\eps)^2\norm*{\bfy}_G
    \]
    for all $\bfy\in\Span(\bfA)$. Rescaling $\eps$ by constant factors gives the desired conclusion.
\end{proof}
\section{Applications: Robust Subspace Approximation}\label{sec:robust-subspace-approx}

By using our improvements in sensitivity sampling, we are able to obtain improvements to the robust subspace approximation algorithms of \cite{ClarksonWoodruff:2015a}.

In the subspace approximation problem, we are given a matrix $\bfA\in\mathbb R^{n\times d}$, which we think of as a set of $n$ points in $\mathbb R^d$. Our task is then to compute a $d\times d$ rank $k$ projection matrix $\bfX$ that minimizes
\[
    \norm*{\bfA\bfX - \bfA}_M,
\]
where here, the $M$-norm of an $n\times d$ matrix $\bfM$ is defined as
\[
    \norm*{\bfM}_M = \parens*{\sum_{i=1}^n M(\norm*{\bfe_i^\top\bfM}_2)}^{1/p_M}
\]
or equivalently, the vector $M$-norm for the vector of $\ell_2$ row norms of $\bfM$. We also define the weighted version
\[
    \norm*{\bfM}_{M,\bfw} = \parens*{\sum_{i=1}^n \bfw_i M(\norm*{\bfe_i^\top\bfM}_2)}^{1/p_M}
\]
We consider approximation algorithms to the above problem, so that we seek a rank $k$ $\bfX$ such that
\[
    \norm*{\bfA\bfX - \bfA}_M \leq (1+\eps)\min_{\rank(\bfY) = k}\norm*{\bfA\bfY - \bfA}_M.
\]
For this problem, \cite{ClarksonWoodruff:2015a} obtained an algorithm with parameters depending on $K = (\log n)^{O(\log k)}$, running in time
\[
    O\parens*{\nnz(\bfA)\log n + (n+d)\poly(K/\eps)}
\]
to reduce the problem to an instance of size $\poly(K/\eps)\times \poly(K/\eps)$. Furthermore, their algorithm makes use of an oblivious subspace embedding, which is only applicable to $M$-estimators with growth bounded by a quadratic, i.e., $p_M\leq 2$. By using our improved sensitivity bounds of Section \ref{sec:sensitivity-bounds}, we match the guarantees of their robust subspace approximation and regression algorithms for $\ell_p$, up to $\poly(\log(n))$ factors, for any $p_M$. That is, we obtain the following improved theorem:

\begin{theorem}
    Let $\norm*{\cdot}_M$ be an $M$-norm. There is an algorithm which computes matrices $\bfT$, $\bfS$, and $\bfU$, and a set of weights $\bfw$ such that
    \[
        \min_{\rank(\bfY) = k}\norm*{\bfT\bfA\bfS^\top - (\bfT\bfA\bfU)\bfY(\bfU\bfS^\top)}_{M,\bfw} \leq (1+\eps)\min_{\rank(\bfY) = k}\norm*{\bfA - \bfA\bfY}_M.
    \]
    Furthermore, $\bfT$ and $\bfS$ only have $\poly(k(\log n)/\eps)$ rows and $\bfU$ only has $\poly(k(\log n)/\eps)$ columns, and the matrices $\bfT\bfA\bfS^\top$, $\bfT\bfA\bfU$, and $\bfU\bfS^\top$ can be formed in
    \[
        O\parens*{[\nnz(\bfA) + (n+d)\poly(k/\eps)]\poly(\log n)}
    \]
    time.
\end{theorem}

We now briefly sketch out the changes, leaving many of the details to the \cite{ClarksonWoodruff:2015a} paper, as the algorithm and analysis is essentially identical to their algorithms for $\ell_p$ losses. The approach of \cite{ClarksonWoodruff:2015a}, like many other coreset algorithms, begins with an initial crude relative error approximation, which is later refined down to $(1+\eps)$ relative error. 

\subsection{Crude Relative Error Approximation}

In this section, we derive the analogues of the \textsc{ConstApprox} algorithm, which is Algorithm 4 and Theorem 47 of \cite{ClarksonWoodruff:2015a}. This computes a $\poly(k\log n)$-dimensional subspace, such that projection to this subspace is within a $\poly(k\log n)$ factor of optimal.

The first step is to reduce the number of columns of $\bfA$ via standard $\ell_2$ dimension reduction techniques, as formalized by the following theorem. 

\begin{theorem}[Theorem 32 of \cite{ClarksonWoodruff:2015a}]\label{thm:cw15-32}
Let $M:\mathbb R_{\geq0}\to\mathbb R_{\geq0}$ be increasing and polynomially bounded with degree $p$ (Definition \ref{dfn:polynomially-bounded}). Let $\bfR\in\mathbb R^{m\times d}$ be an \textsf{OSNAP} sparse embedding \cite{NelsonNguyen:2013} with sparsity parameter $s$. Then, there is $s = O(p^3/\eps)$ and $m = O(k^2/\eps^{O(p)})$ such that with constant probability,
\[
    \min_{\rank(\bfX) = k} \norm*{\bfA\bfR^\top\bfX - \bfA\bfR^\top}_M \leq (1+\eps)\min_{\rank(\bfY) = k}\norm*{\bfA\bfY - \bfA}_M.
\]
\end{theorem}

By applying the above theorem with $\eps$ set to a constant, we make work instead with $\bfA\bfR$, which only has $\poly(k)$ columns. We may then apply our improved sensitivity bounds to efficiently reduce the number of rows for a crude bicriteria guarantee. 

\begin{algorithm}
	\caption{\textsc{ConstApprox}$(\bfA,k)$}
	\textbf{input:} Matrix $\bfA \in \mathbb R^{n \times d}$, rank parameter $k$ \\
	\textbf{output:} Projection $\bfU\bfU^\top$ 
	\begin{algorithmic}[1] 
        \State Let $\bfR$ be a sparse embedding matrix from Theorem \ref{thm:cw15-32} 
        \State Compute $M$-sensitivity upper bounds $\tilde\bfs_i^M(\bfA\bfR)$ by Theorem \ref{thm:m-estimator-alg} with $\tau = \poly(k)$
        \State Let $r$ be sufficiently large in $\poly(k)$
        \State Let $\bfw_i$ be $1/\bfp_i$ w.p. $\bfp_i$ and $0$ otherwise, for $\bfp_i \coloneqq \min\{1,r\cdot \tilde\bfs_i^M(\bfA\bfR)\}$
        \State \Return $\bfU\bfU^\top$, where $\bfU^\top$ is an orthonormal basis for the row space of 
	\end{algorithmic}\label{alg:const-approx}
\end{algorithm}

\begin{lemma}[Improvement to Theorem 47 of \cite{ClarksonWoodruff:2015a}]
    With constant probability, the matrix $\bfU$ output by Algorithm \ref{alg:const-approx} has
    \[
        \norm*{\bfA - \bfA\bfU\bfU^\top}_{M,\bfw} \leq \poly(k)\min_{\rank(\bfY) = k}\norm*{\bfA - \bfA\bfY}_M
    \]
    The running time is
    \[
        O\parens*{[\nnz(\bfA) + (n+d)\poly(k)]\poly\log (n)}
    \]
    with high probability.
\end{lemma}
\begin{proof}
    The proof is nearly identical to the $\ell_p$ case of Theorem 47 of \cite{ClarksonWoodruff:2015a}. Instead of using their Lemma 39 to compute sensitivity upper bounds that sum to approximately $\poly(k)\sqrt n$, we use our Theorem \ref{thm:m-estimator-alg} to compute sensitivity upper bounds that sum to $\poly(k\log n)$. Then, we only require one round of sampling in order to reduce the number of rows to $\poly(k\log n)$, and thus results in guarantees analogous to their $\ell_p$ loss case, up to a $\poly\log(n)$ factor.
\end{proof}

Note that \cite{ClarksonWoodruff:2015a} only achieve a much larger distortion of $\poly(k)^{O(\log\log n)} = (\log n)^{O(\log k)}$ rather than $\poly(k\log n)$, due to their recursive application of sensitivity sampling for a depth of $O(\log\log n)$, which each incurs a distortion of $\poly(k)$. 

\subsection{Refinement to \texorpdfstring{$(1+\eps)$}{(1+eps)}-Approximations}

Using the crude $\poly(k)$ approximation, the next step in the \cite{ClarksonWoodruff:2015a} algorithm is to adapt a result of \cite{DeshpandeVaradarajan:V2007} to obtain a $d\times \poly(k(\log n)/\eps)$ matrix $\bfU$ with orthonormal columns such that
\[
    \min_{\rank(\bfY) = k}\norm*{\bfA - \bfA\bfU\bfY\bfU}_{M,\bfw} \leq (1+\eps)\min_{\rank(\bfY) = k}\norm*{\bfA - \bfA\bfY}_M.
\]
This is the \textsc{DimReduce} algorithm described as Algorithm 2 of \cite{ClarksonWoodruff:2015a} as well as Theorem 46. 

\begin{theorem}[Theorem 46 of \cite{ClarksonWoodruff:2015a}]
    Let $K>0$ and let $\hat\bfX\in\mathbb R^{d\times d}$ be a $m$-dimensional projection such that
    \[
        \norm*{\bfA(\bfI - \hat\bfX)}_{M,\bfw} \leq K\min_{\rank(\bfY) = k}\norm*{\bfA(\bfI - \bfY)}_{M}.
    \]
    Then, there is an algorithm that returns $\bfU\in\mathbb R^{d\times(m+K\poly(k/\eps))}$ such that
    \[
        \min_{\rank(\bfY) = k}\norm*{\bfA - \bfA\bfU\bfY\bfU}_{M,\bfw} \leq (1+\eps)\min_{\rank(\bfY) = k}\norm*{\bfA - \bfA\bfY}_M.
    \]
\end{theorem}

This problem can then by sketched from the left by sparse embedding $\bfS$ with $\poly(k(\log n)/\eps)$ rows so that we are left with solving the problem
\[
    \min_{\rank(\bfY) = k}\norm*{\bfA\bfS^\top - \bfA\bfU\bfY\bfU\bfS^\top}_{M,\bfw}.
\]

Finally, the last step is to conduct another row sampling process to reduce the number of rows $n$ to a $\poly(k(\log n)/\eps)$-sized instance. For this step, we again use our improved sensitivity sampling lemma as in Theorem \ref{thm:m-estimator-alg}. However, there is an additional challenge in this setting in that we cannot afford to compute the matrix $\bfA\bfU$ explicitly, since this requires $\nnz(\bfA)\poly(k(\log n)/\eps)$ time to compute. The solution to this problem described in \cite{ClarksonWoodruff:2015a} is to use an oblivious sparse $\ell_p$ subspace embedding $\bfPi$ and then to compute a QR decomposition of $\bfPi\bfA[\bfS^\top~\bfU] = \bfQ\bfR$ in order to obtain a well-conditioned basis $\bfU = \bfA[\bfS^\top~\bfU]\bfR^{-1}$. Then, the row norms of $\bfU$, which can be used as approximate sensitivity upper bounds, are estimated by multiplying by a Gaussian matrix $\bfG$ with $O(1)$ columns, so that $\bfU\bfG = \bfA([\bfS^\top~\bfU]\bfR^{-1}\bfG)$ can be computed efficiently in $O(\nnz(\bfA))$ time. However, this approach has a drawback, in that the use of an oblivious subspace embedding restricts the range of applicable growth upper bounds $p_M$ to at most $2$. 

Instead, as observed in \cite{ShiWoodruff:2019}, we make use of the fact that Lewis weight computation can be done efficiently as long as matrix vector products can be computed efficiently, so that $\nnz(\bfA)$ dependencies can approximately be replaced by $T(\bfA)$ dependencies, where $T(\bfA)$ denotes the time required to compute a matrix vector product. Thus, for computing the Lewis weights of $\bfA[\bfS^\top~\bfU]$, we can approximately replace all $\nnz(\bfA)$ terms with $\nnz(\bfA) + \poly(k(\log n)/\eps)$, by multiplying the product one matrix at a time. This is implicit in the proof of Theorem 3 of \cite{ShiWoodruff:2019}. 

Thus, following the \textsc{Approx} algorithm described as Algorithm 7 with their sensitivity bound replaced by our sensitivity bound of Theorem \ref{thm:m-estimator-alg}, we may find a set of weights $\bfw'$ such that
\[
    \norm*{\bfA\bfS^\top - \bfA\bfU\bfY\bfU\bfS^\top}_{M,\bfw'} = (1\pm\eps)\norm*{\bfA\bfS^\top - \bfA\bfU\bfY\bfU\bfS^\top}_{M,\bfw}.
\]
for all $\bfY$ with $\rank(\bfY) = k$, with $\poly(k(\log n)/\eps)$ nonzero entries. After this step, all of the matrices involved are of size $\poly(k(\log n)/\eps) \times \poly(k(\log n)/\eps)$, as desired.

\section{Lower Bounds}\label{sec:lower}
We now present various lower bounds on active sampling for $\ell_p$ regression. Our first result obtains a nearly optimal lower bound of $\Omega(d/\eps^2)$ for $p\in(0,1)$, which matches our upper bound of Theorem \ref{thm:main} up to logarithmic factors. Our lower bound is similar to Theorem 5.1 of \cite{ChenDerezinski:2021}, and is based on distinguishing biased coin flips. We also show that the same instance gives a lower bound $\Omega(d/\eps)$ in the range of $p\in(1,2)$, which is also tight with our upper bound in Theorem \ref{thm:main}. For $d = 1$, we note that the active $\ell_p$ regression problem is equivalent to the \emph{$\ell_p$ power means} problem. Thus, these results also improve upon a query complexity lower bound for this problem by \cite{CSS2021}, which shows a lower bound of $\Omega(\eps^{1-p})$ in one dimension. However, this lower bound is useful for us for $p > 2$, which shows that our Theorem \ref{thm:main} is off by a factor of $\eps$ in the $\eps$ dependence.

Our next result shows that $\Omega(d^{\max(1,p/2)})$ samples are required to solve Problem \ref{prob:informal_problem} for the $\ell_p$ norm for any $p \ge 1$ to a constant approximation factor. 
This show that Algorithm \ref{alg:clip} is optimal in its dependence on $d$, up to 
polylogarithmic factors (see Theorem \ref{thm:main}).

Finally, our last lower bound concerns algorithms which solve the $\ell_p$-regression problem $\min_\bfx \|\bfA\bfx-\bfb\|_p$ up to a constant factor in a specific way. Namely, say an algorithm
is a {\it sampling-and-reweighting} algorithm if, given an $n \times d$ input matrix $\bfA$, the algorithm first reads $\bfA$ and then decides on a subset $S$ of $s$ entries of $\bfb$ to read in an arbitrary way. The algorithm also decides on a diagonal rescaling matrix $\bfD \in \R^{s \times s}$ -- $\bfD$ may be arbitrary, except that we require that if two rows of $\bfA$ are identical and are both sampled in $S$, they are given the same weight in $\bfD$. We also assume that the number $s$ of samples is a function of $d, \epsilon,$ and the failure probability $\delta$, and is independent of $n$. These assumptions hold for all importance-based sampling methods for subspace preservation.

After deciding on $S$ and $\bfD$, the algorithm then reads
the entries in $\bfb$ indexed by the set $S$, denoted $\bfb_S$, and sets 
$\bfx' = \textrm{argmin}_x \|\bfD \bfA_S\bfx - \bfD \bfb_S\|_p$, where $\bfA_S$ 
is the subset of rows of $\bfA$ corresponding to the entries in $\bfb_S$. 
We show that any sampling-and-reweighting algorithm which
fails with probability at most $\delta$, necessarily takes $|S| = \Omega(1/\delta^{p-1})$ samples.
Moreover, this remains true even if $\|\bfb\|_p = O(1) \cdot \min_\bfx \|\bfA\bfx-\bfb\|_p$. 

We stress that our main algorithms {\it are not} sampling-and-reweighting algorithms due to the success probability boosting steps of Section \ref{sec:boost} which ensure that $\|\bfb\|_p = O(1) \cdot \min_\bfx \|\bfA\bfx-\bfb\|_p$ and $\norm{\bfS\bfb}_p = O(\norm{\bfb}_p)$ with probability at least $1-\delta$. With these steps, our approach 
%
achieves a $O(\log 1/\delta)$ dependence overall, an exponential improvement over what is possible by simple sampling-and-reweighting algorithms. We also remark that our lower bound becomes vacuous when $p = 1$, which is required -- Chen and Derezi\'{n}ski \cite{ChenDerezinski:2021} as well as Parulekar, Parulekar, and Price \cite{ParulekarParulekarPrice:2021} achieve $O(\log(1/\delta))$ dependence with simple sampling-and-reweighting for $\ell_1$ regression.

\subsection{Lower Bounds for \texorpdfstring{$p\in(0,1)$}{p in (0, 1)}}

We first show an $\Omega(d/\eps^2)$ lower bound for $p\in(0,1)$, which is tight up to logarithmic factors. The idea is essentially the same as the lower bound of \cite{ChenDerezinski:2021}. We use Yao's minimax principle to restrict our attention to deterministic algorithms which must succeed with high probability over a random distribution over input instances.

We first recall the result of \cite{ChenDerezinski:2021}, which provides a generic reduction from $d$-dimensional lower bounds to $1$-dimensional lower bounds via a padding argument. Although \cite{ChenDerezinski:2021} prove a theorem online in the case of $\ell_1$, the following result is an easy generalization that is implicit from their proof:

\begin{theorem}[Theorem 5.1, \cite{ChenDerezinski:2021}]\label{thm:cd-lb}
Let $\mathcal D_0$ and $\mathcal D_1$ be two distributions over label vectors $\bfb\in\mathbb R^m$ such that distinguishing between $\bfb\sim\mathcal D_0$ and $\bfb\sim\mathcal D_1$ with probability at least $2/3$ requires at least $q$ queries to $\bfb$ in expectation, for any deterministic algorithm. Furthermore, suppose that there exists $\bfa\in\mathbb R^m$ such that, with probability at least $99/100$, $\mathcal D_0$ and $\mathcal D_1$ can be distinguished by $\tilde x\in\mathbb R$ such that
\[
    \norm*{\bfa\tilde x-\bfb}_p^p \leq (1+\eps)\min_{x\in\mathbb R}\norm*{\bfa x-\bfb}_p^p.
\]
Finally, suppose that there exist $R>0$ and $c\geq1$ such that $\min_x\norm*{\bfa x-\bfb}_p^p \in [R, cR]$ with probability at least $99/100$ for $\bfb\sim\frac12(\mathcal D_0+\mathcal D_1)$. Then, there exists an $md\times d$ matrix $\bfA$ and a distribution $\mathcal D$ over label vectors $\bfb\in\mathbb R^{md}$ such that any deterministic algorithm which outputs $\tilde\bfx\in\mathbb R^d$ such that
\[
    \Pr\braces*{\norm*{\bfA\tilde\bfx-\bfb}_p^p \leq \parens*{1+\frac{\eps}{200 c}}\min_{\bfx\in\mathbb R^d}\norm*{\bfA\bfx-\bfb}_p^p} \geq \frac{99}{100}
\]
must make at least $\Omega(dq)$ queries to $\bfb$ in expectation.
\end{theorem}

Thus, it suffices to show a $1$-dimensional lower bound which suits the hypotheses of Theorem \ref{thm:cd-lb}. Our hard input distribution will be the same as that of \cite[Theorem 5.1]{ChenDerezinski:2021}. 

\begin{theorem}\label{thm:p-in-0-1-lb}
Let $0 < p < 1$ be a constant. Let $\eps>0$ be sufficiently small and let $n = 100\ceil*{\eps^{-2}}$. Let $\bfa\in\mathbb R^n$ be the all ones vector. Let $\mathcal D_0$ be the distribution over binary vectors $\bfb\in\{0,1\}^n$ which independently draws each coordinate as a Bernoulli with bias $1/2 + \eps$ and let $\mathcal D_1$ be the distribution which independently draws each coordinate as a Bernoulli with bias $1/2 - \eps$. Then, any $\tilde x$ such that
\[
    \norm*{\bfa\tilde x-\bfb}_p^p \leq (1+\eps)\min_{x\in\mathbb R}\norm*{\bfa x-\bfb}_p^p
\]
distinguishes whether $\bfb\sim\mathcal D_0$ or $\bfb\sim\mathcal D_1$ with probability at least $99/100$.
\end{theorem}
\begin{proof}
Note that the optimal $x^*$ minimizing $\norm*{\bfa x-\bfb}_p^p$ over $x\in\mathbb R$ must lie in $[0, 1]$. Indeed, if $x < 0$, then $-x$ has a strictly lower cost than $x$, and if $x > 1$, then $x = 1$ has a strictly lower cost. Thus, the objective function can be written as
\begin{equation}\label{eq:lp-obj-p<1}
    \norm*{\bfa x-\bfb}_p^p = (n-r)\cdot x^p + r\cdot (1-x)^p
\end{equation}
where $r$ is the number of ones in $\bfb$. As noted by \cite{ChenDerezinski:2021}, note that $r \in [(\frac12 + \frac\eps2)n, (\frac12 + \frac{3\eps}2)n]$ with probability at least $99/100$ if $\bfb\sim\mathcal D_0$, and similarly, $r \in [(\frac12 - \frac\eps2)n, (\frac12 - \frac{3\eps}2)n]$ with probability at least $99/100$ if $\bfb\sim\mathcal D_1$. Let this event be denoted as $\mathcal E$, and condition on this event. Write this as $r = n/2 + a\eps n$ for some $a\in[1/2, 3/2]$ if $\bfb\sim\mathcal D_0$ and $a\in[-3/2,-1/2]$ if $\bfb\sim\mathcal D_1$. 

\paragraph{Optimal Solutions.} We will first compute the optimal cost. Since $x\mapsto x^p$ for $p\in(0,1)$ is nonconvex, we have three candidate solutions for the optimum: the endpoints $x = 0$, $x = 1$, and the unique stationary point
\[
    x = \frac1{1+(n/r-1)^{1/(p-1)}}
\]
of Equation \eqref{eq:lp-obj-p<1}. It can easily be seen that the costs for $x = 0$ and $x = 1$ are $r$ and $n-r$, respectively. Now assume that $\bfb\sim\mathcal D_0$, since the other case follows symmetrically. Then,
\[
    (n/r-1)^{1/(p-1)} = \bracks*{\frac{1}{1/2 + a\eps}-1}^{1/(p-1)} = \bracks*{\frac{1/2 + a\eps}{1/2 - a\eps}}^{1/(1-p)} = 1 + O(\eps)
\]
so
\[
    x = \frac1{2 + O(\eps)} = \frac12 - O(\eps).
\]
The objective cost is thus
\begin{align*}
    (n-r)\cdot x^p + r\cdot(1-x)^p &= \parens*{\frac12-a\eps}n\cdot\parens*{\frac12-O(\eps)}^p + \parens*{\frac12+a\eps}n\cdot\parens*{\frac12+O(\eps)}^p \\
    &= \frac{n}{2^p}\bracks*{\parens*{\frac12-a\eps}(1-O(\eps))^p + \parens*{\frac12+a\eps}(1+O(\eps))^p} \\
    &= \frac{n}{2^p}(1\pm O(\eps)).
\end{align*}
Since $p < 1$, this has cost worse than $r$ or $n-r$, so the optimal solution is $n-r$, conditioned on $\mathcal E$. Likewise, if $\bfb\sim\mathcal D_1$, then the optimal solution is $x = 0$ with cost $r$.

\paragraph{Suboptimal Solutions.} We now show that, given a nearly optimal solution $\tilde x$, we can determine whether $\bfb$ is drawn from $\mathcal D_0$ or $\mathcal D_1$ with high probability by testing whether $\tilde x \in [1/2, 1]$ or $\tilde x \in [0, 1/2]$. Again, assume by symmetry that $\bfb\sim\mathcal D_0$. Suppose that $x \in [0, 1/2]$. If $x > 1/2 - O(\eps)$, then as calculated above, $x$ is not even an $\alpha$-factor solution for some constant $\alpha$. Otherwise, we have that
\[
    x \leq \frac12 - O(\eps) = \frac1{1 + (n/r-1)^{1/(p-1)}}
\]
which rearranges to 
\[
    p\cdot (n-r)x^{p-1} - pr\cdot (1 - x)^{p-1} > 0
\]
which means that the objective is increasing on this interval. Thus, for $x\in[0,1/2]$, the smallest that the cost can be is $r$, which is a factor of $(1+\eps)$ larger than $n-r$. Thus, a $(1+\eps)$-factor approximation must distinguish between $\bfb$ drawn from $\mathcal D_0$ and $\mathcal D_1$.
\end{proof}

As discussed in \cite{ChenDerezinski:2021}, the distributions of Theorem \ref{thm:p-in-0-1-lb} require at least $\Omega(\eps^{-2})$ queries to distinguish, by standard arguments. Then, by combining Theorems \ref{thm:cd-lb} and \ref{thm:p-in-0-1-lb}, we arrive at the following:

\begin{theorem}\label{thm:p-in-0-1-lb-d-dim}
Let $p\in (0,1)$ be a constant. Let $\bfA\in\mathbb R^{n\times d}$ and let $\bfb\in\mathbb R^n$. Suppose that with probability at least $99/100$, an algorithm $\mathcal A$ returns $\tilde\bfx\in\mathbb R^d$ such that
\[
    \norm*{\bfA\tilde\bfx - \bfb}_p^p \leq (1+\eps)\min_{\bfx\in\mathbb R^d}\norm*{\bfA\bfx-\bfb}_p^p.
\]
Then, $\mathcal A$ queries $\Omega(d/\eps^2)$ entries of $\bfb$ in expectation.
\end{theorem}

\subsection{Lower Bounds for \texorpdfstring{$p\in(1, 2)$}{p in (1, 2)}}

In the range of $p\in (1,2)$, we analyze the same lower bound instance as in Theorem \ref{thm:p-in-0-1-lb}. However, the nature of the objective function changes in this parameter regime, and our lower bound weakens to $\Omega(d/\eps)$. In particular, the value of the endpoints $x = 0$ and $x = 1$ stay at $r$ and $n-r$, but the value of the stationary point, which is near $x = 1/2$ and has a value of around $n/2^p$, becomes significantly better than the endpoint solutions. This causes a phase transition in the lower bound that we are able to achieve with this method.

We now present our $1$-dimensional lower bound of $\Omega(\eps^{-1})$ for $p\in(1,2)$. For easy reuse of our calculations from Theorem \ref{thm:p-in-0-1-lb}, we state this result as a lower bound of $\Omega(\eps^{-2})$ for any algorithm achieving an $O(\eps^2)$-approximation. This can be reparameterized to an $\Omega(\eps^{-1})$ lower bound for $O(\eps)$-approximations.

\begin{theorem}\label{thm:p-in-1-2-lb}
Let $1 < p < 2$ be a constant. Let $\eps>0$ be sufficiently small and let $n = 100\ceil*{\eps^{-2}}$. Let $\bfa\in\mathbb R^n$ be the all ones vector. Let $\mathcal D_0$ be the distribution over binary vectors $\bfb\in\{0,1\}^n$ which independently draws each coordinate as a Bernoulli with bias $1/2 + \eps$ and let $\mathcal D_1$ be the distribution which independently draws each coordinate as a Bernoulli with bias $1/2 - \eps$. Then, there exists a constant $c$ such that any $\tilde x$ such that
\[
    \norm*{\bfa\tilde x-\bfb}_p^p \leq (1+c\cdot\eps^2)\min_{x\in\mathbb R}\norm*{\bfa x-\bfb}_p^p
\]
distinguishes whether $\bfb\sim\mathcal D_0$ or $\bfb\sim\mathcal D_1$ with probability at least $99/100$.
\end{theorem}
\begin{proof}
Many of our calculations from Theorem \ref{thm:p-in-0-1-lb} directly carry over. Recall our notation of setting $r$ to be the number of ones in $\bfb$, which is $r = n/2 + a\eps n$ for $a\in[1/2,3/2]$ if $\bfb\sim\mathcal D_0$ and $a\in[-3/2,-1/2]$ if $\bfb\sim\mathcal D_1$. Recall also that the unique stationary point, which is the optimum now by convexity, is
\[
    x = \frac1{1+(n/r-1)^{1/(p-1)}}.
\]

\paragraph{Optimal Solutions.}
We now calculate the value of the optimum. We carry out calculations for $\bfb\sim\mathcal D_0$ since $\bfb\sim\mathcal D_1$ gives symmetric results. Note first that
\begin{align*}
    (n/r-1)^{1/(p-1)} &= \bracks*{\frac{1}{1/2 + a\eps}-1}^{1/(p-1)} = \bracks*{\frac{1 - 2a\eps}{1 + 2a\eps}}^{1/(p-1)} = \bracks*{(1 - 2a\eps)\cdot \sum_{i=0}^\infty (-2a\eps)^i}^{1/(p-1)} \\
    &= \bracks*{(1 - 2a\eps) - (2a\eps)(1-2a\eps) + (2a\eps)^2(1-2a\eps) + O(\eps^3)}^{1/(p-1)} \\
    &= \bracks*{1 - 4a\eps + 2(2a\eps)^2 + O(\eps^3)}^{1/(p-1)} \\
    &= 1 - \frac{4a}{p-1}\eps + \frac{8a^2}{p-1} \eps^2 + O(\eps^3)
\end{align*}
so
\[
    x = \frac12 \parens*{1 + \frac{2a}{p-1}\eps - \frac{4a^2}{p-1}\eps^2 + O(\eps^3)}.
\]
Then, the objective value at this $x$ is
\begin{align*}
    &\frac{n}{2^p}\bracks*{\parens*{\frac12-a\eps}\parens*{1 + \frac{2a}{p-1}\eps - \frac{4a^2}{p-1}\eps^2 + O(\eps^3)}^p + \parens*{\frac12+a\eps}\parens*{1 - \frac{2a}{p-1}\eps + \frac{4a^2}{p-1}\eps^2 + O(\eps^3)}^p} \\
    = ~&\frac{n}{2^p}\bracks*{\parens*{\frac12-a\eps}\parens*{1 + \frac{2pa}{p-1}\eps - \frac{4pa^2}{p-1}\eps^2 + O(\eps^3)} + \parens*{\frac12+a\eps}\parens*{1 - \frac{2pa}{p-1}\eps + \frac{4pa^2}{p-1}\eps^2 + O(\eps^3)}} \\
    = ~&\frac{n}{2^p}\bracks*{1 + \parens*{\frac12-a\eps}\parens*{\frac{2pa}{p-1}\eps - \frac{4pa^2}{p-1}\eps^2 + O(\eps^3)} + \parens*{\frac12+a\eps}\parens*{- \frac{2pa}{p-1}\eps + \frac{4pa^2}{p-1}\eps^2 + O(\eps^3)}} \\
    = ~&\frac{n}{2^p}\bracks*{1 - \frac{4pa^2}{p-1}\eps^2 + O(\eps^3)}.
\end{align*}
If $\bfb\sim\mathcal D_1$, then we have that
\[
    x = \frac12 \parens*{1 - \frac{2a}{p-1}\eps + \frac{4a^2}{p-1}\eps^2 + O(\eps^3)}
\]
with the same objective value.

\paragraph{Suboptimal Solutions.}
We now show that, given a nearly optimal solution $\tilde x$, we can determine whether $\bfb$ is drawn from $\mathcal D_0$ or $\mathcal D_1$ with high probability by testing whether $\tilde x \in [1/2, 1]$ or $\tilde x \in [0, 1/2]$. Again, assume by symmetry that $\bfb\sim\mathcal D_0$. Suppose that $x \in [0, 1/2]$. Then,
\[
    x \leq \frac12 < \frac1{1 + (n/r-1)^{1/(p-1)}}
\]
which rearranges to 
\[
    p(n-r)\cdot x^{p-1} - pr\cdot (1 - x)^{p-1} < 0
\]
which means that the objective is decreasing on this interval. Thus, for $x\in[0,1/2]$, the smallest that the cost can be is $x = 1/2$, which gives a value of
\[
    (n-r)\cdot (1/2)^p + r\cdot (1/2)^p = \frac{n}{2^p},
\]
which is a factor of $1 - \Theta(\eps^2)$ larger than the optimal solution. Thus, a $(1+\Theta(\eps^2))$-approximate solution can distinguish between $\mathcal D_0$ and $\mathcal D_1$.
\end{proof}

Then, by combining Theorems \ref{thm:cd-lb} and \ref{thm:p-in-1-2-lb}, we arrive at the following:

\begin{theorem}\label{thm:p-in-1-2-lb-d-dim}
Let $p\in (1,2)$ be a constant. Let $\bfA\in\mathbb R^{n\times d}$ and let $\bfb\in\mathbb R^n$. Suppose that with probability at least $99/100$, an algorithm $\mathcal A$ returns $\tilde\bfx\in\mathbb R^d$ such that
\[
    \norm*{\bfA\tilde\bfx - \bfb}_p^p \leq (1+\eps)\min_{\bfx\in\mathbb R^d}\norm*{\bfA\bfx-\bfb}_p^p.
\]
Then, $\mathcal A$ queries $\Omega(d/\eps)$ entries of $\bfb$ in expectation.
\end{theorem}

\subsection{Lower Bounds for \texorpdfstring{$p>2$}{p > 2}}

In the range of $p>2$, a lower bound of $\Omega(\eps^{1-p})$ follows from Theorem 3 of \cite{CSS2021}, which shows that in the $1$-dimensional $\ell_p$ power mean problem, which is equivalent to $1$-dimensional active $\ell_p$ regression, $\Omega(\eps^{1-p})$ queries are necessary to obtain $(1+\eps)$-approximate solutions. This matches our $\eps$ dependence in the upper bound by at most $\eps$. For $d$ dependence, it is clear that solving Problem \ref{prob:informal_problem} to any non-trivial relative error approximation factor for any norm requires at least $d$ samples, as otherwise, we would have non-zero error in the case that $\bfA\bfx = \bfb$. Here, we show that $\Omega(d^{p/2})$ samples are required to solve the problem for any $\ell_p$ norm with $p \ge 2$. In combination, this shows that the $\tilde O(d^{\max(1,p/2)})$ bound of Theorem \ref{thm:main} is tight up to polylogarithmic factors.
We use the following standard fact from coding theory. 

\begin{lemma}[\cite{parampalli2012construction}]\label{lem:coding}
For any real $q \geq 1$ and $d = 2^k-1$ for some integer $k$, there exists a set
$\bfS \subset \{-1,1\}^d$ and a constant $C_q$ depending on $q$ which satisfy
\begin{enumerate}
\item $|S| = d^q$
\item For any $s, t \in S$ such that $s \neq t$, $|\langle s, t \rangle | \leq C_q \sqrt{d}$.
\end{enumerate}
\end{lemma}

Using Lemma \ref{lem:coding} with $q = p/2$, we can find $r = d^{p/2}$ vectors $\bfx_1, \ldots,\bfx_r$ 
for which $\|x_i\|_2^2 = d$ and $|\langle\bfx_i,\bfx_j \rangle| \leq C \sqrt{d}$ for all $i \neq j$, for a constant $C > 0$. We work with a fixed subset of $s= C' \cdot d^{p/2}$ of these vectors, for
a sufficiently small constant $C' > 0$. Let these $s$ vectors form the rows $a_1,\ldots,a_s$ of $\bfA \in \R^{s \times d}$. 
We choose a uniformly random index $I \in [s]$ and set $\bfb_I = d$, while we set 
$\bfb_j = 0$ for all $j \neq I$. The claim is that any solution to the $\ell_p$-regression problem
can be used to predict the index $I$. However, by standard information-theoretic arguments, 
any algorithm for predicting $I$ with constant
probability requires reading $\Omega(s) = \Omega(d^{p/2})$ entries of $\bfb$. 

To see the claim, consider a $(1+\epsilon)$-approximate solution vector $\bfx$ to 
$\min_\bfx \|\bfA\bfx-\bfb\|_p$, where $\epsilon > 0$ is less than a sufficiently small constant. 
Suppose $| \langle a_j,\bfx \rangle| \geq d/2$ for some $j \neq I$. 
Then the regression cost is at least $(d/2)^p$, since this is the cost on the $j$-th coordinate
alone. Alternatively, suppose $| \langle a_I,\bfx \rangle | \leq d/2$. 
Then the regression cost is at least $(d/2)^p$, since this is the cost on the $I$-th coordinate
alone. On the other hand, consider the solution vector $\bfx' = a_I$. 
Then the regression cost is at most 
$(s-1) \cdot (C \sqrt{d})^{p} \leq d^p \cdot C^{p} \cdot C'.$ 
Since $C' > 0$ can be made an arbitrarily small constant (while suffering a constant factor
in the lower bound on the number of entries of $\bfb$ read),
the regression cost is a constant factor smaller than $(d/2)^p$.

Hence, any $(1+\epsilon)$-approximate solution vector $\bfx$, for $\epsilon > 0$ less
than a small enough constant, must be such that $|\langle a_I,\bfx \rangle | \geq d/2$ and 
$|\langle a_j,\bfx \rangle | < d/2$ for all $j \neq I$. Thus, 
from the solution vector $\bfx$, one can determine the value of $I$. We thus have:

\begin{theorem}
	\label{thm:general_lb}
Let $p \geq 2$ and $\epsilon > 0$ be less than a sufficiently small constant. Any algorithm which outputs
a $(1+\epsilon)$-approximate solution $\bfx$ to the $\ell_p$-regression problem $\min_{\bfx\in\R^d} \|\bfA\bfx-\bfb\|_p^p$
with constant probability requires reading $\Omega(d^{p/2})$ entries of $\bfb$. 
\end{theorem}

Lemma \ref{lem:coding} has been used to obtain lower bounds for the stronger subspace
sketch problem \cite{LWW20}, which is related, though does not immediately give a lower bound for regression.

\subsection{A \texorpdfstring{$1/\delta^{p-1}$}{1/delta**{p-1}} Lower Bound for Sampling-and-Reweighting Algorithms}

We next show that sampling-and-reweighting algorithms for $\ell_p$ regression must pay a polynomial dependence in the failure probability $\delta$, contrasting with the logarithmic dependence achieved by our approach.
\begin{theorem}\label{thm:delta-lb}
Let $p > 1$. Any sampling-and-reweighting algorithm which, with probability at least $1-\delta$, 
outputs a $(1+\epsilon)$-approximate solution $\bfx$
to the $\ell_p$-regression problem, for $\epsilon > 0$ less than a sufficiently small constant,
requires reading $\Omega(1/\delta^{p-1})$ entries of $\bfb$.
\end{theorem}
\begin{proof}
In our hard instance we will have $d = 1$ and require a sufficiently fine constant factor approximation
with failure probability $\delta$. Suppose, with these parameters, that there is an algorithm reading
$s$ entries. We set $n = s/\delta$, and will show that the algorithm cannot output
a constant factor approximation to the $\ell_p$-regression problem with probability at least 
$1-\delta$. 

Let $\bfA$ be a single column of $n$ $1$s. Since the entries of $\bfA$ are indistinguishable from each other, we can assume without loss of generality that the sampling-and-reweighting algorithm samples entries uniformly at random. By assumption, since the rows of $\bfA$ are all identical, the algorithm reweights the sampled rows uniformly (equivalently assigns weight $1$ to each sampled entry). We choose $\bfb = \bfe_I$ for a random standard basis vector $\bfe_I$.
For the optimal $x$, necessarily $0 \leq x \leq 1$ since if $x < 0$, replacing $x$ with $-x $ 
would give lower cost. Similarly, if $x > 1$, then replacing $x$ with $1$ would give lower cost. 
Then the cost is $(1-x )^p + (n - 1)x^p$. This is convex and differentiable for $p > 1$, and is minimized when 
the derivative is $0$. Differentiating, the optimal $x$ satisfies
$-p(1-x )^{p-1} + p (n-1)x^{p-1} = 0$, or
$(1-x )^{p-1} = (n - 1)x^{p-1}$. Taking $(p-1)$-th roots,
$1-x = (n - 1)^{1/(p-1)}x$, or $x = 1/(1 + (n-1)^{1/(p-1)})$. The optimal cost is therefore
$$\left (1-\frac{1}{1 + (n - 1)^{1/(p-1)}} \right )^p + \frac{n-1}{(1 + (n-1)^{1/(p-1)})^p}.$$
For $n = \omega(1)$, this is
$\Theta(1 + n^{1-p/(p-1)}) = O(1)$ for any constant $p > 1$.

On the other hand, given that $n = s/\delta$, with probability at least $\delta$, the algorithm's sample includes $\bfb_I$. 
If this is the case, for the sampled problem, the cost is 
$(1-x )^{p-1} + (s - 1)x^p$ for a given $x$. Setting the derivative to $0$,
 we now have that the optimal $x'$ for the sampled problem is:
$x' = \frac{1}{1 + (s-1)^{1/(p-1)}}$.
Computing the cost of using $x'$ for the original problem, our cost of using $x'$ is
$$\left (1-\frac{1}{(1 + (s - 1)^{1/(p-1)})^p} \right ) + \frac{(n-1)}{\left (1 + (s-1)^{1/(p-1)} \right )^p}.$$
The cost is at least the second term, which for $n = \omega(1)$ and $s = \delta n = \omega(1)$ is
$\Theta((s/\delta) / s^{p/(p-1)})$. This term must be $O(1)$ to be an 
$O(1)$-approximation, by our above calculation of the optimal cost. Hence,
$s^{p/(p-1) - 1} = \Omega(1/\delta)$, or $s^{1/(p-1)} = \Omega(1/\delta)$, or
$s = \Omega(1/\delta^{p-1})$.
\end{proof}

%
%

\section{Acknowledgements}
Cameron Musco's work on this project was supported in part by NSF Grants 2046235 and 1763618, along with an Adobe Research Grant. Christopher Musco was supported by NSF Grant 2045590. David P.\ Woodruff and Taisuke Yasuda were supported by ONR grant N00014-18-1-2562 and a Simons Investigator Award. Taisuke Yasuda thanks Cody Johnson and Yi Li for helpful discussions. We thank anonymous reviewers for comments which helped improve the presentation of the paper.

\bibliographystyle{alpha}
\bibliography{LPRegression}

\newcommand{\etalchar}[1]{$^{#1}$}
\begin{thebibliography}{SWMW89}

\bibitem[ABKS21]{ABKS2021}
Deeksha Adil, Brian Bullins, Rasmus Kyng, and Sushant Sachdeva.
\newblock Almost-linear-time weighted $\ell_p$-norm solvers in slightly dense
  graphs via sparsification.
\newblock In Nikhil Bansal, Emanuela Merelli, and James Worrell, editors, {\em
  48th International Colloquium on Automata, Languages, and Programming,
  {ICALP} 2021, July 12-16, 2021, Glasgow, Scotland (Virtual Conference)},
  volume 198 of {\em LIPIcs}, pages 9:1--9:15. Schloss Dagstuhl -
  Leibniz-Zentrum f{\"{u}}r Informatik, 2021.

\bibitem[AKM{\etalchar{+}}19]{AvronKapralovMusco:2019}
Haim Avron, Michael Kapralov, Cameron Musco, Christopher Musco, Ameya
  Velingker, and Amir Zandieh.
\newblock A universal sampling method for reconstructing signals with simple
  fourier transforms.
\newblock In {\em \STOC{2019}}, 2019.

\bibitem[AKPS19]{AKPS2019}
Deeksha Adil, Rasmus Kyng, Richard Peng, and Sushant Sachdeva.
\newblock Iterative refinement for $\ell_p$-norm regression.
\newblock In Timothy~M. Chan, editor, {\em Proceedings of the Thirtieth Annual
  {ACM-SIAM} Symposium on Discrete Algorithms, {SODA} 2019, San Diego,
  California, USA, January 6-9, 2019}, pages 1405--1424. {SIAM}, 2019.

\bibitem[ALS{\etalchar{+}}18]{AndoniLinSheng:2018}
Alexandr Andoni, Chengyu Lin, Ying Sheng, Peilin Zhong, and Ruiqi Zhong.
\newblock Subspace embedding and linear regression with {O}rlicz norm.
\newblock In {\em \ICML{2018}}, 2018.

\bibitem[BCLL18]{BCLL2018}
S{\'{e}}bastien Bubeck, Michael~B. Cohen, Yin~Tat Lee, and Yuanzhi Li.
\newblock An homotopy method for l\({}_{\mbox{p}}\) regression provably beyond
  self-concordance and in input-sparsity time.
\newblock In Ilias Diakonikolas, David Kempe, and Monika Henzinger, editors,
  {\em Proceedings of the 50th Annual {ACM} {SIGACT} Symposium on Theory of
  Computing, {STOC} 2018, Los Angeles, CA, USA, June 25-29, 2018}, pages
  1130--1137. {ACM}, 2018.

\bibitem[BLM89]{BourgainLindenstraussMilman:1989}
Jean Bourgain, Joram Lindenstrauss, and Vitali Milman.
\newblock Approximation of zonoids by zonotopes.
\newblock {\em Acta Mathematica}, 162(1):73--141, 1989.

\bibitem[BMN01]{BMN2001}
Aharon Ben{-}Tal, Tamar Margalit, and Arkadi Nemirovski.
\newblock The ordered subsets mirror descent optimization method with
  applications to tomography.
\newblock {\em SIAM J. Optim.}, 12(1):79--108, 2001.

\bibitem[BO10]{BravermanOstrovsky:2010}
Vladimir Braverman and Rafail Ostrovsky.
\newblock Zero-one frequency laws.
\newblock In Leonard~J. Schulman, editor, {\em Proceedings of the 42nd {ACM}
  Symposium on Theory of Computing, {STOC} 2010, Cambridge, Massachusetts, USA,
  5-8 June 2010}, pages 281--290. {ACM}, 2010.

\bibitem[BS12]{BaldaufSilva:2012}
Markus Baldauf and JMC~Santos Silva.
\newblock On the use of robust regression in econometrics.
\newblock {\em Economics Letters}, 114(1):124--127, 2012.

\bibitem[BSS12]{BatsonSpielmanSrivastava:2012}
Joshua Batson, Daniel~A. Spielman, and Nikhil Srivastava.
\newblock Twice-{Ra}manujan sparsifiers.
\newblock {\em SIAM Journal on Computing}, 41(6):1704--1721, 2012.
\newblock \pSTOC{2009}.

\bibitem[CD21]{ChenDerezinski:2021}
Xue Chen and Micha{\l} Derezi{\'n}ski.
\newblock Query complexity of least absolute deviation regression via robust
  uniform convergence.
\newblock In {\em \COLT{2021}}, 2021.

\bibitem[CDL13]{CohenDavenportLeviatan:2013}
Albert Cohen, Mark~A. Davenport, and Dany Leviatan.
\newblock On the stability and accuracy of least squares approximations.
\newblock {\em Foundations of Computational Mathematics}, 13(5):819--834, 2013.

\bibitem[CKNS15]{ChaudhuriKakadeNetrapalli:2015}
Kamalika Chaudhuri, Sham~M. Kakade, Praneeth Netrapalli, and Sujay Sanghavi.
\newblock Convergence rates of active learning for maximum likelihood
  estimation.
\newblock In {\em \NIPS{2015}}, 2015.

\bibitem[CKPS16]{ChenKanePrice:2016}
Xue Chen, Daniel~M. Kane, Eric Price, and Zhao Song.
\newblock Fourier-sparse interpolation without a frequency gap.
\newblock In {\em \FOCS{2016}}, pages 741--750, 2016.
\newblock \farXiv{1609.01361}.

\bibitem[CLM{\etalchar{+}}15]{CohenLeeMusco:2015}
Michael~B. Cohen, Yin~Tat Lee, Cameron Musco, Christopher Musco, Richard Peng,
  and Aaron Sidford.
\newblock Uniform sampling for matrix approximation.
\newblock In {\em \ITCS{2015}}, pages 181--190, 2015.

\bibitem[CM17]{CohenMigliorati:2017}
Albert Cohen and Giovanni Migliorati.
\newblock Optimal weighted least-squares methods.
\newblock {\em {SMAI} Journal of Computational Mathematics}, 3:181--203, 2017.

\bibitem[CP15]{CohenPeng:2015}
Michael~B Cohen and Richard Peng.
\newblock $l_p$ row sampling by {L}ewis weights.
\newblock In {\em \STOC{2015}}, pages 183--192, 2015.

\bibitem[CP19]{ChenPrice:2019a}
Xue Chen and Eric Price.
\newblock Active regression via linear-sample sparsification active regression
  via linear-sample sparsification.
\newblock In {\em \COLT{2019}}, 2019.

\bibitem[CSS21]{CSS2021}
Vincent Cohen{-}Addad, David Saulpic, and Chris Schwiegelshohn.
\newblock Improved coresets and sublinear algorithms for power means in
  euclidean spaces.
\newblock In A.~Beygelzimer, Y.~Dauphin, P.~Liang, and J.~Wortman Vaughan,
  editors, {\em Advances in Neural Information Processing Systems}, 2021.

\bibitem[CW09]{CW09}
Kenneth~L. Clarkson and David~P. Woodruff.
\newblock Numerical linear algebra in the streaming model.
\newblock In {\em Proceedings of the 41st Annual {ACM} Symposium on Theory of
  Computing, {STOC} 2009, Bethesda, MD, USA, May 31 - June 2, 2009}, pages
  205--214, 2009.

\bibitem[CW15a]{ClarksonWoodruff:2015a}
Kenneth~L. Clarkson and David~P. Woodruff.
\newblock Input sparsity and hardness for robust subspace approximation.
\newblock In Venkatesan Guruswami, editor, {\em {IEEE} 56th Annual Symposium on
  Foundations of Computer Science, {FOCS} 2015, Berkeley, CA, USA, 17-20
  October, 2015}, pages 310--329. {IEEE} Computer Society, 2015.

\bibitem[CW15b]{ClarksonWoodruff:2015b}
Kenneth~L. Clarkson and David~P. Woodruff.
\newblock Sketching for \emph{M}-estimators: {A} unified approach to robust
  regression.
\newblock In Piotr Indyk, editor, {\em Proceedings of the Twenty-Sixth Annual
  {ACM-SIAM} Symposium on Discrete Algorithms, {SODA} 2015, San Diego, CA, USA,
  January 4-6, 2015}, pages 921--939. {SIAM}, 2015.

\bibitem[CWW19]{ClarksonWangWoodruff:2019}
Kenneth Clarkson, Ruosong Wang, and David Woodruff.
\newblock Dimensionality reduction for {T}ukey regression.
\newblock In {\em \ICML{2019}}, pages 1262--1271. PMLR, 2019.

\bibitem[DDH{\etalchar{+}}09]{DasguptaDrineasHarb:2009}
Anirban Dasgupta, Petros Drineas, Boulos Harb, Ravi Kumar, and Michael~W
  Mahoney.
\newblock Sampling algorithms and coresets for $\ell_p$ regression.
\newblock {\em SIAM Journal on Computing}, 38(5):2060--2078, 2009.

\bibitem[DJS{\etalchar{+}}19]{DJSSW19}
Huaian Diao, Rajesh Jayaram, Zhao Song, Wen Sun, and David~P. Woodruff.
\newblock Optimal sketching for kronecker product regression and low rank
  approximation.
\newblock In {\em Advances in Neural Information Processing Systems 32: Annual
  Conference on Neural Information Processing Systems 2019, NeurIPS 2019,
  December 8-14, 2019, Vancouver, BC, Canada}, pages 4739--4750, 2019.

\bibitem[DSSW18]{dssw18}
Huaian Diao, Zhao Song, Wen Sun, and David~P. Woodruff.
\newblock Sketching for kronecker product regression and p-splines.
\newblock In {\em International Conference on Artificial Intelligence and
  Statistics, {AISTATS} 2018, 9-11 April 2018, Playa Blanca, Lanzarote, Canary
  Islands, Spain}, pages 1299--1308, 2018.

\bibitem[DV07]{DeshpandeVaradarajan:V2007}
Amit Deshpande and Kasturi~R. Varadarajan.
\newblock Sampling-based dimension reduction for subspace approximation.
\newblock In David~S. Johnson and Uriel Feige, editors, {\em Proceedings of the
  39th Annual {ACM} Symposium on Theory of Computing, San Diego, California,
  USA, June 11-13, 2007}, pages 641--650. {ACM}, 2007.

\bibitem[DWH18]{DerezinskiWarmuthHsu:2018}
Michal Derezinski, Manfred K.~K Warmuth, and Daniel~J Hsu.
\newblock Leveraged volume sampling for linear regression.
\newblock In {\em \NIPS{2018}}, 2018.

\bibitem[EMM20]{ErdelyiMuscoMusco:2020}
Tam\'{a}s Erd\'{e}lyi, Cameron Musco, and Christopher Musco.
\newblock Fourier sparse leverage scores and approximate kernel learning.
\newblock {\em \NIPS{2020}}, 2020.

\bibitem[FL11]{FeldmanLangberg:2011}
Dan Feldman and Michael Langberg.
\newblock A unified framework for approximating and clustering data.
\newblock In Lance Fortnow and Salil~P. Vadhan, editors, {\em Proceedings of
  the 43rd {ACM} Symposium on Theory of Computing, {STOC} 2011, San Jose, CA,
  USA, 6-8 June 2011}, pages 569--578. {ACM}, 2011.

\bibitem[FLPS21]{FazelLeePadmanabhan:2021}
Maryam Fazel, Yin~Tat Lee, Swati Padmanabhan, and Aaron Sidford.
\newblock Computing {L}ewis weights to high precision.
\newblock {\em \arXiv{2110.15563}}, 2021.

\bibitem[Fox02]{Fox:2002}
John Fox.
\newblock {Robust Regression: Appendix to an R and S-PLUS Companion to Applied
  Regression}, 2002.

\bibitem[FS12]{FS2012}
Dan Feldman and Leonard~J. Schulman.
\newblock Data reduction for weighted and outlier-resistant clustering.
\newblock In Yuval Rabani, editor, {\em Proceedings of the Twenty-Third Annual
  {ACM-SIAM} Symposium on Discrete Algorithms, {SODA} 2012, Kyoto, Japan,
  January 17-19, 2012}, pages 1343--1354. {SIAM}, 2012.

\bibitem[GPV21]{GPV2021}
Mehrdad Ghadiri, Richard Peng, and Santosh~S Vempala.
\newblock Faster $p$-norm regression using sparsity.
\newblock {\em arXiv preprint arXiv:2109.11537}, 2021.

\bibitem[GS99]{GuittonSymes:1999}
Antoine Guitton and William~W Symes.
\newblock Robust and stable velocity analysis using the huber function.
\newblock In {\em SEG Technical Program Expanded Abstracts 1999}, pages
  1166--1169. Society of Exploration Geophysicists, 1999.

\bibitem[GVL13]{golubVanLoan2013Book}
Gene~H. Golxub and Charles~F. Van~Loan.
\newblock {\em Matrix computations}.
\newblock Johns Hopkins Studies in the Mathematical Sciences. Johns Hopkins
  University Press, Baltimore, MD, 2013.

\bibitem[HD15]{HamptonDoostan:2015}
Jerrad Hampton and Alireza Doostan.
\newblock Coherence motivated sampling and convergence analysis of least
  squares polynomial chaos regression.
\newblock {\em Computer Methods in Applied Mechanics and Engineering},
  290:73--97, 2015.

\bibitem[HLP73]{HardyLittlewoodPolya:1973}
Godfrey~H. Hardy, John~Edensor Littlewood, and George P{\'o}lya.
\newblock {\em Inequalities}.
\newblock Cambridge mathematical library. Cambridge Univ. Press, Cambridge,
  1973.

\bibitem[Hub92]{Huber:1992}
Peter~J Huber.
\newblock Robust estimation of a location parameter.
\newblock In {\em Breakthroughs in statistics}, pages 492--518. Springer, 1992.

\bibitem[JS01]{JS2001}
William~B. Johnson and Gideon Schechtman.
\newblock Finite dimensional subspaces of {$L_p$}.
\newblock In {\em Handbook of the geometry of {B}anach spaces, {V}ol. {I}},
  pages 837--870. North-Holland, Amsterdam, 2001.

\bibitem[KW59]{KieferWolfowitz:1959}
Jack Kiefer and Jacob Wolfowitz.
\newblock Optimum designs in regression problems.
\newblock {\em Annals of Mathematical Statistics}, 30(2):271--294, 1959.

\bibitem[Lew78]{Lewis:1978}
D~Lewis.
\newblock Finite dimensional subspaces of ${L}_p$.
\newblock {\em Studia Mathematica}, 63(2):207--212, 1978.

\bibitem[Loh17]{Loh:2017}
Po-Ling Loh.
\newblock Statistical consistency and asymptotic normality for high-dimensional
  robust {$M$}-estimators.
\newblock {\em Ann. Statist.}, 45(2):866--896, 2017.

\bibitem[Loh18]{Loh:2018}
Po{-}Ling Loh.
\newblock Scale calibration for high-dimensional robust regression.
\newblock {\em CoRR}, abs/1811.02096, 2018.

\bibitem[LS18]{LeeSun:2018}
Yin~Tat Lee and He~Sun.
\newblock Constructing linear-sized spectral sparsification in almost-linear
  time.
\newblock {\em SIAM Journal on Computing}, 47(6):2315--2336, 2018.

\bibitem[LT91]{LT2011}
Michel Ledoux and Michel Talagrand.
\newblock {\em Probability in {B}anach spaces}.
\newblock Classics in Mathematics. Springer-Verlag, Berlin, 1991.

\bibitem[LWW20]{LWW20}
Yi~Li, Ruosong Wang, and David~P. Woodruff.
\newblock Tight bounds for the subspace sketch problem with applications.
\newblock In {\em Proceedings of the 2020 {ACM-SIAM} Symposium on Discrete
  Algorithms, {SODA} 2020, Salt Lake City, UT, USA, January 5-8, 2020}, pages
  1655--1674, 2020.

\bibitem[MM00]{MangasarianMusicant:2000}
Olvi~L. Mangasarian and David~R. Musicant.
\newblock Robust linear and support vector regression.
\newblock {\em {IEEE} Trans. Pattern Anal. Mach. Intell.}, 22(9):950--955,
  2000.

\bibitem[MMR21]{MaiMuscoRao:2021}
Tung Mai, Cameron Musco, and Anup~B. Rao.
\newblock Coresets for classification -- simplified and strengthened.
\newblock {\em \arXiv{2106.04254}}, 2021.

\bibitem[MSSW18]{MunteanuSchwiegelshohnSohler:2018}
Alexander Munteanu, Chris Schwiegelshohn, Christian Sohler, and David~P
  Woodruff.
\newblock On coresets for logistic regression.
\newblock {\em \NIPS{2018}}, 2018.

\bibitem[NN13]{NelsonNguyen:2013}
Jelani Nelson and Huy~L. Nguyen.
\newblock {OSNAP:} faster numerical linear algebra algorithms via sparser
  subspace embeddings.
\newblock In {\em 54th Annual {IEEE} Symposium on Foundations of Computer
  Science, {FOCS} 2013, 26-29 October, 2013, Berkeley, CA, {USA}}, pages
  117--126. {IEEE} Computer Society, 2013.

\bibitem[Pan03]{Pan2003}
Dmitry Panchenko.
\newblock Symmetrization approach to concentration inequalities for empirical
  processes.
\newblock {\em Ann. Probab.}, 31(4):2068--2081, 2003.

\bibitem[PPP21]{ParulekarParulekarPrice:2021}
Aditya Parulekar, Advait Parulekar, and Eric Price.
\newblock $l_1$ regression with {L}ewis weights subsampling.
\newblock {\em \arXiv{2105.09433}}, 2021.

\bibitem[PTB12]{parampalli2012construction}
Udaya Parampalli, Xiaohu Tang, and Serdar Boztas.
\newblock On the construction of binary sequence families with low correlation
  and large sizes.
\newblock {\em IEEE transactions on information theory}, 59(2):1082--1089,
  2012.

\bibitem[Puk06]{Pukelsheim:2006}
Friedrich Pukelsheim.
\newblock {\em Optimal Design of Experiments}.
\newblock Society for Industrial and Applied Mathematics, 2006.

\bibitem[Sar06]{Sarlos:2006}
Tamas Sarlos.
\newblock Improved approximation algorithms for large matrices via random
  projections.
\newblock In {\em \FOCS{2006}}, pages 143--152, 2006.

\bibitem[Sch87]{Schechtman:1987}
Gideon Schechtman.
\newblock More on embedding subspaces of {$L_p$} in {$l^n_r$}.
\newblock {\em Compositio Math.}, 61(2):159--169, 1987.

\bibitem[SW11]{SohlerWoodruff:2011}
Christian Sohler and David~P. Woodruff.
\newblock Subspace embeddings for the l1-norm with applications.
\newblock In {\em \STOC{2011}}, pages 755--764, 2011.

\bibitem[SW19]{ShiWoodruff:2019}
Xiaofei Shi and David~P. Woodruff.
\newblock Sublinear time numerical linear algebra for structured matrices.
\newblock In {\em The Thirty-Third {AAAI} Conference on Artificial
  Intelligence, {AAAI} 2019, The Thirty-First Innovative Applications of
  Artificial Intelligence Conference, {IAAI} 2019, The Ninth {AAAI} Symposium
  on Educational Advances in Artificial Intelligence, {EAAI} 2019, Honolulu,
  Hawaii, USA, January 27 - February 1, 2019}, pages 4918--4925. {AAAI} Press,
  2019.

\bibitem[SWMW89]{SacksWelchMitchell:1989}
Jerome Sacks, William~J. Welch, Toby~J. Mitchell, and Henry~P. Wynn.
\newblock Design and analysis of computer experiments.
\newblock {\em Statistical Science}, 4(4):409--423, 1989.

\bibitem[SWY{\etalchar{+}}19]{SongWangYang:2019}
Zhao Song, Ruosong Wang, Lin~F. Yang, Hongyang Zhang, and Peilin Zhong.
\newblock Efficient symmetric norm regression via linear sketching.
\newblock In Hanna~M. Wallach, Hugo Larochelle, Alina Beygelzimer, Florence
  d'Alch{\'{e}}{-}Buc, Emily~B. Fox, and Roman Garnett, editors, {\em Advances
  in Neural Information Processing Systems 32: Annual Conference on Neural
  Information Processing Systems 2019, NeurIPS 2019, December 8-14, 2019,
  Vancouver, BC, Canada}, pages 828--838, 2019.

\bibitem[SWZ19]{SongWoodruffZhong:2019}
Zhao Song, David~P. Woodruff, and Peilin Zhong.
\newblock Towards a zero-one law for column subset selection.
\newblock In Hanna~M. Wallach, Hugo Larochelle, Alina Beygelzimer, Florence
  d'Alch{\'{e}}{-}Buc, Emily~B. Fox, and Roman Garnett, editors, {\em Advances
  in Neural Information Processing Systems 32: Annual Conference on Neural
  Information Processing Systems 2019, NeurIPS 2019, December 8-14, 2019,
  Vancouver, BC, Canada}, pages 6120--6131, 2019.

\bibitem[SZ01]{SchechtmanZvavitch:2001}
Gideon Schechtman and Artem Zvavitch.
\newblock Embedding subspaces of ${L}_p$ into $\ell_p^n$, $0< p< 1$.
\newblock {\em Mathematische Nachrichten}, 227(1):133--142, 2001.

\bibitem[Tal90]{Tal1990}
Michel Talagrand.
\newblock Embedding subspaces of {$L_1$} into {$l^N_1$}.
\newblock {\em Proc. Amer. Math. Soc.}, 108(2):363--369, 1990.

\bibitem[Tal95]{Tal1995}
Michel Talagrand.
\newblock Embedding subspaces of {$L_p$} in {$l^N_p$}.
\newblock In {\em Geometric aspects of functional analysis ({I}srael,
  1992--1994)}, volume~77 of {\em Oper. Theory Adv. Appl.}, pages 311--325.
  Birkh\"{a}user, Basel, 1995.

\bibitem[TMF20]{TukanMaaloufFeldman:2020}
Murad Tukan, Alaa Maalouf, and Dan Feldman.
\newblock Coresets for near-convex functions.
\newblock In Hugo Larochelle, Marc'Aurelio Ranzato, Raia Hadsell,
  Maria{-}Florina Balcan, and Hsuan{-}Tien Lin, editors, {\em Advances in
  Neural Information Processing Systems 33: Annual Conference on Neural
  Information Processing Systems 2020, NeurIPS 2020, December 6-12, 2020,
  virtual}, 2020.

\bibitem[Ver18]{Ver2018}
Roman Vershynin.
\newblock {\em High-dimensional probability}, volume~47 of {\em Cambridge
  Series in Statistical and Probabilistic Mathematics}.
\newblock Cambridge University Press, Cambridge, 2018.

\bibitem[VL92]{vanbook}
Charles~F Van~Loan.
\newblock {\em Computational frameworks for the fast {F}ourier transform},
  volume~10 of {\em Frontiers in Applied Mathematics}.
\newblock Society for Industrial and Applied Mathematics (SIAM), Philadelphia,
  PA, 1992.

\bibitem[VLP93]{vanLoanKron}
Charles~F Van~Loan and N.~Pitsianis.
\newblock Approximation with {K}ronecker products.
\newblock In {\em Linear algebra for large scale and real-time applications
  ({L}euven, 1992)}, volume 232 of {\em NATO Adv. Sci. Inst. Ser. E Appl.
  Sci.}, pages 293--314. Kluwer Acad. Publ., Dordrecht, 1993.

\bibitem[Wal74]{walker1974new}
Alastair~J Walker.
\newblock New fast method for generating discrete random numbers with arbitrary
  frequency distributions.
\newblock {\em Electronics Letters}, 10(8):127--128, 1974.

\bibitem[Wal77]{walker1977efficient}
Alastair~J Walker.
\newblock An efficient method for generating discrete random variables with
  general distributions.
\newblock {\em ACM Transactions on Mathematical Software (TOMS)},
  3(3):253--256, 1977.

\bibitem[Woj91]{Wojtaszczyk:1996}
P.~Wojtaszczyk.
\newblock {\em Banach spaces for analysts}, volume~25 of {\em Cambridge Studies
  in Advanced Mathematics}.
\newblock Cambridge University Press, Cambridge, 1991.

\bibitem[Woo14]{Woodruff:2014}
David~P. Woodruff.
\newblock Sketching as a tool for numerical linear algebra.
\newblock {\em Foundations and Trends in Theoretical Computer Science},
  10(1--2):1--157, 2014.

\bibitem[Zva00]{Zvavitch:2000}
A.~Zvavitch.
\newblock More on embedding subspaces of {$L_p$} into {$l^N_p$}, {$0<p<1$}.
\newblock In {\em Geometric aspects of functional analysis}, volume 1745 of
  {\em Lecture Notes in Math.}, pages 269--280. Springer, Berlin, 2000.

\end{thebibliography}

\appendix

\section{Missing Proofs for Section \ref{sec:introduction}}\label{sec:huber-appendix}
\subsection{Proof of Lemma \ref{lem:generalized-huber-inequality}}

We will need the following short lemma.

\begin{lemma}\label{lem:lp-convexity}
    Let $\bfy\in\mathbb R^n$, $q>0$, $p \geq q$, and $\gamma>0$. Suppose that $\norm*{\bfy}_\infty \leq \gamma^{1/q}\norm*{\bfy}_q$. Then,
    \[
        \norm*{\bfy}_p^p \leq \frac1\gamma\parens*{\gamma^{1/q}\norm*{\bfy}_1}^p.\qedhere
    \]
\end{lemma}
\begin{proof}
    We have
    \begin{align*}
        \norm*{\bfy}_p^p &= \sum_{i=1}^n \abs*{\bfy_i}^p \leq \sum_{i=1}^n \norm*{\bfy}_\infty^{p-q} \abs*{\bfy_i}^{q} = \norm*{\bfy}_\infty^{p-q}\norm*{\bfy}_q^q  \\
        &\leq \gamma^{(p-q)/q}\norm*{\bfy}_q^{p-q}\norm*{\bfy}_q^q  \leq \gamma^{p/q-1}\norm*{\bfy}_q^p = \frac1\gamma (\gamma^{1/q}\norm*{\bfy}_q)^p.
    \end{align*}
\end{proof}

We now proceed to the main proof of this section.

\begin{proof}[Proof of Lemma \ref{lem:l2-lq-inequality}]
    If $\norm*{\bfy}_M^2 \geq 2n$, then $\norm*{\bfy}_H^2 = \Theta(\norm*{\bfy}_q^q)$, since the Huber mass of the entries at most $1$ is at most $n$ over the $n$ entries. Furthermore, if $\norm*{\bfy}_\infty \leq 1$, then $\norm*{\bfy}_M^2 = \norm*{\bfy}_2^2$. Thus, in what follows, we assume that $\norm*{\bfy}_M^2 < 2n$ and $\norm*{\bfy}_\infty > 1$. Note then that $\norm{\bfy}_q^q \leq \norm*{\bfy}_M^2 + n < 3n$. 

    \paragraph{Reduction to $\norm*{\bfy}_\infty \leq \gamma\norm*{\bfy}_M^2$.}

    Suppose that $\norm*{\bfy}_M^2 > 2\norm*{\bfy\mid_T}_M^2$. Then, at least a $1/2$ fraction of the squared Huber norm mass is located in coordinates of size at least $\gamma\norm*{\bfy}_M^2$. If $\gamma\norm*{\bfy}_M^2 \geq 1$, then we have that $\norm*{\bfy}_M^2 = \Theta(\norm*{\bfy}_q^q)$, since over $1/2$ fraction of the mass is in the $\ell_q$ part. Otherwise, we have that $\gamma \norm*{\bfy}_M^2 \leq 1$ so $\norm*{\bfy}_\infty^q \leq \norm*{\bfy}_M^2 \leq 1/\gamma$. Then, letting $S = \braces*{i\in[n] : \abs*{\bfy_i} \leq 1}$,
    \begin{align*}
        \norm*{\bfy}_2^2 &= \norm*{\bfy\mid_S}_2^2 + \norm*{\bfy\mid_{\overline S}}_2^2 \\
        &\leq \norm*{\bfy\mid_S}_2^2 + \norm*{\bfy}_\infty^{2-q} \norm*{\bfy\mid_{\overline S}}_q^q \\
        &\leq \norm*{\bfy}_\infty^{2-q} \norm*{\bfy}_M^2 \\
        &\leq \frac1{\gamma^{(2-q)/q}} \norm*{\bfy}_M^2.
    \end{align*}
    Thus, in either case, we have the conclusion that
    \[
        \norm*{\bfy}_M^2 \geq c\gamma^{2/q-1}\min_{p\in\{1,2\}}\norm*{\bfy}_p^p.
    \]
    We thus assume that $\norm*{\bfy}_M^2 \leq  2\norm*{\bfy\mid_T}_M^2$ in what follows. Furthermore, by restricting our attention to the set $T$, we assume that we have a vector $\bfy$ such that $\norm*{\bfy}_\infty \leq \gamma\norm*{\bfy}_M^2$ and work on showing that
    \[
        \norm*{\bfy}_M^2 \geq c\frac1{(\gamma n)^{\beta}}\min_{p\in[1,2]}\norm*{\bfy}_p^p
    \]
    for such a vector, where $\beta$ is as stated in the statement of the lemma.

    \paragraph{Reduction to Spiked Vectors.}

    Let $S = \braces*{i\in[n] : \abs*{\bfy_i} \leq 1}$. Note that if $S = [n]$, then $\norm*{\bfy}_M^2  = \norm*{\bfy}_2^2$ so we may assume that $S \subsetneq [n]$. Similarly, if $S = \varnothing$, then $\norm*{\bfy}_M^2  = \Theta(\norm*{\bfy}_q^q)$ so we may assume that $S \neq \varnothing$. 
    
    Furthermore, we will assume that $\norm*{\bfy\mid_S}_q^q = \Theta(\norm*{\bfy\mid_{S}}_2^2)$. This is without loss of generality, because $\norm*{\bfy\mid_{\overline S}}_q^q \leq 3n$ and $\norm*{\bfy\mid_{S}}_2^2$ can be as large as $n$, so we can always increase the smaller of $\norm*{\bfy\mid_{\overline S}}_q^q$ and $\norm*{\bfy\mid_{S}}_2^2$ to match the other, which only increases $\norm*{\bfy}_M^2$ by a constant factor. However, this can only increase $\norm*{\bfy}_p^p$ for any $p$, so restricting to such vectors can only affect the claimed inequalities by a constant factor.
    
    Now consider the vector $\bfz$, defined as follows:
    \[
        \bfz_i = \begin{cases}
            \frac{1}{\sqrt{n}}\norm*{\bfy}_M & \text{if $i\in[n-1/\gamma]$} \\
            (\gamma\norm*{\bfy}_M^2)^{1/q} & \text{if $i \in [1/\gamma] + n-1/\gamma$}
        \end{cases}
    \]
    Then,
    \[
        \norm*{\bfz}_M^2 = \frac1\gamma \cdot [(\gamma \norm*{\bfy}_M^2)^{1/q}]^q + (n-1/\gamma) \cdot \parens*{\frac{\norm*{\bfy}_M}{\sqrt n}}^2 = \Theta(\norm*{\bfy}_M^2)
    \]
    and for any $p\in[1,2]$, we have that
    \begin{align*}
        \norm*{\bfy}_p^p &= \norm*{\bfy\mid_{\overline S}}_p^p + \norm*{\bfy\mid_{S}}_p^p \\
        &\leq O(1)\frac1\gamma \cdot (\gamma \norm*{\bfy\mid_{\overline S}}_q^q)^{p/q} + (n^{1/p - 1/2}\norm*{\bfy\mid_{S}}_2)^p && \text{Lemma \ref{lem:lp-convexity}} \\
        &\leq O(1)\frac1\gamma \cdot (\gamma \norm*{\bfy}_M^2)^{p/q} + n\cdot \parens*{\frac{\norm*{\bfy}_M}{\sqrt n}}^p \\
        &= O(\norm*{\bfz}_p^p)
    \end{align*}
    Thus, by similar reasoning as before, it suffices to consider vectors of the form of $\bfz$. These vectors can be parameterized by a single number $r\in[0,1]$ by setting $\norm*{\bfy}_M = n^r$, so that the first $[n-1/\gamma]$ coordinates are $n^{r - 1/2}$ and the last $1/\gamma$ coordinates are $\gamma n^{2r}$. Thus,
    \begin{align*}
        \sup_{\bfy\in\mathbb R^n, \norm*{\bfy}_\infty \leq \gamma\norm*{\bfy}_M^2 }\frac{\min_{p\in[1,2]}\norm*{\bfy}_p^p}{\norm*{\bfy}_M^2} &= \Theta(1)\sup_{r\in[0,1]}\min_{p\in[q,2]}\frac{(1/\gamma)(\gamma n^{2r})^{p/q} + n\cdot n^{p(r - 1/2)}}{n^{2r}} \\
        &= \Theta(1) \sup_{r\in[0,1]} \min_{p\in[q,2]}\max\braces*{n^{\alpha + (2r-\alpha)p/q - 2r}, n^{1 + p(r - 1/2) - 2r}}.
    \end{align*}
    
    \paragraph{Estimates.}
    
    To find the above minimum, we wish to choose the $p$ such that the two exponents are equal, so we solve for
    \[
        \alpha + (2r-\alpha)p/q - 2r = 1 + p(r - 1/2) - 2r,
    \]
    which gives
    \[
        p = \frac{2(1 - \alpha)}{(2r-\alpha)(2/q) + (1-2r)} = \frac{2(1 - \alpha)}{(1-\alpha)(2/q) - (1-2r)(2/q-1)}
    \]
    Plugging this back into the exponent gives
    \[
        1 + p(r-1/2) - 2r = (1-2r) - \frac{p}{2}(1-2r) = (1-2r) - \frac{(1-\alpha)(1-2r)}{(1-\alpha)(2/q) - (1-2r)(2/q-1)}.
    \]
    We now wish to optimize this exponent over $r\in[0,1]$. Writing $x = 1-2r$, $a = 2/q > 1$, and $c = 1-\alpha$, this the objective is
    \begin{align*}
        x - \frac{cx}{ca-(a-1)x} &= \frac{1}{a-1} \bracks*{(a-1)x - c\cdot\frac{(a-1)x}{ca-(a-1)x}} \\
        &= \frac{1}{a-1} \bracks*{ca - \parens*{[ca - (a-1)x] + c\cdot\frac{(a-1)x}{ca-(a-1)x}}} \\
        &= \frac{1}{a-1} \bracks*{c(a+1) - \parens*{[ca - (a-1)x] + c\cdot\frac{ca}{ca-(a-1)x}}} \\
        &= \frac{1}{a-1} \bracks*{c(a+1) - 2c\sqrt a} && \text{AM-GM} \\
        &= c\frac{1}{a-1} \bracks*{(a+1) - 2\sqrt a} \\
        &= c\beta
    \end{align*}
    as long as $ca - (a-1)x\geq 0$, which we have for $\alpha\leq 1/a = q/2$. Note that equality holds when $x = c(a-\sqrt a)/(a-1)$. We thus have that
    \[
        \sup_{\bfy\in\mathbb R^n, \norm*{\bfy}_\infty \leq \gamma\norm*{\bfy}_H^2 }\frac{\min_{p\in[1,2]}\norm*{\bfy}_p^p}{\norm*{\bfy}_H^2} = \Theta(n^{(1-\alpha)\beta}) = \Theta((\gamma n)^{\beta})
    \]
\end{proof}

\section{Missing Proofs for Section \ref{sec:prelim}}\label{sec:prelim-appendix}

\subsection{Properties of Lewis Bases}

\begin{theorem}[$\ell_p$ Lewis Bases \cite{Lewis:1978}, Theorem 2.1 of \cite{SchechtmanZvavitch:2001}]\label{thm:lewis-bases}
	Let $\bfA\in\mathbb R^{n\times d}$ and $0<p<\infty$. There exists a basis matrix $\bfU\in\mathbb R^{n\times d}$ of the column space of $\bfA$, known as the \emph{$\ell_p$ Lewis basis}, such that if $\bfD$ is the diagonal matrix with $\bfD_{i,i} = \norm*{\bfe_i^\top\bfU}_2$, then $\bfD^{p/2 - 1}\bfU$ is an orthonormal matrix. 
\end{theorem}

The fact that $\bfD^{p/2 - 1}\bfU$ is orthonormal from Theorem \ref{thm:lewis-bases} implies the following bound, which is stated in III.B Proposition 7 of \cite{Wojtaszczyk:1996} for $1 < p < \infty$, but requires a short proof for $0 < p < 1$. 

\begin{lemma}\label{lem:lewis-bases-alpha-bound}
Let $\bfA\in\mathbb R^{n\times d}$ and $0<p<\infty$. Let $\bfU$ be as defined in Theorem \ref{thm:lewis-bases}. Then,
\[
    \sum_{i=1}^n \norm*{\bfe_i^\top \bfU}_2^{p} = d.
\]
\end{lemma}
\begin{proof}
	Let $\bfD$ be as defined in Theorem \ref{thm:lewis-bases}. Note that
\begin{align*}
    d &= \norm*{\bfD^{p/2 - 1}\bfU}_F^2  && \text{Frobenius norm of orthonormal matrices} \\
    &= \sum_{i=1}^n \norm*{\bfe_i^\top \bfD^{p/2-1}\bfU}_2^2 \\
    &= \sum_{i=1}^n \norm*{\bfe_i^\top \bfU}_2^{p-2} \norm*{\bfe_i^\top \bfU}_2^2 \\
    &= \sum_{i=1}^n \norm*{\bfe_i^\top \bfU}_2^{p}.\qedhere
\end{align*}
\end{proof}

We then have the following fact is crucial for obtaining sensitivity bounds using Lewis weights.
	
\begin{lemma}\label{lem:lewis-beta-bound}
Let $\bfA\in\mathbb R^{n\times d}$ and $0<p<\infty$. Let $\bfD$ and $\bfU$ be as defined in Theorem \ref{thm:lewis-bases}. Then for all $\bfx\in\mathbb R^d$, the following inequalities hold:
\begin{itemize}
\item if $0 < p < 2$, then
\[
	\norm*{\bfx}_2 \leq \norm*{\bfU\bfx}_p \leq d^{1/p - 1/2}\norm*{\bfx}_2
\]
\item if $2\leq p < \infty$, then
\[
	\norm*{\bfU\bfx}_p \leq \norm*{\bfx}_2 \leq d^{1/2 - 1/p}\norm*{\bfU\bfx}_p.
\]
\end{itemize}
\end{lemma}
\begin{proof}[Proof of Lemma \ref{lem:lewis-beta-bound}]
    For $p\geq 2$, the result is recorded in III.B Lemma 8 of \cite{Wojtaszczyk:1996}. For $0 < p < 2$, the computation exactly follows III.B Lemma 8 of \cite{Wojtaszczyk:1996} for $1 < p < 2$, which we reproduce for completeness. We first bound
    \begin{align*}
        \norm*{\bfx}_2^2 &= \norm*{\bfD^{p/2-1}\bfU\bfx}_2^2 \\
        &= \sum_{i=1}^n \abs*{\bfe_i^\top \bfD^{p/2-1}\bfU\bfx}^2 \\
        &= \sum_{i=1}^n \norm*{\bfe_i^\top\bfU}_2^{p-2} \abs*{\bfe_i^\top \bfU\bfx}^2 \\
        &= \sum_{i=1}^n \norm*{\bfe_i^\top\bfU}_2^{p-2} \abs*{\bfe_i^\top \bfU\bfx}^{2-p}\abs*{\bfe_i^\top \bfU\bfx}^p \\
        &\leq \max_{i=1}^n \bracks*{\norm*{\bfe_i^\top\bfU}_2^{p-2} \abs*{\bfe_i^\top \bfU\bfx}^{2-p}} \norm*{\bfU\bfx}_p^p && \text{H\"older, $2-p>0$} \\
        &\leq \max_{i=1}^n \bracks*{\norm*{\bfe_i^\top\bfU}_2^{p-2} \norm*{\bfe_i^\top \bfU}_2^{2-p}\norm*{\bfx}_2^{2-p}} \norm*{\bfU\bfx}_p^p && \text{Cauchy-Schwarz} \\
        &= \norm*{\bfx}_2^{2-p}\norm*{\bfU\bfx}_p^p.
    \end{align*}
    Rearranging then gives
    \[
        \norm*{\bfx}_2 \leq \norm*{\bfU\bfx}_p.
    \]
    
    For the upper bound, we have
    \begin{align*}
        \norm*{\bfU\bfx}_p^p &= \sum_{i=1}^n \abs*{\bfe_i^\top\bfU\bfx}^p \\
        &= \sum_{i=1}^n \abs*{\bfe_i^\top\bfU\bfx}^p (\norm*{\bfe_i^\top\bfU}_2^{p-2})^{p/2}(\norm*{\bfe_i^\top\bfU}_2^{p-2})^{-p/2} \\
        &\leq \parens*{\sum_{i=1}^n\abs*{\bfe_i^\top\bfU\bfx}^2 \norm*{\bfe_i^\top\bfU}_2^{p-2}}^{p/2}\parens*{\sum_{i=1}^n (\norm*{\bfe_i^\top\bfU}_2^{p-2})^{(-p/2)(2/(2-p))}}^{(2-p)/2} \\
        &= \norm*{\bfD^{p/2-1}\bfU\bfx}_2^p \parens*{\sum_{i=1}^n \norm*{\bfe_i^\top\bfU}_2^{p}}^{(2-p)/2} \\
        &= \norm*{\bfD^{p/2-1}\bfU\bfx}_2^p d^{(2-p)/2} \\
        &= \norm*{\bfx}_2^p d^{(2-p)/2}
    \end{align*}
    where the inequality is by H\"older with exponent $2/p$, and the second to last identity uses Lemma \ref{lem:lewis-bases-alpha-bound}. Taking $1/p$th powers on both sides gives
    \[
        \norm*{\bfU\bfx}_p \leq d^{1/p - 1/2} \norm*{\bfx}_2.\qedhere
    \]
\end{proof}

Finally, using Lemmas \ref{lem:lewis-bases-alpha-bound} and \ref{lem:lewis-beta-bound}, we obtain an $\ell_p$ sensitivity bound for all $0 < p < \infty$ via Lewis weights. 

\begin{lemma}[Lewis weights bound $\ell_p$ sensitivities]\label{lem:lewis-sensitivity-bound}
Let $\bfU$ be as defined in Theorem \ref{thm:lewis-bases}. Then for all $\bfx\in\mathbb R^d$,
\[
	\frac{\abs*{\bfe_i^\top\bfU\bfx}^p}{\norm*{\bfU\bfx}_p^p} \leq d^{\max\{0,p/2-1\}}\norm*{\bfe_i^\top\bfU}_2^p = d^{\max\{0,p/2-1\}}\bfw_i^p(\bfA)
\]
and
\[
	\sum_{i=1}^n d^{\max\{0,p/2-1\}}\norm*{\bfe_i^\top\bfU}_2^p = \sum_{i=1}^n d^{\max\{0,p/2-1\}}\bfw_i^p(\bfA) = d^{\max\{1,p/2\}}.
\]
\end{lemma}
\begin{proof}
Let $\bfx\in\mathbb R^d$. By using Lemma \ref{lem:lewis-beta-bound}, we have
\[
	\abs*{\bfe_i^\top\bfU\bfx}^p \leq \norm*{\bfe_i^\top\bfU}_2^p \norm*{\bfx}_2^p \leq d^{\max\{0,p/2-1\}}\norm*{\bfe_i^\top\bfU}_2^p \norm*{\bfU\bfx}_p^p
\]
which rearranges to
\[
	\frac{\abs*{\bfe_i^\top\bfU\bfx}^p}{\norm*{\bfU\bfx}_p^p} \leq d^{\max\{0,p/2-1\}}\norm*{\bfe_i^\top\bfU}_2^p.
\]
To bound the sum of these, we use Lemma \ref{lem:lewis-bases-alpha-bound} to conclude. 
\end{proof}
\section{Additional Properties of \texorpdfstring{$M$}{M}-norms}\label{sec:m-norm-appendix}

We prove additional properties of $M$-norms that we need in this section. These results are generally relatively straightforward generalizations of results in \cite{ClarksonWoodruff:2015a}. 

\subsection{Consequences of Polynomial Boundedness}

The polynomial boundedness condition allows us to compare the $M$-norm with the $\ell_{p_M}$ norm. 

\begin{lemma}[Lemma 37 of \cite{ClarksonWoodruff:2015a}]\label{lem:m-norm-vs-p-norm}
Let $\bfx\in\mathbb R^n$. Let $\bfw\geq\mathbf{1}_n$ be a set of weights. Then,
\[
    M\parens*{\norm*{\bfx}_{p_M,\bfw}} \leq c_U^3 \norm*{\bfx}_{M,\bfw}^{p_M}
\]
Furthermore, if $M$ is strictly increasing, then
\[
    \norm*{\bfx}_{p_M,\bfw} \leq M^{-1}(c_U^3 \norm*{\bfx}_{M,\bfw}^{p_M}).
\]
If $\bfw = \mathbf{1}_n$, then the $c_U^3$ may be replaced by a $c_U^2$.
\end{lemma}
\begin{proof}[Proof of Lemma \ref{lem:m-norm-vs-p-norm}]
    Let $\bfW$ be the $n\times n$ diagonal matrix with $\bfw$ down the diagonal. Then,
    \[
        \norm*{\bfx}_{p_M,\bfw} = \norm*{\bfW^{1/p_M}\bfx}_{p_M}.
    \]
    We have
    \begin{align*}
        M\parens*{\norm*{\bfW^{1/p_M}\bfx}_{p_M}} &\leq c_U\frac{\norm*{\bfW^{1/p_M}\bfx}_{p_M}^{p_M}}{\norm*{\bfW^{1/p_M}\bfx}_{\infty}^{p_M}}M\parens*{\norm*{\bfW^{1/p_M}\bfx}_\infty} \\
        &= c_U\sum_{i=1}^n \frac{\bfw_i\abs*{\bfx_i}^{p_M}}{\norm*{\bfW^{1/p_M}\bfx}_{\infty}^{p_M}}M\parens*{\norm*{\bfW^{1/p_M}\bfx}_\infty} \\
        &\leq c_U^2\sum_{i=1}^n \frac{\bfw_i\abs*{\bfx_i}^{p_M}}{\norm*{\bfW^{1/p_M}\bfx}_{\infty}^{p_M}}\frac{\norm*{\bfW^{1/p_M}\bfx}_\infty^{p_M}}{\bfw_i\abs*{\bfx_i}^{p_M}}M(\bfw_i^{1/p_M}\abs*{\bfx_i}) \\
        &\leq c_U^3\sum_{i=1}^n \bfw_i M(\abs*{\bfx_i}) \\
        &= c_U^3\norm*{\bfx}_{M,\bfw}^{p_M}.
    \end{align*}
    If $\bfw = \mathbf{1}_n$, then inequality is not needed, so $c_U^3$ can be replaced by $c_U^2$. If $M$ is strictly increasing, then $M^{-1}$ is strictly increasing and thus we may apply $M^{-1}$ on both sides of the above inequality.
\end{proof}

The following is a simple property of Definition \ref{dfn:polynomially-bounded}. 

\begin{lemma}
    Let $M:\mathbb R_{\geq0}\to\mathbb R_{\geq0}$ be strictly increasing. Then:
    \begin{itemize}
        \item If $M$ is polynomially bounded above with degree $p$ and constant $c_U\geq 1$. Then, $M^{-1}$ is polynomially bounded below with degree $1/p$ and constant $1/c_U^{1/p}$. 
        \item If $M$ is polynomially bounded below with degree $q$ and constant $c_L\geq 1$. Then, $M^{-1}$ is polynomially bounded above with degree $1/q$ and constant $1/c_L^{1/q}$. 
    \end{itemize}
\end{lemma}
\begin{proof}
    Let $M$ be polynomially bounded above. Let $y/x = \kappa > 1$. Then,
    \[
        M\parens*{\parens*{\frac{\kappa}{c_U}}^{1/p} M^{-1}(x)} \leq c_U\parens*{\frac{\kappa}{c_U}}M(M^{-1}(x)) = \kappa x.
    \]
    Applying $M^{-1}$ on both sides gives that
    \[
        M^{-1}(\kappa x) \geq \frac1{c_U^{1/p}} \kappa^{1/p}M^{-1}(x).
    \]
    Similarly, let $M$ be polynomially bounded below. Let $y/x = \kappa > 1$. Then,
    \[
        M\parens*{\parens*{\frac{\kappa}{c_L}}^{1/q} M^{-1}(x)} \geq c_L\parens*{\frac{\kappa}{c_L}}M(M^{-1}(x)) = \kappa x.
    \]
    Applying $M^{-1}$ on both sides gives that
    \[
        M^{-1}(\kappa x) \leq \frac1{c_L^{1/q}} \kappa^{1/q}M^{-1}(x).\qedhere
    \]
\end{proof}

The next lemma compares important entries using the polynomial boundedness condition.

\begin{lemma}\label{lem:max-sensitivity-comparison}
    Let $M$ be polynomially bounded above with degree $p$ and constant $c_U\geq 1$. Let $\bfx\in\mathbb R^n$ be a vector with entries arranged in order, i.e., $\abs*{\bfx_1} \geq \abs*{\bfx_2} \geq \dots \geq \abs*{\bfx_n}$. Then,
    \[
        \frac{M(\abs*{\bfx_1})}{\norm*{\bfx}_M^p} \leq c_U \frac{\abs*{\bfx_1}^p}{\norm*{\bfx}_p^p}.
    \]
\end{lemma}
\begin{proof}
    Note that for all $i \geq 2$, we have by the polynomially boundedness condition that
    \[
        \frac{M(\abs*{\bfx_1})}{M(\abs*{\bfx_i})} \leq c_u \parens*{\frac{\abs*{\bfx_1}}{\abs*{\bfx_i}}}^p = c_U \frac{\abs*{\bfx_1}^p}{\abs*{\bfx_i}^p}.
    \]
    Then,
    \begin{align*}
        \frac{M(\abs*{\bfx_1})}{\norm*{\bfx}_M^p} &= \frac{M(\abs*{\bfx_1})}{M(\abs*{\bfx_1}) + \sum_{i=2}^n M(\abs*{\bfx_i})} \\
        &= \bracks*{1 + \sum_{i=2}^n \frac{M(\abs*{\bfx_i})}{M(\abs*{\bfx_1})}}^{-1} \leq \bracks*{1 + \sum_{i=2}^n \frac1{c_U}\frac{\abs*{\bfx_i}^p}{\abs*{\bfx_1}^p}}^{-1} \\
        &= \frac{\abs*{\bfx_1}^p}{\abs*{\bfx_1}^p + c_U^{-1}\sum_{i=2}^n \abs*{\bfx_i}^p} \leq c_U\frac{\abs*{\bfx_1}^p}{\norm*{\bfx}_p^p}.\qedhere
    \end{align*}
\end{proof}

\subsection{The Triangle Inequality}\label{sec:m-triangle-inequality}

It will be useful to have the triangle inequality for the $M$-norm. For this we will need to assume that $M^{1/p_M}$ is subadditive.

\begin{lemma}\label{lem:triangle-inequality-m}
Let $M:\mathbb R_{\geq0}\to\mathbb R_{\geq0}$ satisfy the conditions of Definition \ref{dfn:m-norm}, and furthermore that $M^{1/p_M}$ is subadditive. Let $\bfw\geq\mathbf{1}_n$ be a set of weights. Then, for all $\bfy_1, \bfy_2\in\mathbb R^n$,
\[
    \norm*{\bfy_1 + \bfy_2}_{M,\bfw} \leq \norm*{\bfy_1}_{M,\bfw} + \norm*{\bfy_2}_{M,\bfw}.
\]
\end{lemma}
\begin{proof}
The proof is by a simple modification of the proof of Minkowski's inequality. We first bound
\begin{align*}
    \norm*{\bfy_1 + \bfy_2}_{M,\bfw}^{p_M} &= \sum_{i=1}^n \bfw_i M(\abs*{(\bfy_1)_i + (\bfy_2)_i}) \\
    &= \sum_{i=1}^n \bfw_i \parens*{M(\abs*{(\bfy_1)_i + (\bfy_2)_i})^{1/p_M}} \parens*{M(\abs*{(\bfy_1)_i + (\bfy_2)_i})^{(p_M-1)/p_M}} \\
    &\leq \sum_{i=1}^n \bfw_i[M^{1/p_M}(\abs*{(\bfy_1)_i}) + M^{1/p_M}(\abs*{(\bfy_2)_i})]\parens*{M(\abs*{(\bfy_1)_i + (\bfy_2)_i})^{(p_M-1)/p_M}} \tag{Subadditivity} \\
    &= \sum_{i=1}^n \bfw_i M^{1/p_M}(\abs*{(\bfy_1)_i})M(\abs*{(\bfy_1)_i + (\bfy_2)_i})^{(p_M-1)/p_M} + \\
    &\hspace{3em}\sum_{i=1}^n \bfw_i M^{1/p_M}(\abs*{(\bfy_2)_i})M(\abs*{(\bfy_1)_i + (\bfy_2)_i})^{(p_M-1)/p_M} \\
    &\leq \parens*{\sum_{i=1}^n \bfw_i M(\abs*{(\bfy_1)_i})}^{1/p_M}\parens*{\sum_{i=1}^n \bfw_i M(\abs*{(\bfy_1)_i + (\bfy_2)_i})}^{(p_M - 1)/p_M} + \\ &\hspace{3em}\parens*{\sum_{i=1}^n \bfw_i M(\abs*{(\bfy_2)_i})}^{1/p_M}\parens*{\sum_{i=1}^n \bfw_i M(\abs*{(\bfy_1)_i + (\bfy_2)_i})}^{(p_M - 1)/p_M} \tag{H\"older} \\
    &= \parens*{\norm*{\bfy_1}_{M,\bfw} + \norm*{\bfy_2}_{M,\bfw}} \norm*{\bfy_1 + \bfy_2}_{M,\bfw}^{p_M - 1}.
\end{align*}
Rearranging then gives the desired result.
\end{proof}

We note that the subadditivity condition in fact implies the polynomially boundedness condition with constant $2^{p_M}$.

\begin{lemma}[Growth Bounds for Subadditive Functions]\label{lem:subadditive-growth-bound}
Let $M:\mathbb R_{\geq0}\to\mathbb R_{\geq0}$ be any increasing function such that $M^{1/p_M}$ is subadditive. Then, $M$ is polynomially bounded above with degree $p_M$ and constant $2^{p_M}$. 
\end{lemma}
\begin{proof}
Let $k\geq 0$ be an integer and $f$ a subadditive and increasing function. Then, it is easy to see by subadditivity that
\[
    f(2^k x) \leq 2^k f(x)
\]
by repeatedly applying subadditivity $k$ times. Now let $u > v$. Then,
\[
    2^k v < u \leq 2^{k+1} v
\]
for some integer $k\geq 0$. Then, by using that $f$ is increasing,
\[
    \frac{f(u)}{f(v)} \leq \frac{f(2^{k+1}v)}{f(v)} \leq \frac{2^{k+1}f(v)}{f(v)} < 2\frac{u}{v}.
\]
Applying the above to $f = M^{1/p_M}$ so that
\[
    \frac{M^{1/p_M}(y)}{M^{1/p_M}(x)} \leq 2\parens*{\frac{y}{x}}.
\]
Raising both sides to the $p_M$th power leads to the desired conclusion. 
\end{proof}

\subsection{Nets for \texorpdfstring{$M$}{M}-Estimators}\label{sec:m-net}
It is often desirable to construct nets in the $M$-norm in order to design row sampling algorithms for $M$-estimators. For this, we additionally need a polynomial lower bound as in Definition \ref{dfn:polynomially-bounded}. 

We first define balls, spheres, and covers for the $M$-norm. Note that we specialize our definitions to a subspace specified by an $n\times d$ matrix $\bfA$. The dependence on $\bfA$ will be clear from context and thus implicit. 

\begin{definition}[Net/Cover]
Let $\bfA\in\mathbb R^{n\times d}$ and let $\mathcal V = \Span(\bfA)$. Let $M:\mathbb R_{\geq0}\to\mathbb R_{\geq0}$ satisfy the conditions of Definition \ref{dfn:m-norm}, and let $\bfw\geq\mathbf{1}_n$ be a set of weights. Let $\eps>0$. Define an $\eps$-net (or $\eps$-cover) of a set $\mathcal A$ as a set of points $\mathcal N \subseteq\mathcal A$ such that
\[
    \mathcal A \subseteq \bigcup_{\bfy\in \mathcal N} (\bfy + \mathcal B_{\eps}^{M,\bfw})
\]
where $\mathcal B_\eps^{M,\bfw}$ is the $M$ ball of radius $\eps$ (see Definition \ref{dfn:m-ball-sphere}). 
\end{definition}

For norms, the following is known.

\begin{lemma}[Lemma 2.4 of \cite{BourgainLindenstraussMilman:1989}]\label{lem:standard-net}
Let $\mathcal B$ be the unit ball of a $d$-dimensional subspace of a normed space $(X,\norm*{\cdot})$. Then for any $0<\eps<1$, there is a net $\mathcal N$ with 
\[
    \log\abs*{\mathcal N} = O\parens*{d\log\frac1\eps}
\]
such that, for any $\bfy\in\mathcal B$, there is a $\bfy'\in\mathcal N$ such that $\norm*{\bfy-\bfy'} \leq \eps$. 
\end{lemma}

We give the following analogous result on covers of $M$ balls. 

\begin{lemma}[Lemma 33 of \cite{ClarksonWoodruff:2015a}]\label{lem:m-norm-net-size}
Let $\bfA\in\mathbb R^{n\times d}$. Let $M:\mathbb R_{\geq0}\to\mathbb R_{\geq0}$ satisfy the conditions of Definition \ref{dfn:m-norm}, and furthermore that $M$ is polynomially bounded below with degree $q_M$ and constant $c_L$ (see Definition \ref{dfn:polynomially-bounded}). Let $\bfw\geq\mathbf{1}_n$ be a set of weights. Consider the subspace $\mathcal V = \Span(\bfA)$. Let $\mathcal C\subset\mathcal V$. Then, there is an $\eps\rho$-covering $\mathcal N$ of $\mathcal B_\rho\cap \mathcal C$ of size at most
\[
    \log\abs*{\mathcal N} \leq O\parens*{d\log\frac1{\eps^{p_M/q_M}}}.
\]
\end{lemma}
\begin{proof}
Note that by Lemma \ref{lem:subadditive-growth-bound}, $M$ is polynomially bounded above with degree $p_M$ and constant $c_U = 2^{p_M}$. Thus,
\[
    c_L \parens*{\frac{y}{x}}^{q_M} \leq \frac{M(y)}{M(x)} \leq c_U\parens*{\frac{y}{x}}^{p_M}.
\]
Then for any $\kappa \geq 1$ and $x\geq 0$, this implies that
\[
    c_L \kappa^{q_M} M(x) \leq M(\kappa x) \leq c_U \kappa^{p_M} M(x).
\]
This in turn implies the ``scale insensitivity'' condition
\begin{equation}\label{eqn:scale-insensitivity}
    c_L\kappa^{q_M}\norm*{\bfy}_{M,\bfw}^{p_M}\leq \norm*{\kappa\bfy}_{M,\bfw}^{p_M} \leq c_L\kappa^{p_M} \norm*{\bfy}_{M,\bfw}^{p_M}
\end{equation}
for $\kappa \geq 1$ for the $M$-norm. Now let $\alpha = (c_L\eps^{p_M})^{1/q_M}$. Then,
\[
    \norm*{\alpha\bfy}_{M,\bfw}^{p_M} \leq \frac{\alpha^{q_M}}{c_L}\norm*{\bfy}_{M,\bfw}^{p_M} = \eps^{p_M}\norm*{\bfy}_{M,\bfw}^{p_M}
\]
so $\alpha\mathcal B_\rho \subseteq \mathcal B_{\eps\rho}$. Then, the volume argument in Lemma 33 of \cite{ClarksonWoodruff:2015a} then shows that at most $(1/\alpha)^d$ translates of $\alpha\mathcal B_{\rho}$ can fit in $\mathcal B_{\rho}$, which in turn implies that there exists an $\eps\rho$-cover of $\mathcal B_\rho$ of size at most $O(1/\alpha)^d$. 

\end{proof}

\subsection{From Nets to Balls}

Next, we show an adaptation of Lemma 34 of \cite{ClarksonWoodruff:2015a}, which shows that approximation guarantees on nets over the ball imply approximation guarantees on the entire ball. This requires both the triangle inequality as discussed in Section \ref{sec:m-triangle-inequality} and the polynomial lower bound as discussed in Section \ref{sec:m-net}. 

We first show that the two additional assumptions give us the continuity of the $M$-norm. 

\begin{lemma}[Continuity of the $M$-norm]\label{lem:m-norm-continuity}
Let $M:\mathbb R_{\geq0}\to\mathbb R_{\geq0}$ satisfy the conditions of Definition \ref{dfn:m-norm}, and furthermore that
\begin{itemize}
    \item $M^{1/p_M}$ is subadditive
    \item $M$ is polynomially bounded below with degree $q_M$ and constant $c_L$ (see Definition \ref{dfn:polynomially-bounded})
\end{itemize}
Let $\bfw\geq \mathbf{1}_n$ be a set of weights. Let $\bfx\in\mathbb R^n$ and $\bfb\in\mathbb R^n$. Then,
\[
    \kappa \mapsto \norm*{\kappa\bfx - \bfb}_{M,\bfw}
\]
is a continuous function. 
\end{lemma}
\begin{proof}
Let $\kappa > 0$ and $\eps > 0$. Let $\kappa'$ be such that
\[
    \abs*{\kappa' - \kappa} \leq \parens*{c_L\bracks*{\frac{\eps}{\norm*{\bfx}_{M,\bfw}}}^{p_M}}^{1/q_M}
\]
Then,
\begin{align*}
    \abs*{\norm*{\kappa'\bfx-\bfb}_{M,\bfw} - \norm*{\kappa\bfx-\bfb}_{M,\bfw}} &\leq \norm*{\kappa'\bfx - \kappa\bfx}_{M,\bfw} && \text{triangle inequality} \\
    &= \norm*{(\kappa' - \kappa)\bfx}_{M,\bfw} \\
    &\leq \bracks*{\frac{\abs*{\kappa' - \kappa}^{q_M}}{c_L}}^{1/p_M}\norm*{\bfx}_{M,\bfw} && \text{Equation \eqref{eqn:scale-insensitivity}} \\
    &\leq \eps
\end{align*}
as desired.
\end{proof}

With the continuity of the $M$-norm in hand, we will next show that we may translate approximation guarantees on a net for the $M$-norm ball to approximation guarantees on the entire ball. For norms, the following is known. 

\begin{lemma}[Lemma 2.5 of \cite{BourgainLindenstraussMilman:1989}]\label{lem:norm-net-to-sphere}
    Let $\eps\in(0,1/2)$. Let $\norm*{\cdot}_X$ and $\norm*{\cdot}_Y$ be two norms over a subspace $\mathcal V\subseteq\mathbb R^n$. Let $\mathcal N$ be an $\eps$-net over the unit sphere $\mathcal S_{X} = \braces*{\bfy\in \mathcal V : \norm*{\bfy}_X = 1}$ for $\norm*{\cdot}_X$. Then if
    \[
        \abs*{\norm*{\bfy}_X - \norm*{\bfy}_Y} \leq \eps
    \]
    for all $\bfy\in\mathcal N$, then
    \[
        \abs*{\norm*{\bfy}_X - \norm*{\bfy}_Y} \leq 4\eps
    \]
    for all $\bfy\in\mathcal S_X$. 
\end{lemma}

The analogous result for $M$-norms is the following:

\begin{lemma}[Lemma 34 of \cite{ClarksonWoodruff:2015a}]\label{lem:m-norm-net-to-sphere}
Let $\bfA\in\mathbb R^{n\times d}$ and let $\mathcal V = \Span(\bfA)$. Let $\bfb\in\mathbb R^n$. Let $M:\mathbb R_{\geq0}\to\mathbb R_{\geq0}$ satisfy the conditions of Definition \ref{dfn:m-norm}, and furthermore that
\begin{itemize}
    \item $M^{1/p_M}$ is subadditive
    \item $M$ is polynomially bounded below with degree $q_M$ and constant $c_L$ (see Definition \ref{dfn:polynomially-bounded})
\end{itemize}
Let $\bfw, \bfv\geq \mathbf{1}_n$ be two sets of weights. Let $\eps\in(0,1/2)$ and $\rho > 0$. Let $\mathcal N$ be an $\alpha\rho$-cover of $\mathcal B_\rho^{M,\bfw}$ for
\[
    \alpha = \frac{c_L^{1/q_M}}{c_U^{1/p_M}}\eps^{p_M/q_M}.
\]
Then, if
\[
    \abs*{\norm*{\bfy}_{M,\bfv} - \norm*{\bfy}_{M,\bfw}} \leq \eps\rho
\]
for all $\bfy\in(\mathcal N - \bfb)$ and $\bfy\in\mathcal N$, then
\[
    \abs*{\norm*{\bfy}_{M,\bfv} - \norm*{\bfy}_{M,\bfw}} \leq 4\eps\rho
\]
for all $\bfy\in(\mathcal B_\rho^{M,\bfw} - \bfb)$ and $\bfy\in\mathcal B_\rho^{M,\bfw}$. 
\end{lemma}
\begin{proof}
Let
\[
    \Delta \coloneqq \sup\braces*{\frac{\abs*{\norm*{\bfy}_{M,\bfv} - \norm*{\bfy}_{M,\bfw}}}\rho  : \bfy\in (\mathcal B_\rho^{M,\bfw} - \bfb)\cup\mathcal B_\rho^{M,\bfw}}.
\]

Let $\bfy\in(\mathcal B_\rho^{M,\bfw} - \bfb)\cup\mathcal B_\rho^{M,\bfw}$. Let $\bfy'\in(\mathcal N - \bfb)\cup\mathcal N$ be such that $\bfy - \bfy' \in\mathcal V$ (by choosing $\bfy'$ to be in the appropriate net to cancel out the possible $\bfb$) and
\[
    \norm*{\bfy - \bfy'}_{M,\bfw} \leq \alpha\rho
\]
and let $\kappa\geq1$ be such that $\kappa(\bfy - \bfy') \in \mathcal S_\rho^{M,\bfw}$, which exists by Lemma \ref{lem:m-norm-continuity}. By the polynomial lower bound condition, we have the scale insensitivity condition (see Equation \eqref{eqn:scale-insensitivity}), which in turn implies that
\[
    \rho = \norm*{\kappa(\bfy-\bfy')}_{M,\bfw} \leq c_U^{1/p_M}\kappa\norm*{\bfy-\bfy'}_{M,\bfw} \leq c_U^{1/p_M}\kappa \alpha \rho = \kappa ((c_L\eps^{p_M})^{1/q_M}) \rho
\]
so
\begin{equation}\label{eqn:kappa-lower-bound}
    \kappa \geq \frac1{(c_L\eps^{p_M})^{1/q_M}}
\end{equation}
Then,
\begin{align*}
    \norm*{\bfy - \bfy'}_{M,\bfv} &\leq \frac1{(c_L\kappa^{q_M})^{1/p_M}}\norm*{\kappa(\bfy - \bfy')}_{M,\bfv} && \text{scale insensitivity} \\
    &\leq \eps\norm*{\kappa(\bfy - \bfy')}_{M,\bfv} && \text{Equation \eqref{eqn:kappa-lower-bound}} \\
    &\leq \eps(1+\Delta)\rho.
\end{align*}
Then,
\begin{align*}
    \norm*{\bfy}_{M,\bfv} \leq \norm*{\bfy'}_{M,\bfv} + \norm*{\bfy-\bfy'}_{M,\bfv} \leq \norm*{\bfy'}_{M,\bfw} + \eps\rho + \eps(1+\Delta)\rho
\end{align*}
so
\[
    \frac{\norm*{\bfy}_{M,\bfv} - \norm*{\bfy'}_{M,\bfw}}\rho \leq \eps + \eps(1+\Delta).
\]
By taking supremums over both sides,
\[
    \Delta \leq 2\eps + \eps\Delta \implies \Delta \leq \frac{2\eps}{1-\eps} \leq 4\eps
\]
We then also have that
\begin{align*}
    \norm*{\bfy}_{M,\bfv} &\geq \norm*{\bfy'}_{M,\bfv} - \norm*{\bfy-\bfy'}_{M,\bfv} \\
    &\geq \norm*{\bfy'}_{M,\bfw}-\eps\rho - \eps(1+4\eps)\rho \\
    &\geq \norm*{\bfy'}_{M,\bfw} - 4\eps\rho.\qedhere
\end{align*}

\end{proof}

One way to cope with the scale invariance is to show looser approximations for all ``large'' and all ``small'' scales, given a guarantee on a sphere of a single radius. 

\begin{lemma}[Lemma 35 of \cite{ClarksonWoodruff:2015a}]\label{lem:net-to-ball}
    Let $\bfA\in\mathbb R^{n\times d}$ and let $\mathcal V = \Span(\bfA)$. Let $M:\mathbb R_{\geq0}\to\mathbb R_{\geq0}$ satisfy the conditions of Definition \ref{dfn:m-norm}, and furthermore that
    \begin{itemize}
        \item $M^{1/p_M}$ is subadditive
        \item $M$ is polynomially bounded below with degree $q_M$ and constant $c_L$ (see Definition \ref{dfn:polynomially-bounded})
    \end{itemize}
    Let $\bfw, \bfv\geq \mathbf{1}_n$ be two sets of weights. Let $\eta > 0$. Then,
    \begin{itemize}
        \item If $\norm*{\bfy}_{M,\bfw} \leq \eta\rho$ for all $\bfy\in\mathcal S_\rho^{M,\bfw}$, then $\norm*{\bfy}_{M,\bfw} \leq \eta\rho/c_L^{1/p_M}$ for all $\bfy\in\mathcal B_\rho^{M,\bfw}$.
        \item If $\norm*{\bfy}_{M,\bfw} \geq \eta\rho$ for all $\bfy\in\mathcal S_\rho^{M,\bfw}$, then $\norm*{\bfy}_{M,\bfw} \geq c_L^{1/p_M}\eta\rho$ for all $\bfy\notin\mathcal B_\rho^{M,\bfw}$.
    \end{itemize}
\end{lemma}
\begin{proof}
    Let $\bfy\in\mathcal B_\rho^{M,\bfw}$ and let $\kappa\geq 1$ be such that $\norm*{\kappa\bfy}_{M,\bfw} = \rho$. Then,
    \[
        \norm*{\bfy}_{M,\bfv} \leq \frac1{(c_L\kappa^{q_M})^{1/p_M}}\norm*{\kappa\bfy}_{M,\bfv} \leq \frac1{c_L^{1/p_M}}\eta\rho.
    \]
    Let $\bfy\notin\mathcal B_\rho^{M,\bfw}$ and let $\alpha\leq 1$ be such that $\norm*{\alpha\bfy}_{M,\bfw} = \rho$. Then,
    \[
        \norm*{\bfy}_{M,\bfv} \geq \parens*{c_L\frac1{\alpha^{q_M}}}^{1/p_M}\norm*{\alpha\bfy}_{M,\bfv} \geq c_L^{1/p_M}\eta\rho.\qedhere
    \]
\end{proof}

\subsection{Proof of Lemma \ref{lem:m-sensitivity-sampling}}
\begin{proof}[Proof of Lemma \ref{lem:m-sensitivity-sampling}]
    By Lemma \ref{lem:m-norm-net-to-sphere}, it suffices to prove the approximation guarantee on a $\alpha\rho$-cover $\mathcal N$ of $\mathcal S_\rho^{M}$, where $\alpha = O(\eps^{p_M/q_M})$. By Lemma \ref{lem:m-norm-net-size}, this has size at most
    \[
        \log\abs*{\mathcal N} \leq O\parens*{d\log\frac1\eps}.
    \]
    Fix a $\bfy\in\mathcal N$ and define the random variable
    \[
        W_i \coloneqq \bfw_i' M(\bfy(i))
    \]
    for each $i\in[n]$. Then,
    \[
        \E\bracks*{\sum_{i=1}^n W_i} = \sum_{i=1}^n \frac{\bfw_i}{\bfp_i}M(\bfy(i))\cdot \bfp_i = \norm*{\bfy}_{M,\bfw}^{p_M}.
    \]
    We next bound the variance:
    \[
        \Var\bracks*{\sum_{i=1}^n W_i} = \sum_{i=1}^n \Var[W_i] \leq \sum_{i=1}^n \frac{\bfw_i^2}{\bfp_i^2}M(\bfy(i))^2\cdot \bfp_i = \sum_{i=1}^n \frac{\bfw_i}{\bfp_i}M(\bfy(i))\cdot \bfw_i M(\bfy(i))
    \]
    Note that
    \[
        \frac{\bfw_i}{\bfp_i}M(\bfy(i)) \leq \frac{\tilde\bfs_i^{M,\bfw}(\bfA)\norm*{\bfy}_{M,\bfw}^{p_M}}{m\cdot \tilde\bfs_i^{M,\bfw}(\bfA)} \leq \frac1m \norm*{\bfy}_{M,\bfw}^{p_M}
    \]
    so the variance is bounded by
    \[
        \Var\bracks*{\sum_{i=1}^n W_i} \leq \frac1m \norm*{\bfy}_{M,\bfw}^{p_M} \sum_{i=1}^n \bfw_i M(\bfy(i)) = \frac{1}{m}\norm*{\bfy}_{M,\bfw}^{2p_M}.
    \]
    Then by Bernstein's inequality,
    \begin{align*}
        \Pr\braces*{\abs*{\sum_{i=1}^n W_i - \norm*{\bfy}_{M,\bfw}^{p_M}} > t} &\leq 2\exp\parens*{-\Theta(1)\frac{t^2}{\frac{1}{m}\norm*{\bfy}_{M,\bfw}^{2p_M} + \frac1m \norm*{\bfy}_{M,\bfw}^{p_M} t}} \\
        &= 2\exp\parens*{-\Theta(1)\frac{mt^2}{\norm*{\bfy}_{M,\bfw}^{p_M}(\norm*{\bfy}_{M,\bfw}^{p_M} + t)}}.
    \end{align*}
    For $t = \eps\norm*{\bfy}_{M,\bfw}^{p_M}$, this gives a bound of
    \[
        2\exp\parens*{-\Theta(1)\frac{m(\eps\norm*{\bfy}_{M,\bfw}^{p_M})^2}{\norm*{\bfy}_{M,\bfw}^{2p_M}}} =  2\exp\parens*{-\Theta(1)m\eps^2}.
    \]
    We then set
    \[
        m = O\parens*{\frac{d}{\eps^2}\parens*{\log\frac1\eps}\parens*{\log\frac1\delta}}
    \]
    which is enough to union bound over the net $\mathcal N$ with failure probability at most $\delta$. 

    Finally,
    \[
        \E\nnz(\bfw') = \sum_{i=1}^n \bfp_i \leq m\sum_{i=1}\tilde\bfs_i^{M,\bfw}(\bfA) = m\tilde{\mathcal T}^{M,\bfw}(\bfA).\qedhere
    \]
\end{proof}
\section{Missing Proofs for Section \ref{sec:orlicz}}\label{sec:orlicz-appendix}
\subsection{Proof of Lemma \ref{lem:sensitivity-bound-weighted-orlicz}}
\begin{proof}[Proof of Lemma \ref{lem:sensitivity-bound-weighted-orlicz}]
    Suppose that $\bfA\bfx$ satisfies
\[
    1 = \sum_{i=1}^n \bfw_i G(\abs*{[\bfA\bfx](i)}) = \sum_{j=1}^N \sum_{i\in T_j}  \bfw_i G(\abs*{[\bfA\bfx](i)}).
\]
Then for each $j\geq 0$,
\[
    \sum_{i\in T_j}  G_j (\abs*{[\bfA\bfx](i)}) = \sum_{i\in T_j}  2^{j-1}  G(\abs*{[\bfA\bfx](i)})\leq \sum_{i\in T_j}  \bfw_i G(\abs*{[\bfA\bfx](i)}) \leq 1.
\]
Then using our previous result, we have that
\[
    G_j (\abs*{[\bfA\bfx](i)}) \leq \bfs_i^{G_j}(\bfA\mid_{T_j})
\]
where $\bfA\mid_{T_j}$ is the restriction of $\bfA$ to the rows of $T_j$, and
\[
    \sum_{i\in T_j}\bfs_i^{G_j}(\bfA\mid_{T_j}) \leq O(d^{\max\{1,p_G/2\}}\log n). 
\]
Then,
\[
    \bfw_i G(\abs*{[\bfA\bfx](i)}) \leq 2^{j+1}\cdot G(\abs*{[\bfA\bfx](i)}) \leq 2\cdot G_j (\abs*{[\bfA\bfx](i)}) \leq 2\bfs_i^{G_j}(\bfA\mid_{T_j})
\]
so
\[
    \bfs_i^{G,\bfw}(\bfA) \leq 2\cdot \bfs_i^{G_j}(\bfA\mid_{T_j})
\]
and
\[
    \sum_{i=1}^n \bfs_i^{G,\bfw}(\bfA) = \sum_{j=1}^N \sum_{i\in T_j}\bfs_i^{G,\bfw}(\bfA) \leq \sum_{j=1}^N O(d^{\max\{1,p_G/2\}}\log n) \leq O(N\cdot d^{\max\{1,p_G/2\}}\log n)
\]
as claimed. 
\end{proof}

\end{document}